\documentclass[11pt,oneside]{article}

\pdfoutput=1 

\usepackage{style_mjh_forarXiv}
\usepackage[numbers]{natbib}
\usepackage{amsmath}                                          
\usepackage{amsthm}                                           
\theoremstyle{definition} \newtheorem{defn}{Definition}       
\theoremstyle{plain}      
\theoremstyle{plain} \newtheorem{thm}[defn]{Theorem}          
\theoremstyle{plain} \newtheorem{lem}[defn]{Lemma}            
\theoremstyle{plain}         
\theoremstyle{remark} \newtheorem{rmk}[defn]{Remark}          
\theoremstyle{remark} \newtheorem{ex}[defn]{Example}          

\usepackage{enumitem}
\makeatletter
\def\namedlabel#1#2{\begingroup
    #2%
    \def\@currentlabel{#2}%
    \phantomsection\label{#1}\endgroup
}
\makeatother

\usepackage[pdfauthor={Matthew James Holland},%
pdftitle={holland2018rgdmult},%
colorlinks=true,%
linkcolor=blue,%
citecolor=blue,%
pdftex]{hyperref} %

\begin{document}

\title{\textbf{Robust descent using smoothed multiplicative noise}}
\author{
  Matthew J.~Holland\thanks{Please direct correspondence to \texttt{matthew-h@ids.osaka-u.ac.jp}.}\\
  Osaka University\\
  Yamada-oka 2-8, Suita, Osaka, Japan
}
\date{} 

\maketitle

\begin{abstract}
To improve the off-sample generalization of classical procedures minimizing the empirical risk under potentially heavy-tailed data, new robust learning algorithms have been proposed in recent years, with generalized median-of-means strategies being particularly salient. These procedures enjoy performance guarantees in the form of sharp risk bounds under weak moment assumptions on the underlying loss, but typically suffer from a large computational overhead and substantial bias when the data happens to be sub-Gaussian, limiting their utility. In this work, we propose a novel robust gradient descent procedure which makes use of a smoothed multiplicative noise applied directly to observations before constructing a sum of soft-truncated gradient coordinates. We show that the procedure has competitive theoretical guarantees, with the major advantage of a simple implementation that does not require an iterative sub-routine for robustification. Empirical tests reinforce the theory, showing more efficient generalization over a much wider class of data distributions.
\end{abstract}

\section{Introduction}\label{sec:intro}

The risk minimization model of learning is ubiquitous in machine learning, and it effectively captures the key facets of any effective learning algorithm: we must have reliable statistical inference procedures, and practical implementations of these procedures. Formulated using the expected loss, or risk $\risk(\ww) \defeq \exx\loss(\ww;\zz)$, induced by a loss $\loss$, where $\ww$ is the parameter (vector, function, set, etc.) to be learned, and expectation is taken with respect to $\zz$. In practice, all we are given is data $\zz_{1},\ldots,\zz_{n}$, and based on this the algorithm outputs some candidate $\wwhat$. If $\risk(\wwhat)$ is small with high confidence over the random sample, it provides some evidence for good generalization, subject to the assumptions placed on the underlying distribution. The statistical side is important because the risk $\risk$ is always unknown, and the implementation is important since the only $\wwhat$ we ever have in practice is one we can actually compute given finite data, time, and memory.

The vast majority of popular algorithms used today can be viewed as different implementations of empirical risk minimization (ERM), which admits any minimizer of $n^{-1}\sum_{i=1}^{n}\loss(\cdot;\zz_{i})$. From an algorithmic perspective, ERM is ambiguous; there are countless ways to implement the ERM procedure, and important work in recent years has highlighted the fact that a tremendous gap exists between the quality of good and bad ERM solutions \citep{feldman2016a}, for tasks as simple as multi-class pattern recognition \citep{daniely2014a}, let alone tasks with unbounded losses. Furthermore, even tried-and-true implementations such as ERM by gradient descent (ERM-GD) only have appealing guarantees when the data is distributed sharply around the mean in a sub-Gaussian sense, as demonstrated in important work by \citet{lin2016a}. These facts are important because ERM is ubiquitous in modern learning algorithms, and heavy-tailed data by no means exceptional \citep{finkenstadt2003Extreme}. Furthermore, these works suggest that procedures which have been designed to deal with finite samples of heavy-tailed data may be much more efficient than traditional ERM-based approaches, and indeed the theoretical promise of robust learning algorithms is being studied rigorously \citep{lecue2017a,lugosi2016a,lugosi2017a,lugosi2017b}.

\paragraph{Review of related work}

Here we review the technical literature most closely related to our work. The canonical benchmark to be compared against is ERM-GD, for which \citet{lin2016a} in pathbreaking work provide generalization guarantees under sub-Gaussian data. There are naturally two points of interest: (1) How do competing algorithms perform in settings when ERM is optimal? (2) What about robustness to settings in which ERM is sub-optimal? Many interesting robust learning algorithms have been studied in the past few years. One important procedure is from \citet{brownlees2015a}, based on fundamental results due to \citet{catoni2012a}. The basic idea is to minimize an M-estimator of the risk, namely
\begin{align*}
\wwhat & = \argmin_{\ww} \widehat{\loss}(\ww)\\
\widehat{\loss}(\ww) & = \argmin_{\theta \in \RR} \sum_{i=1}^{n} \rho\left(\loss(\ww;\zz_{i})-\theta\right).
\end{align*}
While the statistical guarantees are near-optimal under weak assumptions on the data, and the proxy loss $\widehat{\loss}$ can be computed accurately by an iterative procedure, its definition is implicit, and leads to rather significant computational roadblocks. Even if $\loss$ and $\risk$ and convex, the proxy loss need not be, and the non-linear optimization required by this method can be both unstable and costly in high dimensions.

Another important body of work looks at generalization of the classical ``median of means'' technique to higher dimensions. From \citet{minsker2015a} and \citet{hsu2016a}, the core idea is to partition the data into $k$ disjoint subsets $\DD_{1} \cup \cdots \cup \DD_{k} = \{1,2,\ldots,n\}$, obtain ERM solutions on each subset, and then robustly aggregate these solutions such that poor candidates are effectively ignored. For example, using the geometric median approach of aggregation, we have
\begin{align*}
\wwhat & = \argmin_{\ww} \sum_{m=1}^{k} \|\ww - \wwtil_{m}\|\\
\wwtil_{m} & = \argmin_{\ww} \sum_{i \in \DD_{m}} \loss(\ww;\zz_{i}), \quad m=1,\ldots,k.
\end{align*}
These robust aggregation methods can be implemented \citep{vardi2000a}, and have appealing formal properties. An application of this technique to construct a robust loss was very recently proposed by \citet{lecue2018a}. The main limitation of all these approaches is practical: when sample size $n$ is small relative to the number of parameters to be determined, very few subsets can be created, and significant error due to bias occurs; conversely, when $n$ is large enough to make many candidates, cheaper and less sophisticated methods often suffice. Furthermore, in the case of \citet{lecue2018a} where an expensive sub-routine must be run at every iteration, the computational overhead is substantial.

Also in the recent literature, interesting work has begun to appear looking at ``robust gradient descent'' algorithms, which is to say steepest descent procedures which utilize a robust estimate of the gradient vector of the risk \citep{chen2017a,chen2017b,prasad2018a}. The basic idea is as follows. Assuming partial derivatives exist, writing $\rgrad(\ww) \defeq (\partial_{1} \risk(\ww),\ldots,\partial_{d}\risk(\ww))$ for the risk gradient, we could iteratively solve this task by the following update:
\begin{align}\label{eqn:update_GD_ideal}
\wwstar_{(t+1)} = \wwstar_{(t)} - \alpha_{(t)} \, \rgrad(\wwstar_{(t)})
\end{align}
Naturally, this procedure is ideal, since the underlying distribution is never known in practice, meaning $\risk$ is always unknown. As such, we must approximate this objective function and optimize it with incomplete information. In taking a steepest descent approach, all that is required is an accurate approximation of $\rgrad$. Instead of first approximating $\risk$ and then using that approximation to infer $\rgrad$, computational resources are better spent approximating $\rgrad$ directly with some data-dependent $\rgest$ constructed using the loss gradients $\lgrad(\ww;\zz_{1}),\ldots,\lgrad(\ww;\zz_{n})$, and plugging this in to the iterative update, as
\begin{align}\label{eqn:update_GD_approx}
\wwhat_{(t+1)} = \wwhat_{(t)} - \alpha_{(t)} \, \rgest(\wwhat_{(t)}).
\end{align}
Once again here, the median-of-means idea pops up in the literature, with \citet{prasad2018a} using a robust aggregation of empirical mean estimates of the gradient. That is, after partitioning the data into $k$ subsets as before, the estimate vector $\rgest$ is constructed as
\begin{align*}
\rgest(\ww) & = \argmin_{\uu} \sum_{m=1}^{k} \|\uu - \widetilde{\rgrad}_{m}(\ww)\|\\
\widetilde{\rgrad}_{m}(\ww) & = \frac{1}{|\DD_{m}|} \sum_{i \in \DD_{m}} \lgrad(\ww;\zz_{i}), \quad m=1,\ldots,k.
\end{align*}
and substituted within the gradient update (\ref{eqn:update_GD_approx}). While conceptually a very appealing new proposition, there are a number of issues to be overcome. Their theoretical guarantees have statistical error terms which depend only on $d$ rather than $d^{2}$, where $d$ is the dimension of the gradient, but the same error terms also grow with the number of iterations, leading to very slow error rates when the number of iterations grows with $n$, as is required for $\varepsilon$-good performance (see Remark \ref{rmk:compare_rgd}). Furthermore, computing $\rgest$ via a geometric median sub-routine introduces the exact same overhead and bias issues as the procedures of \citet{hsu2016a} just discussed, only that this time these costs are incurred at \textit{each step} of the gradient descent procedure, and thus these costs and errors accumulate, and can propagate over time. Iterative approximations at each update take time and are typically distribution-dependent, while fast approximations leave a major gap between the estimators studied in theory and those used in practice.

\paragraph{Our contributions}

To address the limitations of both ERM-GD and the robust alternatives discussed above, we take an approach that allows us to obtain a robust gradient estimate directly, removing the need for iterative approximations, without losing the theoretical guarantees. In this paper, we provide both theoretical and empirical evidence that using the proposed procedure, paying a small price in terms of bias and computational overhead is worth it when done correctly, leading to a large payout in terms of distributional robustness. Key contributions are as follows:

\begin{itemize}
\item A practical learning algorithm which can closely mimic ERM-GD when ERM is optimal, but which performs far better under heavy-tailed data when ERM deteriorates.

\item Finite-sample risk bounds that hold with high probability for the proposed procedure, under weak moment assumptions on the distribution of the loss gradient.

\item We demonstrate the ease of use and flexibility of our procedure in a series of experiments, testing performance using both controlled simulations and real-world datasets, and compared with numerous standard competitors.
\end{itemize}

\paragraph{Content overview}

Key concepts and basic technical ideas underlying the proposed algorithm are introduced in section \ref{sec:algorithm}. We fill in the details and provide performance guarantees via theoretical analysis in section \ref{sec:theory}, and reinforce our formal argument with a variety of empirical tests in section \ref{sec:tests}. Finally, concluding remarks and a look ahead are given in section \ref{sec:conclusion}.

\begin{figure}[t]
\centering
\includegraphics[width=1.0\textwidth]{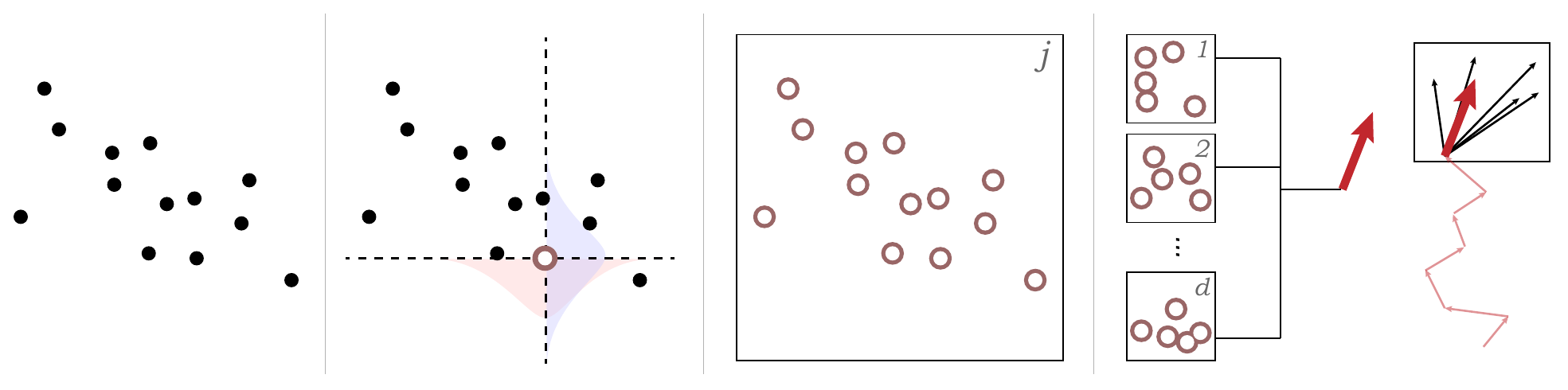}
\caption{Illustration of the key elements of our proposed algorithm. From the far left, points represent loss gradient coordinates evaluated at different observations in our sample. To each point, we consider multiplication by Gaussian noise centered at 1. This noise is smoothed out by integration over the noise distribution, and applied to each gradient coordinate to generate a robust update direction.}
\label{fig:algo_illustration}
\end{figure}

\section{Overview of proposed algorithm}\label{sec:algorithm}

Our proposed procedure can be derived in a few simple steps. Let us begin with a one-dimensional example, in which for random variable $x$ we try to estimate $\exx x$ based on sample $x_{1},\ldots,x_{n}$. Our problem of interest is the setting in which the underlying distribution may be heavy-tailed, but it also may not be, and this information is not available to the learner \textit{a priori}.

\paragraph{Scaling and truncation} The first step involves a very primitive technique for ensuring the bias is small under well-behaved data, all while constraining the impact of outlying points. We re-scale, apply a soft truncation $\psi$, and then put the truncated arithmetic mean back in the original scale, namely
\begin{align*}
\frac{s}{n} \sum_{i=1}^{n} \psi\left(\frac{x_{i}}{s}\right) \approx \exx x.
\end{align*}
Here $\psi$ should have the symmetry of an odd function ($\psi(-u)=-\psi(u)$), be non-decreasing on $\RR$, with a slope of $\psi^{\prime}(u) \to 1$ as $u \to 0$, and be bounded on $\RR$. A simple example is the hyperbolic tangent function, $\tanh(u)$, but we shall consider other examples shortly. If the scale $s>0$ is set such that $|x_{i}|/s$ is near zero for all but errant observations, the impact of the non-deviant terms to the arithmetic mean will be approximately equal, while the deviant points will have a disproportionately small impact.

\paragraph{Noise multiplication} The second step involves applying multiplicative noise, albeit the purpose is rather unique. Let $\epsilon_{1},\ldots,\epsilon_{n}$ be our independent random noise, generated from a common distribution $\epsilon \sim \prior$ with $\exx \epsilon = 0$. We multiply each datum by $1+\epsilon$, and then pass each modified datum $x_{i}(1+\epsilon_{i}) = x_{i} + x_{i}\epsilon_{i}$ through the truncation function as above, yielding
\begin{align*}
\widetilde{x}(\mv{\epsilon}) = \frac{s}{n} \sum_{i=1}^{n} \psi\left( \frac{x_{i}+\epsilon_{i}x_{i}}{s} \right).
\end{align*}
Multiplicative noise has received much attention in recent years in the machine learning literature, in particular with ``dropout'' in deep neural networks via Bernoulli random variables \citep{srivastava2014a}, and more recent investigations using Gaussian multiplicative noise \citep{nalisnick2015a}. In using multiplicative noise with mean $1$, the basic idea is as follows. For typical points, an increase or decrease of a certain small fraction should not change the estimator output much. On the other hand, for wildly deviant points, a push further in the wrong direction is likely to be harmless due to $\psi$, while a push in the right direction could earn an additional valid point for the estimator.

\paragraph{Noise smoothing} In the third and final step, we smooth out the multiplicative noise by taking the expectation of this estimator with respect to the noise distribution. This smoothed version of the estimator, still a random variable dependent on the original sample, is the final estimator of interest, defined
\begin{align}\label{eqn:estimator_defn}
\widehat{x} \defeq \exx \widetilde{x}(\mv{\epsilon}) = \frac{s}{n} \sum_{i=1}^{n} \int \psi\left( \frac{x_{i}+\epsilon_{i}x_{i}}{s} \right) \, d\prior(\epsilon_{i}).
\end{align}
Computationally, in order to obtain $\widehat{x}$ to approximate $\exx x$, we will not actually have to generate the $\epsilon_{i}$ and multiply the $x_{i}$ by $(1+\epsilon_{i})$, but instead will have to evaluate the integral.

\paragraph{Computational matters} Before we move to the high-dimensional setting of interest, how can we actually compute this $\widehat{x}$? Numerical integration is not appealing as the overhead will be too much for a sub-routine to be repeated many times. Naturally, the computational approach will depend on the noise distribution $\prior$, and the truncation function $\psi$. Using recent results in the statistical literature from \citet{catoni2017a}, if we set the truncation function to be
\begin{align}\label{eqn:influence_cat17}
\psi(u) \defeq
\begin{cases}
u - u^{3}/6, & -\sqrt{2} \leq u \leq \sqrt{2}\\
2\sqrt{2}/3, & u > \sqrt{2}\\
-2\sqrt{2}/3, & u < -\sqrt{2}
\end{cases}
\end{align}
and set the noise distribution to be $\prior = N(0,1/\beta)$, then the integral of interest can be given in an explicit form that is simple to compute, requiring no numerical integration or approximation. Written in a general form with shift parameter $a \in \RR$ and scale parameter $b > 0$, we can express the integral as
\begin{align}\label{eqn:compute_integral}
\exx_{\prior} \psi\left( a + b \sqrt{\beta} \epsilon \right) = a\left(a - \frac{b^{2}}{2}\right) - \frac{a^{3}}{6} + \corr(a,b)
\end{align}
where $\corr(a,b)$ is a correction term that is complicated to write, but extremely simple to implement (see Appendix \ref{sec:tech_computation} for exact form).

\paragraph{Proposed learning algorithm} Let us now return to the high-dimensional setting of interest. At any candidate $\ww$, we can evaluate the $\loss(\ww;\zz_{i})$ and $\lgrad(\ww;\zz_{i})$ for all points $i = 1,\ldots,n$. The heart of our proposal: apply the sub-routine specified in (\ref{eqn:estimator_defn}) to each coordinate of the loss gradients, which can be computed directly using (\ref{eqn:compute_integral}), and plug the resulting ``robust gradient estimate'' into the usual first-order update (\ref{eqn:update_GD_approx}). Pseudocode for the proposed procedure is provided in Algorithm \ref{algo:rgdmult}. All operations on vectors in the pseudo-code are element-wise, e.g., $\uu^{2} = (u_{1}^{2},\ldots,u_{d}^{2})$, $|\uu|=(|u_{1}|,\ldots,|u_{d}|)$, $\uu/\vv = (u_{1}/v_{1},\ldots,u_{d}/v_{d})$, and so forth. For readability, we abbreviate $\lgrad_{i}(\ww) \defeq \lgrad(\ww;\zz_{i})$.
\begin{algorithm}
\caption{Outline of robust gradient descent learning algorithm}
\label{algo:rgdmult}
\begin{algorithmic}
\State \textbf{inputs:} $\wwhat_{(0)}$, $T>0$
\For{$t = 0,1,\ldots,T-1$} 
  \medskip
  \State \textbf{Scale set to minimize error bound, via Remark \ref{rmk:scaling_method} and Lemma \ref{lem:uniform_grad_accuracy}.}
  \medskip
  \State $\displaystyle \mv{s}_{(t)} = \sqrt{n \mv{v}_{(t)}/2\log(\delta^{-1})}, \text{ where } \mv{v}_{(t)} \geq \exx_{\ddist} \lgrad_{i}(\wwhat_{(t)})^{2}$
  \medskip
  \State \textbf{Corrected gradient estimate, via (\ref{eqn:estimator_defn}) and (\ref{eqn:compute_integral}):}
  \medskip
  \State $\displaystyle \rgest_{(t)} = \frac{1}{n} \sum_{i=1}^{n} \left(\lgrad_{i}(\wwhat_{(t)})\left(1 - \frac{\lgrad_{i}(\wwhat_{(t)})^{2}}{2\mv{s}_{(t)}^{2}\beta}\right) - \frac{\lgrad_{i}(\wwhat_{(t)})^{3}}{6\mv{s}_{(t)}^{2}}\right) + \frac{1}{n}\sum_{i=1}^{n}C\left(\frac{\lgrad_{i}(\wwhat_{(t)})}{\mv{s}_{(t)}}, \frac{|\lgrad_{i}(\wwhat_{(t)})|}{\mv{s}_{(t)}\sqrt{\beta}} \right)$
  \medskip
  \State \textbf{Plug in to gradient-based update (\ref{eqn:update_GD_approx}).}
  \medskip
  \State $\displaystyle \wwhat_{(t+1)} = \wwhat_{(t)} - \alpha_{(t)} \, \rgest_{(t)}(\wwhat_{(t)})$
  \medskip
\EndFor
\State \textbf{return:} $\wwhat_{(T)}$
\end{algorithmic}
\end{algorithm}

As a simple example of the guarantees that are available for this procedure, assuming just finite variance of the gradients, and setting $\alpha_{(t)} = \alpha$ for simplicity, we have that
\begin{align*}
\risk(\wwhat_{(T)}) - \risk^{\ast} \leq O\left(\frac{d(\log(d\delta^{-1})+d\log(\diameter{}n))}{n}\right) + O\left((1-\alpha\SCfactor)^{T}\right)
\end{align*}
with probability at least $1-\delta$ over the random draw of the sample, where $d$ is the dimension of the space the gradient lives in, $\diameter$ is the diameter of $\WW$, and the constant $\SCfactor$ depends only on $\risk(\cdot)$. Theoretically, these results are competitive with existing state of the art methods cited in the previous section, but with the computational benefits of zero computational error, direct computability, and the fact that per-step computation time is independent of the underlying distribution.

\section{Theoretical analysis}\label{sec:theory}

In this section, we carry out some formal analysis of the generalization performance of Algorithm \ref{algo:rgdmult}. More concretely, we provide guarantees in the form of high-probability upper bounds on the excess risk achieved by the proposed procedure, given a finite sample of $n$ observations, and finite budget of $T$ iterations. We start with a general sketch, followed by a precise description of key conditions, representative results, and subsequent discussion. All detailed proofs are given in Appendix \ref{sec:tech_proofs}.

\paragraph{Sketch of the argument}

Our approach can be broken down into three straightforward steps:

\begin{enumerate}
\item Obtain pointwise error bounds for $\rgest(\ww) \approx \rgrad(\ww)$.
\item Extend Step 1 to obtain error bounds uniform in $\ww \in \WW$.
\item Control distance of $\wwhat_{(t)}$ from minimizer at each step.
\end{enumerate}

For the first step, we can leverage new technical results from \citet{catoni2017a} to obtain strong guarantees for a novel estimator of the risk gradient, evaluated at an arbitrary, albeit pre-fixed, parameter $\ww \in \WW$. With this result established, since Algorithm \ref{algo:rgdmult} updates iteratively, and the sequence of parameters $(\wwhat_{(t)})$ is data dependent, bounds which hold for all possible contingencies are required. Since $\WW$ has an infinite number of elements, naive union bounds are useless. However, a rather typical covering number argument offers an appealing solution. As long as a finite number of balls of radius $\varepsilon$ can cover $\WW$, then we can discretize the space: every $\ww$ is $\varepsilon$-close to at least one ball, and by the previous step we have strong pointwise guarantees that we can apply to each of the ball centers. Finally, while in practice any learning meachine has no choice but to use the approximate update (\ref{eqn:update_GD_approx}), when the risk function is convex, we can show that the deviation from the ideal procedure (\ref{eqn:update_GD_ideal}) can be tightly controlled. Indeed, under strong convexity, the distance between $\wwhat_{(t)}$ and a minimizer of $\risk(\cdot)$ can be controlled by a sharp statistical error term, and optimization error equivalent to running the ideal gradient descent procedure (\ref{eqn:update_GD_ideal}).

\paragraph{Notation}
Here we organize the key notation used in the remainder of our theoretical analysis and associated proofs (some are re-statements of definitions above). The observable loss is $\loss:\WW \times \ZZ \to \RR$, where $\WW$ is the model from which the learning machine can select parameters $\ww \in \WW$, and $\ZZ$ is the space housing the data sample, $\zz_{1},\ldots,\zz_{n}$. The data distribution is denoted $\zz \sim \ddist$, and the noise distribution featured in our algorithm is $\epsilon \sim \prior$. The risk to be minimized is $\risk(\ww) \defeq \exx_{\ddist}l(\ww;\zz)$. The risk and loss gradients are respectively $\rgrad$ and $\lgrad$. Estimates of $\rgrad$ based on observations of $\lgrad$ are denoted $\rgest$. We frequently use $\prr$ to denote a generic probability measure, typically the product measure induced by the sample, which should be clear from the context. Unless specified otherwise, $\|\cdot\|$ shall denote the usual $\ell_{2}$ norm on Euclidean space. For integer $k > 0$, write $[k] \defeq \{1,\ldots,k\}$.

\paragraph{Assumptions with examples}

No algorithm can achieve arbitrarily good performance across all possible distributions \citep{shalev2014a}. In order to obtain meaningful theoretical results, we must place conditions on the underlying distribution, as well as the model and objective functions used. We give concrete examples to illustrate that these assumptions are reasonable, and that they include scenarios that allow for both sub-Gaussian and heavy-tailed data.

\begin{enumerate}
\item[\namedlabel{asmp:model_compact}{A0}.] $\WW$ is a closed, convex subset of $\RR^{d}$, with diameter $\Delta \defeq \sup\{\|\uu-\vv\|: \uu,\vv \in \WW\} < \infty$.

\item[\namedlabel{asmp:loss_smooth}{A1}.] Loss function $\loss(\,\cdot\,;\zz)$ is $\paraSmooth$-smooth on $\WW$.

\item[\namedlabel{asmp:risk_smooth}{A2}.] $\risk(\cdot)$ is $\paraSmooth$-smooth, and continuously differentiable on $\WW$.

\item[\namedlabel{asmp:risk_flat_min}{A3}.] There exists $\wwstar \in \WW$ at which $\rgrad(\wwstar)=0$.

\item[\namedlabel{asmp:risk_strong_convex}{A4}.] $\risk(\cdot)$ is $\paraSC$-strongly convex on $\WW$.

\item[\namedlabel{asmp:variance_bound}{A5}.] There exists $v < \infty$ such that $\exx_{\ddist}(\lgsub_{j}(\ww;\zz))^{2} \leq v$, for all $\ww \in \WW$, $j \in [d]$.
\end{enumerate}

Of these assumptions, assuredly \ref{asmp:model_compact} is simplest: any ball (here in the $\ell_{2}$ norm) with finite radius will suffice, though far more exotic examples are assuredly possible. The remaining assumptions require some checking, but hold under very weak assumptions on the underlying distribution, as the following examples show.

\begin{ex}[Concrete example of assumption \ref{asmp:loss_smooth}]\label{ex:loss_smooth}
Consider the linear regression model $y = \langle \wwstar, \xx \rangle + \eta$, where $\xx$ is almost surely bounded (say $\prr\{\|\xx\| \leq c \} = 1$), but the noise $\eta$ can have any distribution we desire. Consider the squared loss $\loss(\ww;\zz) = (\langle \ww, \xx \rangle - y)^{2}$, and observe that for any $\ww, \ww^{\prime} \in \WW$, we have
\begin{align*}
\lgrad(\ww;\zz) - \lgrad(\ww^{\prime};\zz) = 2(\langle \ww - \ww^{\prime}, \xx \rangle)\xx
\end{align*}
and thus
\begin{align*}
\|\lgrad(\ww;\zz) - \lgrad(\ww^{\prime};\zz)\| \leq 2\|\xx\|^{2}\|\ww-\ww^{\prime}\| \leq 2c^{2}\|\ww-\ww^{\prime}\|.
\end{align*}
Thus we have smoothness with $\paraSmooth = 2c^{2}$, satisfying \ref{asmp:loss_smooth}.
\end{ex}

\begin{ex}[Concrete example of assumptions \ref{asmp:risk_smooth} and \ref{asmp:risk_flat_min}]\label{ex:risk_smooth}
Consider a similar setup as in Example \ref{ex:loss_smooth}, but instead of requiring $\xx$ to be bounded, weaken the assumption to $\exx\|\xx\|^{2} < \infty$. Since taking the derivative under the integral we have $\rgrad(\ww) = \exx \lgrad(\ww;\zz) = 2(\langle \ww-\wwstar, \xx \rangle-\eta)\xx$. Clearly, $\rgrad(\wwstar)=0$, satisfying \ref{asmp:risk_flat_min}. Furthermore, it follows that
\begin{align*}
\rgrad(\ww)-\rgrad(\ww^{\prime}) & = \exx\left(\lgrad(\ww;\zz) -  \lgrad(\ww^{\prime};\zz)\right)\\
& = 2 \exx (\langle \ww-\ww^{\prime}, \xx \rangle) \xx.
\end{align*}
We thus have
\begin{align*}
\|\rgrad(\ww)-\rgrad(\ww^{\prime})\| \leq 2\exx\|\xx\|^{2}\|\ww-\ww^{\prime}\| \leq 2c^{2}\|\ww-\ww^{\prime}\|,
\end{align*}
meaning smoothness of the risk holds with $\paraSmooth = 2\exx\|\xx\|^{2}$, satisfying \ref{asmp:risk_smooth}.
\end{ex}

\begin{ex}[Concrete example of assumption \ref{asmp:variance_bound}]\label{ex:variance_bound}
Again consider a setting similar to Examples \ref{ex:loss_smooth}--\ref{ex:risk_smooth}, but with added assumptions that $\exx \xx = 0$, that the noise $\eta$ and input $\xx$ are independent, and that the components of $\xx = (x_{1},\ldots,x_{d})$ are independent of each other. Some straightforward algebra shows that
\begin{align*}
\exx(\lgsub_{j}(\ww;\zz))^{2} & = 4\left( \exx x_{j}^{2} \langle \ww-\wwstar, \xx \rangle^{2} + \exx\eta^{2}\exx x_{j}^{2} \right)\\
& \leq 4\left( \|\ww-\wwstar\|^{2}\exx{}x_{j}^{2}\|\xx\|^{2} + \exx\eta^{2}\exx x_{j}^{2} \right).
\end{align*}
It follows that as long as the noise $\eta$ has finite variance ($\exx\eta^{2} < \infty$), and all inputs have finite fourth moments $\exx x_{j}^{4} < \infty$, then using assumption \ref{asmp:model_compact}, we get 
\begin{align*}
\exx(\lgsub_{j}(\ww;\zz))^{2} \leq 4\left( \diameter^{2}\exx{}x_{j}^{2}\|\xx\|^{2} + \exx\eta^{2}\exx x_{j}^{2} \right) < \infty.
\end{align*}
This holds for all $\ww \in \WW$, satisfying \ref{asmp:variance_bound}.
\end{ex}

\begin{ex}[Concrete example of assumption \ref{asmp:risk_strong_convex}]\label{ex:risk_strong_convex}
Consider the same setup as Example \ref{ex:variance_bound}. Since $\exx x_{j}\eta = (\exx x_{j})(\exx\eta) = 0$ for each $j \in [d]$, it follows that the risk induced by the squared loss under this model takes a convenient quadratic form,
\begin{align*}
\risk(\ww) = \exx \loss(\ww;\zz) = (\ww-\wwstar)^{T}A(\ww-\wwstar) + b^{2},
\end{align*}
with $A = \exx\xx\xx^{T}$ and $b^{2}=\exx\eta^{2}$. For concreteness, say all the components of $\xx$ have variance $\exx x_{j}^{2} = \sigma^{2}$, recalling that $\exx x_{j}=0$ by assumption. Then the Hessian matrix of $\risk(\cdot)$ is $\risk^{\prime\prime}(\ww)=\exx\xx\xx^{T}=\sigma^{2}I_{d}$, for all $\ww \in \WW$. For any $\ww, \ww^{\prime} \in \WW$, taking an exact Taylor expansion, we have that
\begin{align*}
\risk(\ww) = \risk(\ww^{\prime}) + \langle \rgrad(\ww), \ww-\ww^{\prime} \rangle + \frac{1}{2} \langle \ww-\ww^{\prime}, \risk^{\prime\prime}(\uu)(\ww-\ww^{\prime}) \rangle
\end{align*}
for some appropriate $\uu$ on the line segment between $\ww$ and $\ww^{\prime}$. Since the Hessian is positive definite with factor $\sigma^{2}$, the last term on the right-hand side can be no smaller than $\|\ww-\ww^{\prime}\|^{2}\sigma^{2}/2$. This implies a lower bound,
\begin{align*}
\risk(\ww) \geq \risk(\ww^{\prime}) + \langle \rgrad(\ww), \ww-\ww^{\prime} \rangle + \frac{\sigma^{2}}{2}\|\ww-\ww^{\prime}\|^{2}.
\end{align*}
The exact same inequality holds for any choice of $\ww$ and $\ww^{\prime}$. This is precisely the definition of strong convexity of $\risk(\cdot)$ given in (\ref{eqn:strong_convex_defn}), with convexity parameter $\paraSC = \sigma^{2}$, satisfying \ref{asmp:risk_strong_convex}.
\end{ex}

In the analysis that follows, \ref{asmp:model_compact}--\ref{asmp:variance_bound} are assumed to hold.

\paragraph{Analysis of Algorithm \ref{algo:rgdmult} with discussion}

Here we consider the learning performance of the proposed procedure given by Algorithm \ref{algo:rgdmult}. Almost every step of the procedure is given explicitly, save for the means of setting the moment bound $\mv{v}_{(t)}$ (see section \ref{sec:tests}), and the step size setting of $\alpha_{(t)}$. In the subsequent analysis of section \ref{sec:theory}, we shall specify exact settings of $\alpha_{(t)}$, and show how these settings impact the final guarantees that can be made.

We begin with a general fact that shows the sub-routine used to estimate each element of the risk gradient has sharp guarantees under weak assumptions.
\begin{lem}[Pointwise accuracy]\label{lem:pointwise_accuracy}
Consider data $x_{1},\ldots,x_{n}$, with distribution $x \sim \ddist$. Assume finite second moments, and a known upper bound $\exx_{\ddist}x^{2} \leq v < \infty$. With probability no less than $1-\delta$, the estimator $\widehat{x}$ defined in (\ref{eqn:estimator_defn}), scaled using $s = \sqrt{nv/(2\log(\delta^{-1}))}$, satisfies
\begin{align*}
|\widehat{x} - \exx_{\ddist}x| \leq \sqrt{\frac{2v\log(\delta^{-1})}{n}} + \sqrt{\frac{v}{n}}.
\end{align*}
\end{lem}
\begin{rmk}[Scaling in Lemma \ref{lem:pointwise_accuracy}]\label{rmk:scaling_method}
Note that the scale setting $s > 0$ in the above lemma depends on the sample size $n$ and the second moment of the underlying distribution. This can be derived from an exponential tail bound on the deviations of this estimator, namely we have that for any choice of $s>0$,
\begin{align*}
\prr\left\{ |\widehat{x} - \exx_{\ddist}x| \leq \frac{v}{2s} + \frac{s\log(\delta^{-1})}{n} + \sqrt{\frac{v}{n}} \right\} \geq 1-\delta.
\end{align*}
Choosing $s$ to minimize this upper bound yields the final results given in the lemma. Having a scale large relative to the second moments ensures the estimator behaves similarly regardless of the underlying scale of the data. Note that it also increases with sample size $n$: this is intuitive since in most cases, given a larger sample, we can allow the estimator to be more sensitive to outliers, earning a reduction in bias.
\end{rmk}
Of critical importance is that Lemma \ref{lem:pointwise_accuracy} only assumes finite variance, nothing more. Higher-order moments may be infinite or undefined, and the results still hold. This means the results hold for both Gaussian-like well-behaved data, and heavy-tailed data which are prone to errant observations. Next, we show that this estimator has a natural continuity property.
\begin{lem}[Estimator is Lipschitz]\label{lem:est_lipschitz}
Considering the estimator $\widehat{x}$ defined in (\ref{eqn:estimator_defn}) as a function of the data $\xx \defeq (x_{1},\ldots,x_{n}) \in \RR^{n}$, it satisfies the following Lipschitz property:
\begin{align*}
|\widehat{x}(\xx)-\widehat{x}(\xx^{\prime})| \leq \frac{c_{\prior}}{n}\|\xx - \xx^{\prime}\|_{1}, \quad \text{ for all } \xx,\xx^{\prime} \in \RR^{n}
\end{align*}
where the factor $c_{\prior}$ takes the form
\begin{align*}
c_{\prior} = 1 - 2\Phi\left(-\sqrt{\beta}\right) + \sqrt{\frac{2}{\beta\pi}}\exp\left(-\frac{\beta}{2}\right)
\end{align*}
where $\Phi(u) \defeq \prr\{N(0,1) \leq u\}$, the cumulative distribution function of the standard Normal distribution.
\end{lem}
At each step in an iterative procedure, we have some candidate $\ww$, at which we can evaluate the loss $\loss(\ww;\zz_{i})$ and/or the gradient $\lgrad(\ww;\zz_{i})$ over some or all data points $i \in [n]$. In traditional ERM-GD, one simply uses the empirical mean of the loss gradients to approximate $\rgrad(\ww)$. In our proposed robust gradient descent procedure, instead of just doing summation, we feed the loss gradients as data into the robust procedure (\ref{eqn:estimator_defn}), highlighted in Lemma \ref{lem:pointwise_accuracy}. Running this sub-routine for each dimension results in a novel estimator $\rgest(\ww)$ of the risk gradient $\rgrad(\ww)$, to be plugged into (\ref{eqn:update_GD_approx}), constructing a novel steepest descent update. Since the candidate $\ww$ at any step will depend on the random draw of the data set $\zz_{1},\ldots,\zz_{n}$, upper bounds on the estimation error must be uniform in $\ww \in \WW$ in order to capture all contingencies. More explicitly, we require for some bound $0 < \varepsilon < \infty$ that
\begin{align}\label{eqn:conf_uniform}
\prr\left\{ \max_{t \leq T} \|\rgest(\wwhat_{(t)})-\rgrad(\wwhat_{(t)})\| \leq \varepsilon \right\} \geq 1-\delta.
\end{align}
Using the following lemma, we can show that such a bound does exist, and its form can be readily characterized.
\begin{lem}[Uniform accuracy of gradient estimates]\label{lem:uniform_grad_accuracy}
Consider the risk gradient approximation $\rgest = (\rgestsub_{1}, \ldots, \rgestsub_{d})$, defined at $\ww$ (for $j \in [d]$) by
\begin{align}\label{eqn:best_settings_RGD}
\rgestsub_{j}(\ww) \defeq \frac{s_{j}}{n} \sum_{i=1}^{n} \int \psi\left(\frac{\lgsub_{j}(\ww;\zz_{i})(1+\epsilon_{i})}{s_{j}}\right) \, d\prior(\epsilon_{i}), \text{ scaled as } s_{j}^{2} = \frac{nv_{j}}{2\log(\delta^{-1})},
\end{align}
with $v_{j}$ any valid bound satisfying $v_{j} \geq \exx_{\ddist}|\lgsub_{j}(\ww;\zz)|^{2}$, for all $\ww \in \WW$. Then, with probability no less than $1-\delta$, for any choice of $\ww \in \WW$, we have that
\begin{align*}
\|\rgest(\ww) - \rgrad(\ww)\| \leq \frac{\widetilde{\varepsilon}}{\sqrt{n}},
\end{align*}
where writing $V \defeq \max_{j \in [d]} v_{j}$, the error $\widetilde{\varepsilon}$ is
\begin{align*}
\widetilde{\varepsilon} \defeq \paraSmooth(1+c_{\prior}\sqrt{d}) + \sqrt{2dV(\log(d\delta^{-1})+d\log(3\diameter\sqrt{n}/2))} + \sqrt{V}.
\end{align*}
\end{lem}
We now have that (\ref{eqn:conf_uniform}) is satisfied by our underlying routine, as just proved in Lemma \ref{lem:uniform_grad_accuracy}. The last remaining task is to disentangle the underlying optimization problem (minimization of unknown $\risk(\cdot)$) from the statistical estimation problem (approximating $\rgrad$ with $\rgest$), in order to control the distance between the output of Algorithm \ref{algo:rgdmult} after $T$ iterations, denoted $\wwhat_{(T)}$, and the minimizer $\wwstar$ of $\risk(\cdot)$.
\begin{lem}[Distance control]\label{lem:dist_control_strong}
Consider the general approximate GD update (\ref{eqn:update_GD_ideal}), and assume that (\ref{eqn:conf_uniform}) holds with bound $0 < \varepsilon < \infty$. Then, with probability no less than $1-\delta$, the following statements hold.
\begin{enumerate}
\item Setting $\alpha_{(t)} = \alpha/\SCfactor$, with $0 < \alpha < 1$, we have 
\begin{align*}
\|\wwhat_{(T)}-\wwstar\| \leq (1-\alpha)^{T/2}\|\wwhat_{(0)}-\wwstar\| + \frac{2\varepsilon}{\SCfactor}.
\end{align*}

\item Setting $\alpha_{(t)} = 1/((2+t)\SCfactor)$, we have
\begin{align*}
\|\wwhat_{(T)}-\wwstar\| \leq \frac{1}{\sqrt{t+2}}\|\wwhat_{(0)}-\wwstar\| + \frac{\varepsilon}{\SCfactor}.
\end{align*}
\end{enumerate}
\end{lem}

Our preparatory lemmas are now complete, and we can finally focus on the risk itself. We are considering bounds on the excess risk, namely the difference between the risk achieved by our procedure $\risk(\wwhat_{(T)})$, and $\risk^{\ast} \defeq \inf \{\risk(\ww): \ww \in \WW\} = R(\wwstar)$, namely the best possible performance using $\WW$.

\begin{thm}[Excess risk bounds, fixed step size]\label{thm:riskbd_fixed_strong}
Write $\wwhat_{(T)}$ for the output of Algorithm \ref{algo:rgdmult} after $T$ iterations, assuming step size $\alpha_{(t)} = \alpha/\SCfactor$, and moment bounds $\mv{v}_{(t)} \leq V$ for all $t$. It follows that
\begin{align*}
\risk(\wwhat_{(T)}) - \risk^{\ast} & \leq (1-\alpha)^{T}\paraSmooth\|\wwhat_{(0)} - \wwstar\|^{2} + \frac{4\paraSmooth \widetilde{\varepsilon}}{\paraSC^{2}n}\\
& = O\left((1-\alpha)^{T}\right) + O\left(\frac{d(\log(d\delta^{-1})+d\log(\diameter{}n))}{n}\right)
\end{align*}
with probability no less than $1-\delta$ over the random draw of the sample $\zz_{1},\ldots,\zz_{n}$, where $\widetilde{\varepsilon}$ and $V$ are as defined in Lemma \ref{lem:uniform_grad_accuracy}.
\end{thm}

An immediate observation is that if we let $T$ scale with $n$ such that $T \to \infty$ as $n \to \infty$, we have convergence in probability as for any $\varepsilon > 0$, it holds that
\begin{align*}
\lim\limits_{n \to \infty} \prr\{ \risk(\wwhat_{(T)})-\risk^{\ast} > \varepsilon \} = 0,
\end{align*}
and that indeed to get arbitrarily good performance at confidence $1-\delta$, one requires on the order of $d^{2}\log(\delta^{-1})$ observations. In terms of learning efficiency, the rate of convergence becomes more important than the simple fact of consistency. In the remark below, we compare our results with closely related works in the literature.

\begin{rmk}[Comparison with other RGD]\label{rmk:compare_rgd}
Among recent work in the literature, conceptually the closest work to ours are those proposing and analyzing novel ``robust gradient descent'' algorithms. As with ours, these procedures look to replace the traditional sample mean-based risk gradient estimate with a more robust approximation, within the context of a first-order optimization scheme. As we mentioned in section \ref{sec:intro}, two particularly closely related works are \citet{chen2017b} and \citet{prasad2018a}, both using a generalized median-of-means strategy. On the computational side, our Algorithm \ref{algo:rgdmult} requires only a fixed number of basic operations, applied to each coordinate and each data point; we have no iterative sub-routines here. This provides an advantage over median-of-means procedures which require iterative approximations of the geometric median at each update in the main loop. Theoretically, while the setting of \citet{chen2017b} is that of parallel computing with robustness to failures, their formal guarantees have essentially the same dependence on $d$ and $n$ as our bounds in Theorem \ref{thm:riskbd_fixed_strong}, under comparable assumptions. On the other hand, \citet{prasad2018a}, using the same tactic as \citet{chen2017b}, provide new formal guarantees with better dependence on the dimension, but at the cost of a new $T$ factor in the statistical error term. More concretely, we consider their Theorem 8, in which they provide error bounds on a robust linear regression procedure. Consider assumptions as in our Example \ref{ex:variance_bound}, where $\zz = (\xx,y)$ and $y = \langle \wwstar,\xx \rangle + \eta$. Writing $(\wwtil_{(t)})$ for the sequential output of their procedure. Under such assumptions, they assert $(1-\delta)$-probability bounds of the form
\begin{align*}
\|\wwtil_{(T)}-\wwstar\| \leq O\left(a^{T}\right) + O\left(\frac{\sigma}{1-a}\sqrt{\frac{Td\log(\delta^{-1})}{(n/k)}}\right)
\end{align*}
for a constant $0 < a < 1$, where $k$ is the number of partitions made, and $\exx\xx\xx = \sigma^{2}I_{d}$. The reason for this form is as follows. Using their Lemma 2, based on error bounds from \citet{minsker2015a}, one gets a pointwise bound on the error gradient estimate. In their Theorem 1 proof, they take a union over all algorithm iterations, and conclude with bounds on $\|\wwtil_{(T)}-\wwstar\|$ taking the form just stated. A naive approach re-using the same data over each step $t=0,1,\ldots,T$ of the algorithm means the loss gradient observations are no longer independent, then making Lemma 2 invalid. To get around this, the authors split the original $n$ independent observations into $T$ disjoint subsets to be used for their proposed sub-routine (involving a further $k$-partition). Union bounds over $T$ steps are now perfectly valid, but the data size at each step becomes $n/(Tk)$, and even $T=\Theta(\sqrt{n})$ leads to a very slow rate of $O(n^{-1/4})$. Our bounds have an extra $d$ factor, but are free of $T$ in the statistical error, while also maintaining $O(n^{-1/2})$ rates.
\end{rmk}

\begin{rmk}[Projected RGD]
One implicit assumption in our analysis above is that $\wwhat_{(t)} \in \WW$ for all steps of Algorithm \ref{algo:rgdmult}. To enforce this, running a projection sub-routine after the parameter update step is sufficient, and all theoretical guarantees hold as-is. To see this, note that the update becomes
\begin{align}\label{eqn:update_projection}
\wwhat_{(t+1)} = \pi_{\WW}\left(\wwhat_{(t)} - \alpha_{(t)}\rgest_{(t)}(\wwhat_{(t)})\right)
\end{align}
where $\pi_{\WW}(\vv) \defeq \argmin_{\uu \in \WW}\|\uu-\vv\|$. Under \ref{asmp:model_compact}, this projection is well-defined \citep[Sec.~3.12, Thm.~3.12]{luenberger1969Book}. With this fact in hand, it follows that $\|\pi_{\WW}(\uu)-\pi_{\WW}(\vv)\| \leq \|\uu-\vv\|$ for any choice of $\uu,\vv \in \WW$. Again via \ref{asmp:model_compact}, since $\wwstar \in \WW$ and $\pi_{\WW}(\wwstar)=\wwstar$, we have
\begin{align*}
\left\| \pi_{\WW}\left(\wwhat_{(t)}-\alpha_{(t)}\rgest_{(t)}(\wwhat_{(t)})\right) - \wwstar \right\| \leq \|\wwhat_{(t+1)}-\wwstar\|
\end{align*}
implying that Lemma \ref{lem:dist_control_strong} applies as-is to Algorithm \ref{algo:rgdmult} modified using projection to $\WW$ as in (\ref{eqn:update_projection}), thereby extending all subsequent results based on it.
\end{rmk}

One would expect that with robust estimates of the risk gradient that over a wide variety of distributions, that the updates of Algorithm \ref{algo:rgdmult} should have small variance given enough observations. The following result shows that this is true, with the procedure stabilizing to the best level available under the given sample as the procedure closes in on a valid solution.

\begin{thm}[Control of update variance]\label{thm:variance_control}
Run Algorithm \ref{algo:rgdmult} under the same assumptions as Theorem \ref{thm:riskbd_fixed_strong}, except with step-size $\alpha_{(t)}$ left arbitrary. Then, for any step $0 \geq t$, taking expectation with respect to the sample $\{\zz_{i}\}_{i=1}^{n}$ conditioned on $\wwhat_{(t)}$, we have
\begin{align*}
\exx\|\wwhat_{(t+1)}-\wwhat_{(t)}\|^{2} \leq 2\alpha_{(t)}^{2} \left( \frac{d\sqrt{2\pi{}b^{2}}}{2}\left( \sqrt{\frac{V}{n}}\left( 1 - 2\Phi\left(\frac{-1}{\sqrt{d}}\right) \right) + \sqrt{\frac{2Vd}{n\pi}}e^{-2d} \right) + \|\rgrad(\wwhat_{(t)})\|^{2} \right).
\end{align*}
\end{thm}

\section{Empirical analysis}\label{sec:tests}

In the numerical experiments that follow, our primary goal is to elucidate the relationship that exists between factors of the learning task (e.g., sample size, model dimension, initial value, underlying data distribution) and the performance of the robust gradient descent procedure proposed in Algorithm \ref{algo:rgdmult}. We are interested in how these factors impact algorithm behavior in an absolute sense, as well as performance relative to well-known competitors.

We are considering three basic types of experiments. First, we develop a risk minimization task based on noisy function (and thus noisy gradient) observations. These controlled simulations let us carefully examine how different factors influence performance over time. Next, we consider a regression task under a wide variety of noise distributions, which lets us examine the true utility of competing algorithms under a scenario where the data may or may not be heavy-tailed. Finally, we use real-world benchmark data sets to evaluate performance on classification tasks.

\subsection{Controlled tests}\label{sec:tests_noisyopt}

\paragraph{Experimental setup}

We begin with a ``noisy convex risk minimization'' task, designed as follows. The risk function itself takes a quadratic form, as $R(\ww) = \langle \Sigma\ww, \ww \rangle/2 + \langle \ww, \uu \rangle + c$, where $\Sigma \in \RR^{d \times d}$, $\uu \in \RR^{d}$, and $c \in \RR$ are constants set in advance. The learning task is to find a minimizer of $R(\cdot)$, without direct access to $R$, rather only access to $n$ random function data $r_{1},\ldots,r_{n}$, with $r:\RR^{d} \to \RR$ mapping from parameter space to a numerical penalty. This data is generated independently from a common distribution, and are centered at the true risk, namely $\exx r(\ww) = R(\ww)$ for all $\ww \in \RR^{d}$. More concretely, we generate $r_{i}(\ww)=(\langle \wwstar-\ww, \xx_{i}\rangle + \epsilon_{i})^{2}/2$, $i \in [n]$, with $\xx$ and $\epsilon$ independent. The true minimum is denoted $\wwstar$, and $\Sigma = \exx \xx\xx^{T}$. The inputs $\xx$ are set to have a $d$-dimensional Gaussian distribution with all components uncorrelated. This means that $\Sigma$ is positive definite, and $R$ is strongly convex.

We make use of three metrics for evaluating performance here: average excess empirical risk (averaging of $r_{1},\ldots,r_{n}$), average excess risk (computed using true $R$), and variance of the risk. The latter two are computed by averaging over trials; each trial means a new independent random sample. In all tests, we conduct 250 trials.

Regarding methods tested, we run three representative procedures. First is the idealized gradient descent procedure (\ref{eqn:update_GD_ideal}), denoted \texttt{oracle}, which is possible here since $R$ is designed by us. Second, as a \textit{de facto} standard for most machine learning algorithms, we use ERM-GD, written \texttt{erm}. Here the update direction is simply the sample mean of the loss gradient. Finally, we compare our Algorithm \ref{algo:rgdmult}, written \texttt{rgdmult}, against these two procedures. Variance bounds $\mv{v}_{(t)}$ are computed using the simplest possible procedure, namely the empirical mean of the second moments of $\lgrad(\wwhat_{(t)};\zz)$, divided by two.

\paragraph{Impact of heavy-tailed noise}

Our first inquiry is a basic proof of concept: are there natural problem settings under which using \texttt{rgdmult} over ERM-GD is advantageous? How does this procedure perform when ERM-GD is known to be effectively optimal? Under Gaussian noise, ERM-GD is effectively optimal \citep[Appendix C]{lin2016a}. As a baseline, we start with Gaussian noise (mean $0$, standard deviation $20$), and then consider centered log-Normal noise (log-location $0$, log-scale $1.75$) as a representative example of asymmetric, heavy-tailed data. Performance results are given in Figure \ref{fig:POC}.

\begin{figure}[t]
\centering
\includegraphics[width=1.0\textwidth]{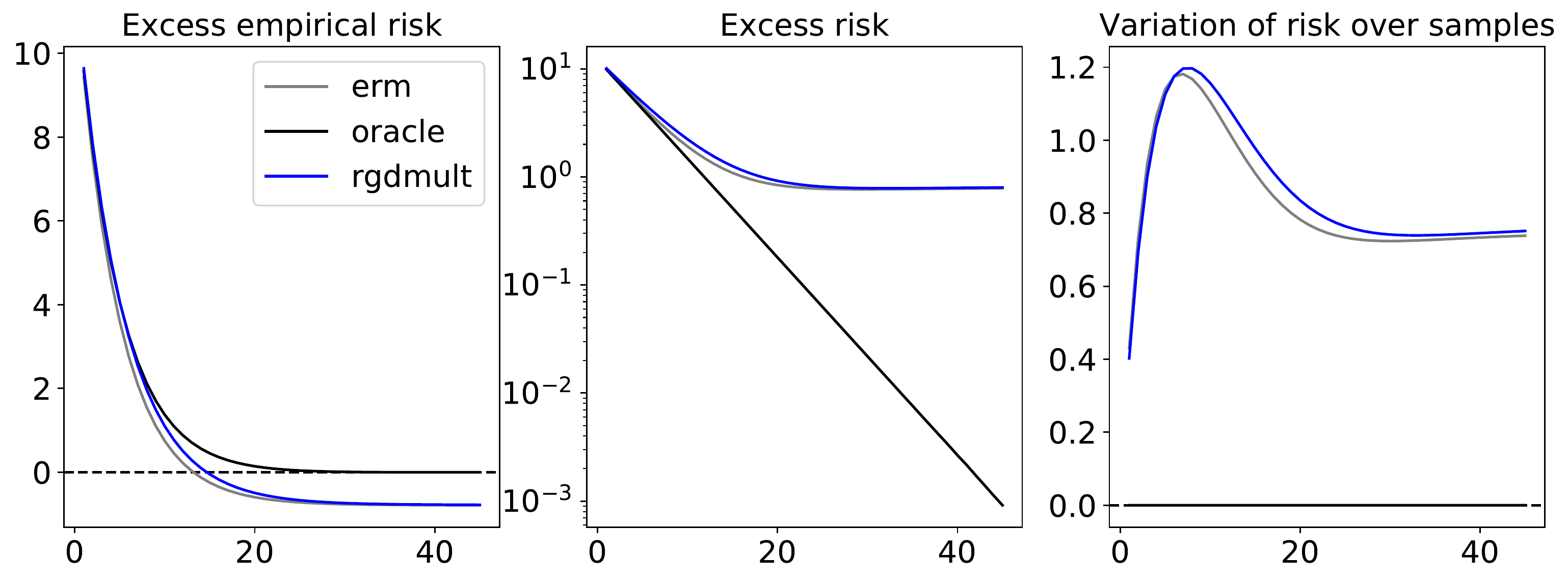}\\
\includegraphics[width=1.0\textwidth]{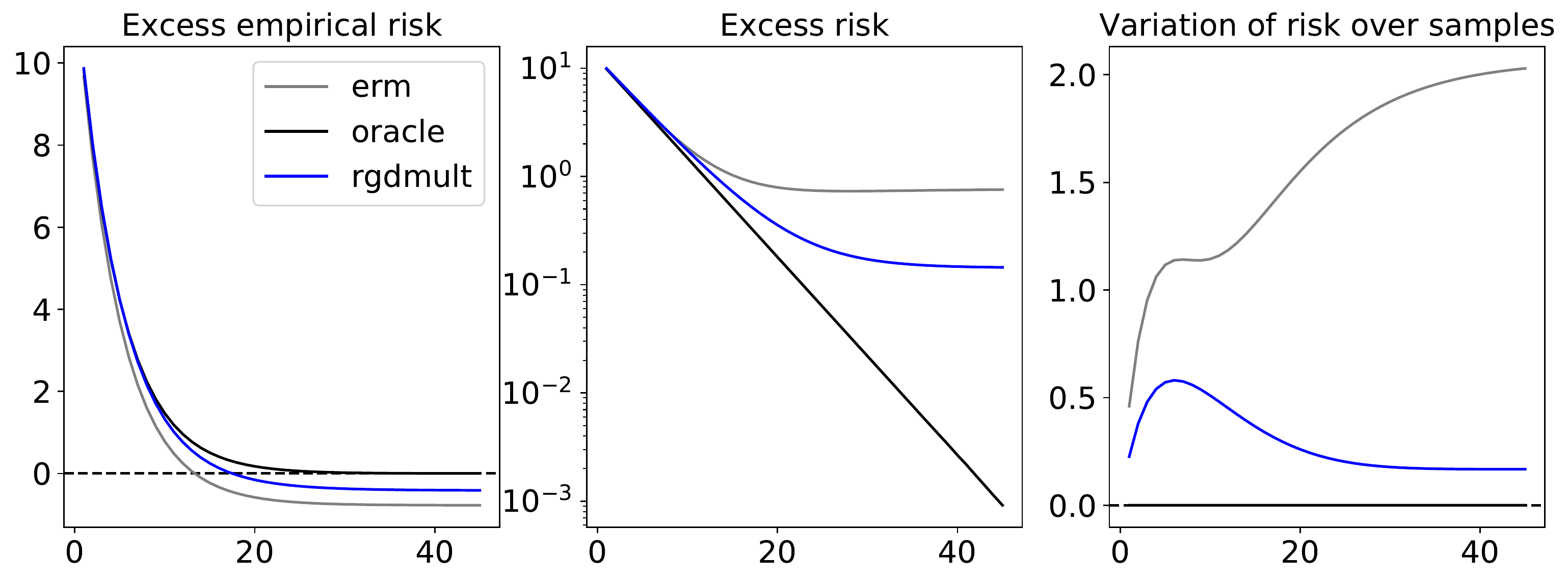}
\caption{Performance metrics as a function of iterative updates. Top row: Normal noise. Bottom row: log-Normal noise. Settings: $n = 500, d=2, \alpha_{(t)}=0.1$ for all $t$.}
\label{fig:POC}
\end{figure}

We see that when ERM-GD is known to be strong, both algorithms perform almost the same. In contrast, when we have heavy-tailed data, we see that \texttt{rgdmult} is far superior in terms of both the solution found, and the stability over the random draw of the sample. While compared with the oracle procedure, there is clearly some overfitting, we see that \texttt{rgdmult} departs from the oracle procedure at a much slower rate than ERM-GD, a desirable property.

Moving forward, we look more systematically at how different experimental settings lead to different performance, all else kept constant.

\paragraph{Impact of initialization}

Having fixed the underlying distribution and number of observations $n$, here we look at the impact of the initial guess $\wwhat_{(0)}$. We look at three initializations, taking the form $\wwstar + \text{Unif}[-\mv{\Delta},\mv{\Delta}]$, with $\mv{\Delta} = (\Delta_{1},\ldots,\Delta_{d})$, and values ranging over $\Delta_{j} \in \{2.5, 5.0, 10.0\}$, $j \in [d]$. Here larger values of $\Delta_{j}$ correspond to potentially worse initialization. Results are displayed in Figure \ref{fig:INIT}.

\begin{figure}[t]
\centering
\includegraphics[width=1.0\textwidth]{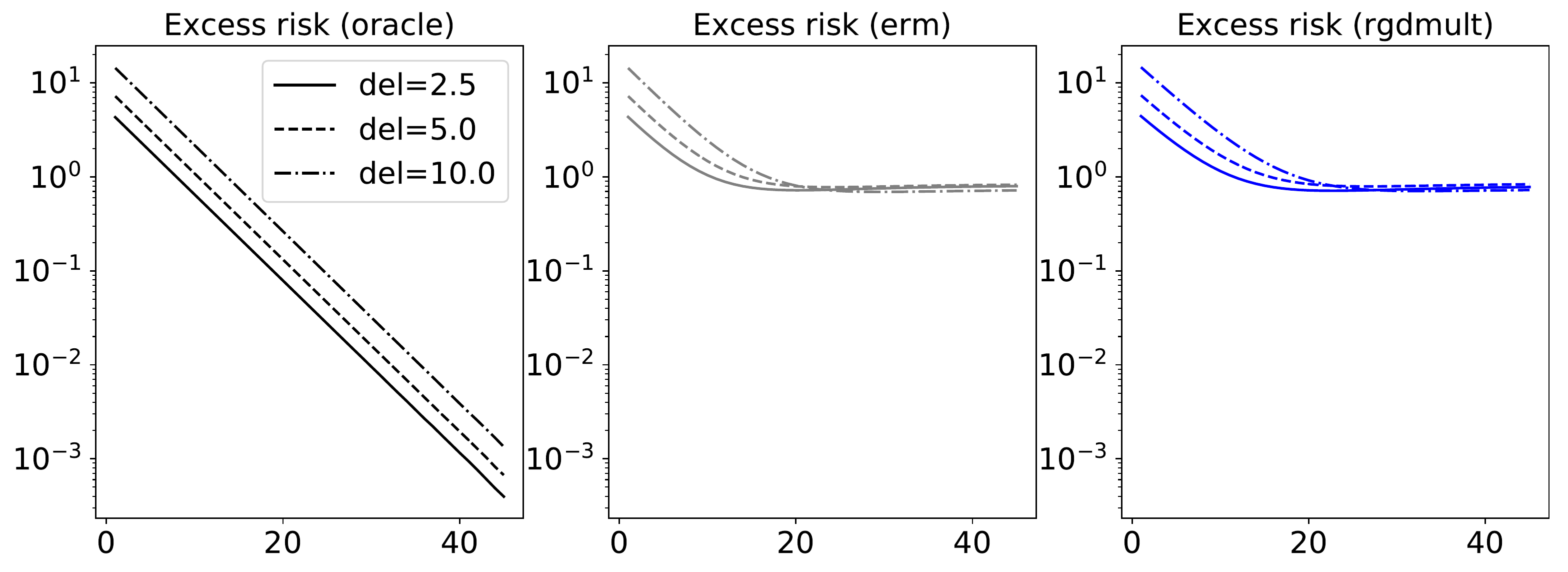}\\
\includegraphics[width=1.0\textwidth]{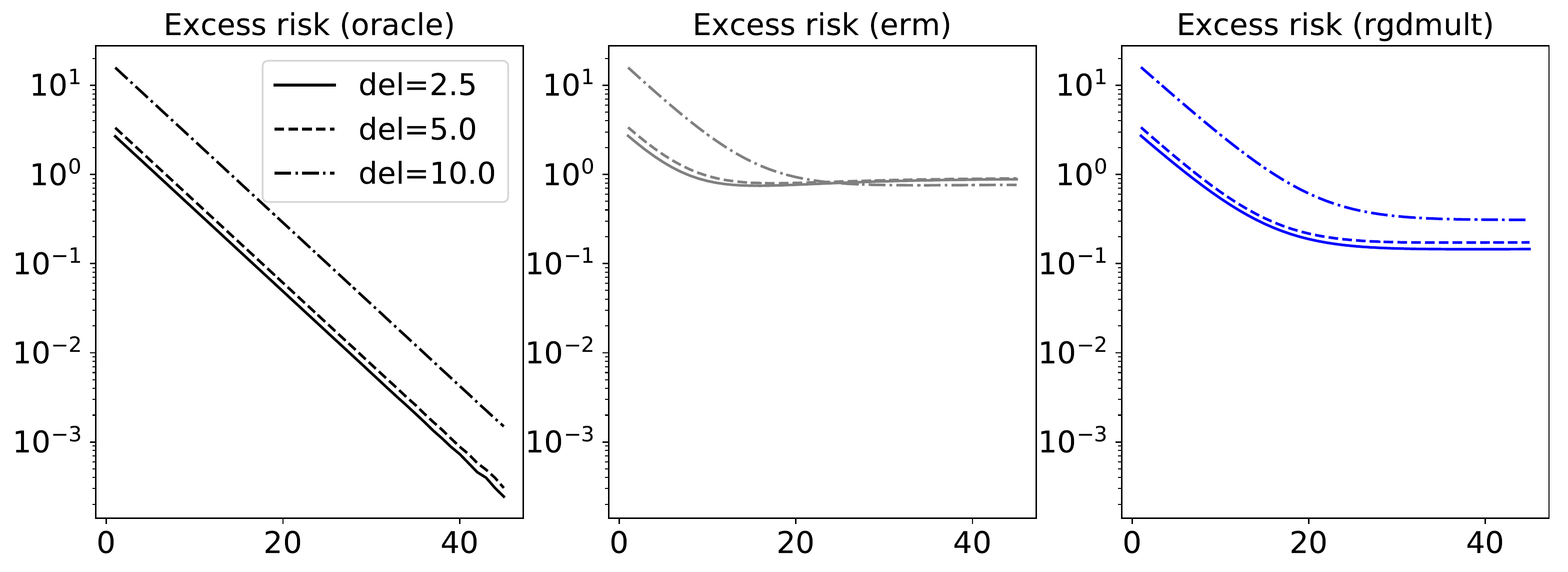}
\caption{Performance over iterations, under strong/poor initialization. Here \texttt{del} refers to $\Delta_{j}$. Top row: Normal noise. Bottom row: log-Normal noise. Settings: $n = 500, d=2, \alpha_{(t)}=0.1$ for all $t$.}
\label{fig:INIT}
\end{figure}

Several trends are clear. First, in the case where ERM-GD is essentially optimal, we see that \texttt{rgdmult} matches it. Furthermore, when the data is heavy-tailed, the proposed procedure is seen to be much more robust to a sub-par initial guess. While a bad start can lead to serious performance issues long-run in ERM-GD, we see that \texttt{rgdmult} can effectively recover.

\paragraph{Impact of distribution}

In our risk minimization task construction, we take advantage of the fact that very distinct loss distributions can still lead to precisely the same risk function. Here we examine performance as the underlying distribution changes; since all changes are of a purely statistical nature the oracle procedure is not affected, only ERM-GD and \texttt{rgdmult}. We consider six settings, three for Gaussian noise, and three for log-Normal noise. Location and scale parameters for the former are $(0,0,0), (1,20,34)$. Log-location and log-scale parameters for the latter are $(0,0,0), (1.25,1.75,1.9)$. Results are given in Figure \ref{fig:DIST}.

\begin{figure}[t]
\centering
\includegraphics[width=1.0\textwidth]{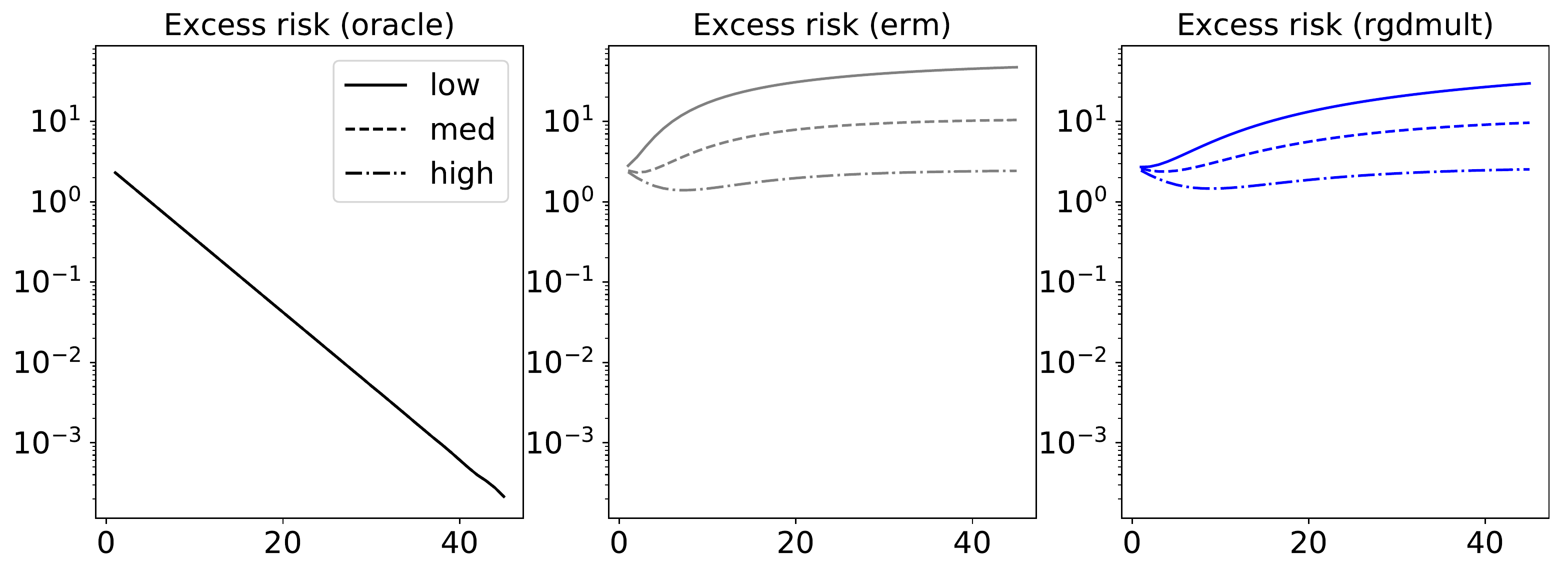}\\
\includegraphics[width=1.0\textwidth]{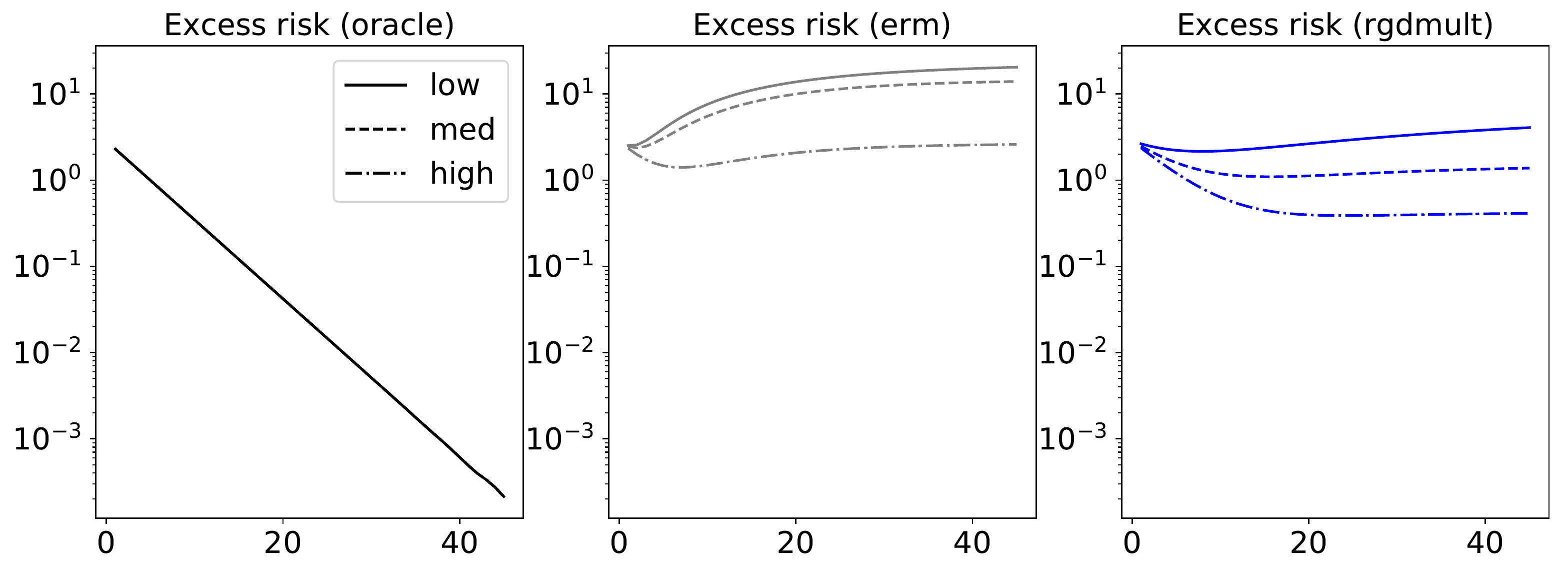}
\caption{Performance over iterations, under varying noise intensities. Here \texttt{low}, \texttt{med}, and \texttt{high} refer to the three noise distribution settings described in the main text. Settings: $n = 500, d = 2, \alpha_{(t)}=0.1$ for all $t$.}
\label{fig:DIST}
\end{figure}

In the case of Gaussian data, where we expect ERM-based methods to perform well, we see that \texttt{rgdmult} is able to match ERM-GD in all settings. Under log-Normal noise, the performance of ERM falls rather sharply, and we see a gap in performance that widens as the variance grows. Guarantees of good performance over a wide class of distributions for Algorithm \ref{algo:rgdmult}, without any prior knowledge of the underlying distribution are the key take-aways of the results culminating in Theorem \ref{thm:riskbd_fixed_strong}, and are reinforced clearly by these empirical test results, as well as those in subsequent sub-sections.

\paragraph{Impact of sample size}

As a direct measure of learning efficiency, we investigate how algorithm performance metrics change with the sample size $n$, with dimension fixed. Figure \ref{fig:NVAL} shows the accuracy of \texttt{erm} and ERM-GD in tests just like those given above. Initial values are common for all methods, and $n$ ranges over $\{10, 40, 160, 640\}$.

\begin{figure}[t]
\centering
\includegraphics[width=1.0\textwidth]{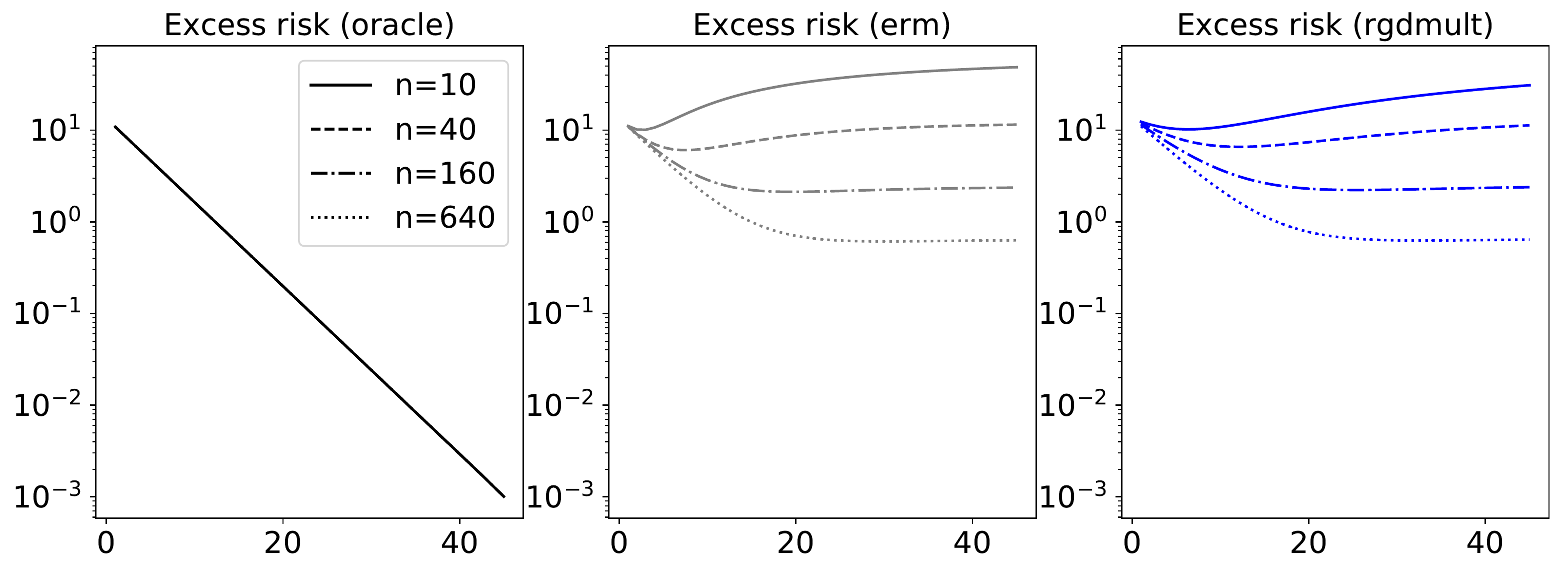}\\
\includegraphics[width=1.0\textwidth]{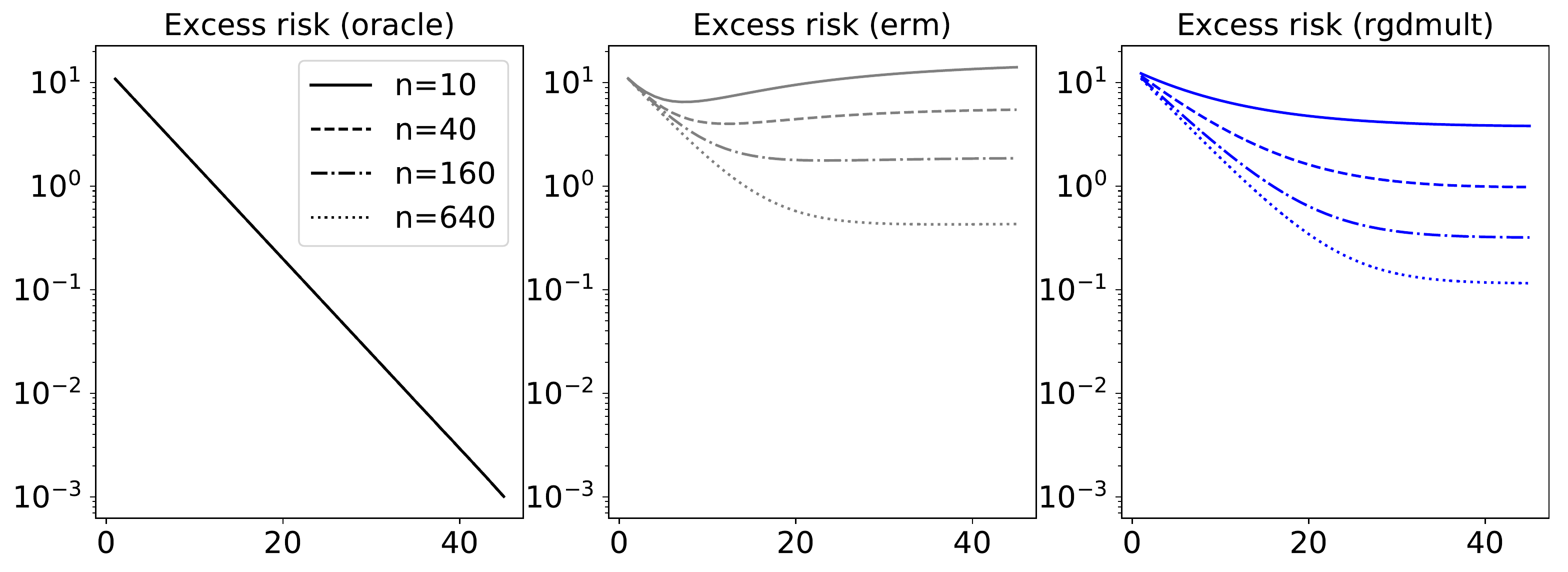}
\caption{Performance over iterations, under different sample sizes. Settings: $d=2, \alpha_{(t)}=0.1$ for all $t$.}
\label{fig:NVAL}
\end{figure}

As is natural, both procedures see monotonic performance improvement over increasing $n$. More salient is the strength of \texttt{rgdmult} under heavy-tailed observations, particularly when data is limited, giving clear evidence of better learning efficiency, in the sense of realizing better generalization in less time, with less data.

\paragraph{Impact of dimension}

The number of parameters the learning algorithm has to determine, here denoted as dimension $d$, makes a major impact on the overall difficulty of the learning task, and what sample sizes should be considered ``small.'' Fixing $n$, we let $d$ range over $\{2, 8, 32, 128\}$, and investigate how each algorithm performs with access to progressively less sufficient information. For our Algorithm \ref{algo:rgdmult}, we set the variance bound $\mv{v}_{(t)}$ to the empirical second moments of $\lgrad(\wwhat_{(t)};\zz)$, multiplied by $1/\sqrt{d}$. Figure \ref{fig:DVAL} gives results for these tests.

\begin{figure}[t]
\centering
\includegraphics[width=1.0\textwidth]{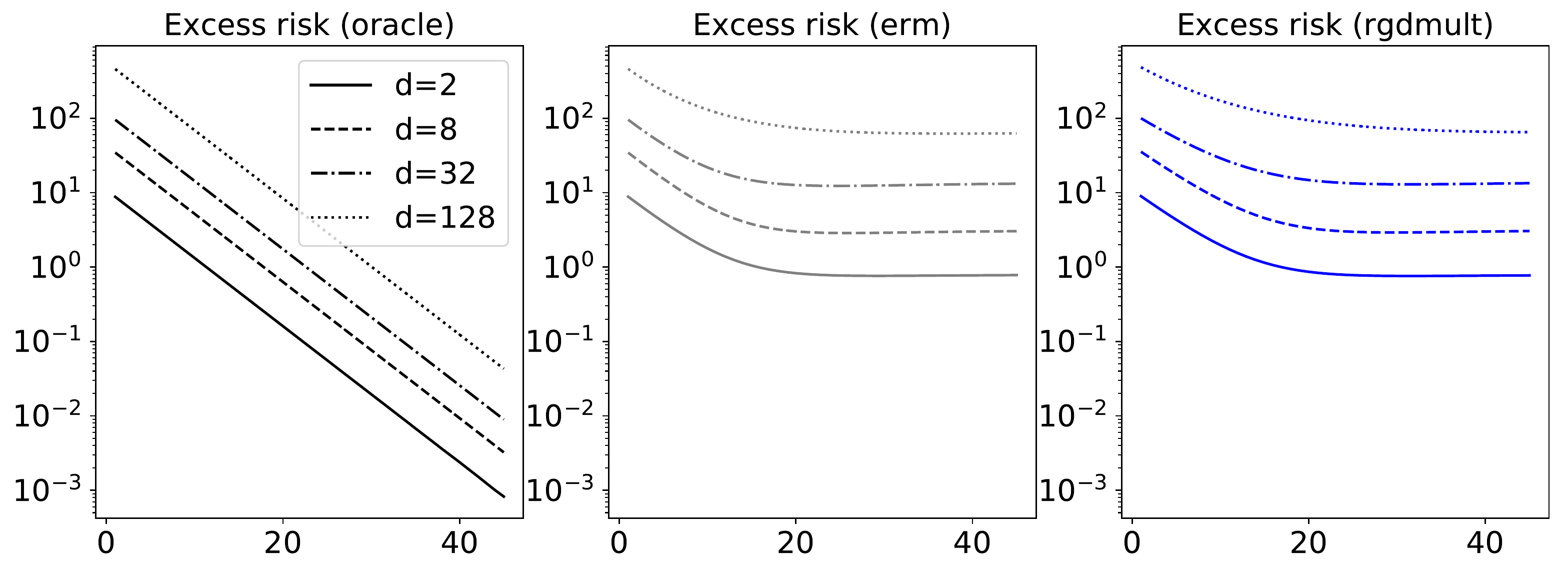}\\
\includegraphics[width=1.0\textwidth]{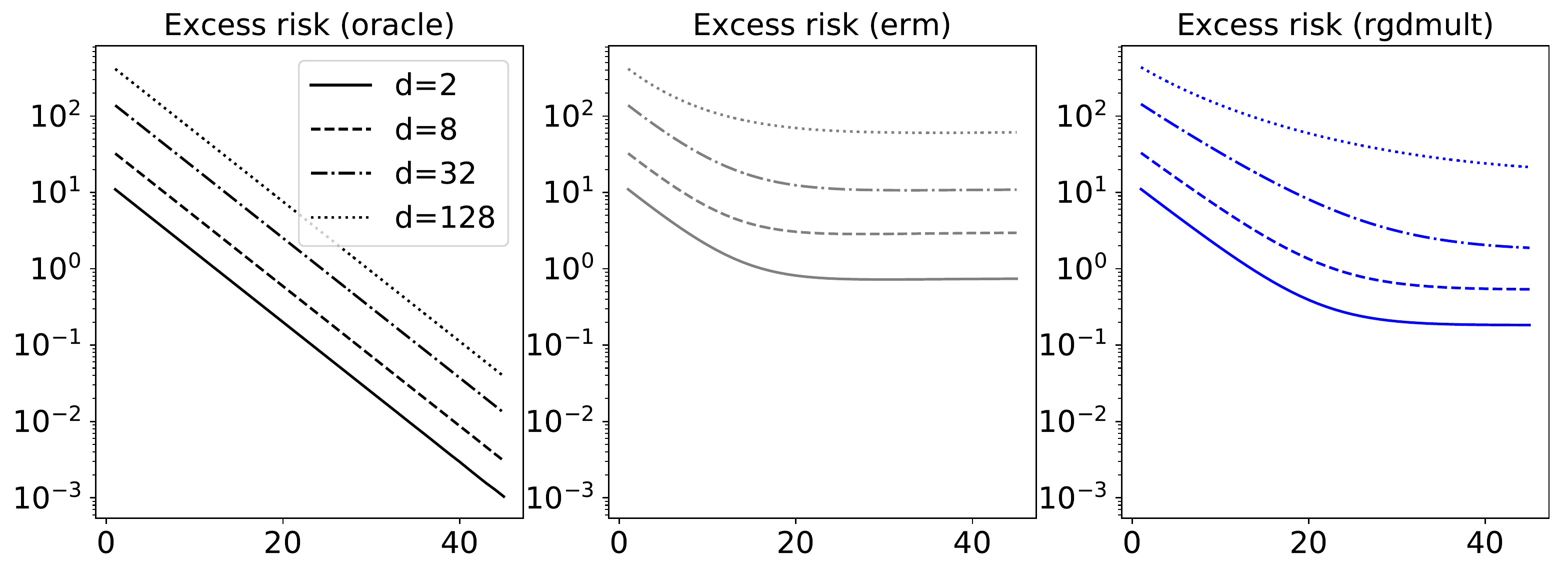}
\caption{Performance over iterations, under increasing dimension. Settings: $n = 500, \alpha_{(t)}=0.1$ for all $t$.}
\label{fig:DVAL}
\end{figure}

We see that with larger $d$, since $n$ is fixed, both non-oracle routines become less efficient, and need more time to converge. As with previous tests, the key difference appears under heavy tails, where we see Algorithm \ref{algo:rgdmult} is superior ove ERM-GD for all $d$ settings. While the ERM procedure saturates rather quickly, our procedure keeps improving for more iterations.

\paragraph{Comparison with robust loss minimizer}

In section \ref{sec:intro}, we cited the important work of \citet{brownlees2015a}, which chiefly considered theoretical analysis of a robust learning procedure that minimizes a robust \textit{objective}, in contrast to our use of a robust update direction. Our proposed procedure enjoys essentially the same theoretical guarantees, and we have claimed that it is more practical. Here we attempt to verify this claim empirically. Denote the method of \citet{brownlees2015a} by \texttt{bjl}. To implement their approach, which does not specify any particular algorithmic technique, we implement \texttt{bjl} using the non-linear conjugate gradient method of Polak and Ribi\`{e}re \citep{nocedal1999a}. This can be found as part of the the \texttt{optimize} module of the SciPy scientific computation library, called \texttt{fmin\_cg}, with default parameter settings. We believe that using this standard first-order solver makes for a fair comparison between \texttt{bjl} and our Algorithm \ref{algo:rgdmult}, again denoted \texttt{rgdmult}, and again with variance bound $\mv{v}_{(t)}$ set to the empirical second moments of $\lgrad(\wwhat_{(t)};\zz)$, multiplied by $1/\sqrt{d}$. For our routine, we have fixed the number of iterations to be $T=30$ for all settings. We compute the time required for computation using the Python \texttt{time} module. Multiple independent trials of each learning task (analogous to those previous) are carried out, with the median time taken over trials (for each $d$ setting) used as the final time record. We consider settings of $d=2,4,8,16,32,64$. These times along with performance results are given in Figure \ref{fig:versusBJL}.

\begin{figure}[t]
\centering
\includegraphics[width=1.0\textwidth]{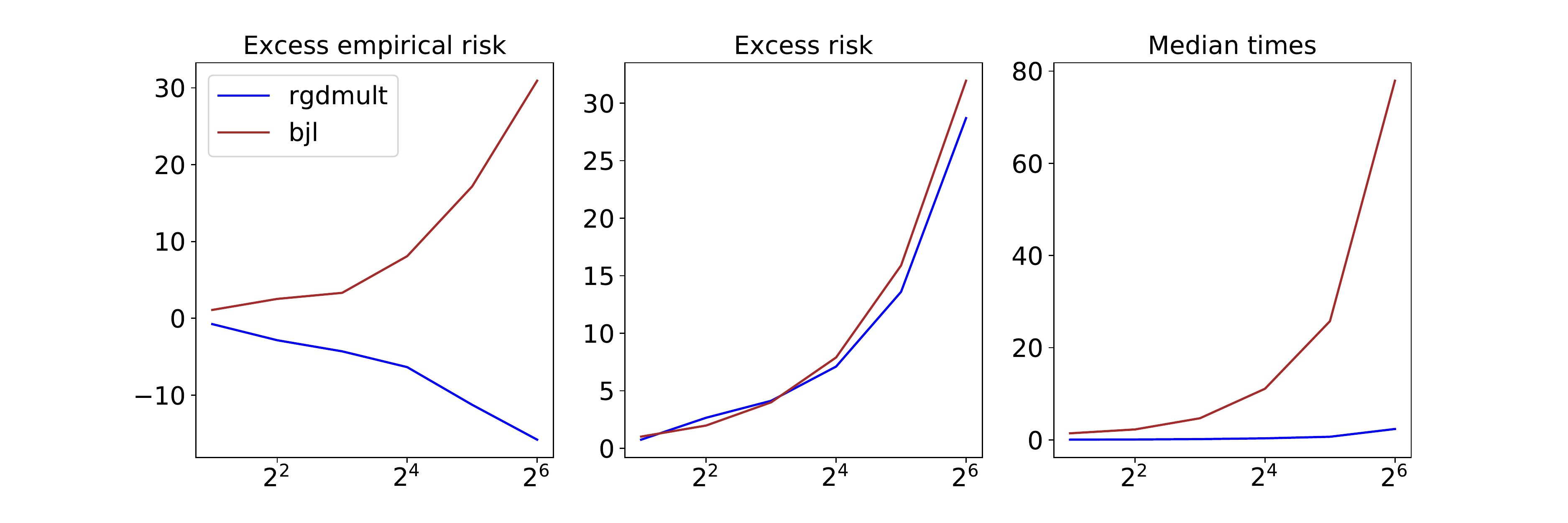}\\
\includegraphics[width=1.0\textwidth]{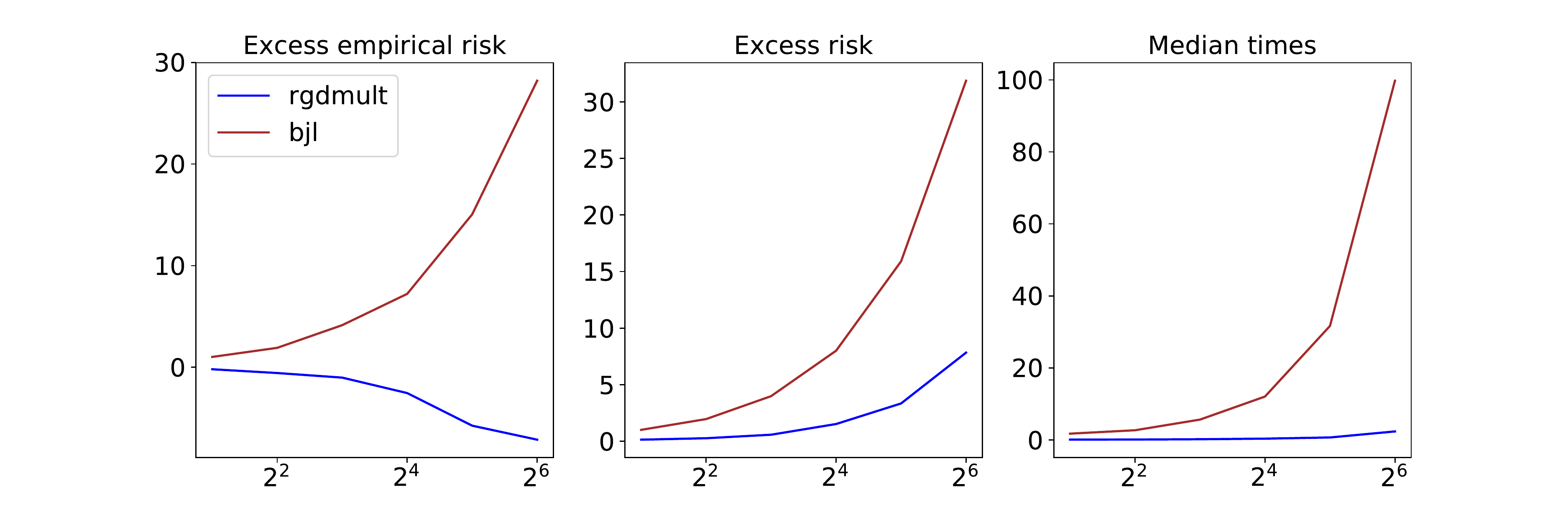}
\caption{Comparison of our robust gradient-based approach with the robust objective-based approach. Top: Normal noise. Bottom: log-Normal noise. Performance is given as a function of the number of $d$, the number of parameters to optimize, given in $\log_{2}$ scale. Settings: $n = 500, \alpha_{(t)}=0.1$ for all $t$.}
\label{fig:versusBJL}
\end{figure}

We can first observe that in low dimensions, and with data subject to Gaussian noise, the performance of both methods is simiular, a reassuring fact considering their conceptual closeness. Moving to higher dimensions, however, and especially under heavy-tailed noise, we see that our \texttt{rgdmult} achieves better performance in much less time. Note that \texttt{bjl} is optimizing a completely different objective function, explaining the deviation in excess empirical risk. There are certainly other ways of implementing \texttt{bjl}, but there is no way of circumventing the optimization of an explicitly-defined objective, which may not be convex. Our proposed \texttt{rgdmult} looks to offer a more practical alternative, which still enjoys the same statistical guarantees.

\paragraph{Regression application}

For our next class of experiments, we look at a more general regression task, under a diverse collection of data distributions. We then compare Algorithm \ref{algo:rgdmult} with well-known procedures specialized to regression, both classical and recent. In each experimental condition, and for each trial, we generate $n$ observations of the form $y_{i} = \xx_{i}^{T}\wwstar + \epsilon_{i}, i \in [n]$ for training. Each condition is defined by the setting of $(n,d)$ and $\mu$. Throughout, we have inputs $\xx$ which are generated from a $d$-dimensional Gaussian distribution, with each coordinate independent of the others. As such, to set $\mu$ requires setting the distribution of the noise, $\epsilon$. We consider several families of distributions, each with 15 distinct parameter settings, or ``noise levels.'' These settings are carried out such that the standard deviation of $\epsilon$ increases over the range $0.3--20.0$, in a roughly linear fashion as we increase from level 1 (lowest) to 15 (highest).

A range of signal/noise ratios can be captured by controlling the norm of the vector $\wwstar \in \RR^{d}$ determining the model. For each trial, we generate $\wwstar$ randomly as follows. Considering the sequence $w_{k} \defeq \pi/4 + (-1)^{k-1}(k-1)\pi/8, k=1,2,\ldots$, sample $i_{1},\ldots,i_{d} \in [d_{0}]$ uniformly, with $d_{0}=500$. The underlying vector is then set as $\wwstar = (w_{i_{1}},\ldots,w_{i_{d}})$. The signal to noise ratio $\text{SN}_{\mu} = \|\wwstar\|_{2}^{2}/\vaa_{\mu}(\epsilon)$ then varies over the range $0.2 \leq SN_{\mu} \leq 1460.6$. Here we consider four noise families: log-logistic (denoted \texttt{llog} in figures), log-Normal (\texttt{lnorm}), Normal (\texttt{norm}), and symmetric triangular (\texttt{tri\_s}). Many more are considered in Appendix \ref{sec:more_test_results}, and even with just these four, we have representative distributions with both bounded and unbounded sub-Gaussian noise, and heavy-tailed data both with and without finite higher-order moments.

Here we do not compute the risk $R$ exactly, but rather use off-sample prediction error as the key metric for evaluating performance. This is computed as excess root mean squared error (RMSE) computed on an independent testing set. Performance is averaged over independent trials. For each condition and trial, a test set of $m$ independent observations is generated identically to the $n$-sized training set that precedes testing. All competing methods use common samples for training and testing, for each condition and trial. In the $k$th trial, each algorithm outputs an estimate $\wwhat(h)$. Using RMSE to approximate the $\ell_{2}$-risk, compute $e_{k}(\wwhat) \defeq (m^{-1}\sum_{i=1}^{m}(\wwhat^{T}\xx_{k,i}-y_{k,i})^{2})^{1/2}$, outputting prediction error as the excess error $e_{k}(\wwhat(k)) - e_{k}(\wwstar(k))$, averaged over $K$ trials. In all experiments, we have $K=250$, $m=1000$.

We consider several methods against which we compare the proposed Algorithm \ref{algo:rgdmult}. As classical choices, we have ordinary least squares (ERM under the squared error, \texttt{ols}) and least absolute deviations (ERM under absolute error, \texttt{lad}). For more recent methods, as described in section \ref{sec:intro}, we consider robust regression routines as given by \citet{minsker2015a} (\texttt{geomed}) and \citet{hsu2016a} (\texttt{hs}). In the former, we partition the data, obtaining the \texttt{ols} solution on each subset, and these candidates are aggregated using the geometric median in the $\ell_{2}$ norm \citep{vardi2000a}. The number of partitions is set to $\max\{2,\lfloor n/(2d) \rfloor\}$. In the latter, we used source code published online by the authors. To compare our Algorithm \ref{algo:rgdmult} with these routines, we initialize \texttt{rgdmult} to the analytical \texttt{ols} solution, with step size $\alpha_{(t)} = 0.01$ for all iterations, and $\delta = 0.005$. Variance bounds $\mv{v}_{(t)}$ are set to the empirical second moments of $\lgrad(\wwhat_{(t)},\zz)$, divided by 2. In total, the number of iterations is constrained by a fixed budget: we allow for $40n$ gradient evaluations in total. Representative results are provided in Figure \ref{fig:multinoise_linreg}.

\begin{figure}[t]
\centering
\includegraphics[width=0.25\textwidth]{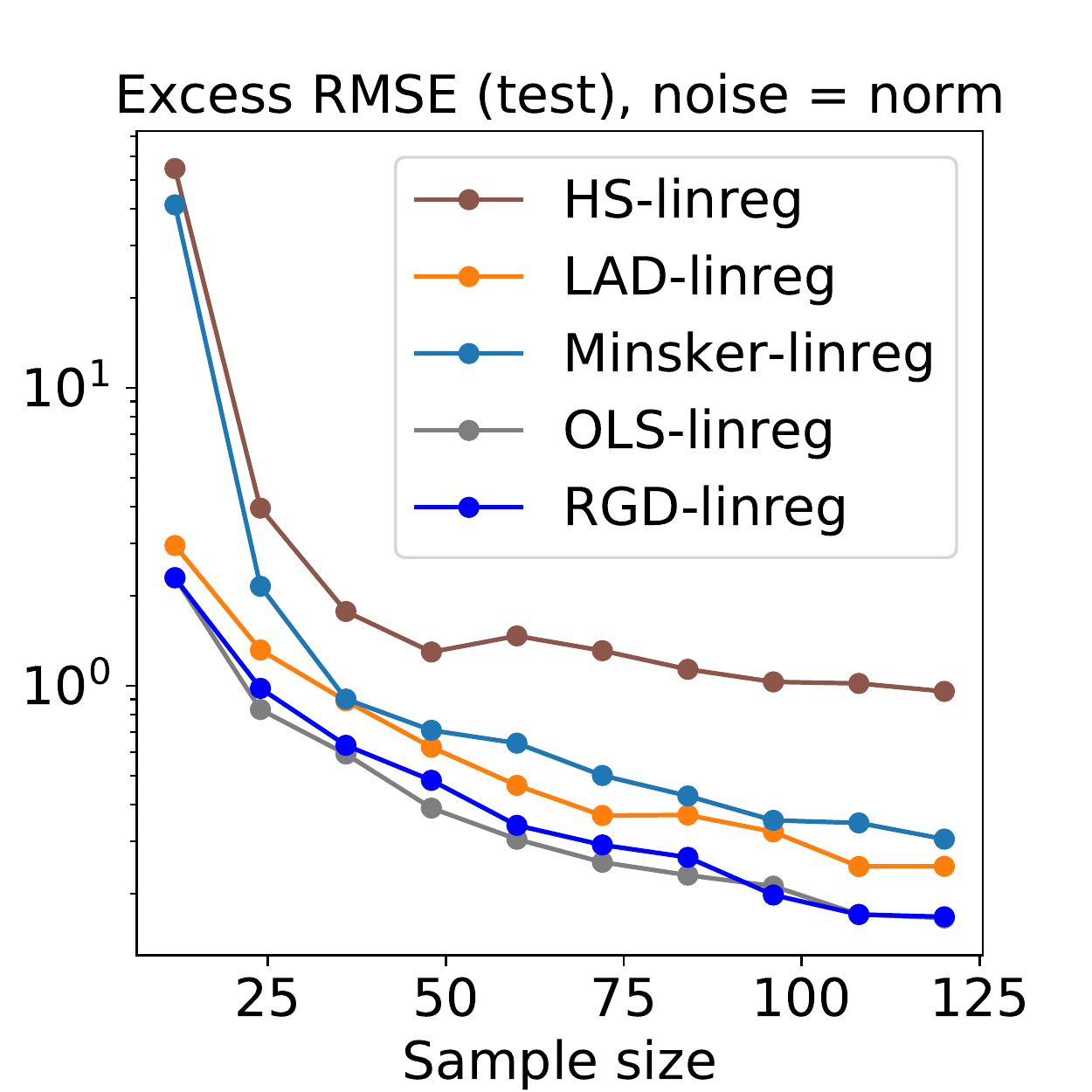}\,\includegraphics[width=0.25\textwidth]{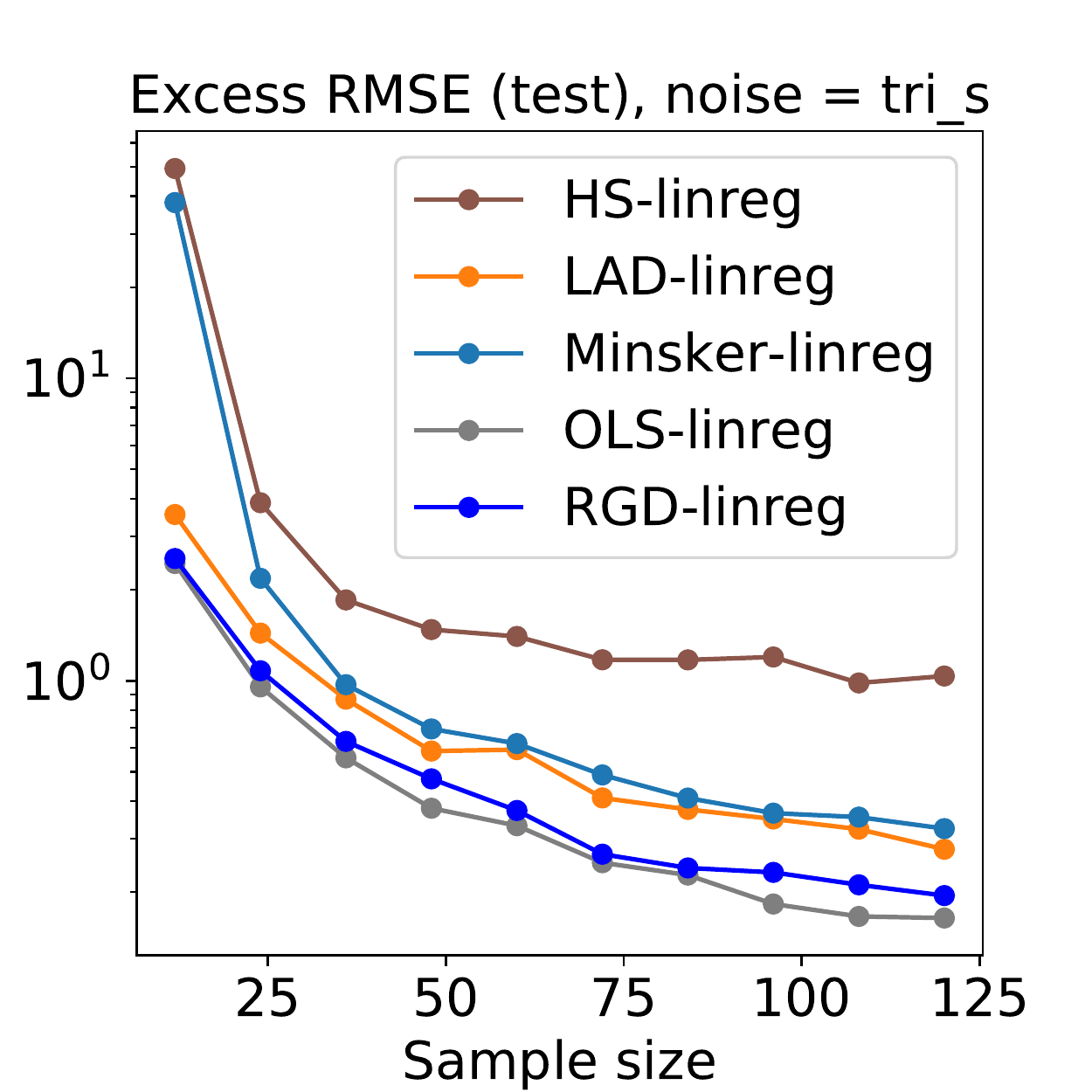}\,\includegraphics[width=0.25\textwidth]{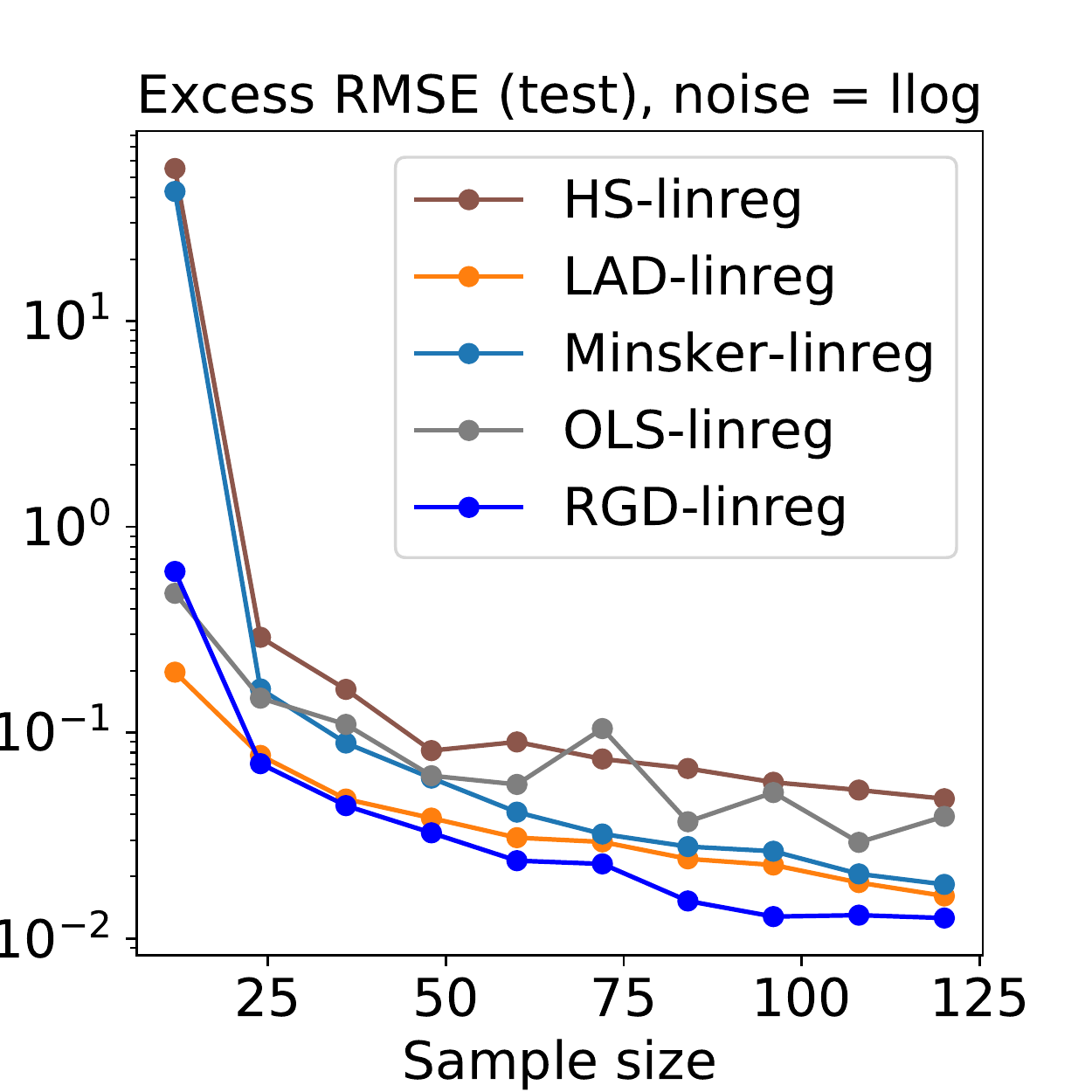}\,\includegraphics[width=0.25\textwidth]{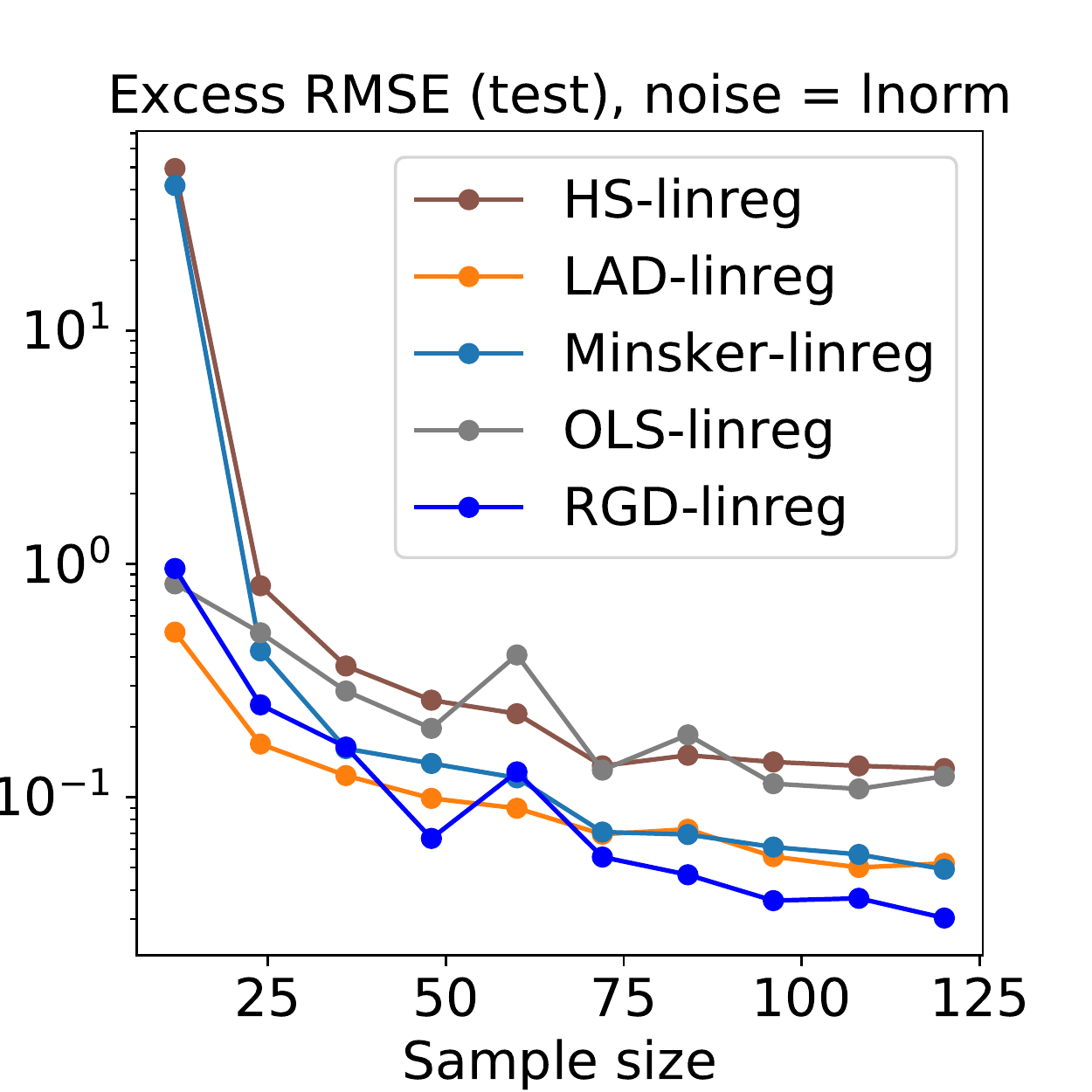}\\
\includegraphics[width=0.25\textwidth]{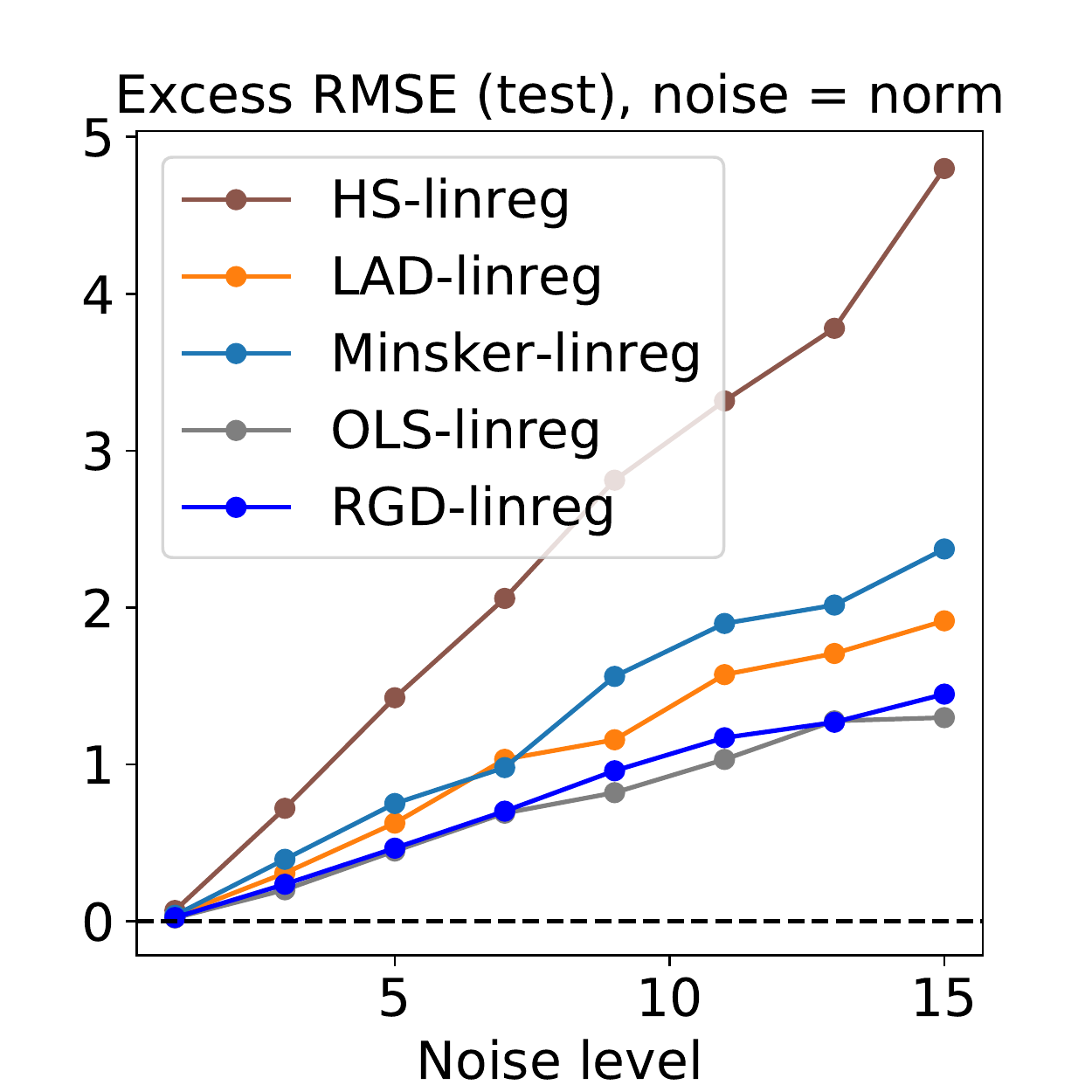}\,\includegraphics[width=0.25\textwidth]{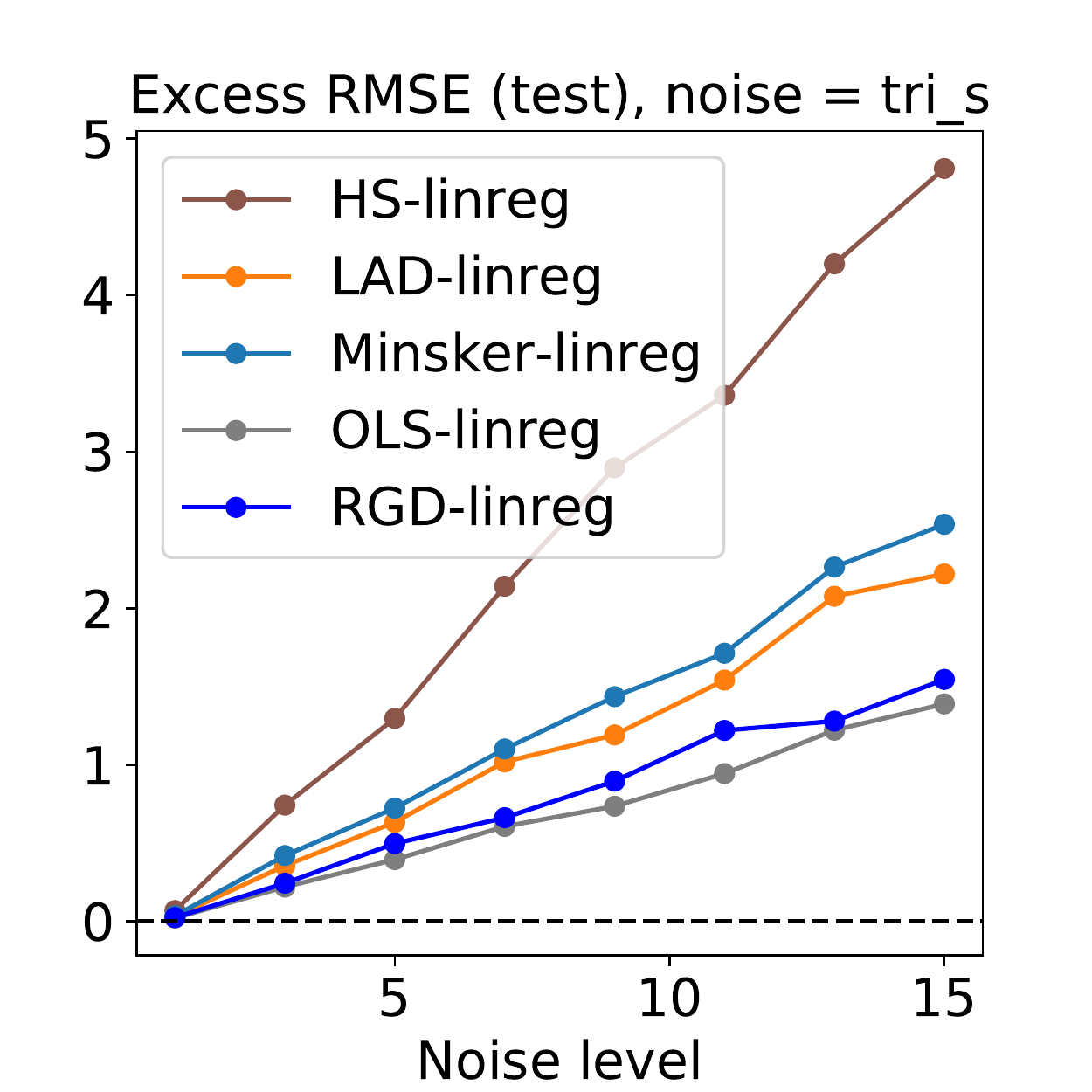}\,\includegraphics[width=0.25\textwidth]{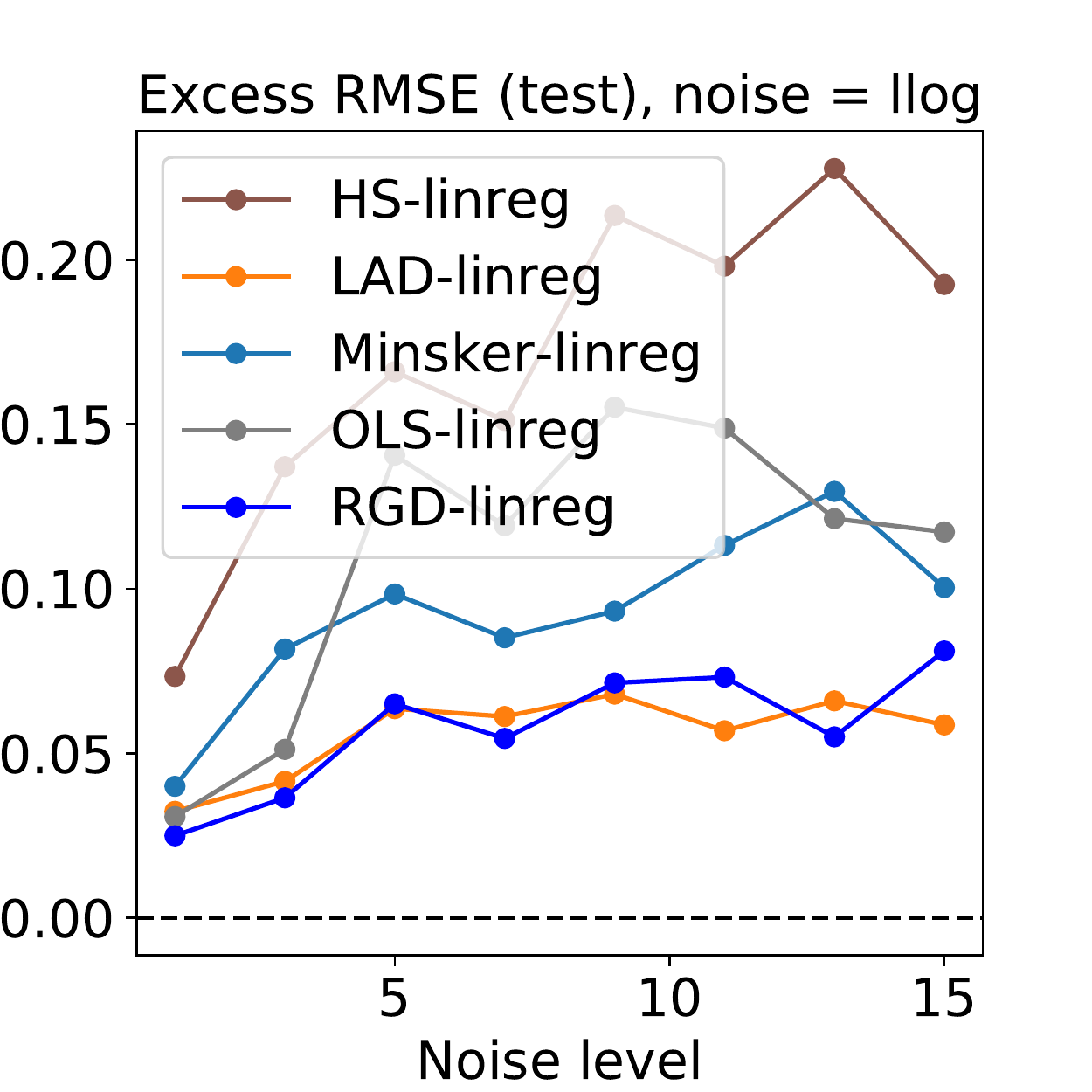}\,\includegraphics[width=0.25\textwidth]{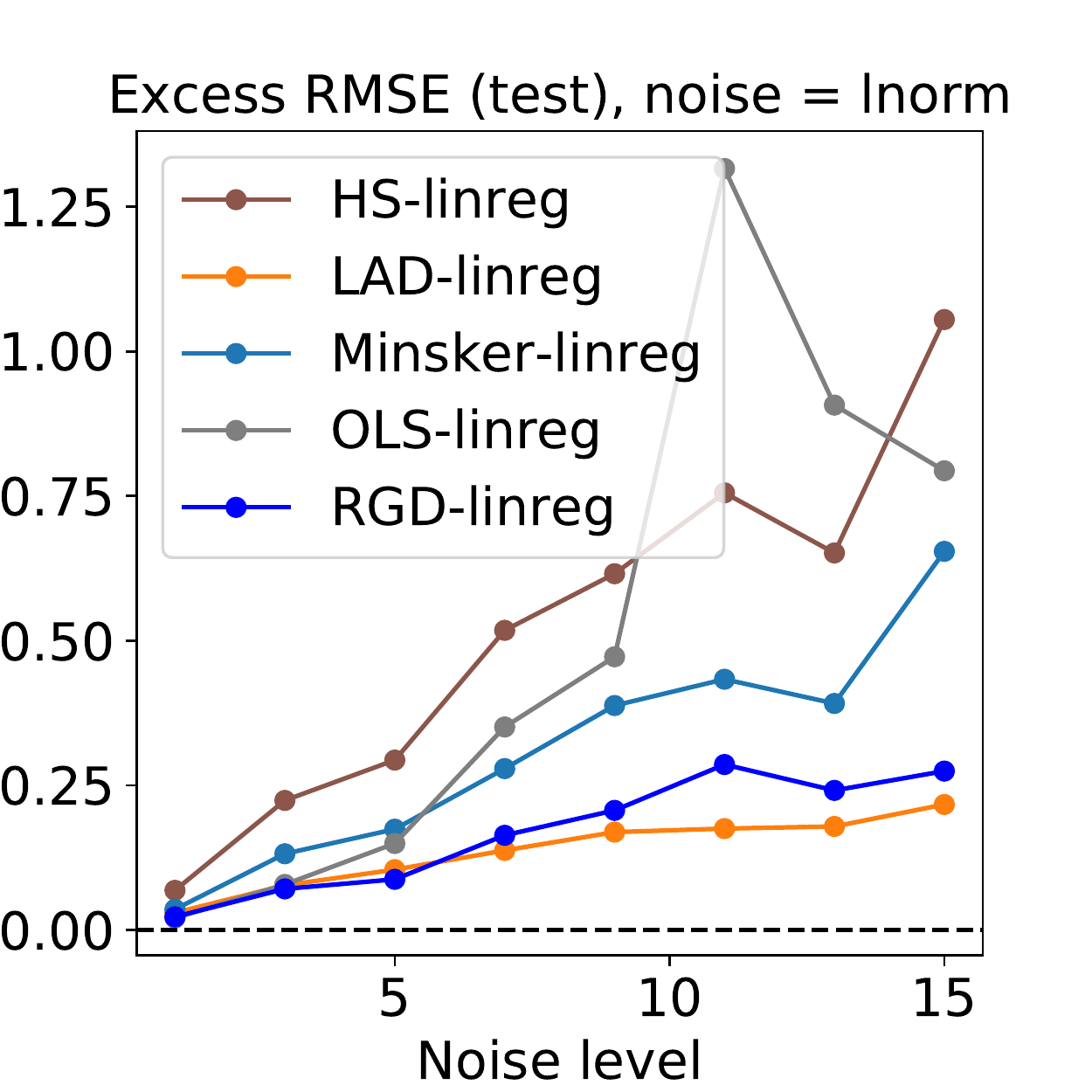}\\
\includegraphics[width=0.25\textwidth]{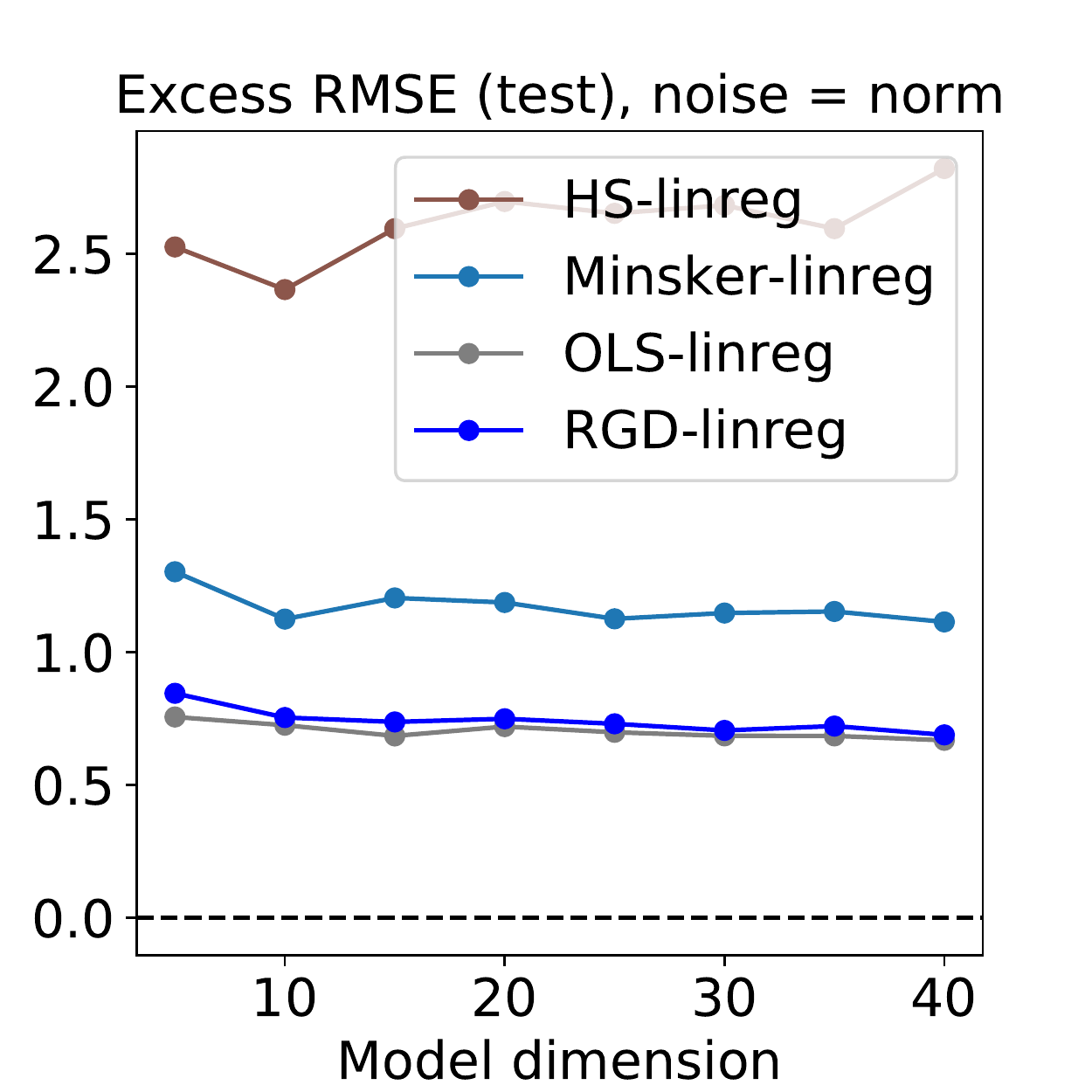}\,\includegraphics[width=0.25\textwidth]{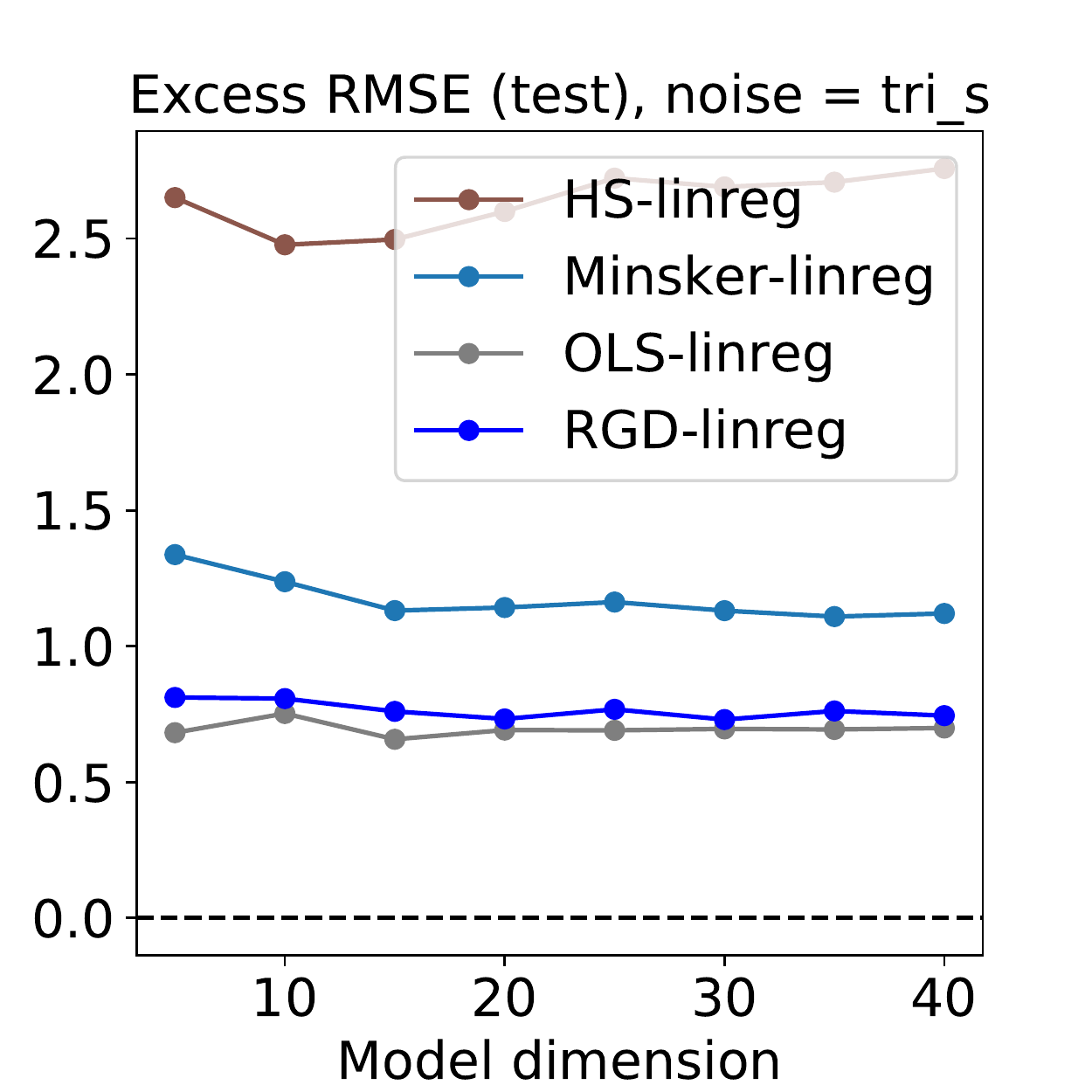}\,\includegraphics[width=0.25\textwidth]{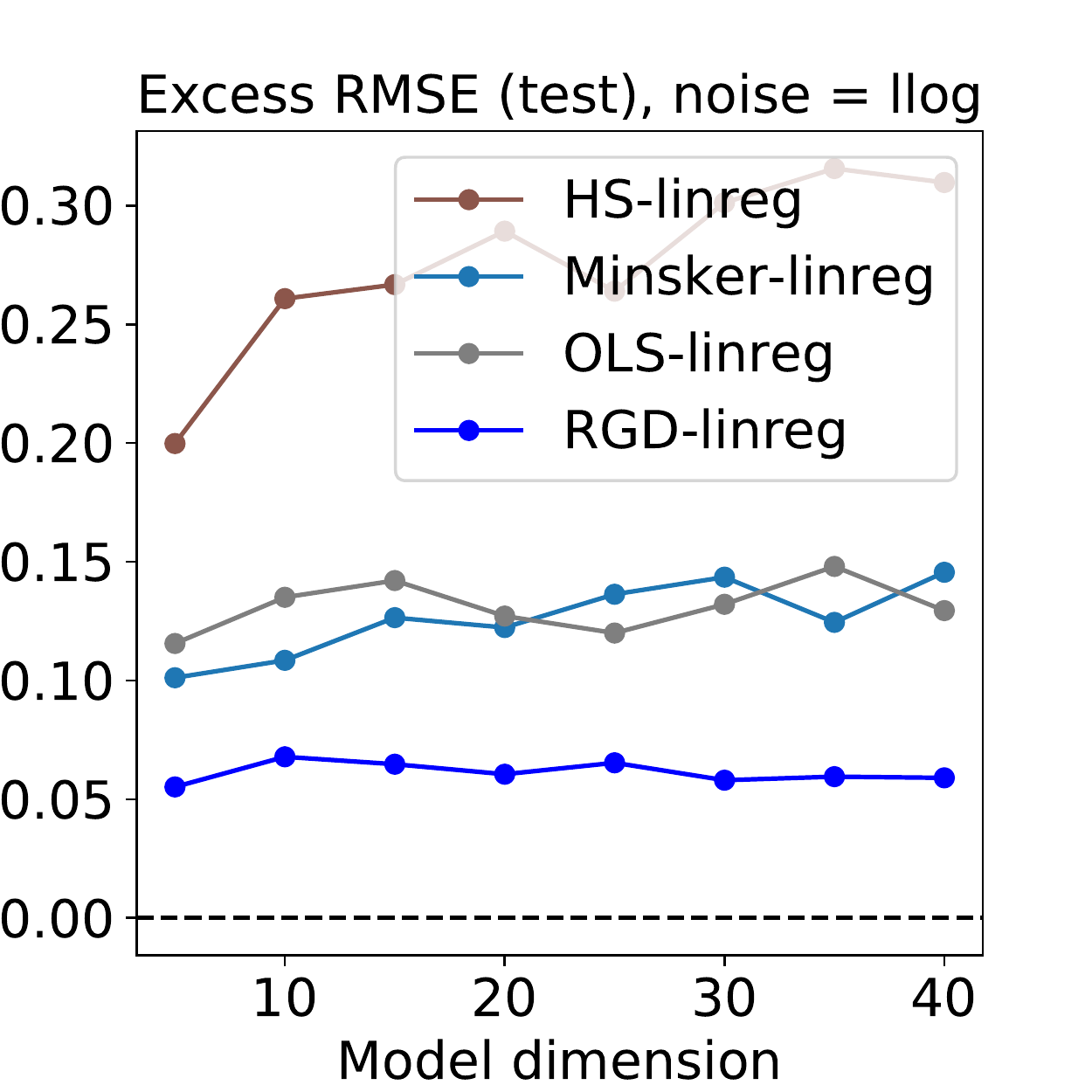}\,\includegraphics[width=0.25\textwidth]{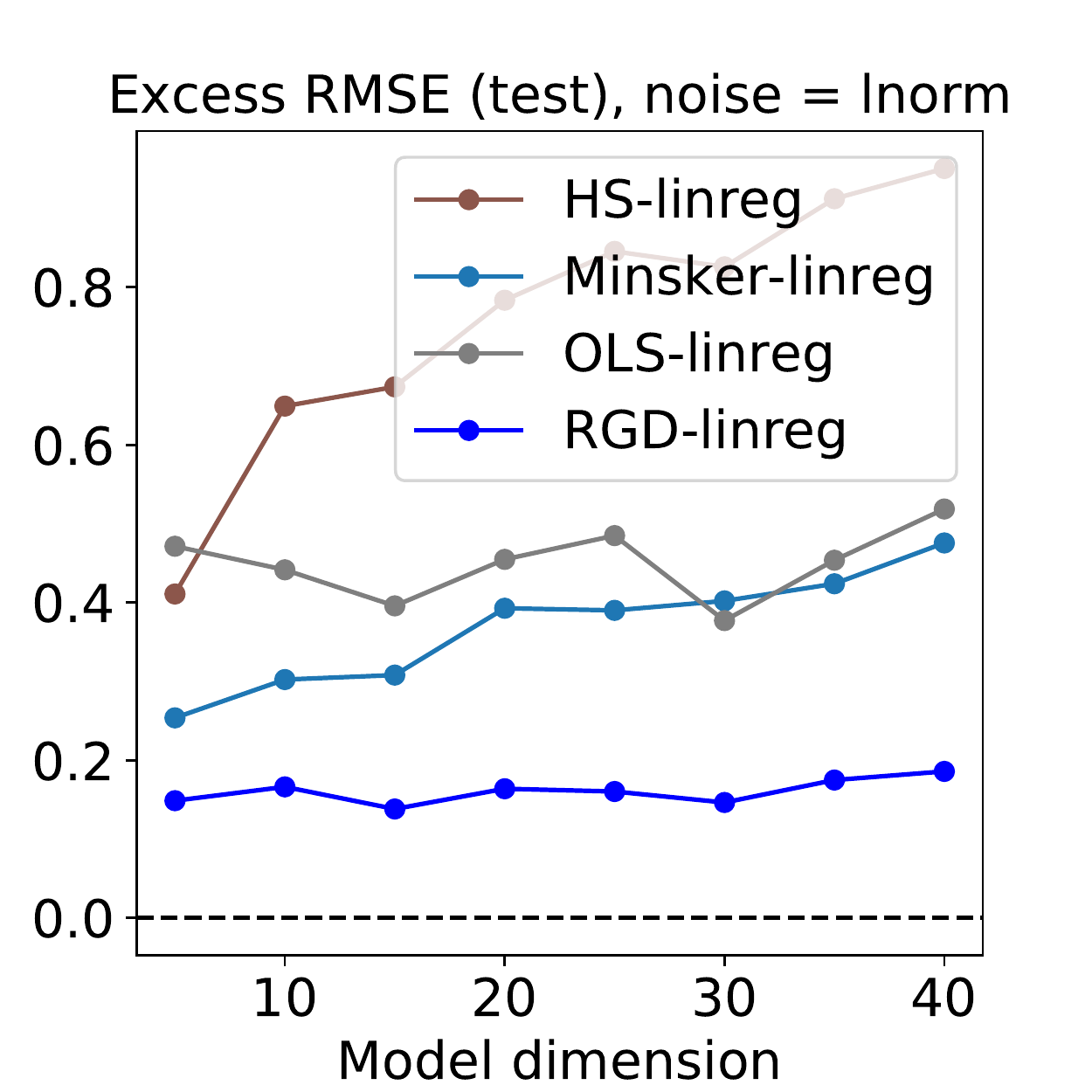}
\caption{Top: Prediction error over sample size $12 \leq n \leq 122$, fixed $d=5$, noise level = $8$. Middle: Prediction error over noise levels, for $n=30, d=5$. Bottom: Prediction error over dimensions $5 \leq d \leq 40$, with ratio $n/d = 6$ fixed, and noise level = $8$. Each column corresponds to a distinct noise family.}
\label{fig:multinoise_linreg}
\end{figure}

To begin, we fix $d$, and look at performance over $n$ settings (first row of Figure \ref{fig:multinoise_linreg}). Regardless of distribution, \texttt{rgdmult} is seen to provide highly competitive performance; whereas other methods perform well on some distributions and very poorly on others, a high level of generalization ability is uniformly maintained by the proposed procedure. We see that \texttt{ols} is strong under sub-Gaussian data (\texttt{norm} and  \texttt{tri\_s}), while it deteriorates under the heavy-tailed data. The more robust methods tend to perform better than \texttt{ols} on heavy-tailed data, but clearly suffer from biased estimates under sub-Gaussian data. These results illustrate how \texttt{rgdmult} realizes the best of both worlds, paying a tolerable price in bias for large payouts in terms of robustness to outliers.

In the second row of Figure \ref{fig:multinoise_linreg}, we examine performance over noise levels. It is encouraging that even with pre-fixed step size and budgets (since $n$ is fixed over all noise levels), the strong performance of \texttt{rgdmult} holds over very diverse settings.

Finally, in the third row of Figure \ref{fig:multinoise_linreg}, the ratio of $n$ to $d$ is fixed, and we see if and how performance changes when $d$ is increased. For all distributions, the performance of \texttt{rgdmult} is essentially constant over $d$ when $n$ scales with $d$, which is what we would hope considering the risk bounds of Theorem \ref{thm:riskbd_fixed_strong}. Certain competitive methods show more sensitivity to the absolute number of free parameters, particularly in the case of heavy-tailed data with asymmetric distributions.

\subsection{Application to real-world benchmarks}\label{sec:tests_real}

As our final class of numerical experiments, we shift our focus to classification tasks, and this time make use of real-world data sets, to be described in detail below.

All methods use a common model, here multi-class logistic regression. If the number of classes is $C$, and we have $F$ input features, then the dimension of the model will be $d=(C-1)F$. A basic property of this model is that the loss function is convex in the parameters, with gradients that exist, thus placing the model firmly within our realm of interest. Furthermore, for all of these tests we shall add a squared $\ell_{2}$-norm regularization term $a\|\ww\|^{2}$ to the loss, where $a$ varies depending on the dataset. Once again, each algorithm is given a fixed budget, this time of $20n$, where $n$ is the size of the training set available, which again depends on the dataset (details below).

Here we give results for two well-known data sets used for benchmarking: the forest cover type dataset from the UCI repository,\footnote{\url{http://archive.ics.uci.edu/ml/datasets/Covertype}} and the protein homology dataset used in a previous KDD Cup.\footnote{\url{http://www.kdd.org/kdd-cup/view/kdd-cup-2004/Tasks}} For each dataset, we execute 10 independent trials, with training/testing subsets randomly sampled without replacement as is described shortly. For all datasets, we normalize input features to the unit interval $[0,1]$ in a per-feature fashion. For the cover type dataset, we consider binary classification of the second type against all other types. With $C=2$ and $F=54$, we have $d=54$ and $a=0.001$, with a training subset of size $n=4d$. The protein homology dataset has highly unbalanced labels, with only 1296 positive labels our of over 145,000 examples. We balance out training and testing data, randomly selecting 296 positive examples and the same number of negative examples, yielding a test set of 592 points. As for the training set size, we use all positive examples not used for testing (1000 points each time), plus a random selection of 1000 negatively labeled examples, so $n=2000$. With $C=2$ and $F=74$, the dimension is $d=74$, and $a=0.001$. In all settings, initialization is done uniformly over the interval $[-0.05,0.05]$.

We investigate the utility of a random mini-batch version of Algorithm \ref{algo:rgdmult} here. We try mini-batch sizes of 10 and 20. Variance bounds $\mv{v}_{(t)}$ are set to $k$ times the empirical mean of the second moments of $\lgrad(\wwhat_{(t)},\zz)$, with $k$ ranging over $\{1/10, 1/5, 1/2, 1, 5, 25, 125, 625\}$. Furthermore, for the high-dimensional datasets, we consider a mini-batch in terms of random selection of which parameters to robustly update. At each iteration, we randomly choose $\min\{100,d\}$ indices, running Algorithm \ref{algo:rgdmult} for the resulting sub-vector, and the sample mean for the remaining coordinates. We compare our proposed algorithm with stochastic gradient descent (SGD), and stochastic variance-reduced gradient descent (SVRG) proposed by \citet{johnson2013a}. For each method, pre-fixed step sizes ranging over $\{0.0001, 0.001, 0.01, 0.05, 0.10, 0.15, 0.20\}$ are tested. SGD has mini-batches of size 1, just as the SVRG inner loop. The inner loop of SVRG has $n/2$ iterations, and all methods continue running until the fixed budget of gradient evaluations is spent.

\begin{figure}[t]
\centering
\includegraphics[width=0.5\textwidth]{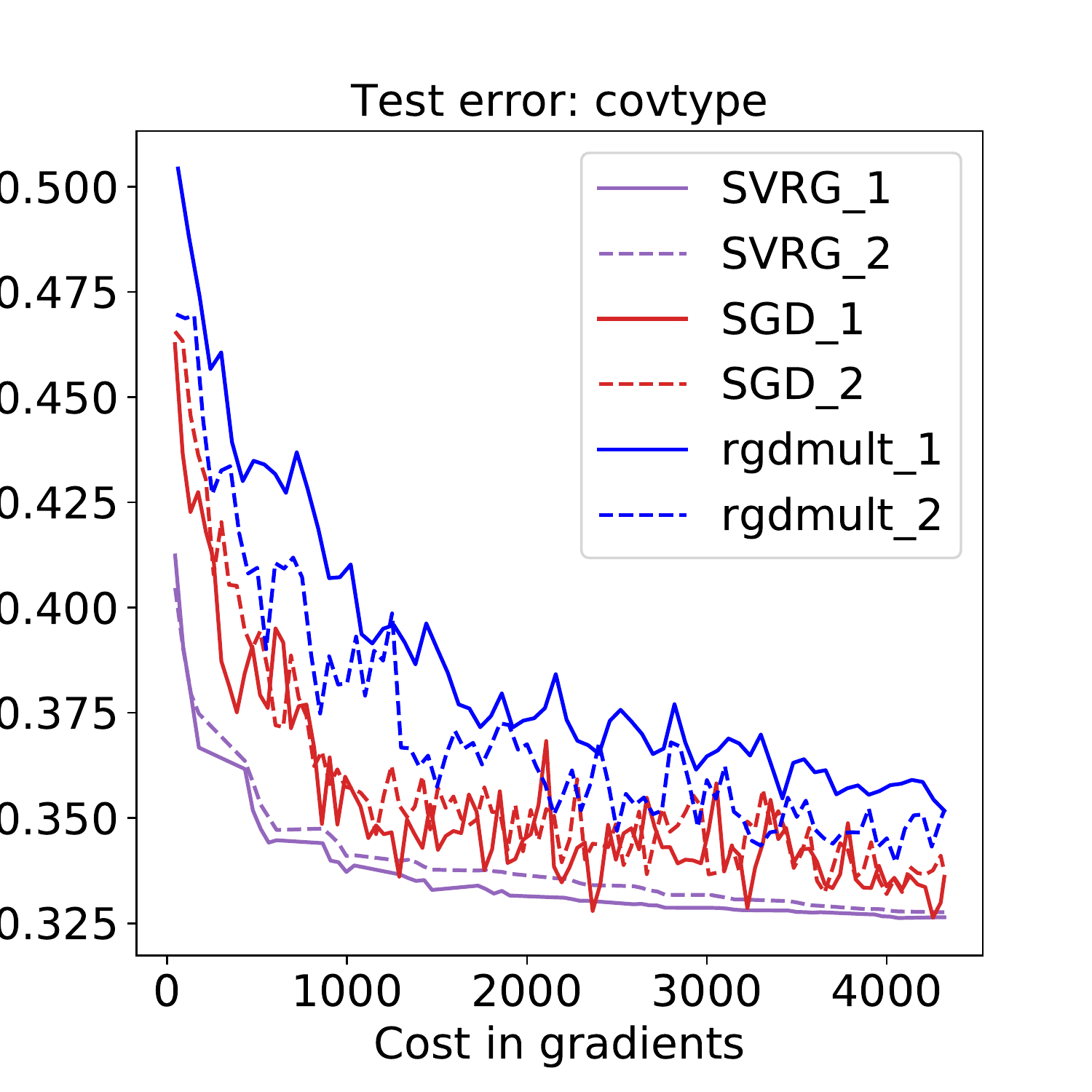}\,\includegraphics[width=0.5\textwidth]{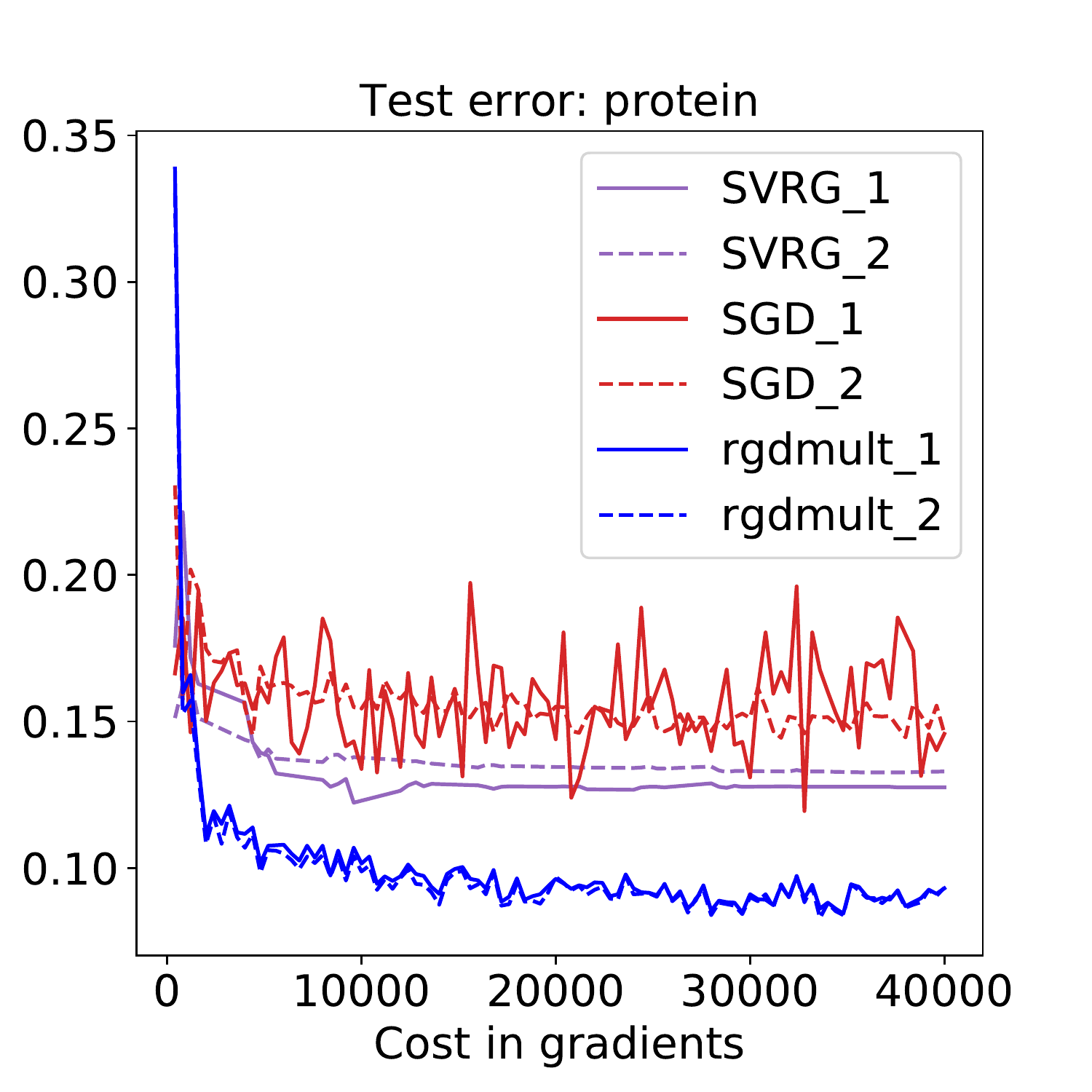}
\caption{Test error (misclassification rate) over budget spent, as measured by gradient computations, for the top two performers within each method class. Each plot corresponds to a distinct dataset.}
\label{fig:tests_real}
\end{figure}

We share representative results in Figure \ref{fig:tests_real}. For each dataset and each method, we chose the top two parameters settings, written \texttt{*\_1} and \texttt{*\_2} here. Here the ``top two'' refers to performance as measured by the median test error for the last five iterations. In general, the proposed procedure is clearly competitive with the best settings of these popular routines, and in the case of the smaller data set (protein homology, right-most plot), we see a significant improvement over competitors. Assuredly, our mini-batch implementation of Algorithm \ref{algo:rgdmult} is merely a nascent application, but strong performance under even a very naive setup is promising in terms of developing even stronger procedures for real-world data.

\section{Conclusion}\label{sec:conclusion}

We introduced and analyzed a novel machine learning algorithm, with a solid theoretical grounding based on firm statistical principles, and with the added benefit of a simple implementation, very few parameters to set, and a fast, non-iterative robustification procedure that does not throw away any data, but which is also not overly sensitive to errant observations. Based on the strong theoretical guarantees and appealing empirical performance, it appears that our approach of paying the price of a small bias for the reward of more distributionally robust gradient estimates is sound as a methodology, realizing better performance using less computational resources (data, time).

In looking ahead, we are particularly interested in moving beyond per-coordinate robustification, and considering operations that operate on the loss gradient vectors themselves as atomic units. The per-coordinate technique is easy to implement and theoretical analysis is also more straightforward, but the risk bounds have an extra $d$ factor that should be removable given more sophisticated procedures. Indeed, the high-dimensional mean estimation discussed by \citet{catoni2017a} has such a vector estimator, but unfortunately there is no way to actually compute the estimator they analyze. Bridging this gap is an important next step, from the perspective of both learning theory and machine learning practice.

\appendix

\section{Technical appendix}

\subsection{Preliminaries}\label{sec:tech_prelims}

Consider two probability measures $P$ and $Q$ on measurable space $(\XX,\mathcal{A})$. We say that $Q$ is absolutely continuous with respect to $P$, written $Q \ll P$, whenever $P(A) = 0$ implies $Q(A) = 0$ for all $A \in \mathcal{A}$. The Radon-Nikodym theorem guarantees that there exists a function $g \geq 0$, $P$-measurable, such that
\begin{align*}
Q(A) = \int_{A} g \, dP, \quad \text{ for all } A \in \mathcal{A}.
\end{align*}
Furthermore, this $g$ is unique in the sense that if another $f$ exists satisfying the above equality, we have $f=g$ almost everywhere $[P]$. It is common to call this function $g$ the Radon-Nikodym derivative of $Q$ with respect to $P$, written $dQ/dP$. The relative entropy, or Kullback-Leibler divergence, between two probability measures $P$ and $Q$ on measurable space $(\XX,\mathcal{A})$ is defined
\begin{align}
\KL(P;Q) \defeq 
\begin{cases}
\displaystyle -\int \log\left(\frac{dQ}{dP}\right) \, dP, & \text{ if } Q \ll P\\
\displaystyle +\infty, & \text{ else.}
\end{cases}
\end{align}

They key property of the $\psi$ truncation function utilized by \citet{catoni2017a}, defined in (\ref{eqn:influence_cat17}), is that for all $u \in \RR$, we have
\begin{align}\label{eqn:Catoni_property}
-\log\left(1 - u + \frac{u^{2}}{2}\right) \leq \psi(u) \leq \log\left(1 + u + \frac{u^{2}}{2}\right).
\end{align}

Let $f:\RR^{d} \to \RR$ be a continuously differentiable, convex, $\paraSmooth$-smooth function. 
\begin{align}
\label{eqn:facts_CO_1}
f(\uu)-f(\vv) & \leq \frac{\paraSmooth}{2}\|\uu-\vv\|^{2} + \langle f^{\prime}(\vv), \uu-\vv \rangle\\
\label{eqn:facts_CO_2}
\frac{1}{2\paraSmooth}\|f^{\prime}(\uu)-f^{\prime}(\vv)\|^{2} & \leq f(\uu)-f(\vv) - \langle f^{\prime}(\vv), \uu-\vv \rangle
\end{align}
for all $\uu,\vv \in \RR^{d}$.

\paragraph{Terminology}
For a function $F: \WW \to \RR$, we say that $F$ is $\paraSmooth$\textit{-Lipschitz} if, for all $\ww_{1},\ww_{2} \in \WW$ we have $|F(\ww_{1})-F(\ww_{2})| \leq \paraSmooth \|\ww_{1}-\ww_{2}\|$. If $F$ is differentiable, and the derivative $\ww \mapsto F^{\prime}(\ww)$ is $\paraSmooth$-Lipschitz, then we say that $F$ is $\paraSmooth$\textit{-smooth}.

If $F$ is a convex function on convex set $\WW$, then we say $F$ is $\paraSC$-\textit{strongly convex} if for all $\ww_{1},\ww_{2} \in \WW$,
\begin{align}\label{eqn:strong_convex_defn}
F(\ww_{1})-F(\ww_{2}) \geq \langle F^{\prime}(\ww_{2}), \ww_{1}-\ww_{2} \rangle + \frac{\paraSC}{2}\|\ww_{1}-\ww_{2}\|^{2}.
\end{align}
This definition can be made for any valid norm space, but we shall be assuming $\WW \subseteq \RR^{d}$ throughout, and use the Euclidean norm. If there exists $\wwstar \in \WW$ such that $F^{\prime}(\wwstar)=0$, then it follows that $\wwstar$ is the unique minimum of $F$ on $\WW$.

\subsection{Proofs of results in the main text}\label{sec:tech_proofs}

\begin{proof}[Proof of Lemma \ref{lem:pointwise_accuracy}]
Let $\PP(\RR)$ denote all probability measures on $\RR$, with an appropriate $\sigma$-field tacitly assumed. Consider any two measures $\prior, \prior_{0} \in \PP(\RR)$, and $h: \RR \to \RR$ a $\prior_{0}$-measurable function. By  \citet[p.~159--160]{catoni2004SLT}, it is proved that a Legendre transform of the mapping $\prior \mapsto \KL(\prior;\prior_{0})$ takes the form of a cumulant generating function, namely
\begin{align}\label{eqn:KL_Legendre_identity}
\sup_{\prior} \left( \int h(u) \, d\prior(u) - \KL(\prior;\prior_{0}) \right) = \log \int \exp(h(u)) \, d\prior_{0}(u),
\end{align}
where the supremum is taken over $\prior \in \PP(\RR)$. This identity is a technical tool, and the choice of $h$ and $\prior_{0}$ are parameters that can be adjusted to fit the application.

In actually setting these parameters, we adapt the general argument of \citet{catoni2017a} to our setting. Recalling the estimator (\ref{eqn:estimator_defn}), we start with a quasi average of the points $x_{1},\ldots,x_{n}$, modified by some data-sensitive additive noise, and passed through a truncation function. The expectation of this sum is then taken over the noise distribution. The $\prior$ in the definition of (\ref{eqn:estimator_defn}) will correspond to $\prior$ here, and thus to reflect the whole estimator within (\ref{eqn:KL_Legendre_identity}), it makes sense to include the data-dependent sum in our choice of $h$. Note that the summands in the estimator definition
\begin{align*}
\psi\left(\frac{x_{i}+\epsilon_{i}x_{i}}{s}\right), \quad i \in [n]
\end{align*}
depend on two random quantities, namely the data $x_{i}$, and the artificial noise $\epsilon_{i}$ (since $s>0$ is assumed pre-fixed). Reflecting dependence on these quantities directly, we write
\begin{align*}
f(\epsilon,x) \defeq \psi\left(\frac{x + \epsilon x}{s}\right), \quad \epsilon, x \in \RR.
\end{align*}
Note that by definition of $\psi$ in (\ref{eqn:influence_cat17}), the function $f: \RR^{2} \to \RR$ is measurable and bounded. With this cleaner notation, let us now set
\begin{align*}
h(\epsilon) = \sum_{i=1}^{n} f(\epsilon,x_{i}) - c(\epsilon)
\end{align*}
where $c(\epsilon)$ is a term to be determined shortly. Plugging this in to (\ref{eqn:KL_Legendre_identity}) yields the following quantity:
\begin{align*}
B & \defeq \sup_{\prior} \left( \int h(\epsilon) \, d\prior(\epsilon) - \KL(\prior;\prior_{0}) \right)\\
& = \log \int \exp\left( \sum_{i=1}^{n} f(\epsilon,x_{i}) - c(\epsilon) \right) \, d\prior(\epsilon).
\end{align*}
Taking the exponential of this $B$ and then taking expectation with respect to the sample, we have
\begin{align*}
\exx_{\ddist}\exp(B) & = \exx_{\ddist} \int \left( \frac{\exp\left(\sum_{i=1}^{n}f(\epsilon,x_{i})\right)}{\exp(c(\epsilon))} \right) \, \prior(\epsilon)\\
& = \int \left( \frac{\prod_{i=1}^{n}\exx_{\ddist}f(\epsilon,x_{i})}{\exp(c(\epsilon))} \right) \, \prior(\epsilon).
\end{align*}
The first equality comes from simple log/exp manipulations, and the second equality from taking the integration over the sample inside the integration with respect to $\prior$, valid via Fubini's theorem. It will be useful to have $\exx_{\ddist}\exp(B) \leq 1$. This can be achieved easily by setting
\begin{align*}
c(\epsilon) = n \log \exx_{\ddist} \exp(f(\epsilon,x)),
\end{align*}
which yields
\begin{align}\label{eqn:exp_B_unity}
\exx_{\ddist}\exp(B) = \int \left( \frac{\prod_{i=1}^{n}\exx_{\ddist}\exp(f(\epsilon,x_{i}))}{\left(\exx_{\ddist}\exp(f(\epsilon,x_{i}))\right)^{n}} \right) \, \prior(\epsilon) = 1.
\end{align}
With this preparation done, we can start on the high-probability upper bound of interest:
\begin{align*}
\prr\{ B \geq \log(\delta^{-1}) \} & = \prr\{ \exp(B) \geq 1/\delta \}\\
& = \exx_{\ddist} I\{ \delta \exp(B) \geq 1 \}\\
& \leq \exx_{\ddist} \delta \exp(B)\\
& = \delta.
\end{align*}
The inequality follows immediately since $\delta\exp(B) \geq 0$, and the final equality holds due to (\ref{eqn:exp_B_unity}). Note that since our setting of $c(\epsilon)$ is such that $c(\cdot)$ is $\prior$-measurable (via measurability of $f$), the resulting $h(\cdot)$ is indeed $\prior$-measurable, as required. Of importance here is the fact that
\begin{align}\label{eqn:uniform_upbd}
\sup_{\prior} \left( \int h(\epsilon) \, d\prior(\epsilon) - \KL(\prior;\prior_{0}) \right) \leq \log(\delta^{-1})
\end{align}
with probability no less than $1-\delta$, noting that the event is uniform in $\prior$. Using (\ref{eqn:KL_Legendre_identity}) once again, and dividing both sides by $n$, we have that with high probability, for any choice of $\prior$, we can bound this generic empirical mean as follows:
\begin{align}\label{eqn:empirical_inequality}
\frac{1}{n}\sum_{i=1}^{n} \int f(\epsilon,x_{i}) \, d\prior(\epsilon) \leq \int \log \exx_{\ddist} \exp\left(f(\epsilon,x)\right) \, d\prior(\epsilon) + \frac{\KL(\prior,\prior_{0}) + \log(\delta^{-1})}{n}.
\end{align}

Bridging the gap between these preparatory facts and the estimator of interest is now easy; since the noise terms $\epsilon_{1},\ldots,\epsilon_{n}$ are assumed to be independent copies of $\epsilon \sim \prior$, it follows immediately that
\begin{align*}
\widehat{x} & = \frac{s}{n} \sum_{i=1}^{n} \int \left( \psi\left( \frac{x_{i}+\varepsilon_{i}x_{i}}{s} \right) \right) \, d\prior(\epsilon_{i})\\
& = \frac{s}{n}\sum_{i=1}^{n} \int f(\epsilon,x_{i}) \, d\prior(\epsilon).
\end{align*}
That is to say, we have
\begin{align}\label{eqn:target_upbd}
\widehat{x} \leq s \int \log \exx_{\ddist} \exp\left( \psi\left(\frac{x(1+\epsilon)}{s}\right) \right) \, d\prior(\epsilon) + \frac{s}{n}\left( \KL(\prior;\prior_{0}) + \log(\delta^{-1}) \right)
\end{align}
on the high-probability event, uniformly in choice of $\prior$. Let us work step by step through each of the terms in the upper bound.

Starting with the first term, recall the definition of the truncation function $\psi$ given in (\ref{eqn:influence_cat17}), and in particular the logarithmic upper/lower bounds given in (\ref{eqn:Catoni_property}). These bounds will be convenient because it offers us polynomial bounds when passing $\psi$ through $\exp(\cdot)$, which is precisely what occurs in (\ref{eqn:target_upbd}) above. To get the first term in (\ref{eqn:target_upbd}) in a more useful form, we can bound it as
\begin{align*}
\int \log \exx_{\ddist} \exp\left( \psi\left(\frac{x(1+\epsilon)}{s}\right) \right) \, d\prior(\epsilon) & \leq \int \log \left( 1 + \frac{(1+\epsilon)\exx_{\ddist}x}{s} + \frac{(1+\epsilon)^{2}\exx_{\ddist}x^{2}}{2s^{2}} \right) \, d\prior(\epsilon)\\
& \leq \int \left( \frac{(1+\epsilon)\exx_{\ddist}x}{s} + \frac{(1+\epsilon)^{2}\exx_{\ddist}x^{2}}{2s^{2}} \right) \, d\prior(\epsilon)\\
& = \frac{\exx_{\prior}(1+\epsilon)\exx_{\ddist}x}{s} + \frac{\exx_{\prior}(1+\epsilon)^{2} \exx_{\ddist}x^{2}}{2s^{2}}\\
& = \frac{\exx_{\ddist}x}{s} + \frac{\exx_{\ddist}x^{2}}{2s^{2}} \left(\frac{1}{\beta} + 1\right).
\end{align*}
The first inequality follows from (\ref{eqn:Catoni_property}), and the second from the fact that $\log(1+u) \leq u$ for all $u > -1$. As for the final equality, note that with $\epsilon \sim \prior = N(0,\beta^{-1})$, it follows immediately that
\begin{align*}
\exx_{\prior}(1+\epsilon)^{2} = \frac{1}{\beta} + (\exx_{\prior}(1+\epsilon))^{2} = \frac{1}{\beta} + 1.
\end{align*}

Moving on to the second term, evaluating $\KL(\prior;\prior_{0})$ depends completely on how we define the pre-fixed $\prior_{0}$. One approach is to set $\prior_{0}$ such that the KL divergence is easily computed; for example, $\prior_{0} = N(1,\beta^{-1})$. In this case,\footnote{Another approach of interest is as follows: consider setting $\prior_{0}=\prior$, causing the KL term to vanish, and $\beta \to \infty$ to be the optimal strategy from the perspective of the upper bound. Then it remains to see how the correction term $\corr(a,b)$ is computed in the limit. Not sure if it's worth the effort, but conceptually it is interesting.} simple computations show that
\begin{align*}
\KL(\prior;\prior_{0}) & = \int_{-\infty}^{\infty} \log \left( \exp\left(\frac{\beta(u-1)^{2}}{2} - \frac{\beta u^{2}}{2}\right) \right) \sqrt{\frac{\beta}{2\pi}}\exp\left(-\frac{\beta u^{2}}{2}\right) \, du\\
& = \int_{-\infty}^{\infty} \frac{(1-2u)\beta}{2} \sqrt{\frac{\beta}{2\pi}}\exp\left(-\frac{\beta u^{2}}{2}\right) \, du\\
& = \frac{\beta}{2}.
\end{align*}
With this computation done, an upper bound is complete, taking the form
\begin{align}\label{eqn:complete_upbd}
\widehat{x} \leq \exx_{\ddist}x + \frac{\exx_{\ddist}x^{2}}{2s} \left(\frac{1}{\beta} + 1\right) + \frac{s}{n}\left( \frac{\beta}{2} + \log(\delta^{-1}) \right).
\end{align}
Optimizing this upper bound with respect to $s > 0$, we have
\begin{align*}
s^{2} = \left(1 + \frac{1}{\beta}\right)\frac{n\exx_{\ddist}x^{2}}{2}\left(\frac{\beta}{2} + \log(\delta^{-1})\right)^{-1}
\end{align*}
and with respect to $\beta > 0$, we have
\begin{align}
\beta^{2} = \frac{n\exx_{\ddist}x^{2}}{s^{2}}.
\end{align}
Plugging this setting of $\beta$ in to the setting of $s$ yields
\begin{align}
s^{2} = \frac{n\exx_{\ddist}x^{2}}{2\log(\delta^{-1})}.
\end{align}
With this setting of $s$, the upper bound (\ref{eqn:complete_upbd}) can be cleaned up to the form
\begin{align*}
\widehat{x} \leq \exx_{\ddist}x + \sqrt{\frac{2\exx_{\ddist}x^{2}\log(\delta^{-1})}{n}} + \sqrt{\frac{\exx_{\ddist}x^{2}}{n}}.
\end{align*}
To get lower bounds on $\widehat{x}-\exx_{\ddist}x$, we can equivalently seek out upper bounds on $(-1)\widehat{x} + \exx_{\ddist}x$. This can be easily done via
\begin{align}\label{eqn:target_lowbd}
-\widehat{x} \leq s \int \log \exx_{\ddist} \exp\left(-\psi\left(\frac{x(1+\epsilon)}{s}\right) \right) \, d\prior(\epsilon) + \frac{s}{n}\left( \KL(\prior;\prior_{0}) + \log(\delta^{-1}) \right).
\end{align}
Only the first term on the right-hand side is different from before. Note that by the lower bound of (\ref{eqn:Catoni_property}), we have
\begin{align*}
\log \exx_{\ddist} \exp\left(-\psi\left(\frac{x(1+\epsilon)}{s}\right)\right) \leq (-1)\frac{(1+\epsilon)\exx_{\ddist}x}{s} + \frac{(1+\epsilon)^{2}\exx_{\ddist}x^{2}}{2s^{2}}.
\end{align*}
The rest plays out analogously to the upper bound, yielding
\begin{align*}
(-1)\widehat{x} \leq (-1)\exx_{\ddist}x + \sqrt{\frac{2\exx_{\ddist}x^{2}\log(\delta^{-1})}{n}} + \sqrt{\frac{\exx_{\ddist}x^{2}}{n}}
\end{align*}
which implies, as desired,
\begin{align}\label{eqn:complete_lowbd}
\widehat{x} - \exx_{\ddist}x \geq \sqrt{\frac{2\exx_{\ddist}x^{2}\log(\delta^{-1})}{n}} + \sqrt{\frac{\exx_{\ddist}x^{2}}{n}}.
\end{align}
Since $-\psi(u) = \psi(-u)$, both of these settings can be interpreted as different settings of the distribution of the noise factor: $(1+\epsilon)$ in the upper bound case, and $-(1+\epsilon)$ in the lower bound case, both with $\epsilon \sim \prior$. Since the inequality (\ref{eqn:uniform_upbd}) is uniform in the distribution of this noise, both bounds hold on the same event, which has probability no less than $1-\delta$. We may thus conclude that with probability at least $1-\delta$ over the random draw of the sample $x_{1},\ldots,x_{n}$, the estimator $\widehat{x}$ satisfies
\begin{align*}
|\widehat{x} - \exx_{\ddist}x| \leq \sqrt{\frac{2\exx_{\ddist}x^{2}\log(\delta^{-1})}{n}} + \sqrt{\frac{\exx_{\ddist}x^{2}}{n}}.
\end{align*}
In practice, since $\exx_{\ddist}x^{2}$ will typically be unknown, this factor can be replaced by any valid upper bound $v \geq \exx_{\ddist}x^{2}$. The only impact to the final upper bound is that the unknown $\exx_{\ddist}x^{2}$ factors are replaced by the known $v$, concluding the proof.
\end{proof}

\begin{proof}[Proof of Lemma \ref{lem:est_lipschitz}]
Consider two data sets, the original $x_{1},\ldots,x_{n}$ and a perturbed version $x_{1}^{\prime},\ldots,x_{n}^{\prime}$. For clean notation, organize these into vectors $\xx = (x_{1},\ldots,x_{n})$ and $\xx^{\prime} = (x_{1}^{\prime},\ldots,x_{n}^{\prime})$. Taking the difference between the estimator evaluated on these distinct data sets, we have
\begin{align*}
\widehat{x}(\xx)-\widehat{x}(\xx^{\prime}) & = \int \frac{s}{n} \sum_{i=1}^{n}\left(\psi\left(\frac{(1+\epsilon)x_{i}}{s}\right) - \psi\left(\frac{(1+\epsilon)x_{i}^{\prime}}{s}\right)\right) \, d\prior(\epsilon)\\
& \leq \int \frac{s}{n} \sum_{i=1}^{n} \left|\frac{1+\epsilon}{s}\right|\left|x_{i}-x_{i}^{\prime}\right| \, d\prior(\epsilon)\\
& = \exx_{\prior}|1+\epsilon| \frac{1}{n} \sum_{i=1}^{n} \left|x_{i}-x_{i}^{\prime}\right|\\
& = \frac{\exx_{\prior}|1+\epsilon|}{n} \|\xx - \xx^{\prime}\|_{1}.
\end{align*}
The first equality follows by linearity and the definition of the estimators. The subsequent inequality follows from the $1$-Lipschitz property of $\psi$ defined in (\ref{eqn:influence_cat17}), which is that for all $u,v \in \RR$, we have that $|\psi(u) - \psi(v)| \leq |u - v|$.

Evaluating $\exx_{\prior}|1+\epsilon|$ is straightforward under the assumption that $\epsilon \sim \prior = N(0,1/\beta)$, since the random variable $|1+\epsilon|$ follows a Folded Normal distribution. More generally, if $X \sim N(a,b^{2})$, then $Y = |X|$ follows a folded normal distribution, with expected value
\begin{align*}
\exx Y = a\left(1-2\Phi\left(\frac{-a}{b}\right)\right) + b \sqrt{\frac{2}{\pi}}\exp\left(\frac{-a^{2}}{2b^{2}}\right).
\end{align*}
Since in our case, we have $a=1$ and $b^{2} = 1/\beta$, it follows that
\begin{align*}
\exx_{\prior}|1+\epsilon| = 1-2\Phi\left(-\sqrt{\beta}\right) + \sqrt{\frac{2}{\beta\pi}}\exp\left(\frac{-\beta}{2}\right).
\end{align*}
Reflecting this factor in the above inequalities concludes the proof.
\end{proof}

\begin{proof}[Proof of Lemma \ref{lem:uniform_grad_accuracy}]
In order to obtain bounds that hold uniformly over the choice of $\ww$, we adopt a rather standard strategy utilizing covering numbers of $\WW$. Using assumption \ref{asmp:model_compact}, since $\WW$ is closed and bounded, the Heine-Borel theorem implies that $\WW$ is compact. This means the number of balls of radius $\varepsilon$ required to cover $\WW$ (denoted $N_{\varepsilon}$) is bounded above\footnote{This is a basic property of covering numbers for compact subsets of Euclidean space \citep{kolmogorov1993SelectWorks3}.} as
\begin{align}\label{eqn:cover_size_compact}
N_{\varepsilon} \leq (3\diameter/2\varepsilon)^{d}.
\end{align}
Denote the centers of this $\varepsilon$-net by $\{\wwtil_{1},\ldots,\wwtil_{N_{\varepsilon}}\}$. Given an abitrary $\ww \in \WW$ and center $\wwtil \in \{\wwtil_{1},\ldots,\wwtil_{N_{\varepsilon}}\}$, we break the quantity to be controlled into three error terms, each to be tackled separately, as
\begin{align}\label{eqn:ineq_3errors}
\|\rgest(\ww) - \rgrad(\ww)\| \leq \|\rgest(\ww)-\rgest(\wwtil)\| + \|\rgrad(\ww) - \rgrad(\wwtil)\| + \|\rgest(\wwtil) - \rgrad(\wwtil)\|.
\end{align}

Let us start with the first term, $\|\rgest(\ww)-\rgest(\wwtil)\|$. Using Lemma \ref{lem:est_lipschitz}, we have that
\begin{align*}
\|\rgest(\ww)-\rgest(\wwtil)\|^{2} & \leq \sum_{j=1}^{d} \left(\frac{c_{\prior}}{n} \sum_{i=1}^{n}|\lgsub_{j}(\ww;\zz_{i})-\lgsub_{j}(\wwtil;\zz_{i})| \right)^{2}\\
& \leq \sum_{j=1}^{d} \left(c_{\prior} \paraSmooth \|\ww - \wwtil\|\right)^{2}\\
& = d c_{\prior}^{2} \paraSmooth^{2} \|\ww - \wwtil\|^{2}.
\end{align*}
The first inequality is via Lemma \ref{lem:est_lipschitz}, and the second via smoothness of the loss (via \ref{asmp:loss_smooth}). We may thus control the first error term as
\begin{align}\label{eqn:bound_error1}
\|\rgest(\ww)-\rgest(\wwtil)\| \leq c_{\prior} \paraSmooth \sqrt{d} \|\ww - \wwtil\|.
\end{align}

Moving on to the second error term in the upper bound, this follows easily by smoothness of the risk (via \ref{asmp:risk_smooth}), namely a Lipschitz property of the risk gradient. It immediately follows that 
\begin{align}\label{eqn:bound_error2}
\|\rgrad(\ww) - \rgrad(\wwtil)\| \leq \paraSmooth \|\ww-\wwtil\|
\end{align}
for any choice of $\ww \in \WW$ and $\varepsilon$-ball center $\wwtil$.

Finally for the third error term, given any center $\wwtil$, as long as $\exx_{\ddist} \lgsub_{j}(\wwtil;\zz)^{2} < \infty$, then we can apply Lemma \ref{lem:pointwise_accuracy}, implying
\begin{align*}
|\rgestsub_{j}(\ww) - \rgsub_{j}(\ww)| \leq \varepsilon_{j} \defeq \sqrt{\frac{2v_{j}\log(\delta^{-1})}{n}} + \sqrt{\frac{v_{j}}{n}}
\end{align*}
where $v_{j} > 0$ is an upper bound on $\exx_{\ddist}|\lgsub_{j}(\ww;\zz)|^{2}$ used in the setting of $s_{j}$, in accordance with Lemma \ref{lem:pointwise_accuracy}. For any pre-fixed $\ww$, then for any $\varepsilon > 0$ we have
\begin{align*}
\prr\left\{ \|\rgest(\ww)-\rgrad(\ww)\| > \varepsilon \right\} & = \prr\left\{ \|\rgest(\ww)-\rgrad(\ww)\|^{2} > \varepsilon^{2} \right\}\\
& \leq \sum_{j=1}^{d} \prr\left\{ |\rgestsub_{j}(\ww) - \rgsub_{j}(\ww)| > \frac{\varepsilon}{\sqrt{d}} \right\}.
\end{align*}
Using $\varepsilon_{j}$ just defined, taking the maximum over $j \in [d]$, it follows that
\begin{align*}
\prr\left\{ \|\rgest(\ww)-\rgrad(\ww)\| > \left(\max_{k}\varepsilon_{k}\right)\sqrt{d} \right\} & \leq \sum_{j=1}^{d}\prr\left\{ |\rgestsub_{j}(\ww) - \rgsub_{j}(\ww)| > \max_{k}\varepsilon_{k} \right\}\\
& \leq \sum_{j=1}^{d}\prr\left\{ |\rgestsub_{j}(\ww) - \rgsub_{j}(\ww)| > \varepsilon_{j} \right\}\\
& \leq d\delta.
\end{align*}
Note that the second inequality follows immediately from $\varepsilon_{j} \leq \max_{k} \varepsilon_{k}$ and monotonicity of probability measures. Writing $V \defeq \max_{j} v_{j}$, it follows immediately that fixing any $\ww \in \WW$, the nearest center $\wwtil \defeq \wwtil(\ww)$ can be determined, and the event
\begin{align*}
\EE(\wwtil) \defeq \left\{ \|\rgest(\wwtil) - \rgrad(\wwtil)\| > \sqrt{\frac{2Vd\log(d\delta^{-1})}{n}} + \sqrt{\frac{V}{n}} \right\}
\end{align*}
has probability no greater than $\delta$. The whole reason for utilizing a $\varepsilon$-cover of $\WW$ in the first place is to avoid having to take a supremum over $\ww \in \WW$, which spoils union bounds, and instead to simply take a maximum over a finite number of $\varepsilon$-covers. The critical fact for our purposes is that
\begin{align*}
\sup_{\ww \in \WW} \| \rgest(\wwtil(\ww)) - \rgrad(\wwtil(\ww)) \| = \max_{k \in [N_{\varepsilon}]} \| \rgest(\wwtil_{k}) - \rgrad(\wwtil_{k}) \|
\end{align*}
holds. The ``good event'' of interest is the one in which the bad event $\EE(\cdot)$ holds for none of the centers on our $\varepsilon$-cover. In other words, the event
\begin{align*}
\EE_{+} = \left(\bigcap_{k \in [N_{\varepsilon}]} \EE(\wwtil_{k})\right)^{c},
\end{align*}
which taking a union bound, occurs with probability no less than $1-\delta N_{\varepsilon}$. To get a $1-\delta$ guarantee, simply pay the price of an extra logarithmic factor in the upper bound; that is to say, we equivalently have
\begin{align}\label{eqn:bound_error3}
\|\rgest(\wwtil(\ww)) - \rgrad(\wwtil(\ww))\| \leq \sqrt{\frac{2Vd\log(dN_{\varepsilon}\delta^{-1})}{n}} + \sqrt{\frac{V}{n}}
\end{align}
with probability no less than $1-\delta$, uniformly in the choice of $\ww \in \WW$.

Taking these intermediate results together, we can form a useful uniform upper bound on (\ref{eqn:ineq_3errors}), taking the form
\begin{align*}
\sup_{\ww \in \WW} \|\rgest(\ww) - \rgrad(\ww)\| & \leq \sup_{\ww \in \WW} \left( c_{\prior} \paraSmooth \sqrt{d} \|\ww - \wwtil\| + \paraSmooth \|\ww-\wwtil\| +  \|\rgest(\wwtil) - \rgrad(\wwtil)\| \right)\\
& \leq c_{\prior}\paraSmooth\sqrt{d}\varepsilon + \paraSmooth\varepsilon + \max_{k \in [N_{\varepsilon}]} \| \rgest(\wwtil_{k}) - \rgrad(\wwtil_{k}) \|\\
& \leq \paraSmooth\varepsilon(1+c_{\prior}\sqrt{d}) + \sqrt{\frac{2Vd\log(dN_{\varepsilon}\delta^{-1})}{n}} + \sqrt{\frac{V}{n}}
\end{align*}
with probability no less than $1-\delta$ over the random draw of the sample. The bounds on the first, second, and third terms in the original upper bound come from (\ref{eqn:bound_error1}), (\ref{eqn:bound_error2}), and (\ref{eqn:bound_error3}) respectively, with the $\varepsilon$ factors following immediately from the definition of an $\varepsilon$-cover. To obtain the desired result, simply bound $N_{\varepsilon}$ as in (\ref{eqn:cover_size_compact}), and set $\varepsilon = 1/\sqrt{n}$, yielding updates to two of the terms, as
\begin{align*}
\paraSmooth\varepsilon(1+c_{\prior}\sqrt{d}) & = \frac{\paraSmooth(1+c_{\prior}\sqrt{d})}{\sqrt{n}}\\
\sqrt{\frac{2Vd\log(dN_{\varepsilon}\delta^{-1})}{n}} & \leq \sqrt{\frac{2Vd(\log(d\delta^{-1})+d\log(3\diameter\sqrt{n}/2))}{n}}
\end{align*}
which, when plugged into the bound just obtained, concludes the proof.
\end{proof}

\begin{proof}[Proof of Lemma \ref{lem:dist_control_strong}]
Given $\wwhat_{(t)}$, running the approximate update (\ref{eqn:update_GD_approx}), we have
\begin{align*}
\|\wwhat_{(t+1)}-\wwstar\| & = \|\wwhat_{(t)}-\alpha_{(t)}\rgest(\wwhat_{(t)})-\wwstar\|\\
& \leq \|\wwhat_{(t)}-\alpha_{(t)}\rgrad(\wwhat_{(t)})-\wwstar\| + \alpha_{(t)}\|\rgest(\wwhat_{(t)})-\rgrad(\wwhat_{(t)})\|.
\end{align*}
The first term looks at the distance from the target given an optimal update, using $\rgrad$. Using the $\paraSC$-strong convexity of $\risk$, via \citet[Thm.~2.1.15]{nesterov2004ConvOpt} it follows that
\begin{align*}
\|\wwhat_{(t)}-\alpha_{(t)}\rgrad(\wwhat_{(t)})-\wwstar\|^{2} \leq \left(1-\frac{2\alpha_{(t)}\paraSC\paraSmooth}{\paraSC+\paraSmooth}\right)\|\wwhat_{(t)}-\wwstar\|^{2}.
\end{align*}
Writing $\SCfactor \defeq 2\paraSC\paraSmooth/(\paraSC+\paraSmooth)$, the coefficient becomes $(1-\alpha_{(t)}\SCfactor)$. 

To control the second term simply requires unfolding the recursion. By hypothesis, we can leverage (\ref{eqn:conf_uniform}) to bound the statistical estimation error by $\varepsilon$ for every step, all on the same $1-\delta$ ``good event.'' For notational ease, write $a_{(t)} \defeq \sqrt{1-\alpha_{(t)}\SCfactor}$. Unfolding the recursion, on the good event, we have
\begin{align}\label{eqn:recursion_unfolded}
\|\wwhat_{(t+1)}-\wwstar\| \leq \|\wwhat_{(0)}-\wwstar\| \prod_{k=0}^{t}a_{(k)} + \varepsilon \left( \alpha_{(t)} + \sum_{k=0}^{t-1}\alpha_{(k)}\prod_{l=k+1}^{t} a_{(l)} \right).
\end{align}
In the case of $\alpha_{(t)} = \alpha/\SCfactor$, things are very simple. We have $a_{(t)} = a \defeq \sqrt{1-\alpha}$ for all $t$. The above inequality simplifies to
\begin{align*}
\|\wwhat_{(t+1)}-\wwstar\| & \leq \|\wwhat_{(0)}-\wwstar\| a^{t+1} + \frac{\varepsilon\alpha}{\SCfactor} \left(1 + a + \cdots + a^{t}\right)\\
& = \|\wwhat_{(0)}-\wwstar\| a^{t+1} + \frac{\varepsilon\alpha}{\SCfactor} \frac{(1-a^{t+1})}{(1-a)}.
\end{align*}
To clean up the second summand in (\ref{eqn:recursion_unfolded}),
\begin{align*}
\frac{\alpha\varepsilon}{\SCfactor}\frac{(1-a^{t+1})}{1-a} & \leq \frac{\alpha\varepsilon}{\SCfactor}\frac{(1+a)}{(1-a)(1+a)}\\
& = \frac{\alpha\varepsilon}{\SCfactor}\frac{(1+\sqrt{1-\alpha})}{\alpha}\\
& \leq \frac{2\varepsilon}{\SCfactor}.
\end{align*}
This gives us the first statement as desired. For the case of $\alpha_{(t)} = 1/((2+t)\SCfactor)$, things are only slightly more complicated. First observe that
\begin{align}\label{eqn:infinite_prod_1}
\prod_{m=2}^{M} \left(1 - \frac{1}{m}\right) = \prod_{m=2}^{M} \frac{m-1}{m} = \frac{1}{M},
\end{align}
where the last equality follows by simply cancelling terms. We can now handle the first summand in (\ref{eqn:recursion_unfolded}) as
\begin{align*}
\left(\prod_{k=0}^{t}a_{(k)}\right)^{2} & = \prod_{k=0}^{t}\left(1 - \alpha_{(k)}\SCfactor\right) = \prod_{k=0}^{t}\left(1-\frac{1}{2+k}\right) = \prod_{k=2}^{t+2}\left(1-\frac{1}{k}\right) = \frac{1}{t+2},
\end{align*}
where the final equality uses (\ref{eqn:infinite_prod_1}). As for the second summand in (\ref{eqn:recursion_unfolded}), first note that for any $k \geq 1$, we have
\begin{align*}
\frac{\alpha_{(k)}}{a_{(k)}\alpha_{(k-1)}} = \frac{(2+k-1)}{(2+k)\left(1-\frac{1}{(2+k)}\right)} = 1.
\end{align*}
Then recalling the second term on the right-hand side of (\ref{eqn:recursion_unfolded}), consider any two consecutive summands within the parentheses, say
\begin{align}\label{eqn:all_terms_equal}
\alpha_{(k)}a_{(k+1)}\cdots{}a_{(t)} \text{ and } \alpha_{(k-1)}a_{(k)}\cdots{}a_{(t)}
\end{align}
for any $1 \leq k < t$. Dividing the first term by the second term, note that almost all the factors cancel, yielding
\begin{align*}
\frac{\alpha_{(k)}a_{(k+1)}\cdots{}a_{(t)}}{\alpha_{(k-1)}a_{(k)}\cdots{}a_{(t)}} = \frac{\alpha_{(k)}}{a_{(k)}\alpha_{(k-1)}} = 1,
\end{align*}
by what we just proved in (\ref{eqn:all_terms_equal}). It follows that all terms inside the parentheses next to $\varepsilon$ are identical, and indeed equal to $\alpha_{(t)}$, which is to say
\begin{align*}
\varepsilon \left( \alpha_{(t)} + \sum_{k=0}^{t-1}\alpha_{(k)}\prod_{l=k+1}^{t} a_{(l)} \right) = (t+1)\alpha_{(t)}\varepsilon = \frac{(t+1)\varepsilon}{(t+2)\SCfactor} \leq \frac{\varepsilon}{\SCfactor}.
\end{align*}
Plugging these two new forms into the original inequality (\ref{eqn:recursion_unfolded}) yields our second desired result, and concludes the proof.
\end{proof}

\begin{proof}[Proof of Theorem \ref{thm:riskbd_fixed_strong}]
Using the strong convexity of $\risk$ (via \ref{asmp:risk_strong_convex}) and  (\ref{eqn:facts_CO_1}), it follows that
\begin{align*}
\risk(\wwhat_{(T)}) - \risk^{\ast} & \leq \frac{\paraSmooth}{2}\|\wwhat_{(T)} - \wwstar\|^{2}\\
& \leq \paraSmooth(1-\alpha)^{T}\|\wwhat_{(0)}-\wwstar\|^{2} + \frac{4\paraSmooth\varepsilon^{2}}{\SCfactor^{2}}.
\end{align*}
The latter inequality holds by direct application of Lemma \ref{lem:dist_control_strong} under fixed step size, followed by the elementary fact $(a+b)^{2} \leq 2(a^{2}+b^{2})$. The particular value of $\varepsilon$ under which Lemma \ref{lem:dist_control_strong} is valid (i.e., under which (\ref{eqn:conf_uniform}) holds) is given by Lemma \ref{lem:uniform_grad_accuracy} as $\widetilde{\varepsilon}$. Setting $\varepsilon = \widetilde{\varepsilon}$ yields the desired result.
\end{proof}

\begin{proof}[Proof of Theorem \ref{thm:variance_control}]
We construct an upper bound using
\begin{align*}
\exx\|\wwhat_{(t+1)}-\wwhat_{(t)}\|^{2} & = \alpha_{(t)}^{2} \exx\|\rgest(\wwhat_{(t)})\|^{2}\\
& \leq \alpha_{(t)}^{2} \exx\left(\|\rgest(\wwhat_{(t)})-\rgrad(\wwhat_{(t)})\|+\|\rgrad(\wwhat_{(t)})\|\right)^{2}\\
& \leq 2\alpha_{(t)}^{2} \left(\exx\|\rgest(\wwhat_{(t)})-\rgrad(\wwhat_{(t)})\|^{2} + \|\rgrad(\wwhat_{(t)})\|^{2} \right).
\end{align*}
Now, if we condition on $\wwhat_{(t)}$, by assumption the loss gradients $\lgrad(\wwhat_{(t)};\zz_{1}),\ldots,\lgrad(\wwhat_{(t)};\zz_{n})$ are iid. With independence, just as in the proof of Lemma \ref{lem:uniform_grad_accuracy}, we have
\begin{align*}
\|\rgest(\wwhat_{(t)}) - \rgrad(\wwhat_{(t)})\| \leq \sqrt{\frac{2Vd\log(d\delta^{-1})}{n}} + \sqrt{\frac{V}{n}},
\end{align*}
with probability no less than $1-\delta$. Setting the right-hand side of this equation to $\varepsilon$ and solving for $\delta$, we have exponential tails of the form
\begin{align*}
\prr\left\{\|\rgest(\wwhat_{(t)}) - \rgrad(\wwhat_{(t)})\| > \varepsilon \right\} \leq d \exp\left(-\frac{(\varepsilon-a)^{2}}{2b^{2}}\right)
\end{align*}
with constants defined
\begin{align*}
a \defeq \sqrt{\frac{V}{n}}, \quad b \defeq \sqrt{\frac{Vd}{n}}.
\end{align*}
Controlling moments using exponential tails can be done as follows. For random variable $X \in \LL_{p}$ for $p \geq 1$, recall the classic inequality
\begin{align*}
\exx|X|^{p} = \int_{0}^{\infty} \prr\{|X|^{p}>t\}\,dt.
\end{align*}
Our setting of interest is $X = \|\rgest(\wwhat_{(t)}) - \rgrad(\wwhat_{(t)})\|$, with $p=2$. It follows that
\begin{align*}
\exx|X|^{2} & = \int_{0}^{\infty} \prr\{|X|^{2}>u\} \, du\\
& = \int_{0}^{\infty} \prr\{X>\sqrt{u}\} \, du\\
& = \int_{0}^{\infty} \prr\{X>u\} \frac{u}{2} \, du\\
& \leq \frac{d}{2} \int_{0}^{\infty} \exp\left(-\frac{(u-a)^{2}}{2b^{2}}\right)u \, du.
\end{align*}
The third equality uses substitution of variables, and the inequality at the end uses the exponential tail inequality given above. This integral is the expectation of the Normal distribution $N(a,b^{2})$ taken over just the positive half-line. A simple upper bound can be constructed by
\begin{align*}
\int_{0}^{\infty} \exp\left(-\frac{(u-a)^{2}}{2b^{2}}\right)u \, du & = \int_{-\infty}^{\infty} I\{u \geq 0\} \exp\left(-\frac{(u-a)^{2}}{2b^{2}}\right)u \, du\\
& \leq \int_{-\infty}^{\infty} \exp\left(-\frac{(u-a)^{2}}{2b^{2}}\right)|u| \, du,
\end{align*}
easily recognized (after rescaling by $1/\sqrt{2\pi{}b^{2}}$) as the expectation of a Folded Normal random variable, induced by $N(a,b^{2})$. Recalling the proof of Lemma \ref{lem:est_lipschitz}, the expected value of this Folded Normal random variable is
\begin{align*}
\int_{-\infty}^{\infty} \frac{1}{\sqrt{2\pi}b} \exp\left(-\frac{(u-a)^{2}}{2b^{2}}\right)|u| \, du & = a\left( 1 - 2\Phi\left(\frac{-a}{b}\right) \right) + b\sqrt{\frac{2}{\pi}}\exp\left(\frac{-a^{2}}{2b^{2}}\right)\\
& = \sqrt{\frac{V}{n}}\left( 1 - 2\Phi\left(\frac{-1}{\sqrt{d}}\right) \right) + \sqrt{\frac{2Vd}{n\pi}}e^{-2d}.
\end{align*}
Taking into account the normalization factor, our upper bound takes the form
\begin{align*}
\exx|X|^{2} \leq \frac{d\sqrt{2\pi{}b^{2}}}{2}\left( \sqrt{\frac{V}{n}}\left( 1 - 2\Phi\left(\frac{-1}{\sqrt{d}}\right) \right) + \sqrt{\frac{2Vd}{n\pi}}e^{-2d} \right).
\end{align*}
Plugging this in for $X = \|\rgest(\wwhat_{(t)})-\rgrad(\wwhat_{(t)})\|$ in the upper bound constructed at the start of this proof yields the desired result.
\end{proof}

\subsection{Computation}\label{sec:tech_computation}

From \citet{catoni2017a}, Lemma 3.2, it follows that the correction term $\corr(a,b)$ used in (\ref{eqn:compute_integral}) can be computed as follows. First, some preparatory definitions to keep notation clean.
\begin{align*}
V_{-} \defeq \frac{\sqrt{2}-a}{b}, \quad V_{+} \defeq \frac{\sqrt{2}+a}{b}
\end{align*}
\begin{align*}
F_{-} \defeq \Phi(-V_{-}), \quad F_{+} \defeq \Phi(-V_{+})
\end{align*}
\begin{align*}
E_{-} \defeq \exp\left(-\frac{V_{-}^{2}}{2}\right), \quad E_{+} \defeq \exp\left(-\frac{V_{+}^{2}}{2}\right).
\end{align*}
As seen in other parts of the text, $\Phi$ denotes the standard Normal CDF. With these atomic elements defined to keep things a bit cleaner, we break the final quantity into five terms to be summed:
\begin{align*}
T_{1} & \defeq \frac{2\sqrt{2}}{3}\left(F_{-}-F_{+}\right)\\
T_{2} & \defeq -\left(a - \frac{a^{3}}{6}\right)\left(F_{-}+F_{+}\right)\\
T_{3} & \defeq \frac{b}{\sqrt{2\pi}} \left(1-\frac{a^{2}}{2}\right)\left(E_{+}-E_{-}\right)\\
T_{4} & \defeq \frac{ab^{2}}{2} \left( F_{+}+F_{-}+\frac{1}{\sqrt{2\pi}}\left(V_{+}E_{+} + V_{-}E_{-}\right) \right)\\
T_{5} & \defeq \frac{b^{3}}{6\sqrt{2\pi}} \left( (2+V_{-}^{2})E_{-} - (2+V_{+}^{2})E_{+} \right).
\end{align*}
With these terms in hand, the final computation is just summation, as
\begin{align*}
C(a,b) = T_{1}+T_{2}+T_{3}+T_{4}+T_{5}.
\end{align*}

\section{Additional test results}\label{sec:more_test_results}

Due to space restrictions and overall readability, we did not include all empirical test results in the tests of section \ref{sec:tests}. Here we provide results for all of the noise distribution families considered in the second class of numerical experiments, namely the regression application considered at the end of section \ref{sec:tests_noisyopt}. The following distribution families are considered: Arcsine (\texttt{asin}), Beta Prime (\texttt{bpri}), Chi-squared (\texttt{chisq}), Exponential (\texttt{exp}), Exponential-Logarithmic (\texttt{explog}), Fisher's F (\texttt{f}), Fr\'{e}chet (\texttt{frec}), Gamma (\texttt{gamma}), Gompertz (\texttt{gomp}), Gumbel (\texttt{gum}), Hyperbolic Secant (\texttt{hsec}), Laplace (\texttt{lap}), Log-Logistic (\texttt{llog}), Log-Normal (\texttt{lnorm}), Logistic (\texttt{lgst}), Maxwell (\texttt{maxw}), Pareto (\texttt{pareto}), Rayleigh (\texttt{rayl}), Semi-circle (\texttt{scir}), Student's t (\texttt{t}), Triangle (asymmetric \texttt{tri\_a}, symmetric \texttt{tri\_s}), U-Power (\texttt{upwr}), Wald (\texttt{wald}), Weibull (\texttt{weibull}).

The content of this section is as follows:
\begin{itemize}
\item Figures \ref{fig:overN_all_distros_1}--\ref{fig:overN_all_distros_2}: performance as a function of sample size $n$. 

\item Figures \ref{fig:overLvl_all_distros_1}--\ref{fig:overLvl_all_distros_2}: performance over noise levels, with fixed $n$ and $d$.

\item Figures \ref{fig:overDim_all_distros_1}--\ref{fig:overDim_all_distros_2}: performance as a function of $d$, with fixed $n/d$ ratio and noise level.
\end{itemize}

\clearpage

\begin{figure}[t]
\centering
\includegraphics[width=0.25\textwidth]{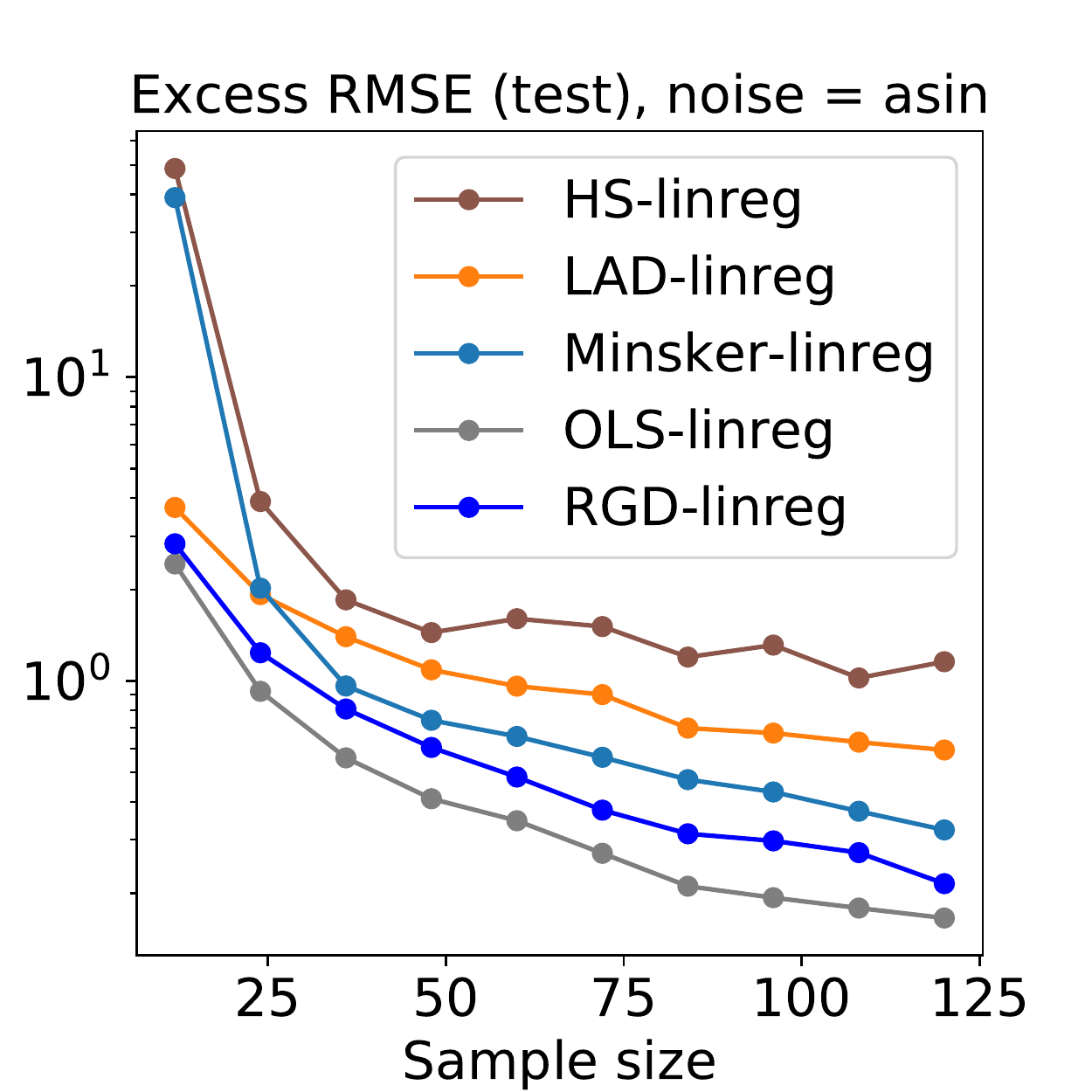}\,\includegraphics[width=0.25\textwidth]{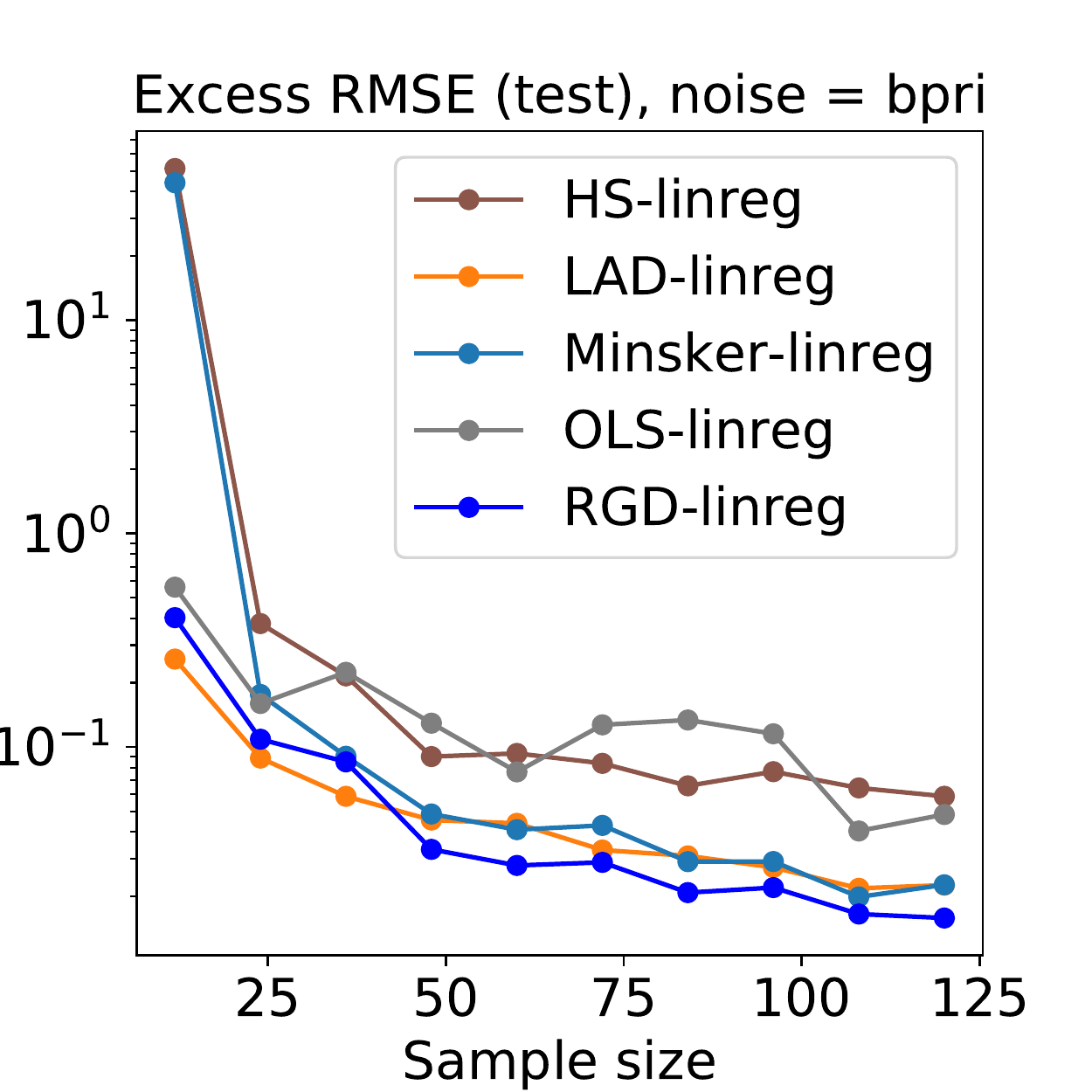}\,\includegraphics[width=0.25\textwidth]{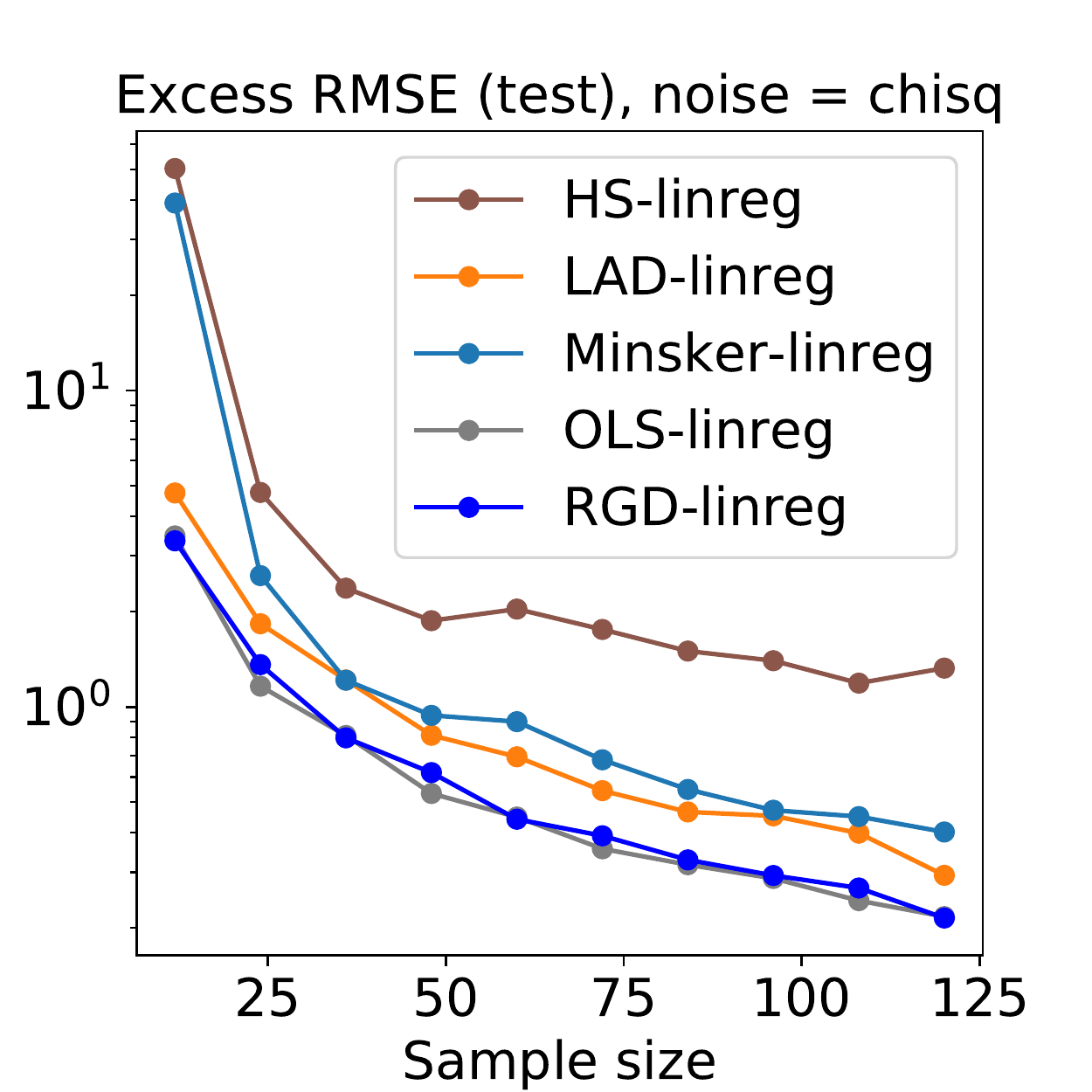}\,\includegraphics[width=0.25\textwidth]{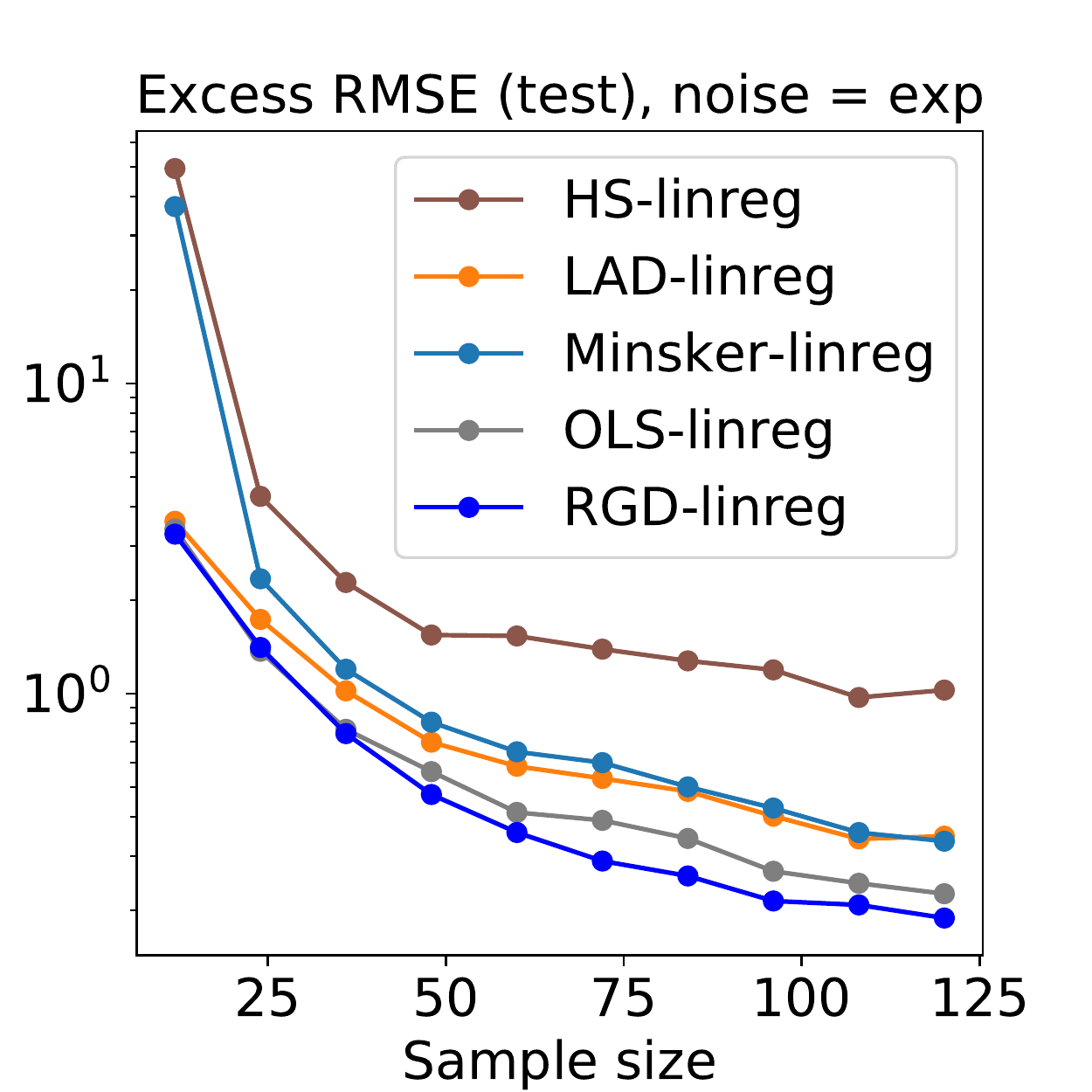}\\
\includegraphics[width=0.25\textwidth]{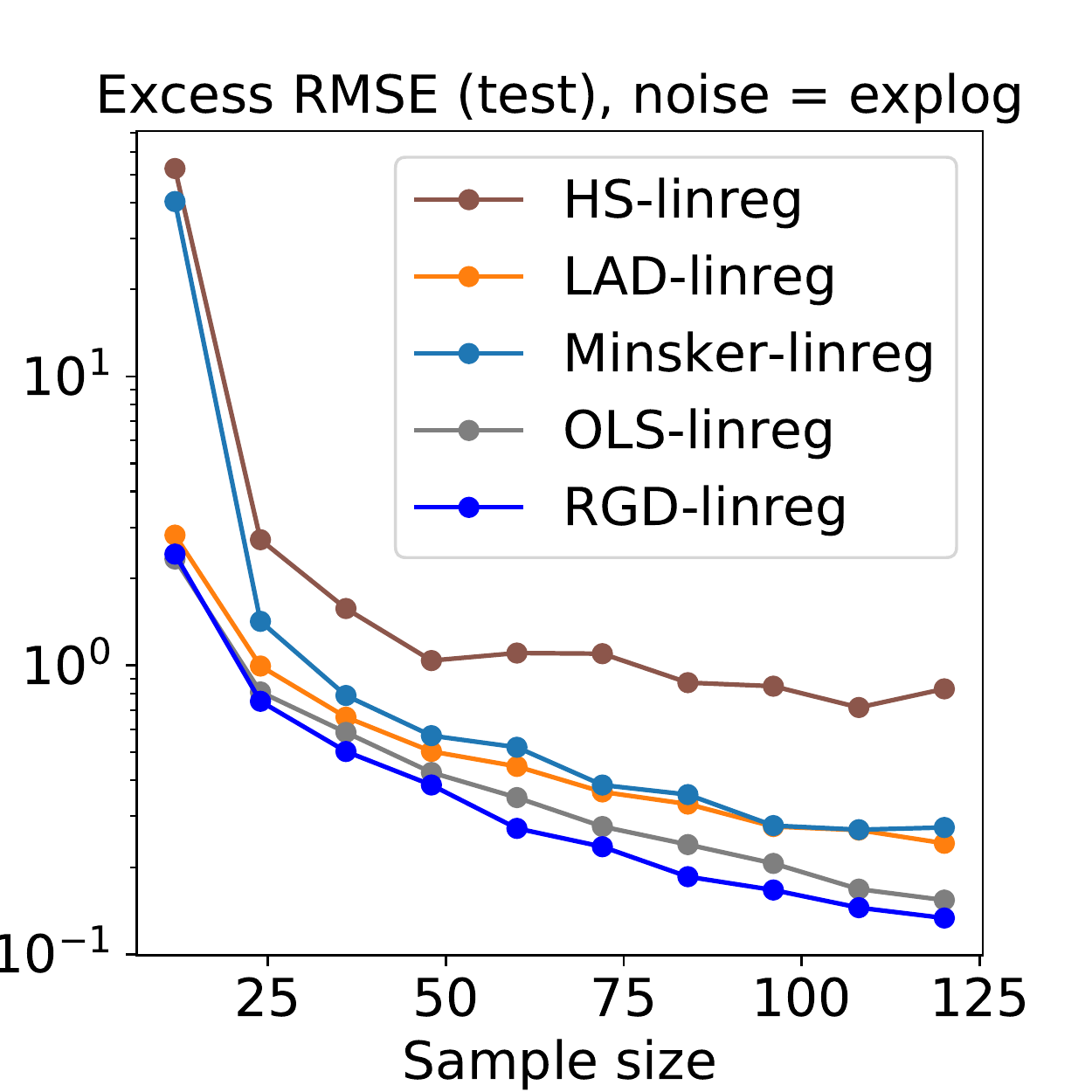}\,\includegraphics[width=0.25\textwidth]{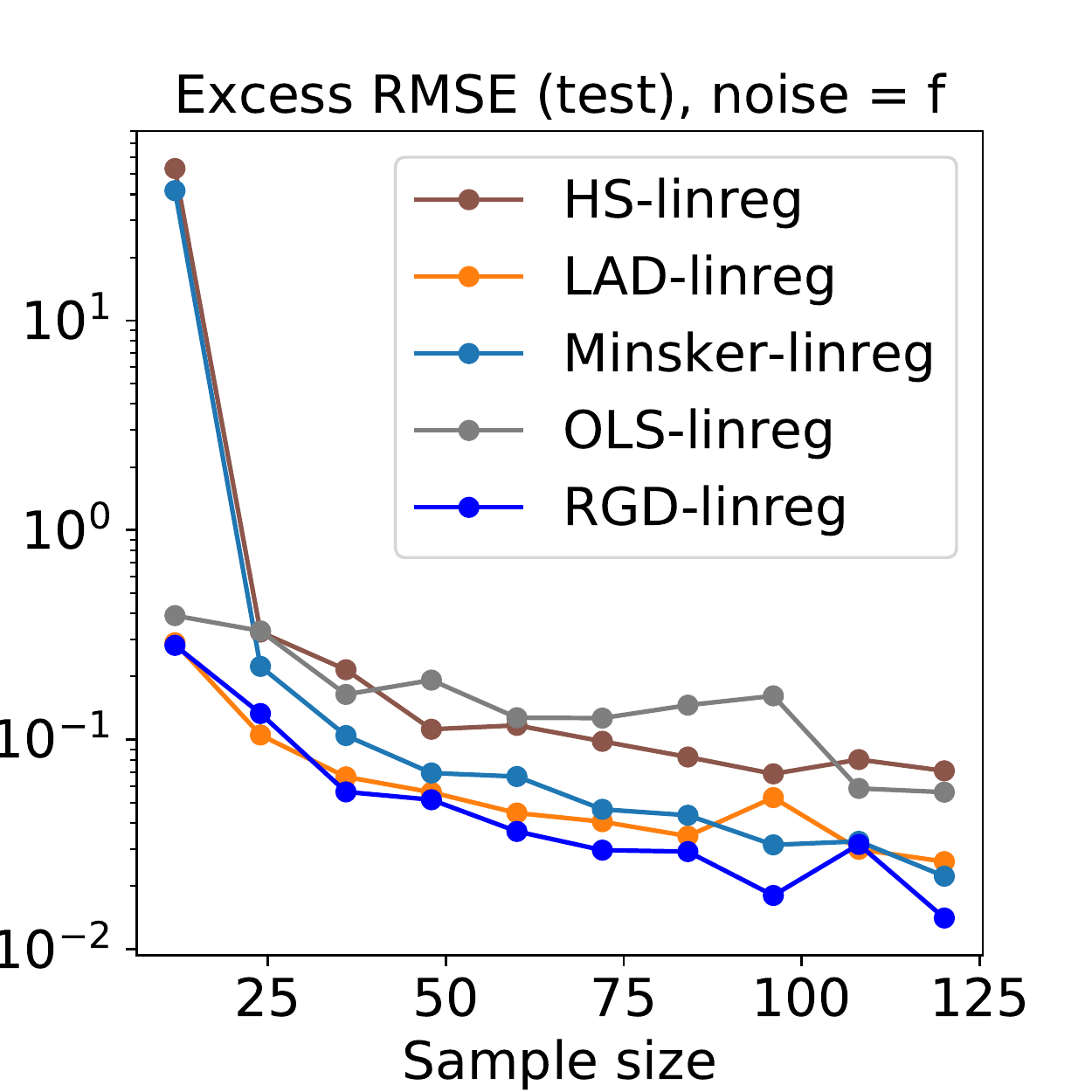}\,\includegraphics[width=0.25\textwidth]{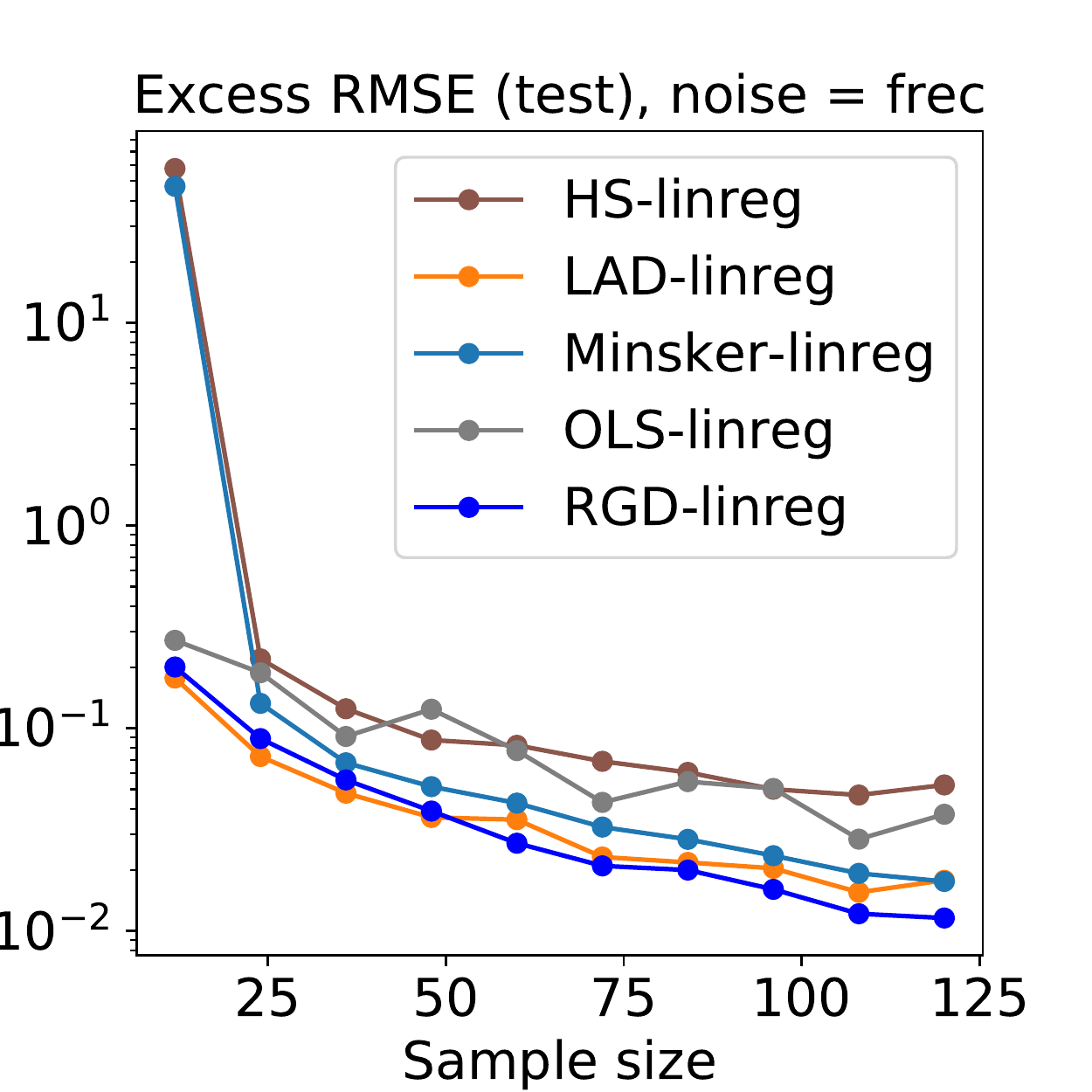}\,\includegraphics[width=0.25\textwidth]{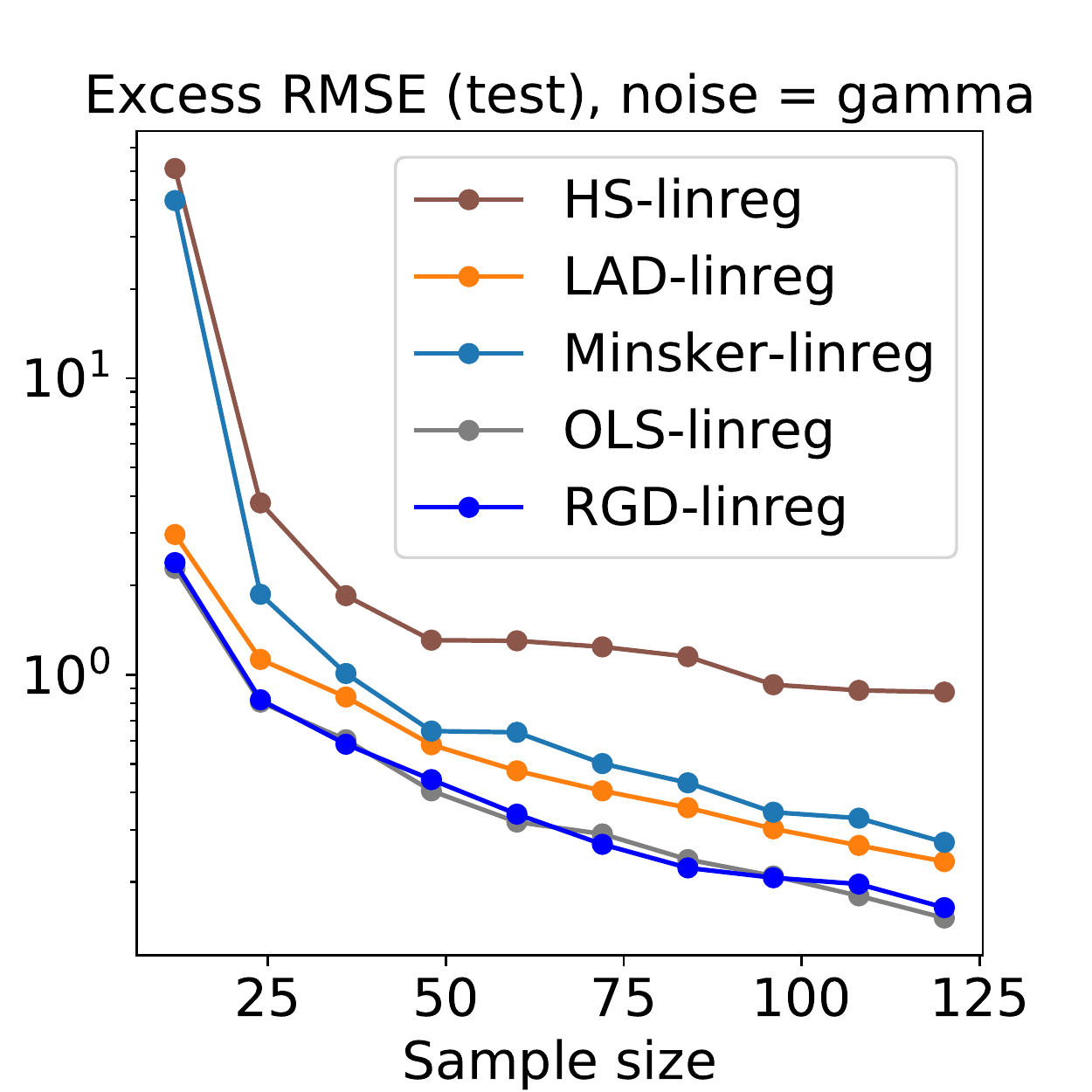}\\
\includegraphics[width=0.25\textwidth]{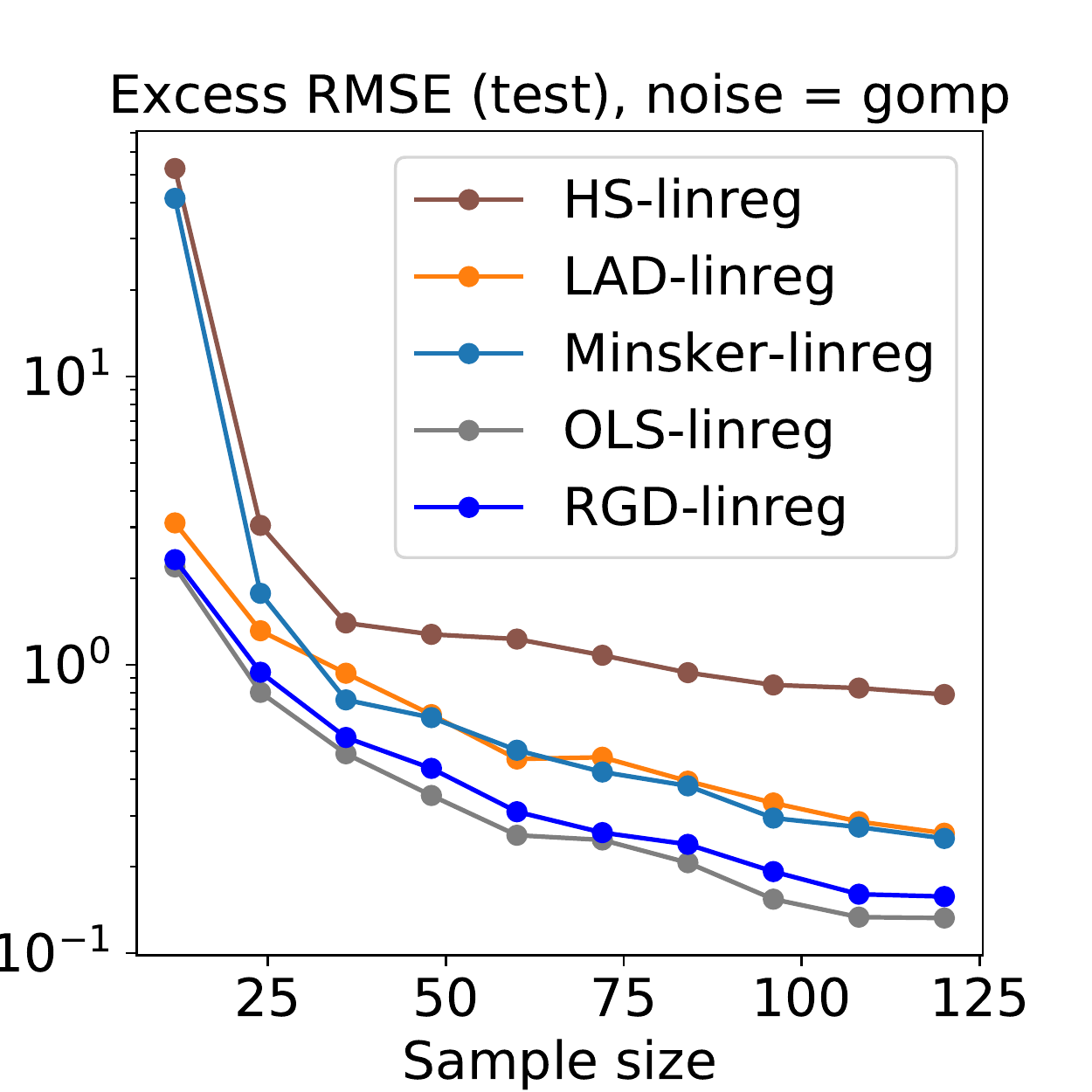}\,\includegraphics[width=0.25\textwidth]{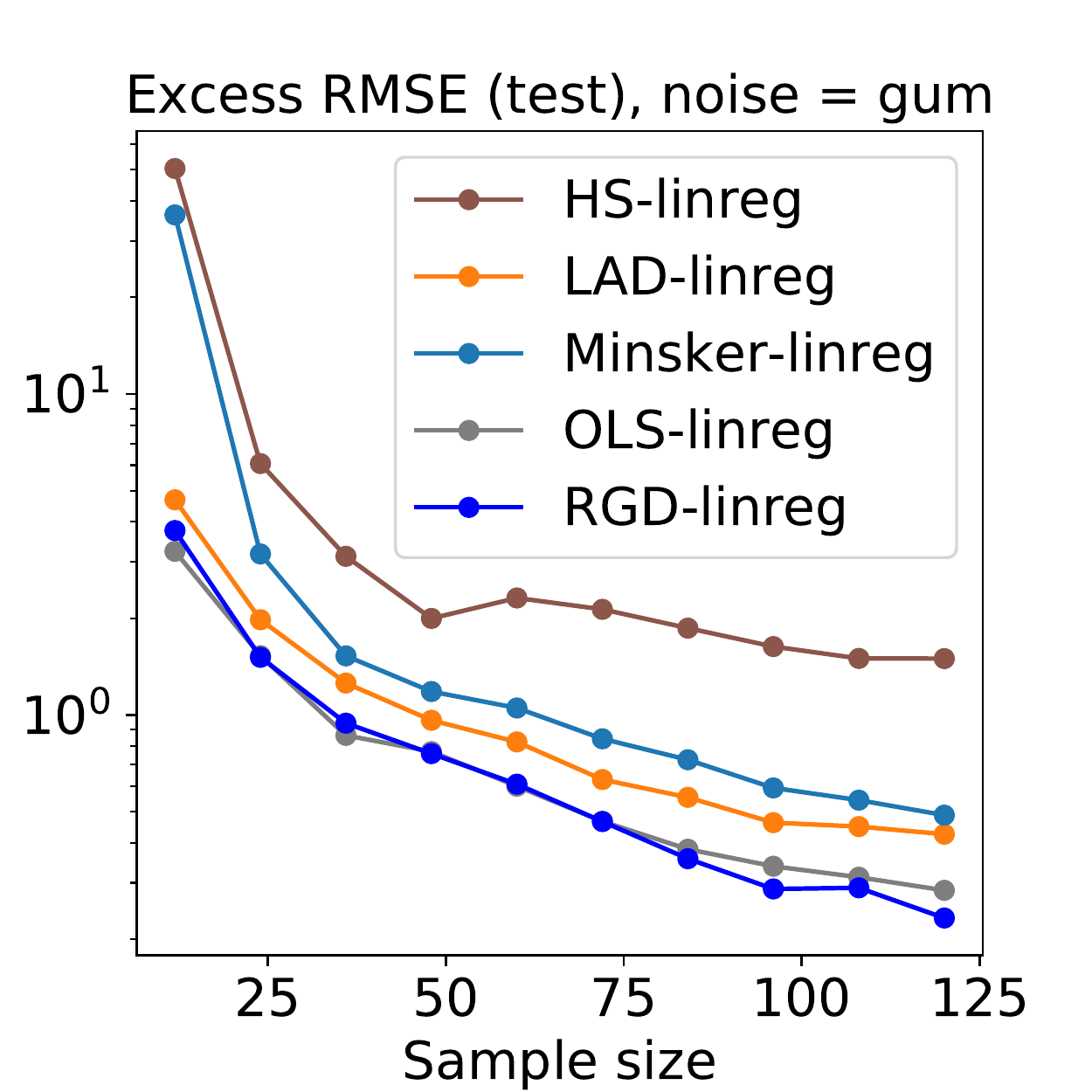}\,\includegraphics[width=0.25\textwidth]{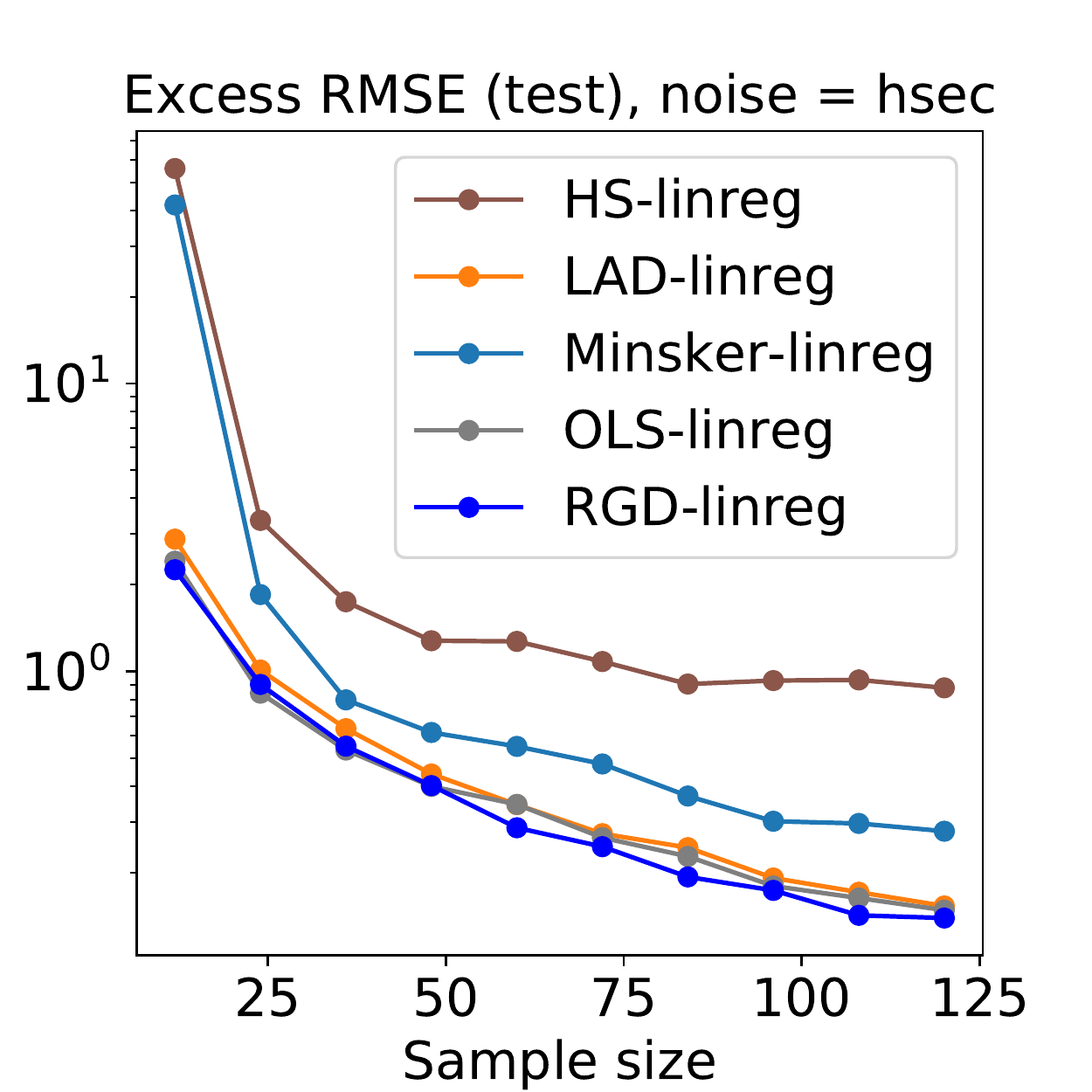}\,\includegraphics[width=0.25\textwidth]{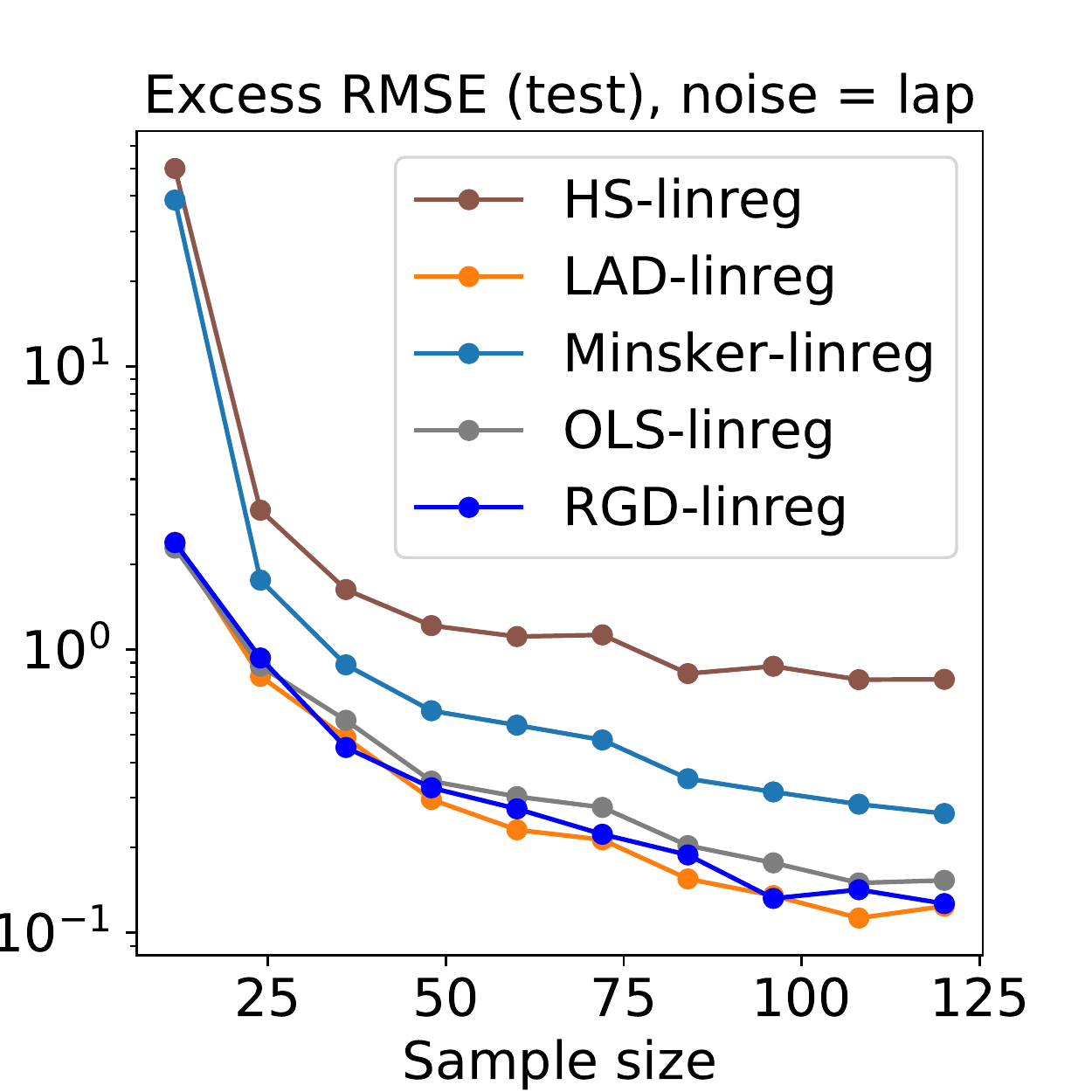}\\
\includegraphics[width=0.25\textwidth]{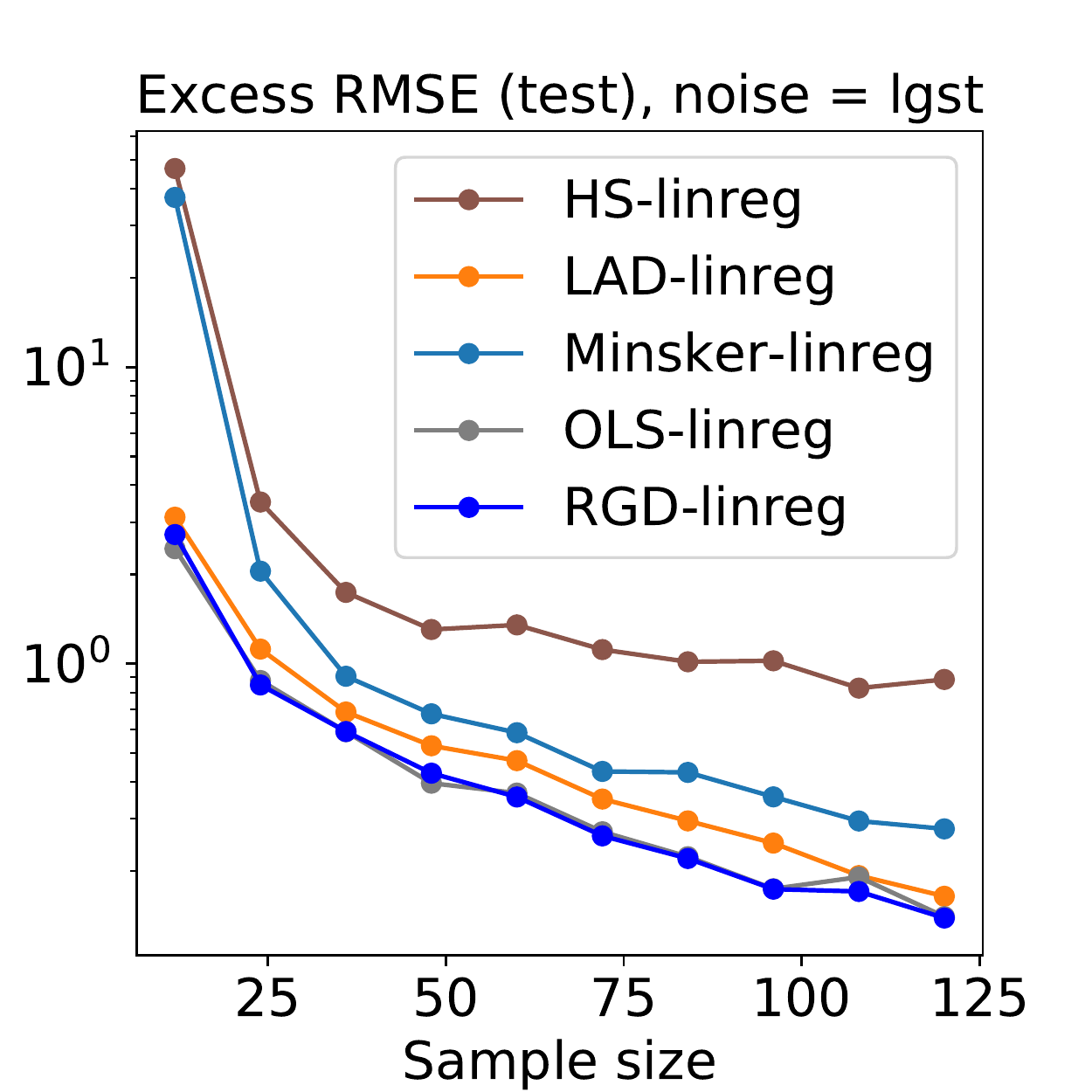}\,\includegraphics[width=0.25\textwidth]{linreg_overN_risk_llog}\,\includegraphics[width=0.25\textwidth]{linreg_overN_risk_lnorm}\,\includegraphics[width=0.25\textwidth]{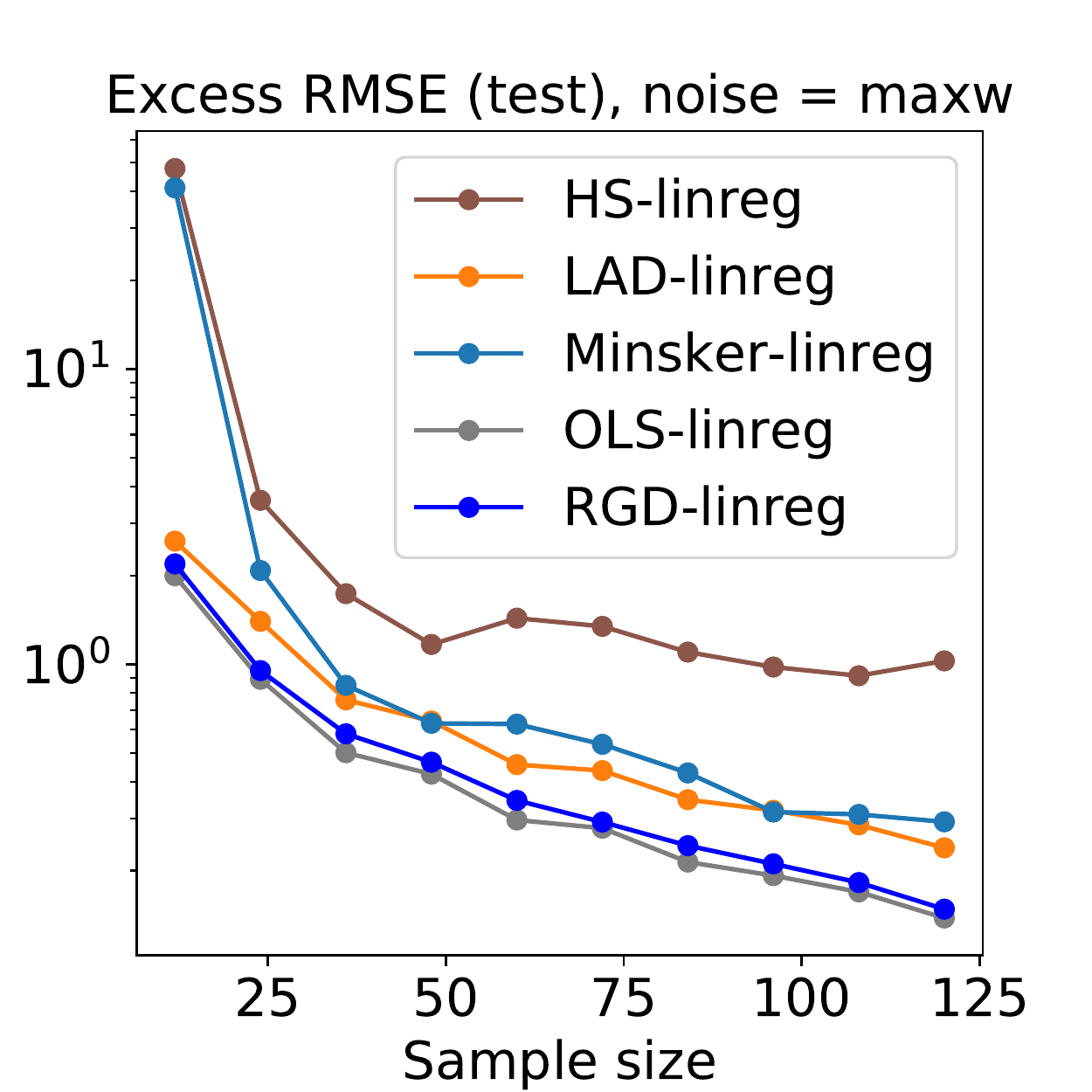}\\
\includegraphics[width=0.25\textwidth]{linreg_overN_risk_norm}\,\includegraphics[width=0.25\textwidth]{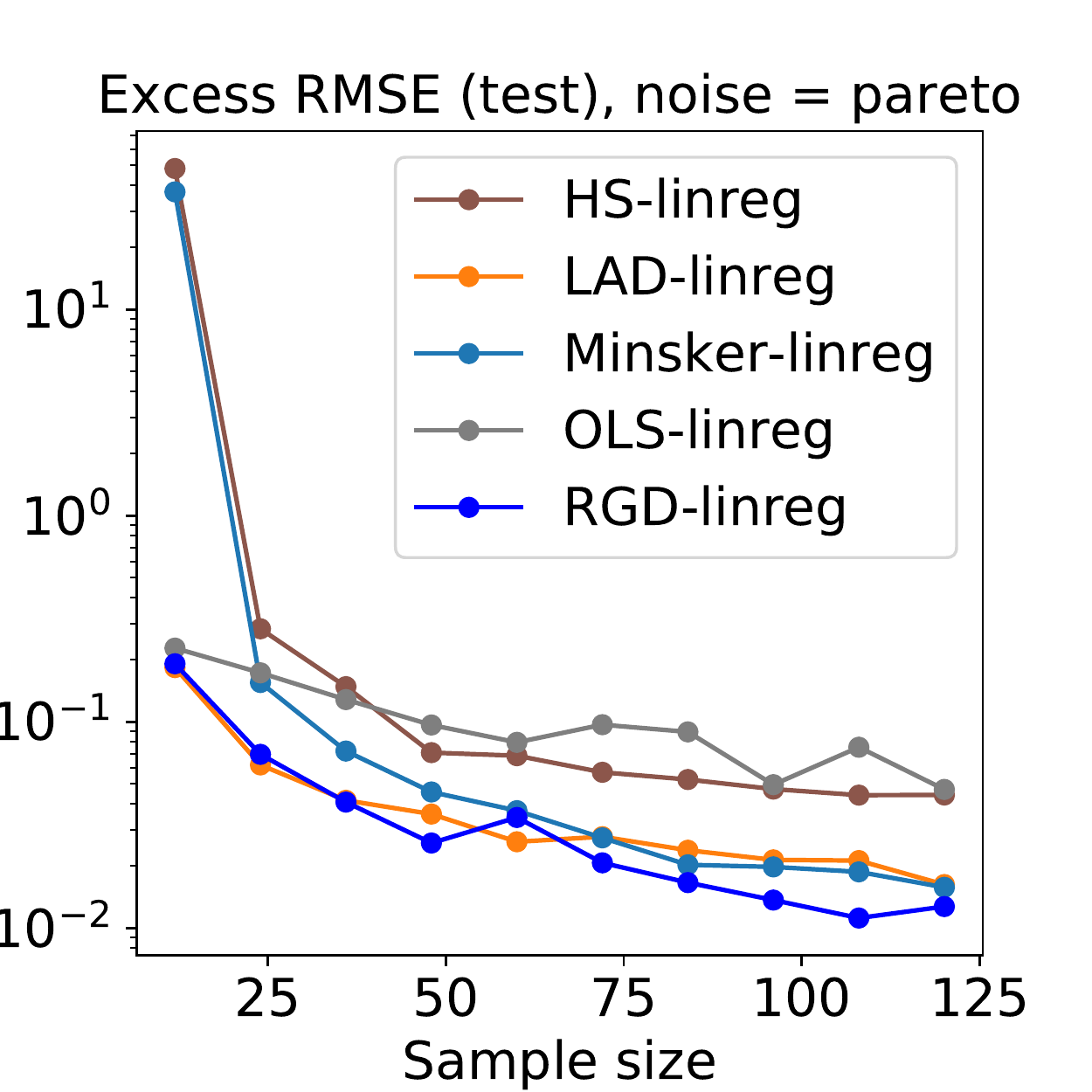}\,\includegraphics[width=0.25\textwidth]{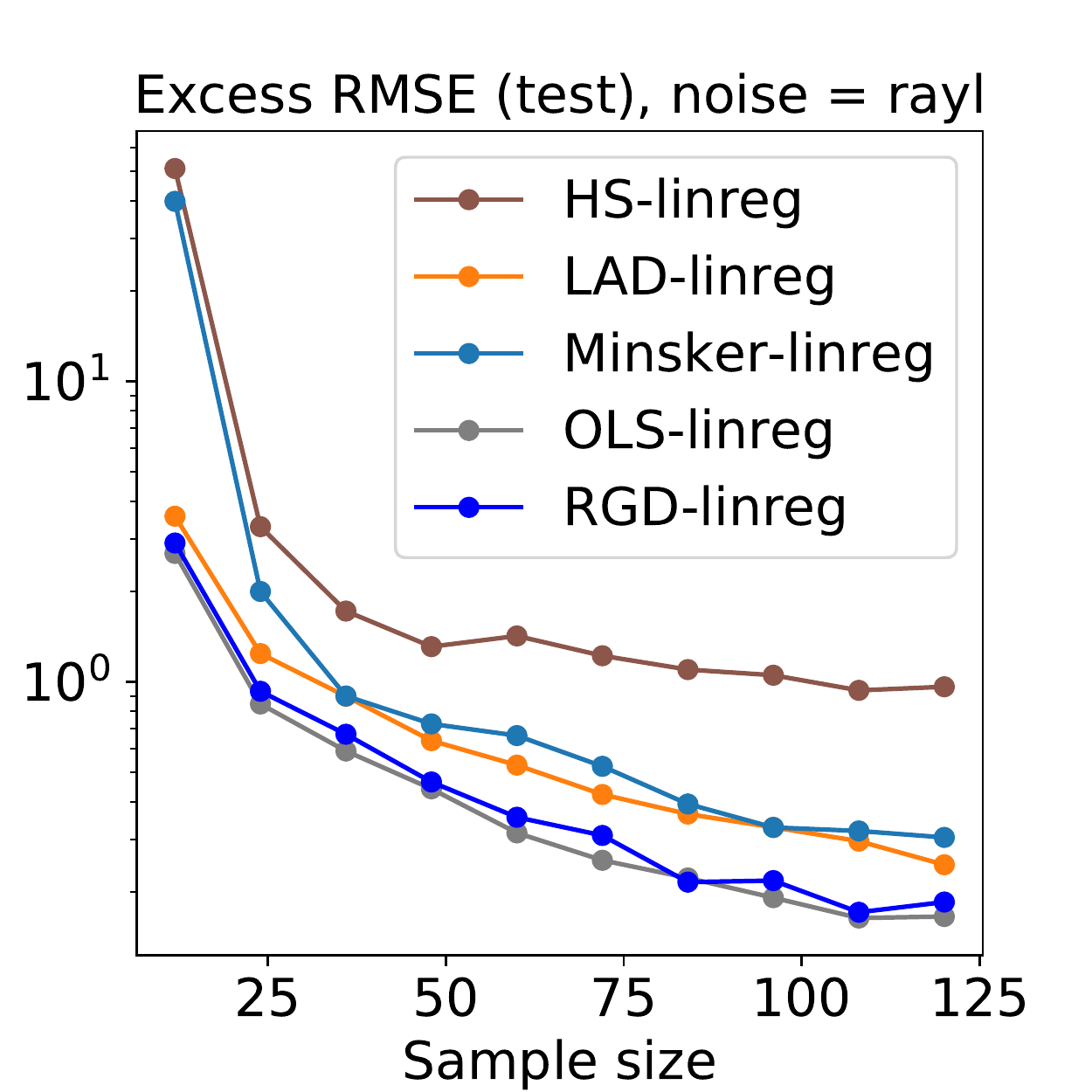}\,\includegraphics[width=0.25\textwidth]{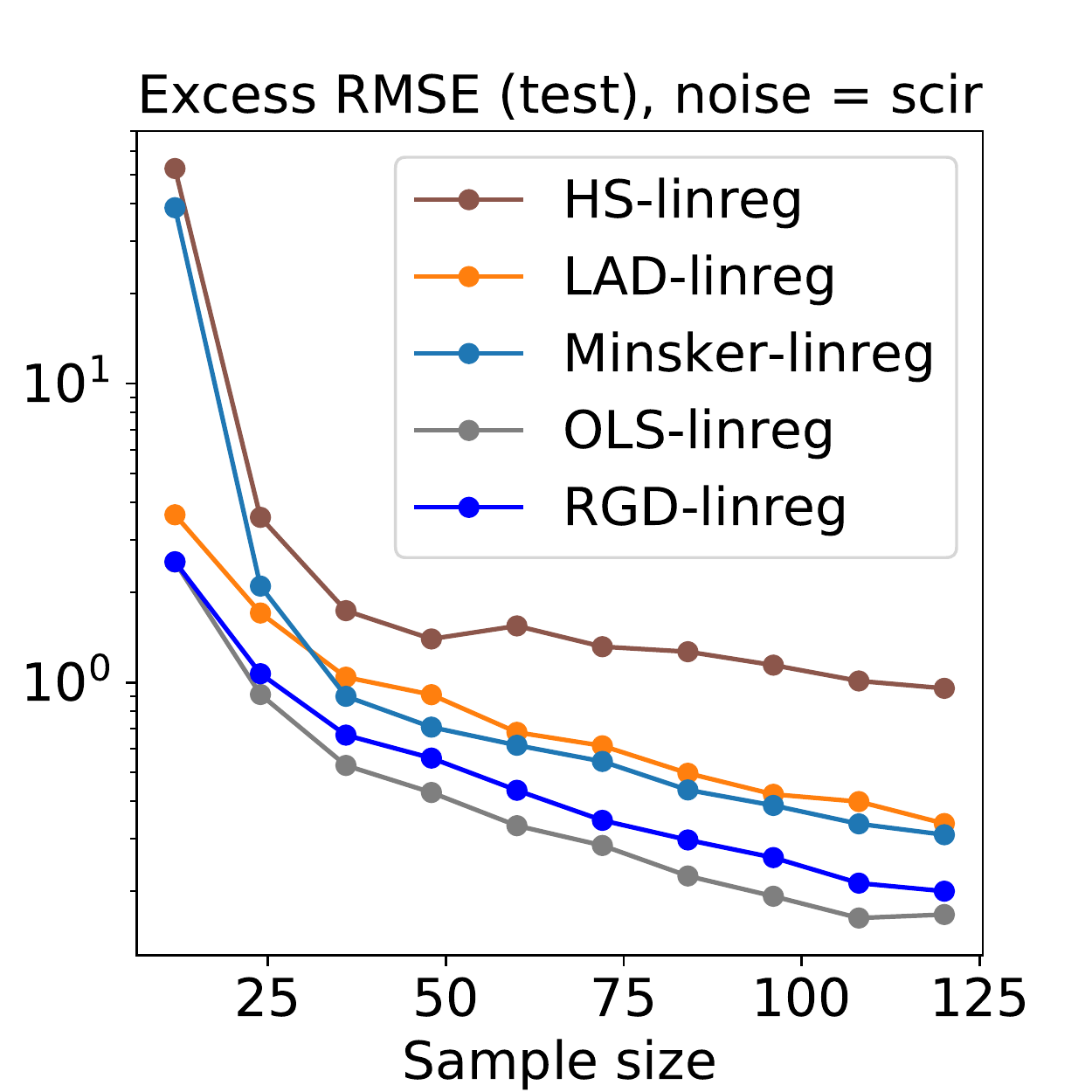}
\caption{Prediction error over sample size $12 \leq n \leq 122$, fixed $d=5$, noise level = $8$. Each plot corresponds to a distinct noise distribution.}
\label{fig:overN_all_distros_1}
\end{figure}

\clearpage

\begin{figure}[t]
\centering
\includegraphics[width=0.25\textwidth]{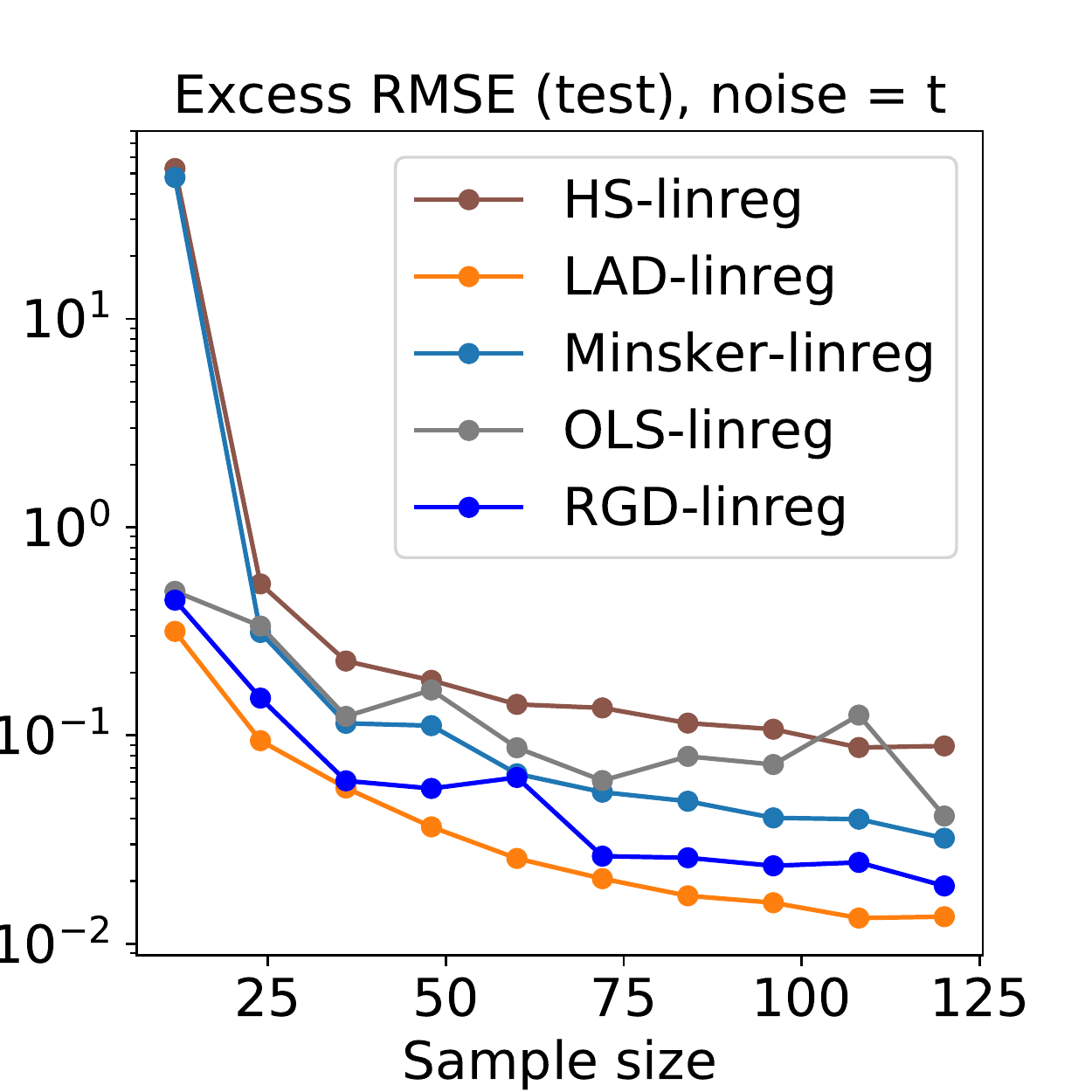}\,\includegraphics[width=0.25\textwidth]{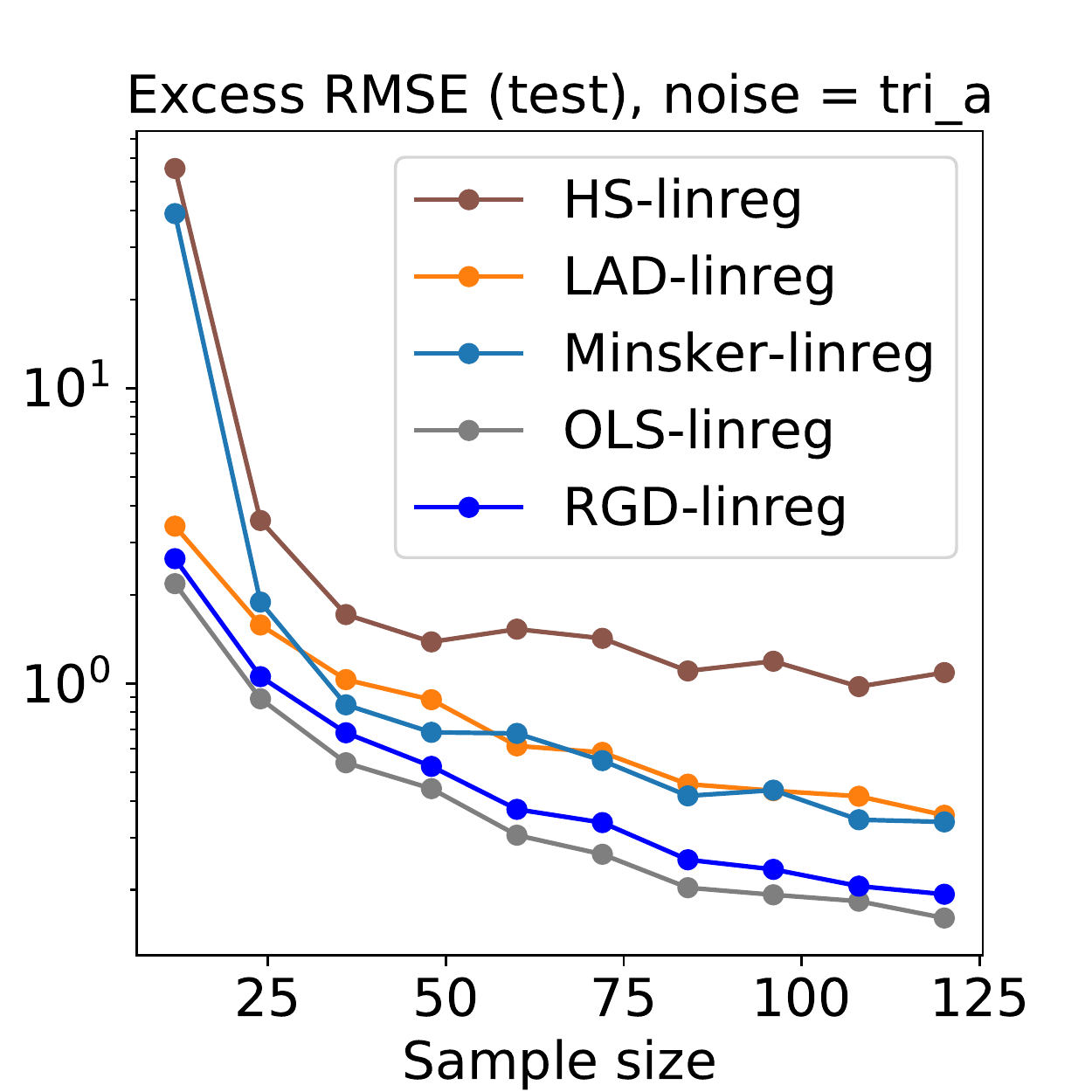}\,\includegraphics[width=0.25\textwidth]{linreg_overN_risk_tri_s}\,\includegraphics[width=0.25\textwidth]{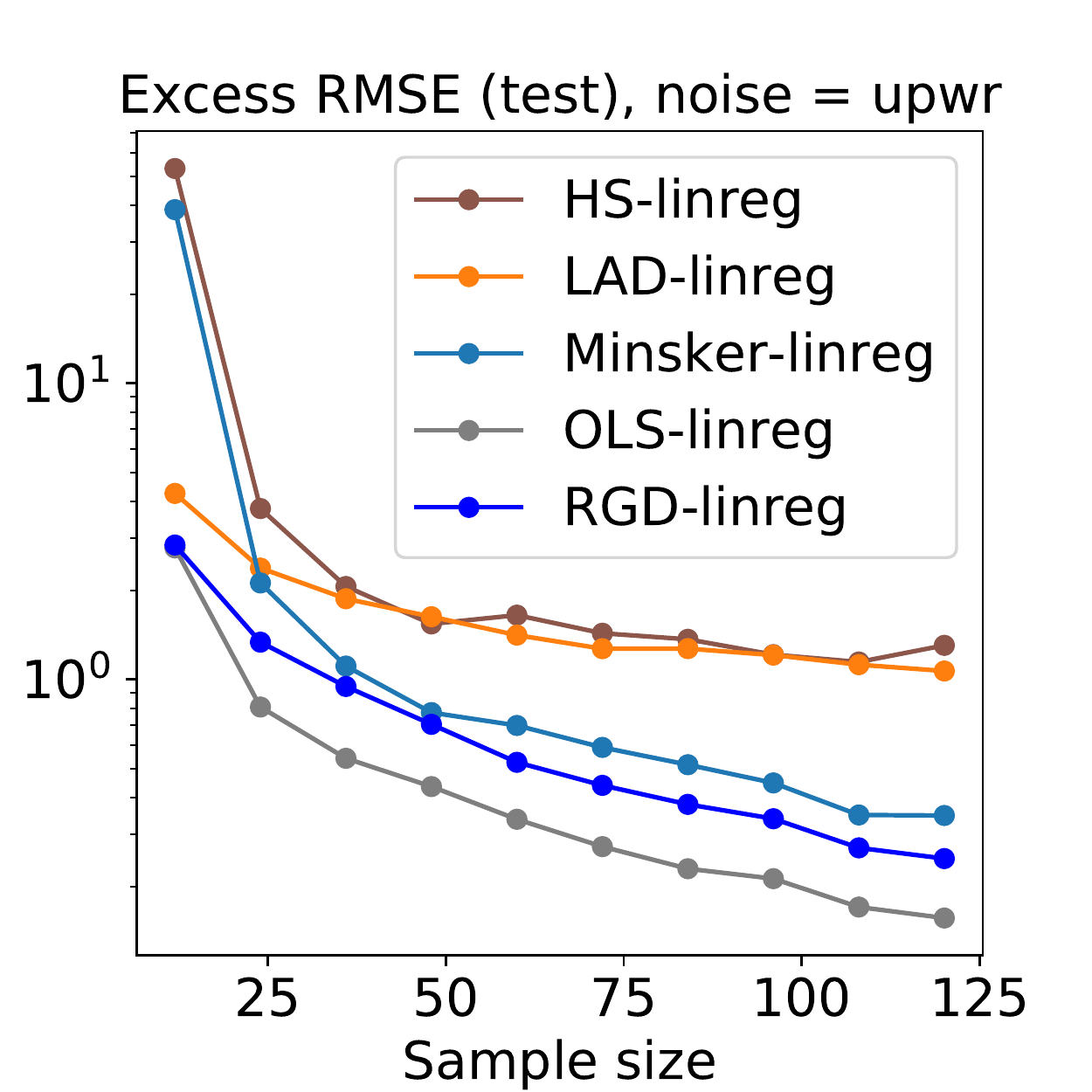}\\
\includegraphics[width=0.25\textwidth]{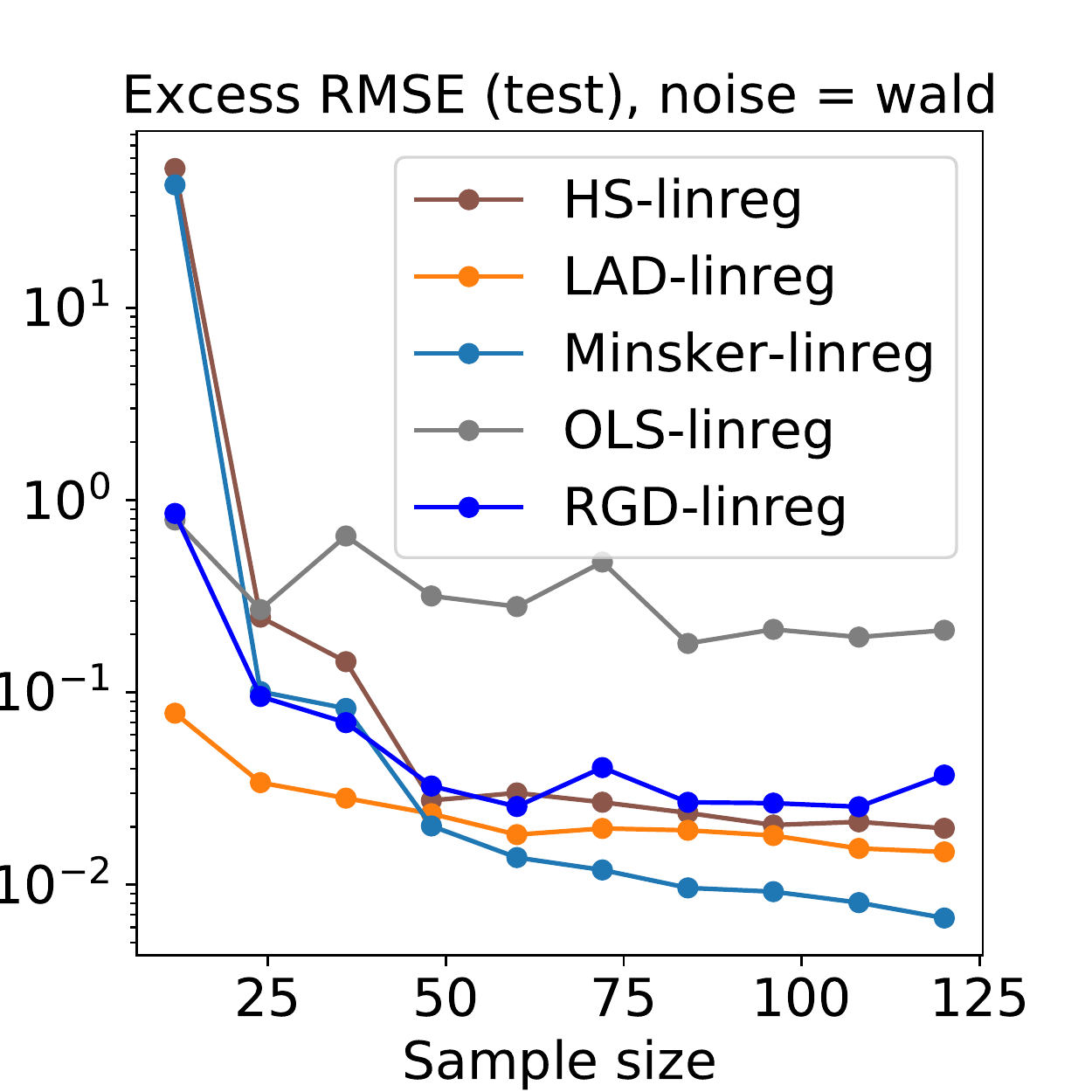}\,\includegraphics[width=0.25\textwidth]{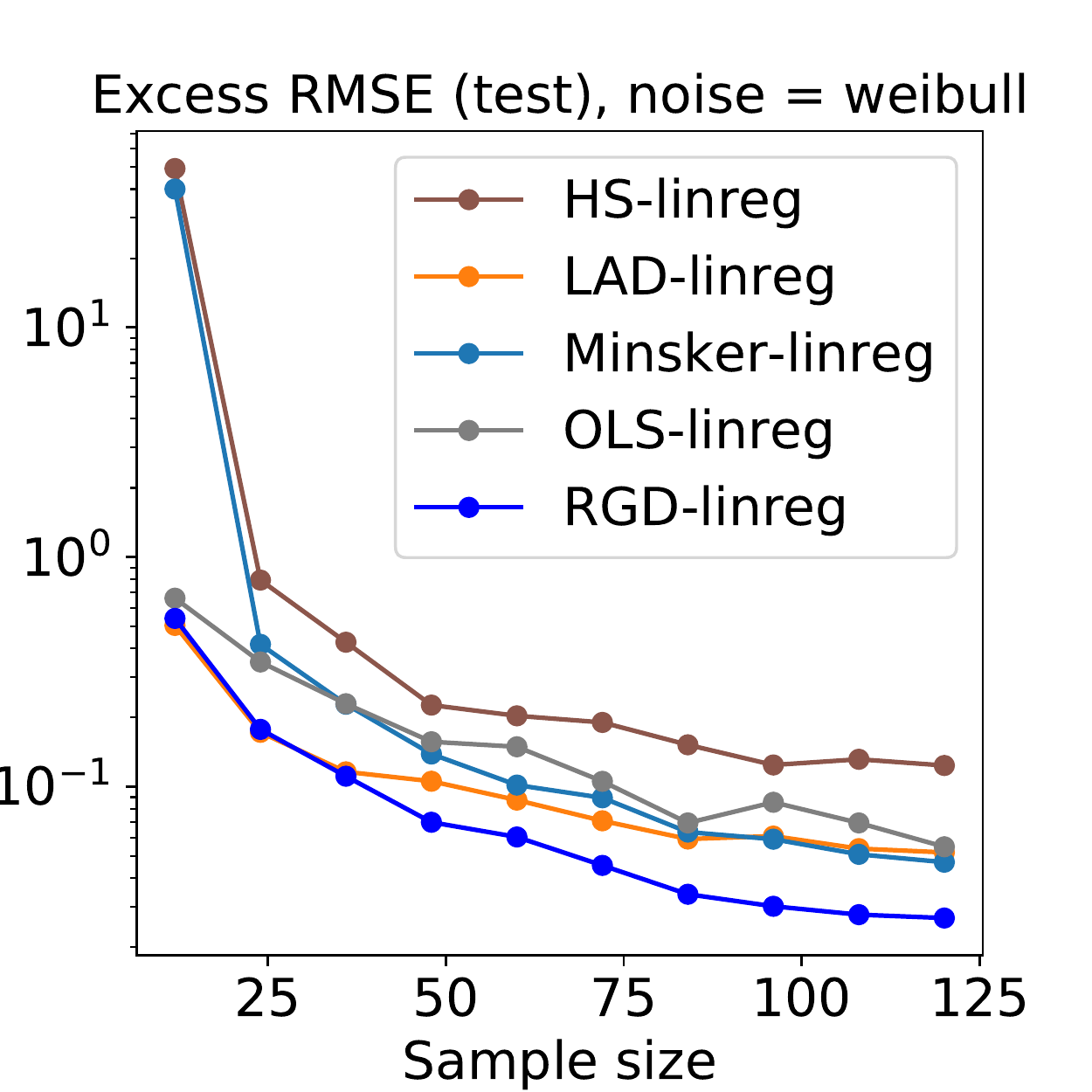}
\caption{Prediction error over sample size $12 \leq n \leq 122$, fixed $d=5$, noise level = $8$. Each plot corresponds to a distinct noise distribution.}
\label{fig:overN_all_distros_2}
\end{figure}

\clearpage

\begin{figure}[t]
\centering
\includegraphics[width=0.25\textwidth]{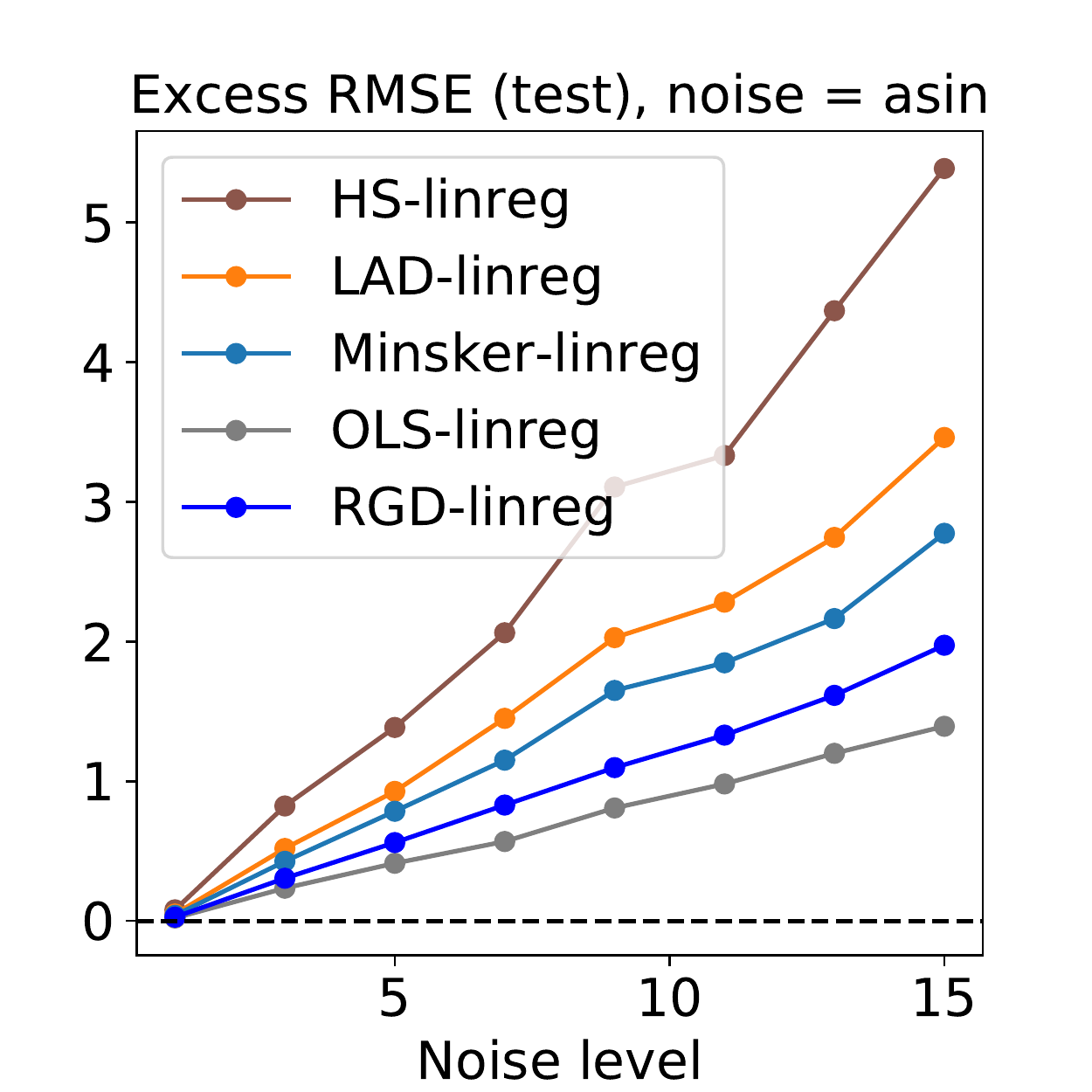}\,\includegraphics[width=0.25\textwidth]{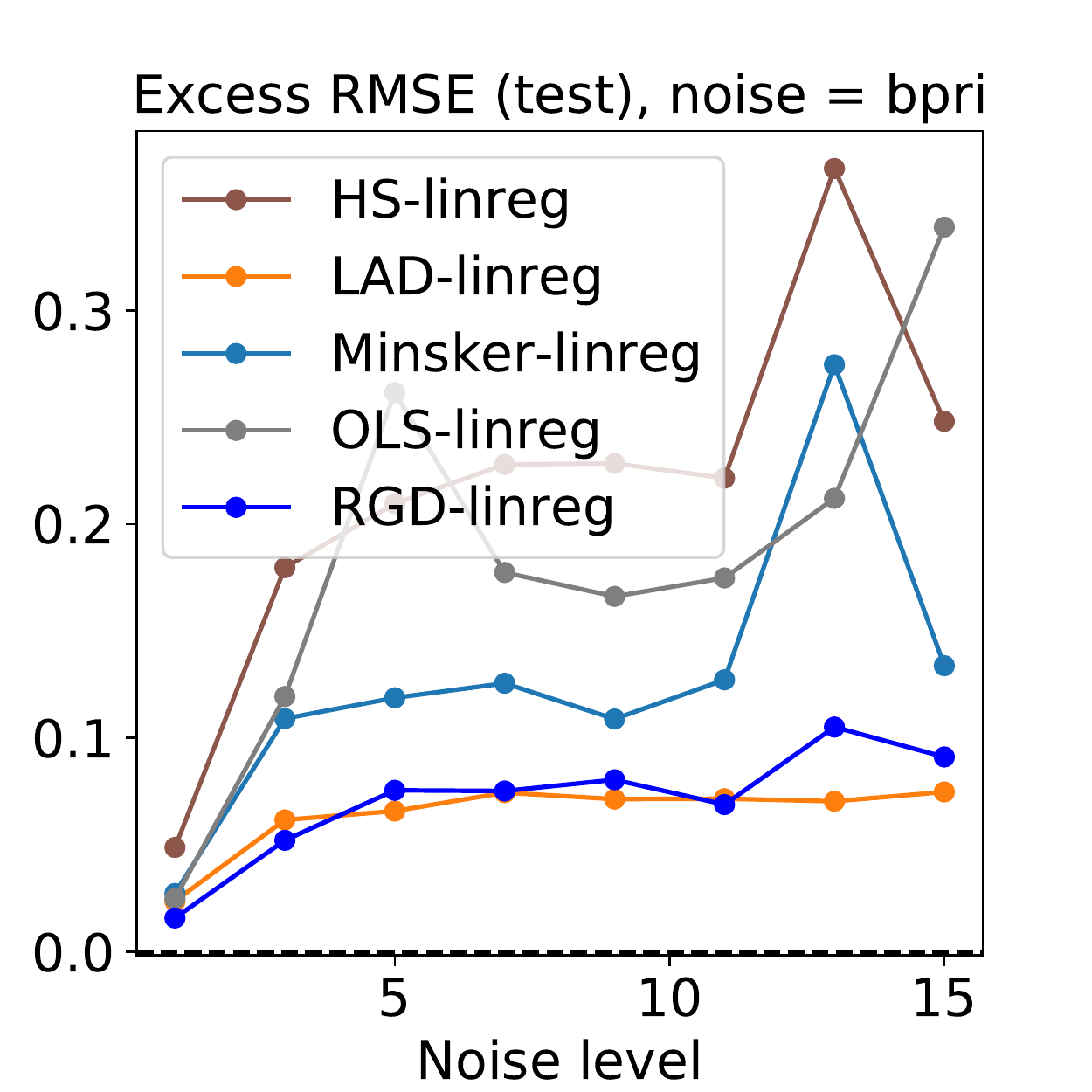}\,\includegraphics[width=0.25\textwidth]{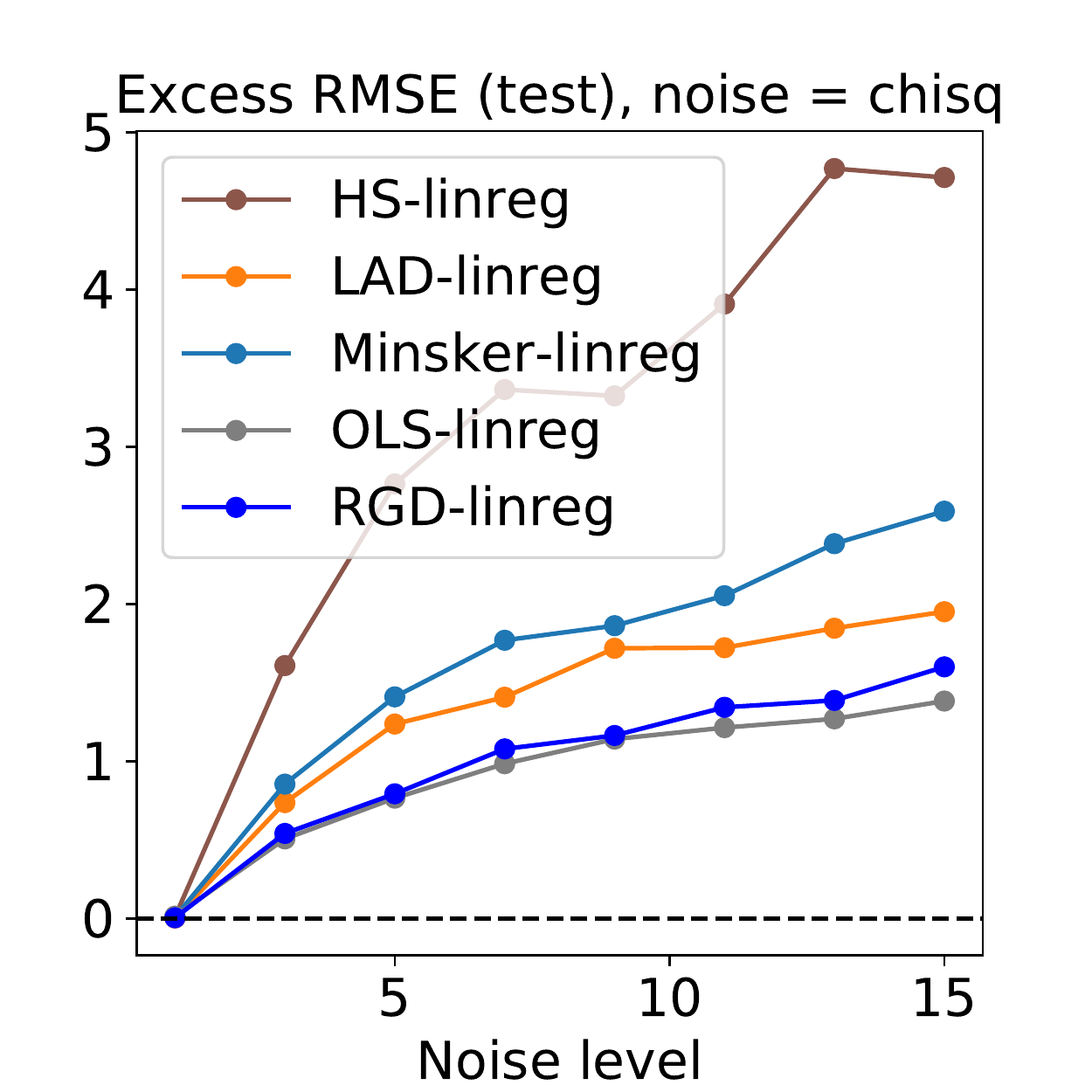}\,\includegraphics[width=0.25\textwidth]{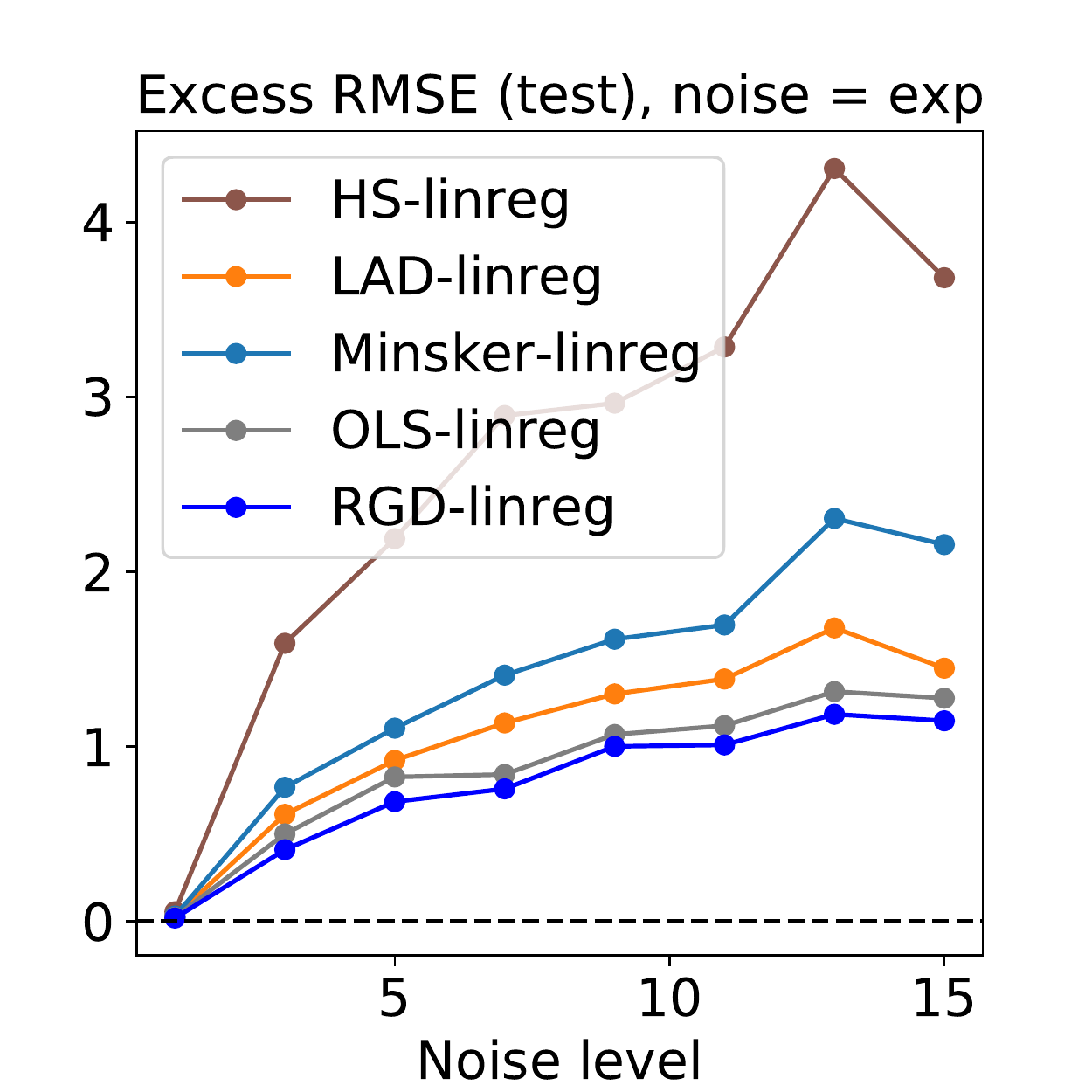}\\
\includegraphics[width=0.25\textwidth]{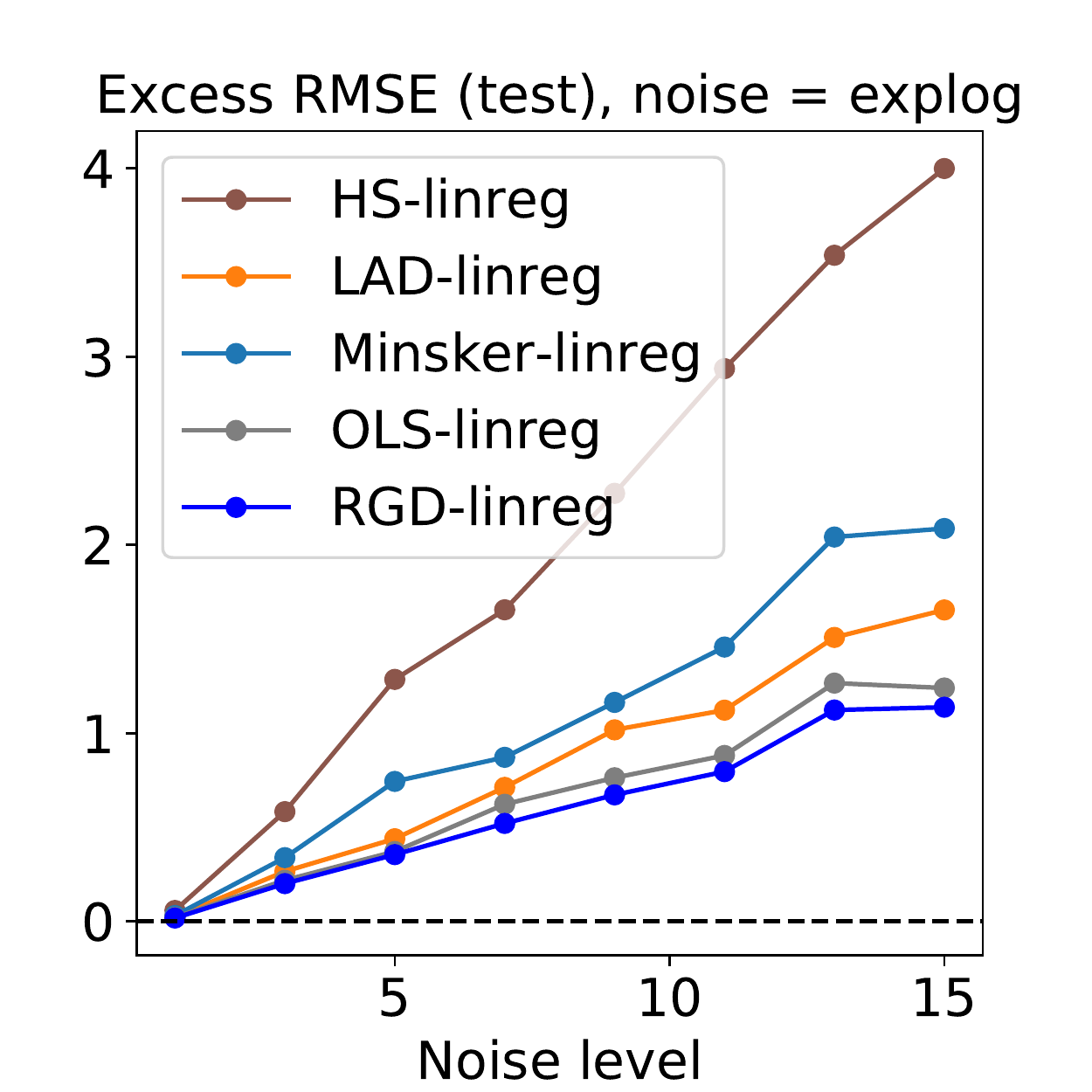}\,\includegraphics[width=0.25\textwidth]{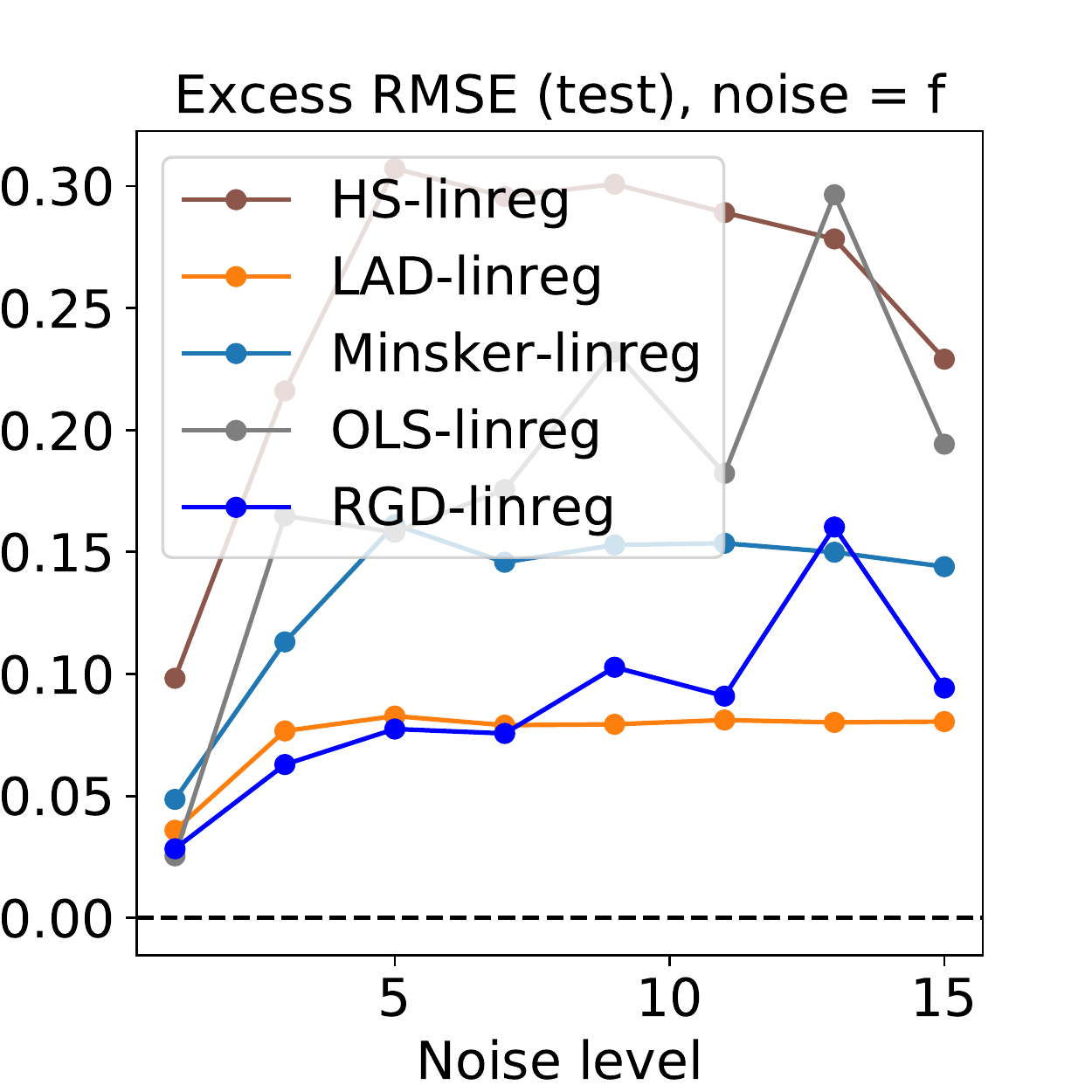}\,\includegraphics[width=0.25\textwidth]{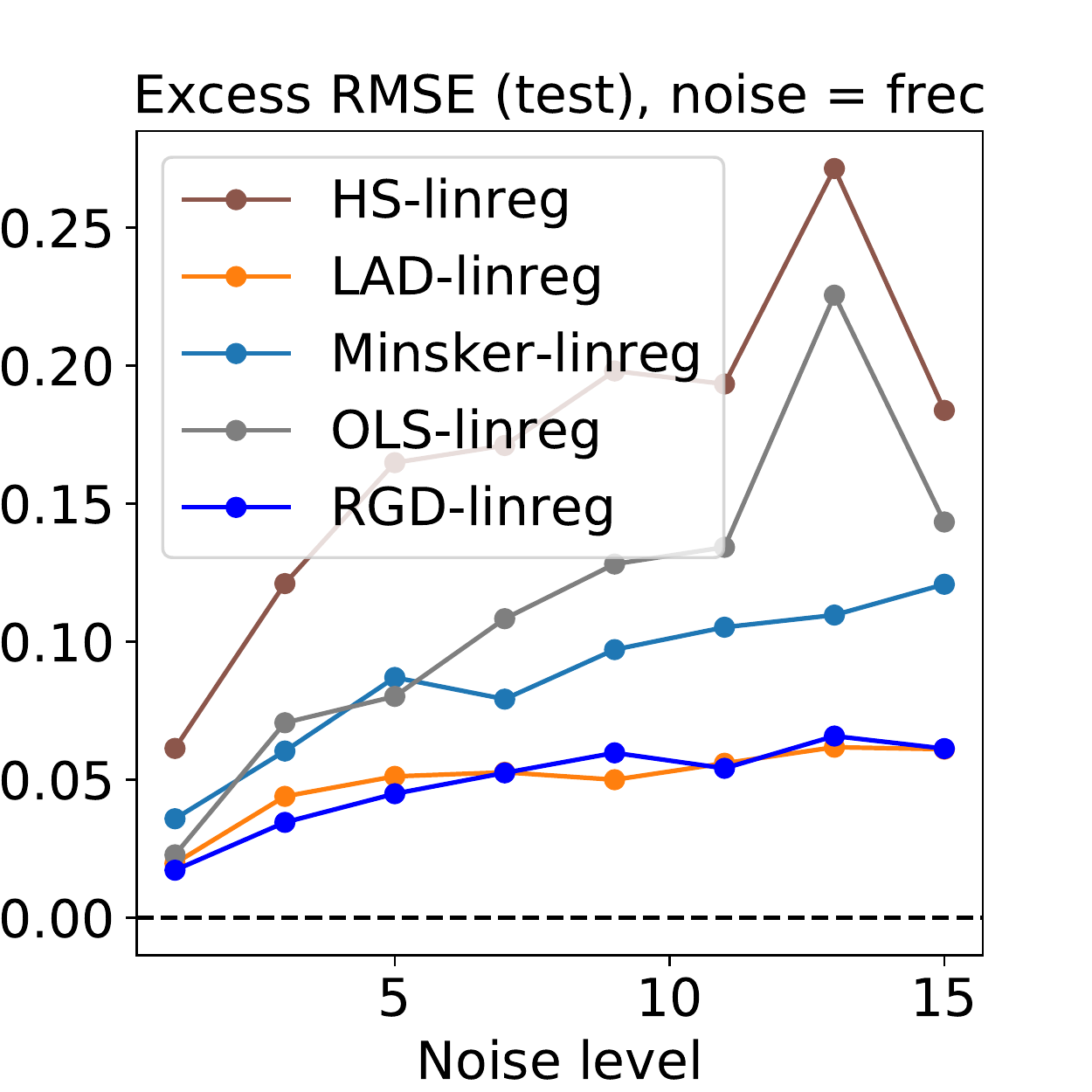}\,\includegraphics[width=0.25\textwidth]{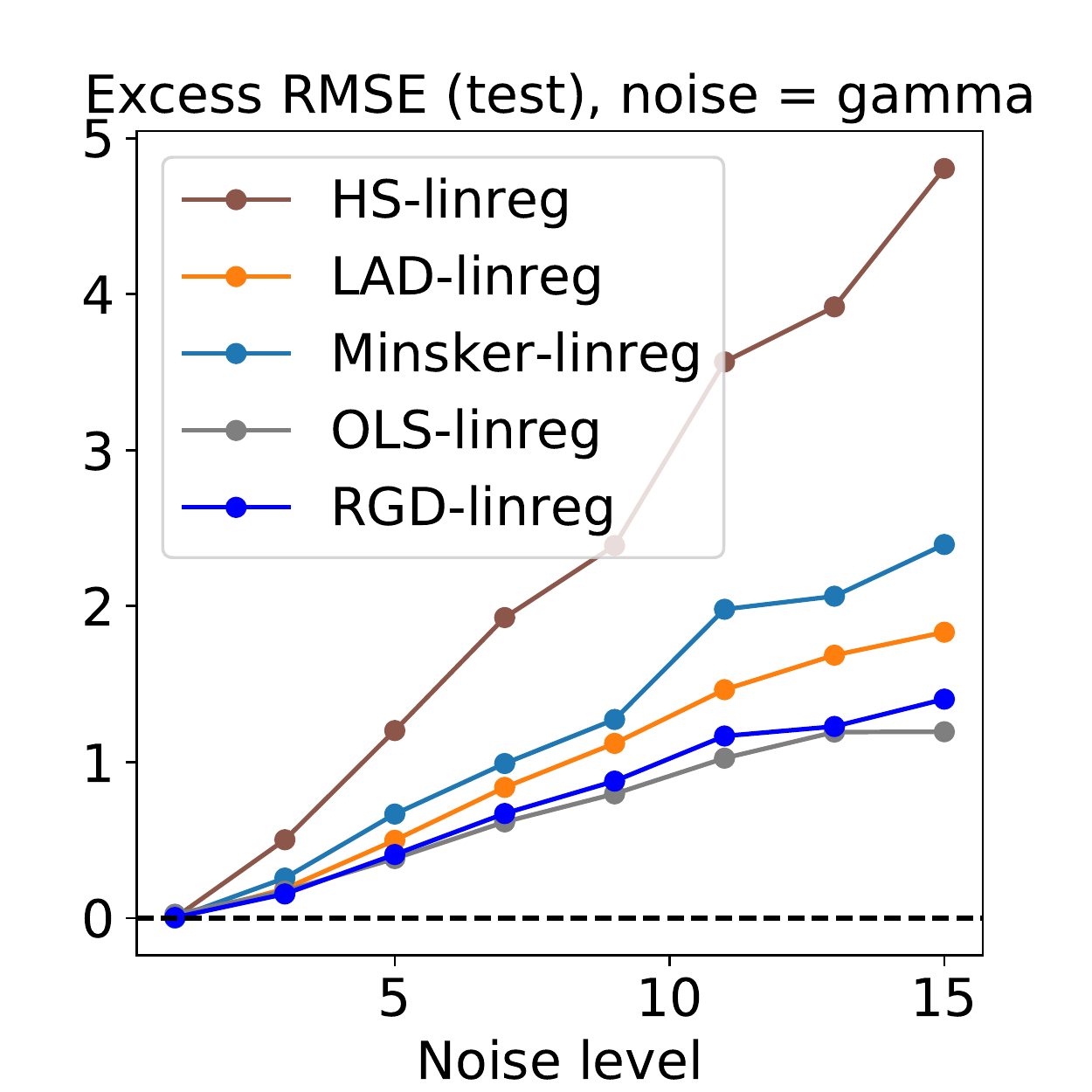}\\
\includegraphics[width=0.25\textwidth]{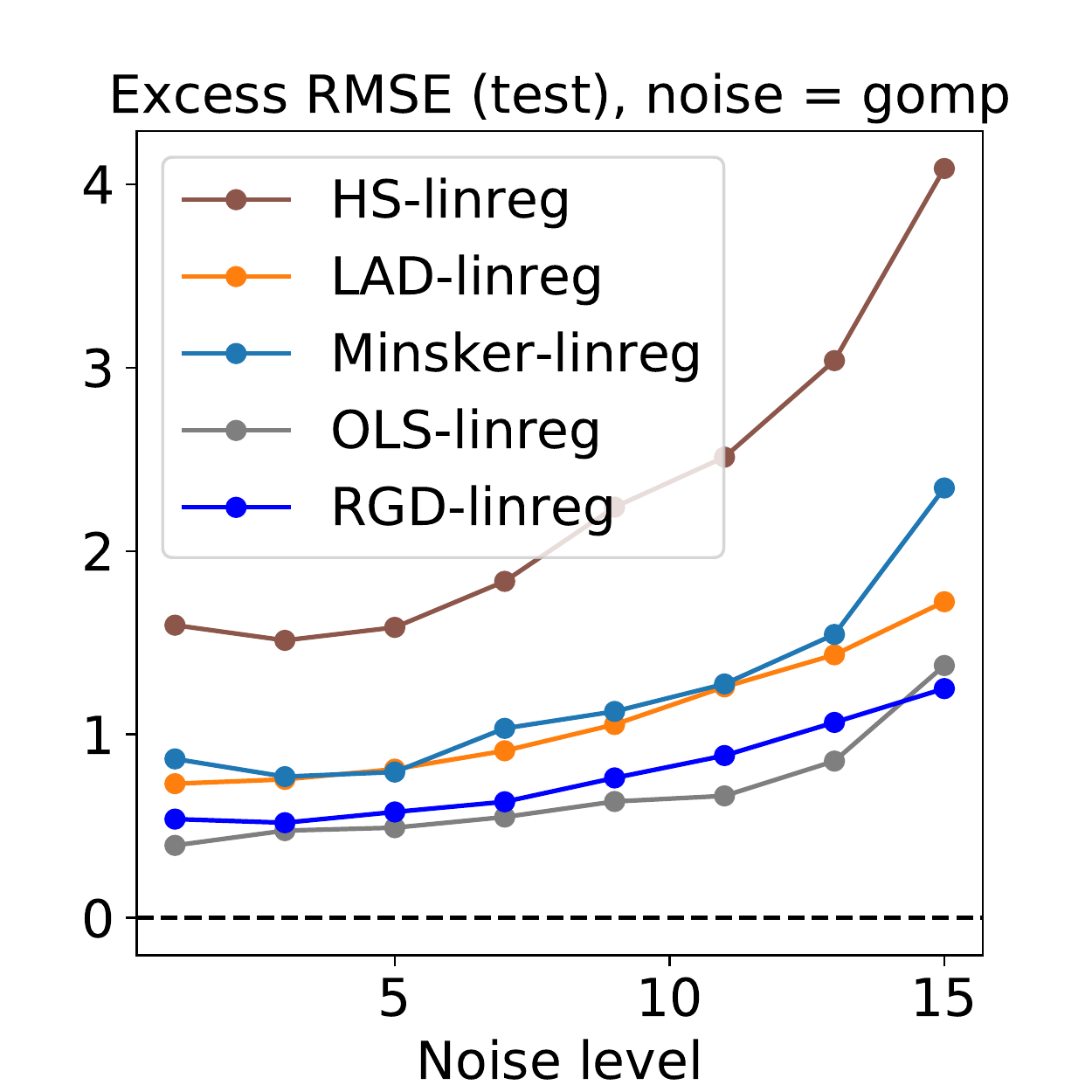}\,\includegraphics[width=0.25\textwidth]{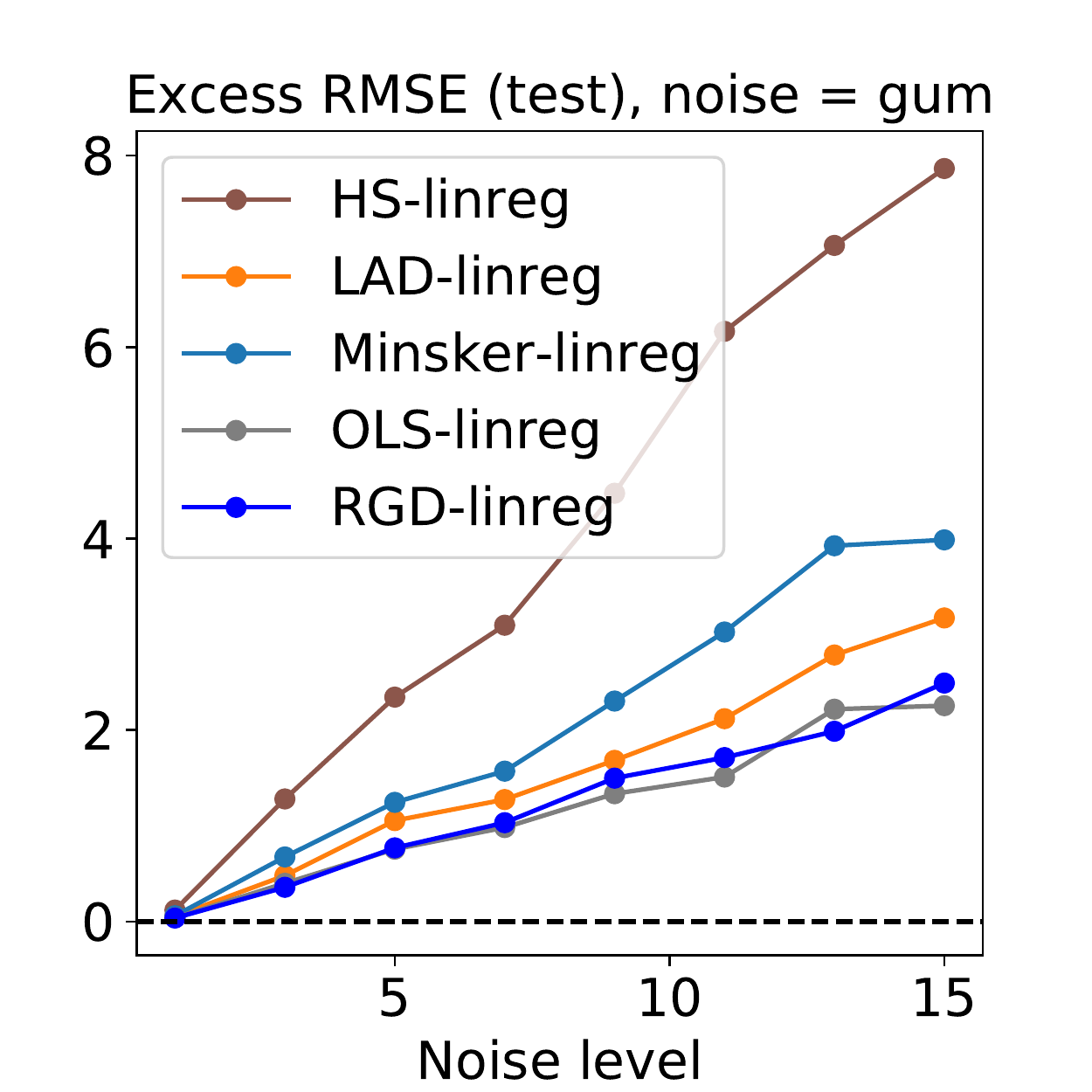}\,\includegraphics[width=0.25\textwidth]{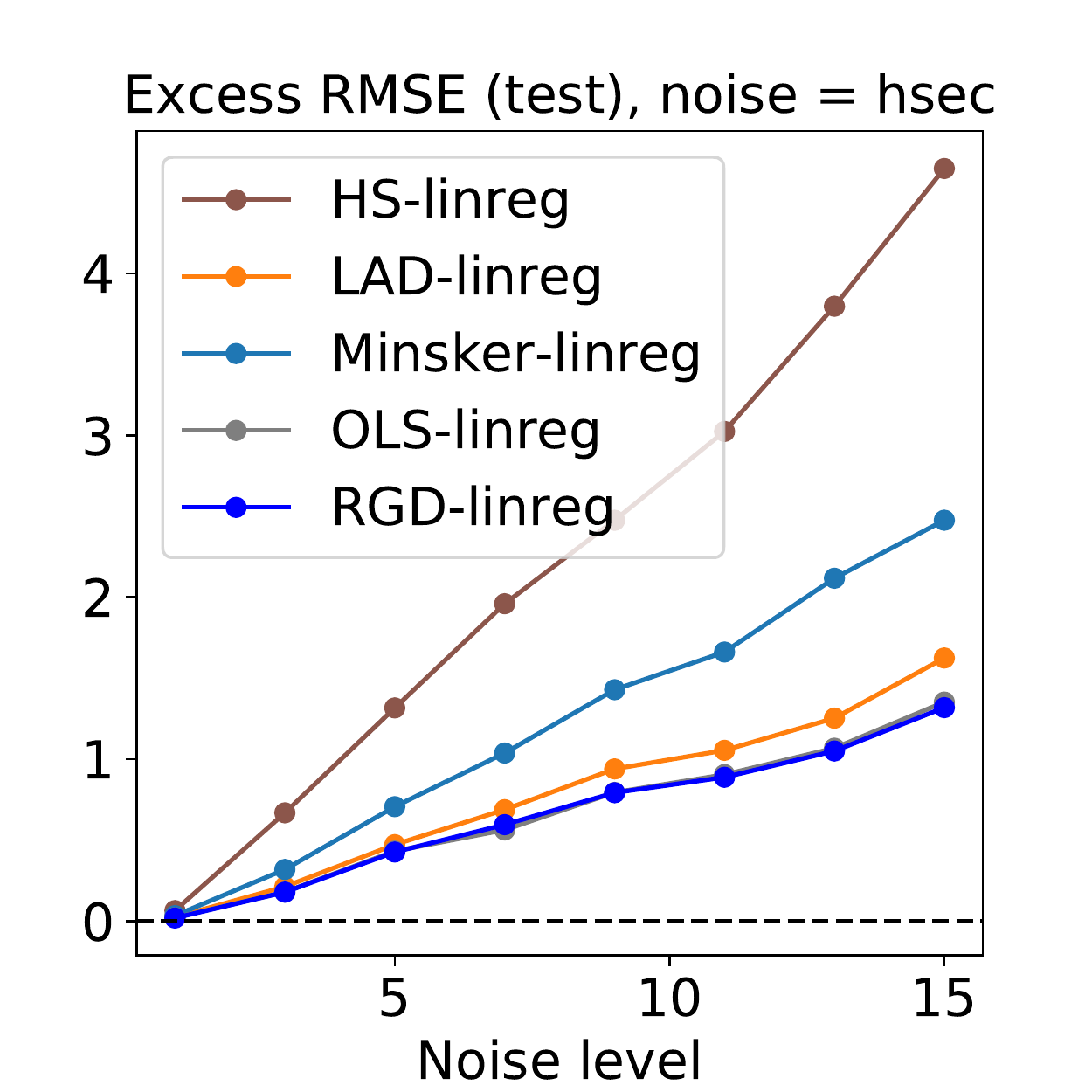}\,\includegraphics[width=0.25\textwidth]{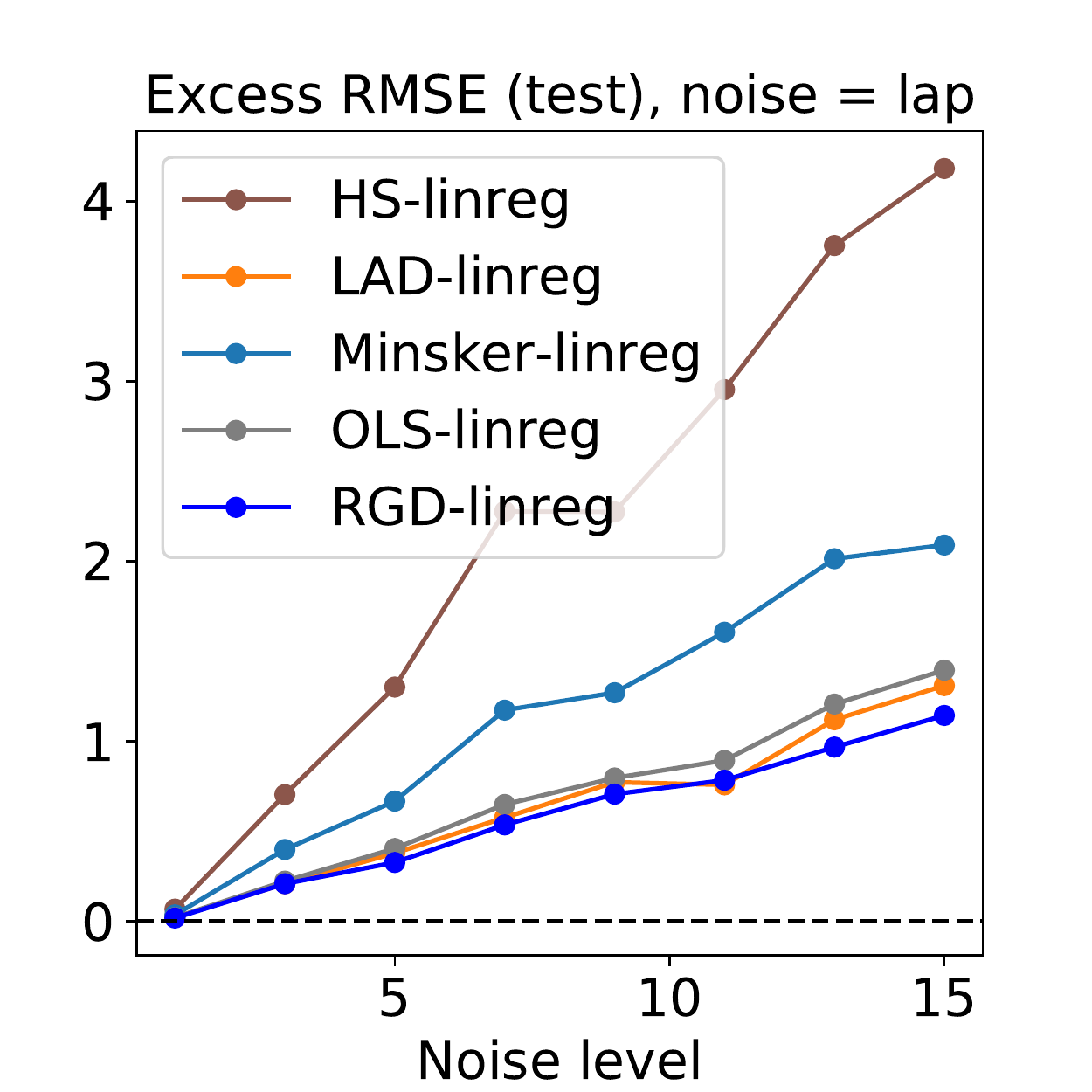}\\
\includegraphics[width=0.25\textwidth]{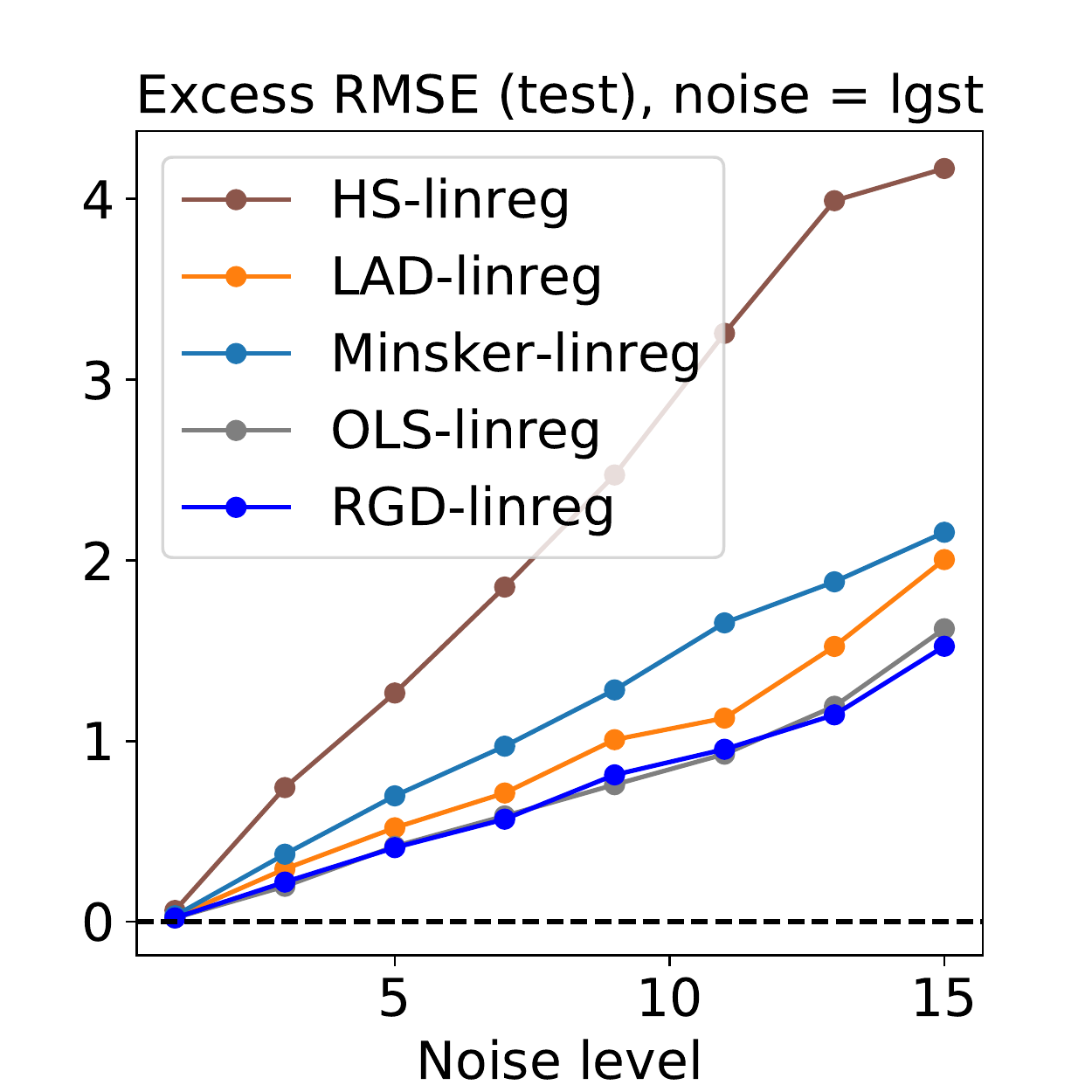}\,\includegraphics[width=0.25\textwidth]{linreg_overLvl_risk_llog}\,\includegraphics[width=0.25\textwidth]{linreg_overLvl_risk_lnorm}\,\includegraphics[width=0.25\textwidth]{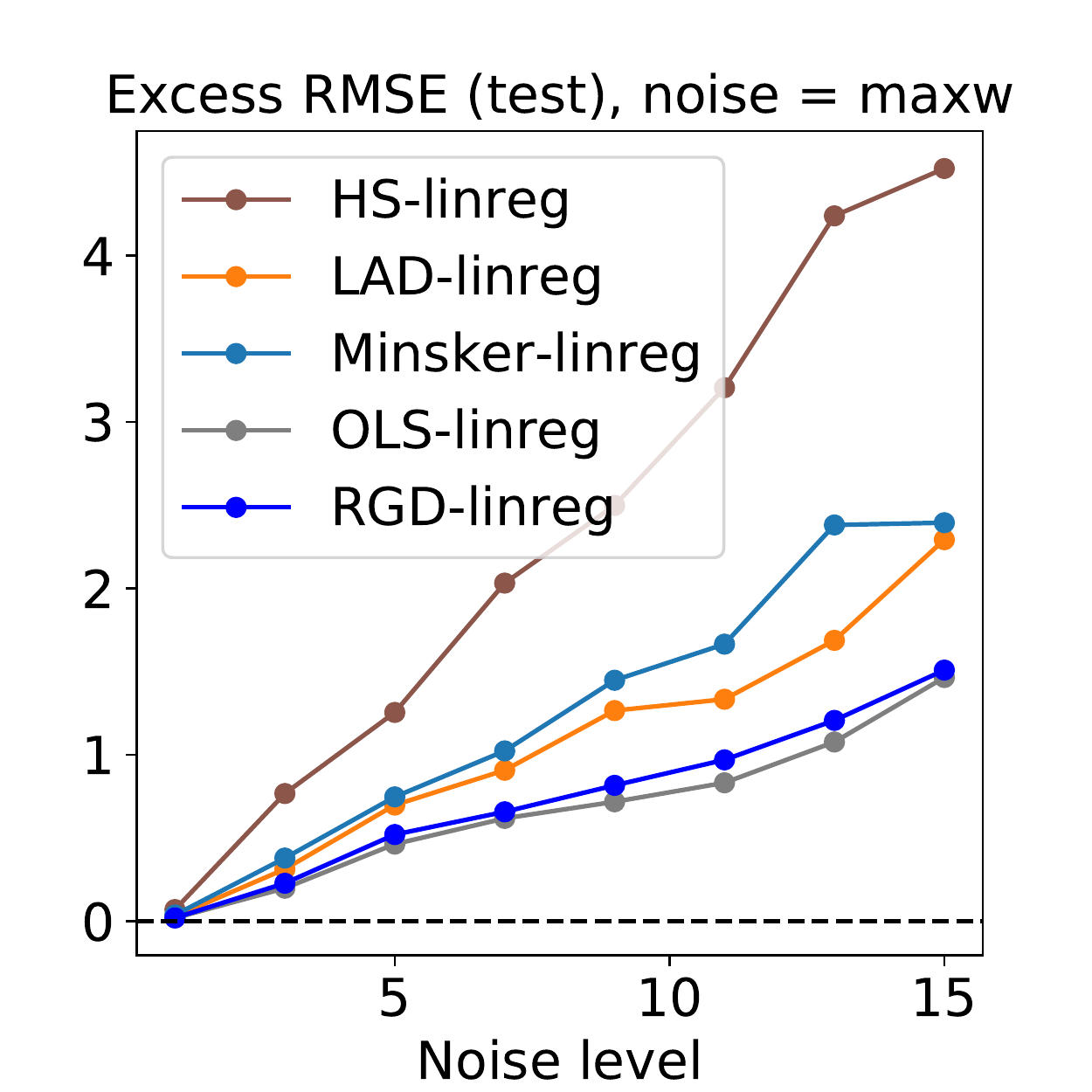}\\
\includegraphics[width=0.25\textwidth]{linreg_overLvl_risk_norm}\,\includegraphics[width=0.25\textwidth]{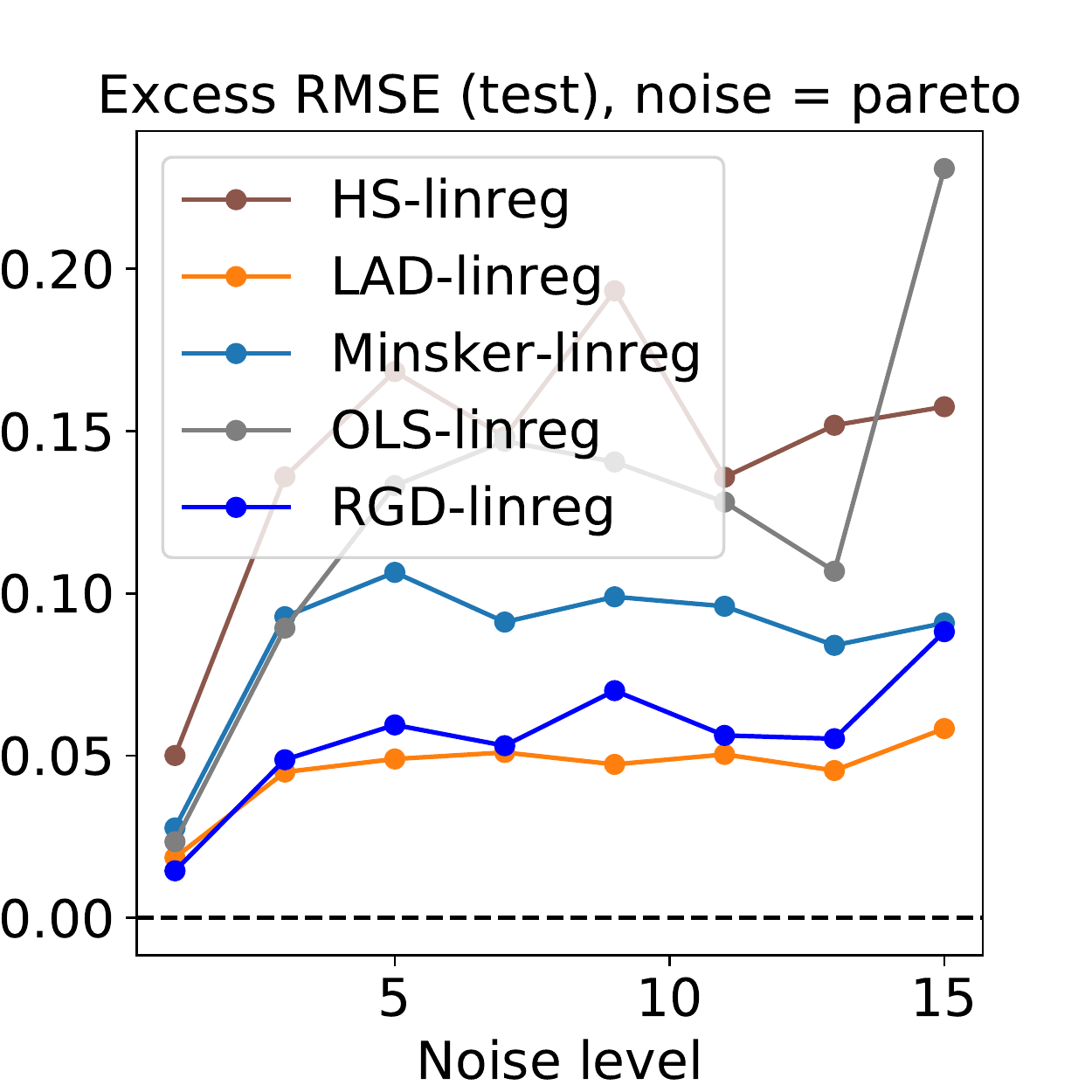}\,\includegraphics[width=0.25\textwidth]{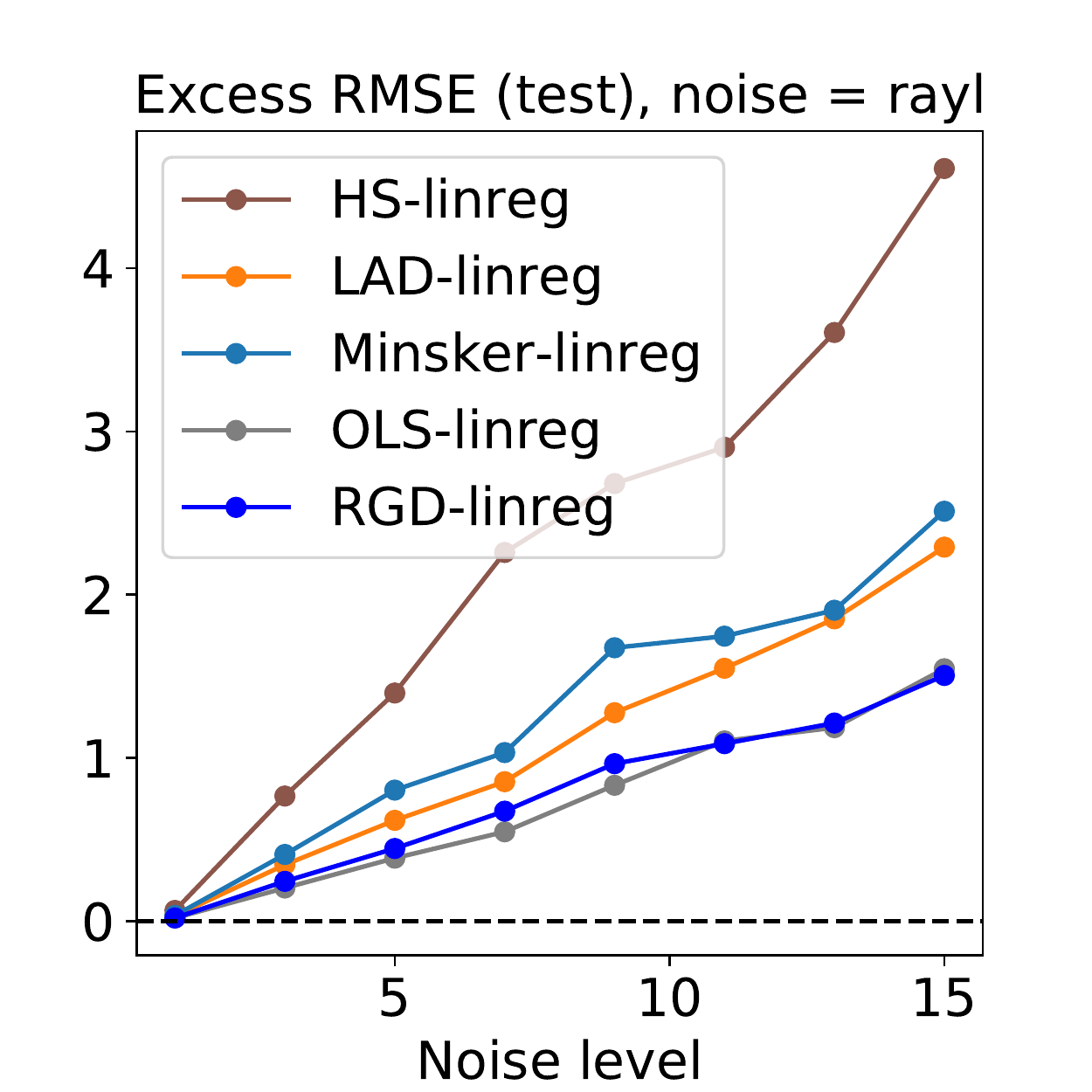}\,\includegraphics[width=0.25\textwidth]{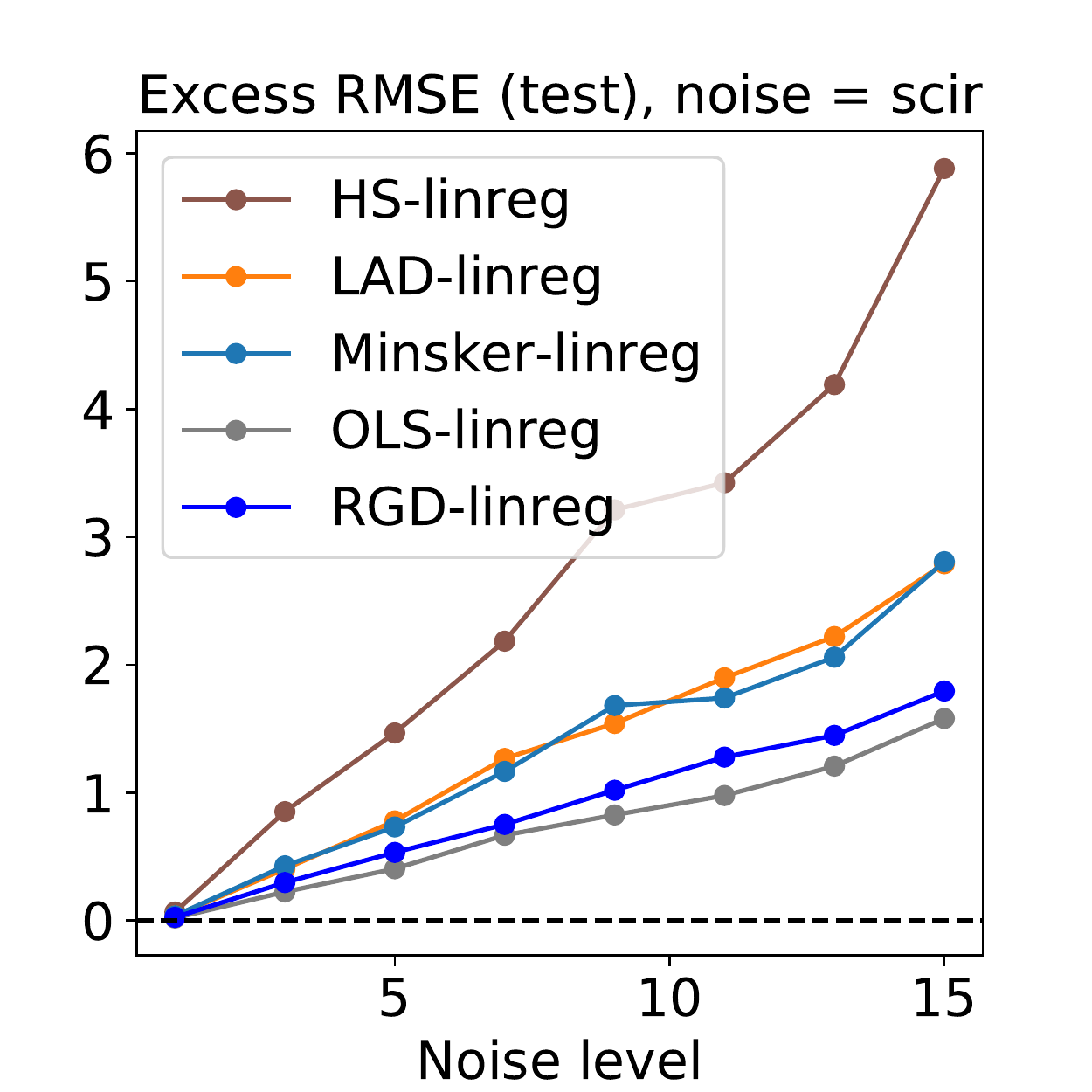}
\caption{Prediction error over noise levels, for $n=30, d=5$. Each plot corresponds to a distinct noise distribution.}
\label{fig:overLvl_all_distros_1}
\end{figure}

\clearpage

\begin{figure}[t]
\centering
\includegraphics[width=0.25\textwidth]{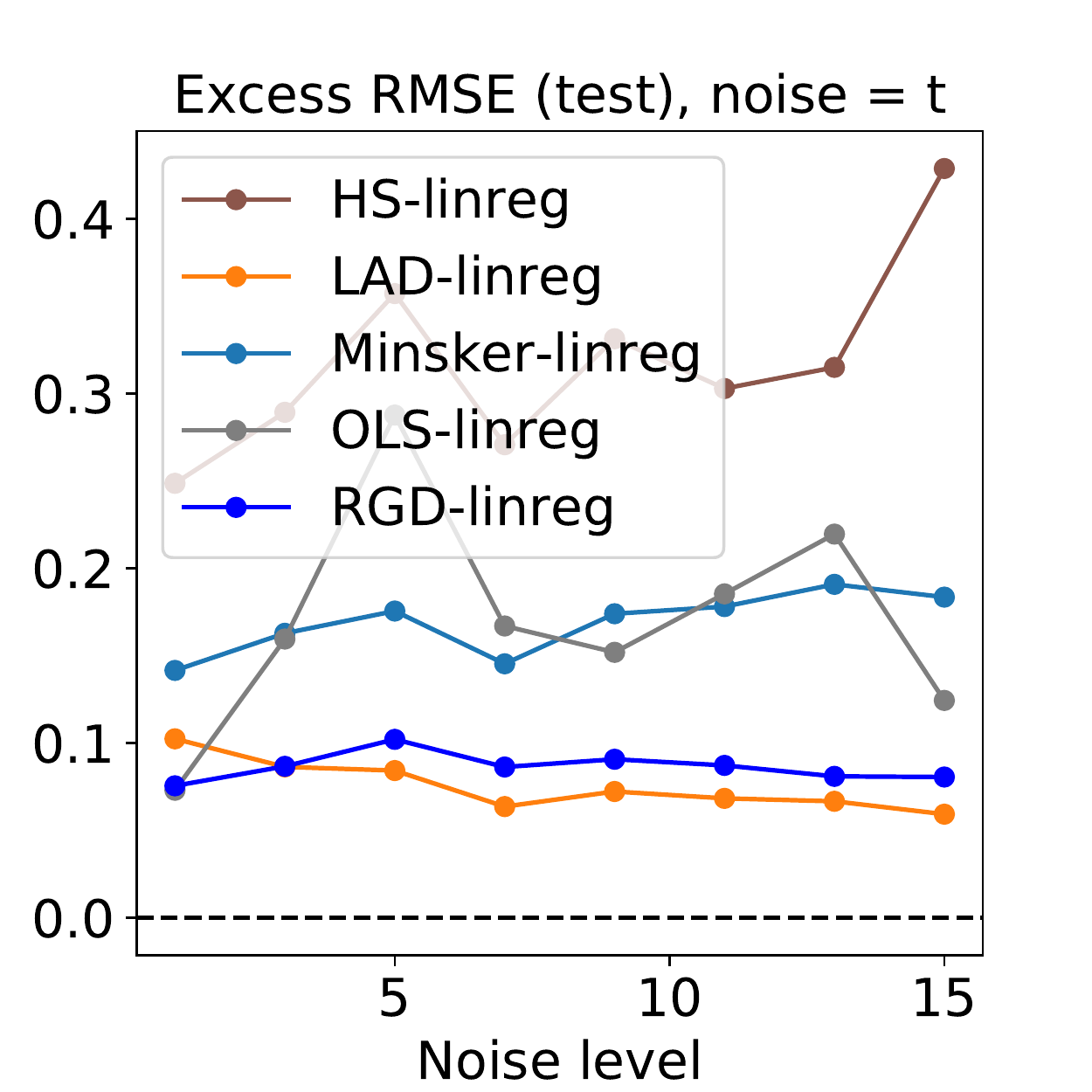}\,\includegraphics[width=0.25\textwidth]{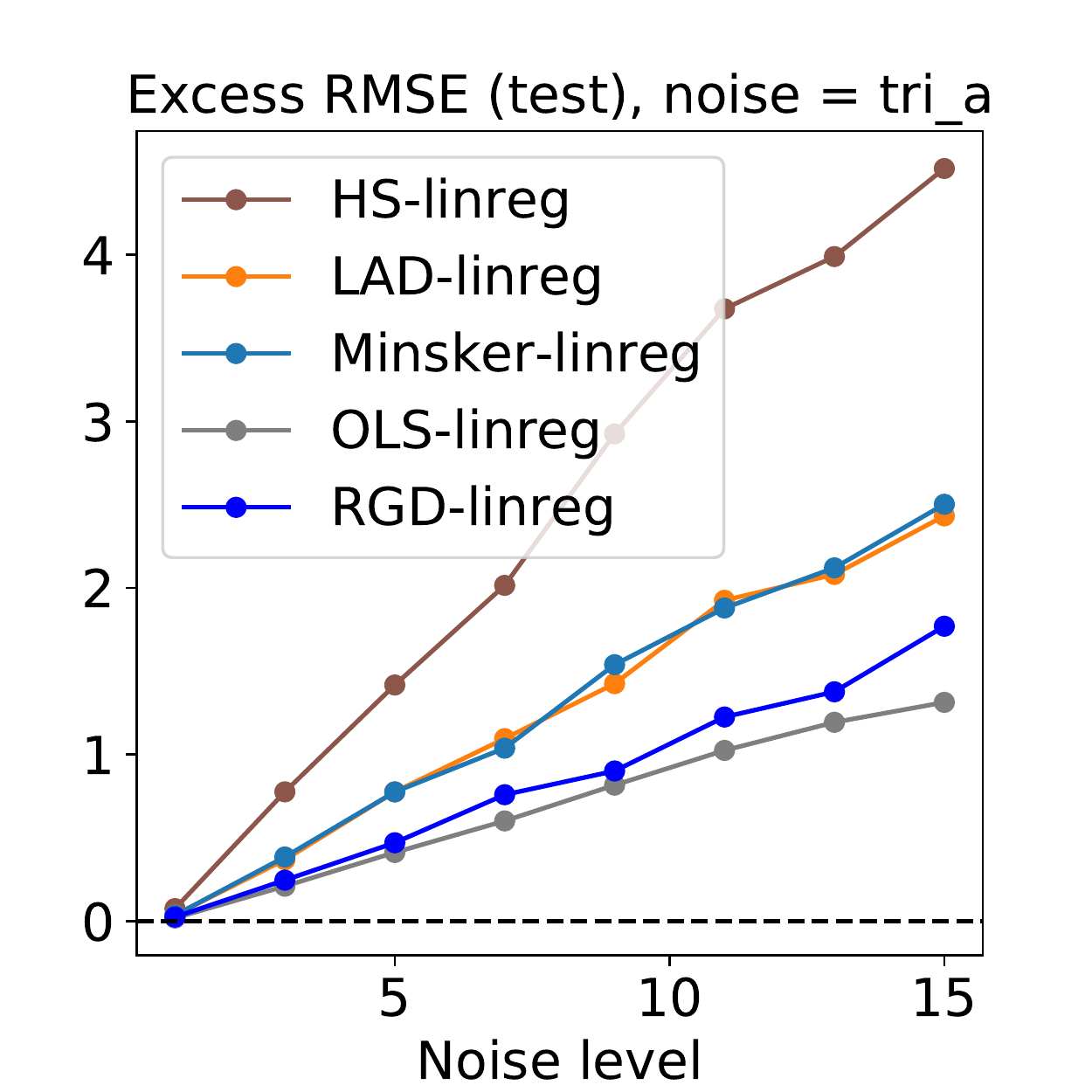}\,\includegraphics[width=0.25\textwidth]{linreg_overLvl_risk_tri_s}\,\includegraphics[width=0.25\textwidth]{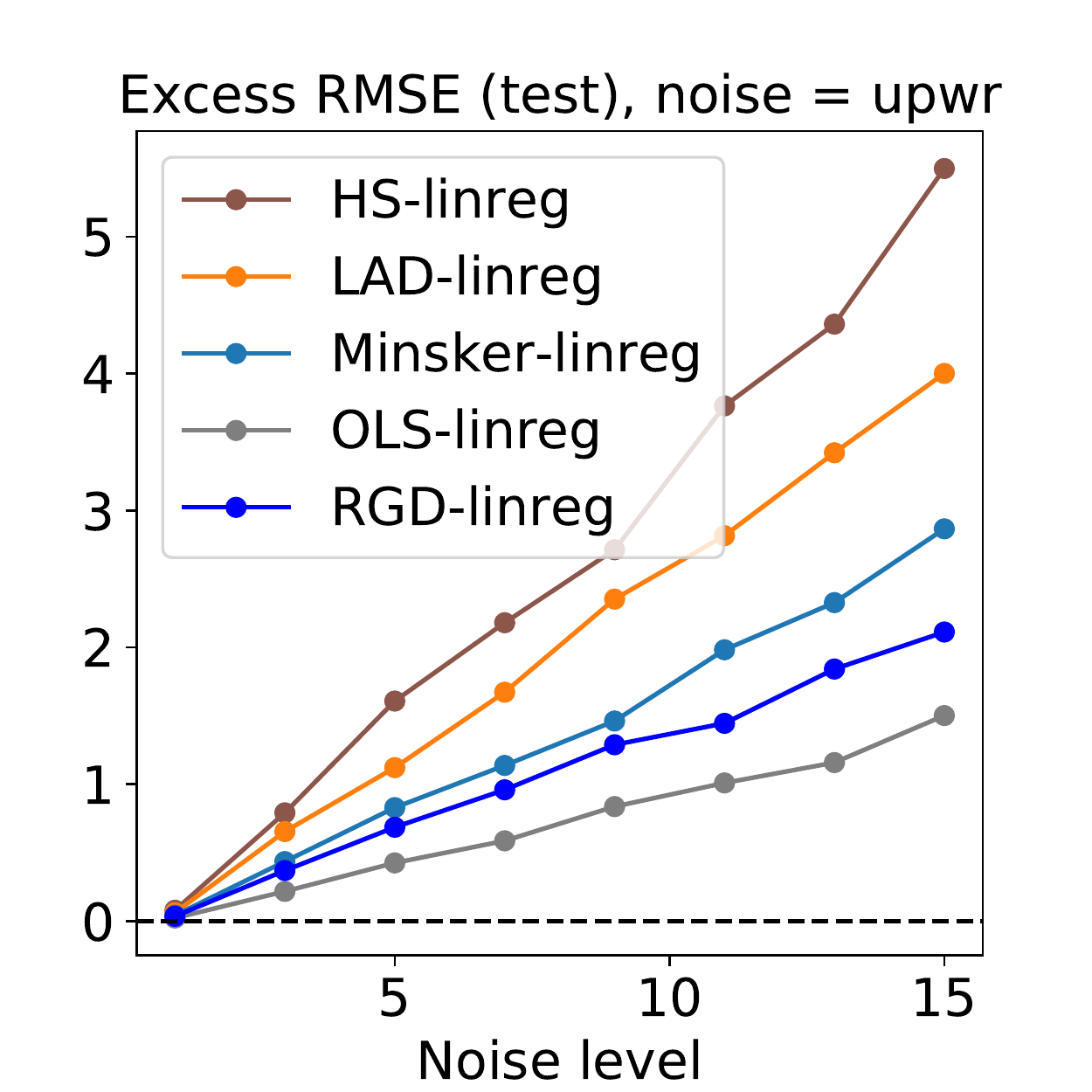}\\
\includegraphics[width=0.25\textwidth]{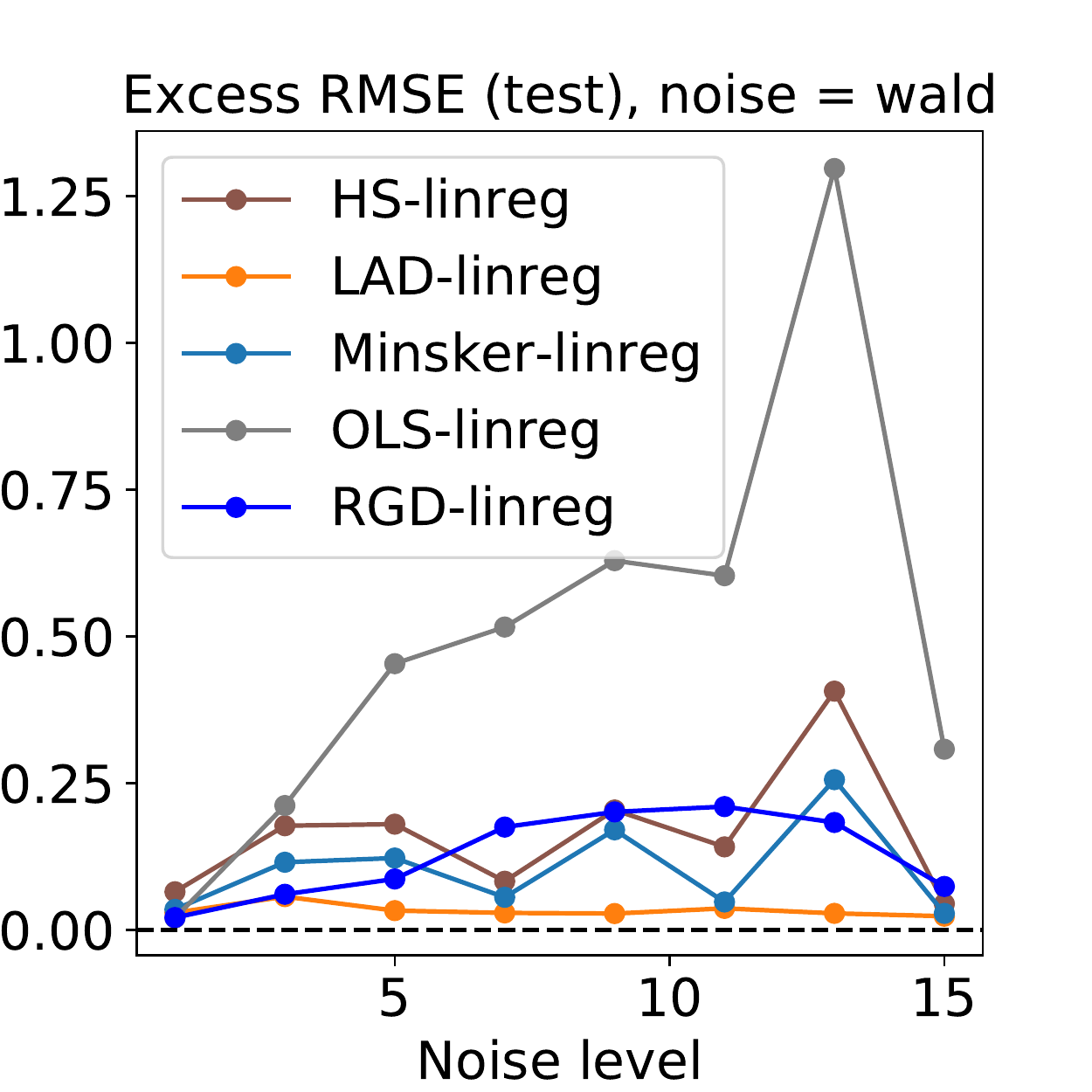}\,\includegraphics[width=0.25\textwidth]{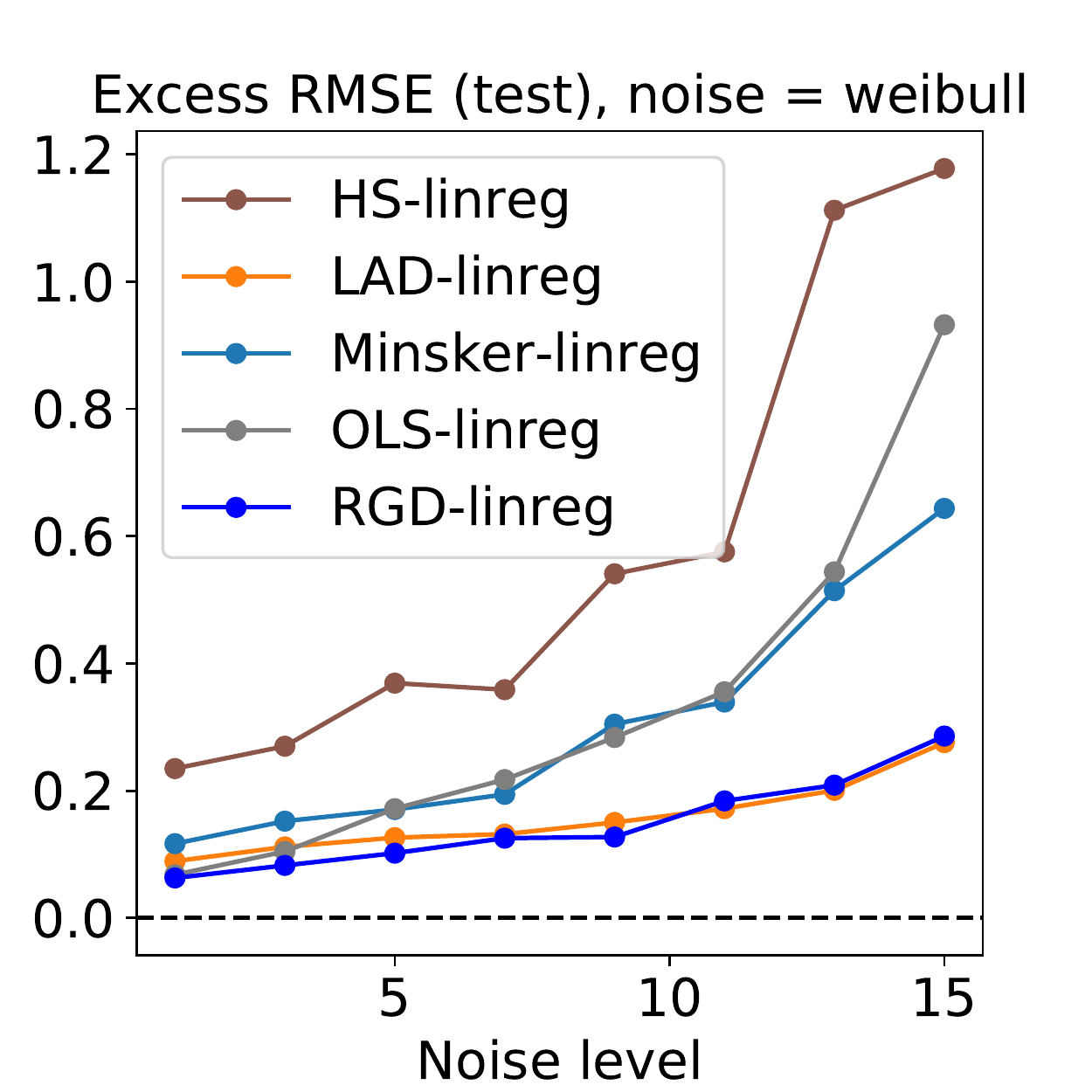}
\caption{Prediction error over noise levels, for $n=30, d=5$. Each plot corresponds to a distinct noise distribution.}
\label{fig:overLvl_all_distros_2}
\end{figure}

\clearpage

\begin{figure}[t]
\centering
\includegraphics[width=0.25\textwidth]{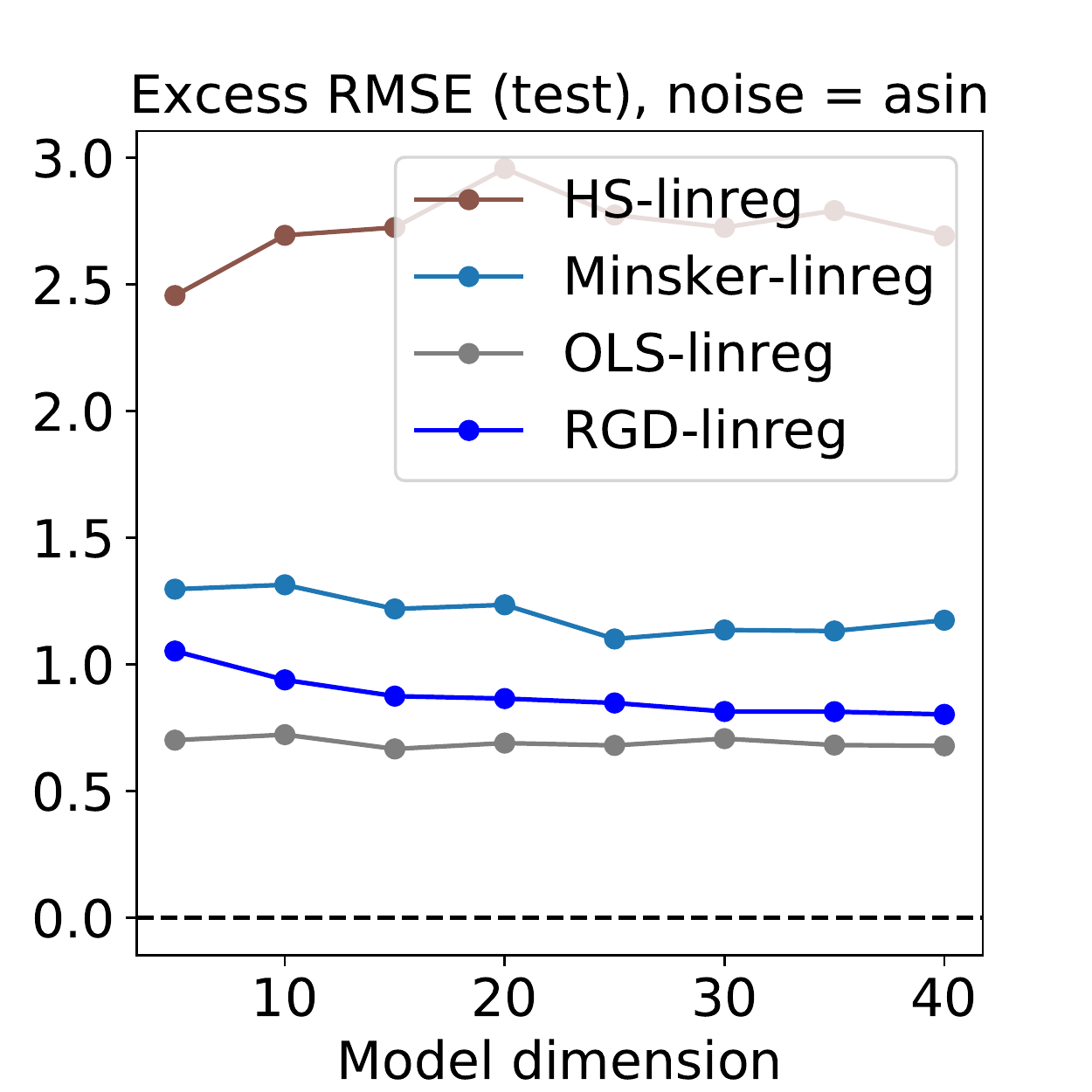}\,\includegraphics[width=0.25\textwidth]{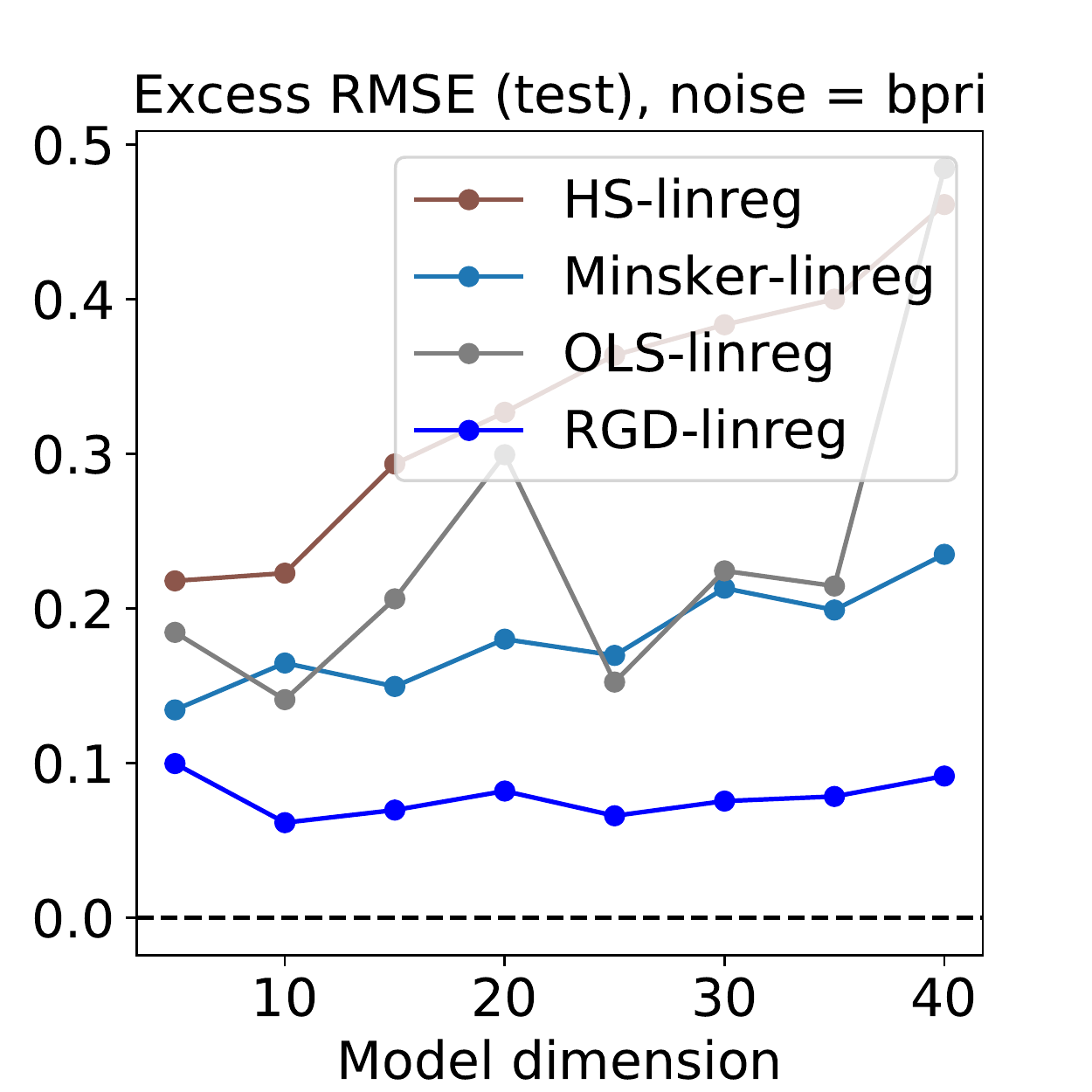}\,\includegraphics[width=0.25\textwidth]{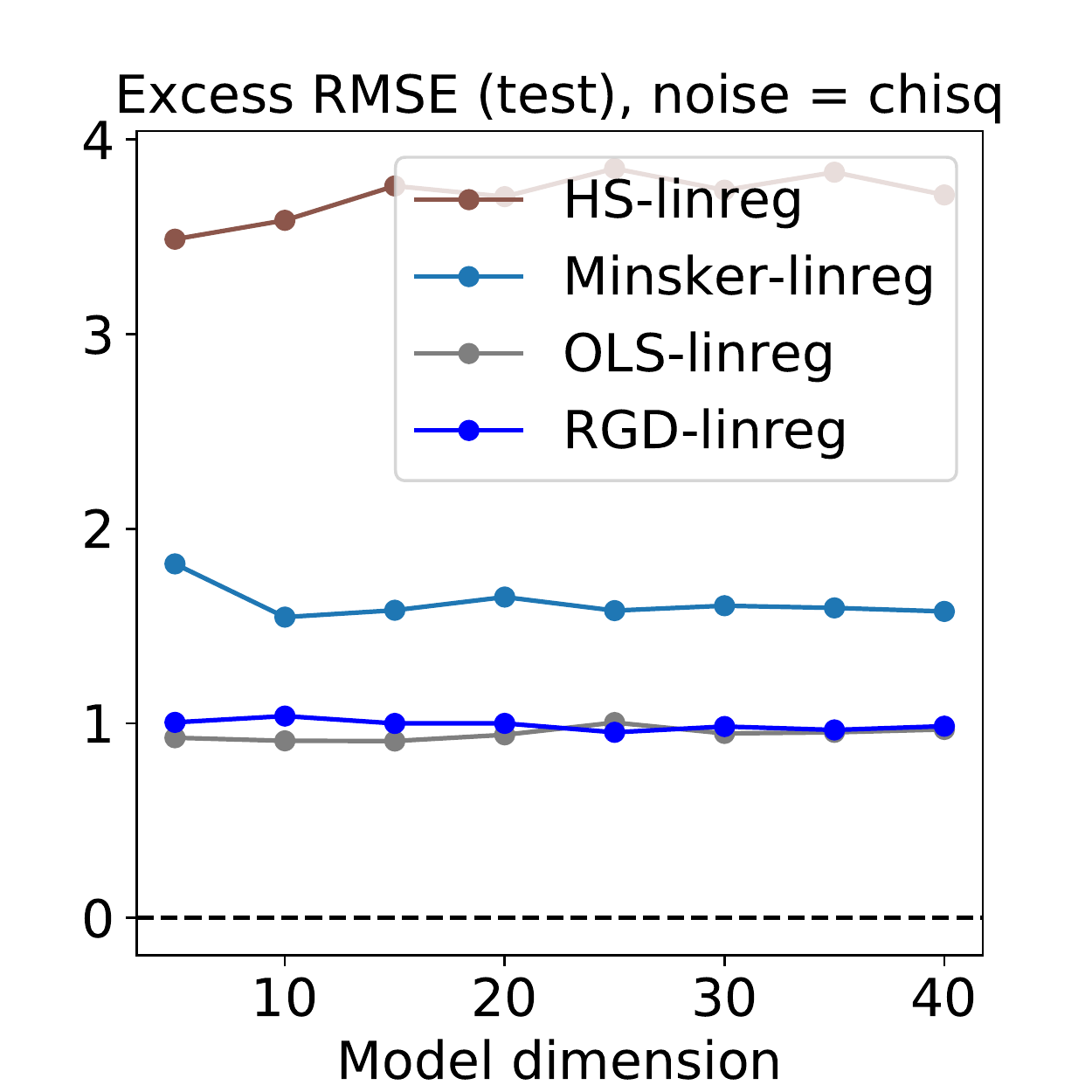}\,\includegraphics[width=0.25\textwidth]{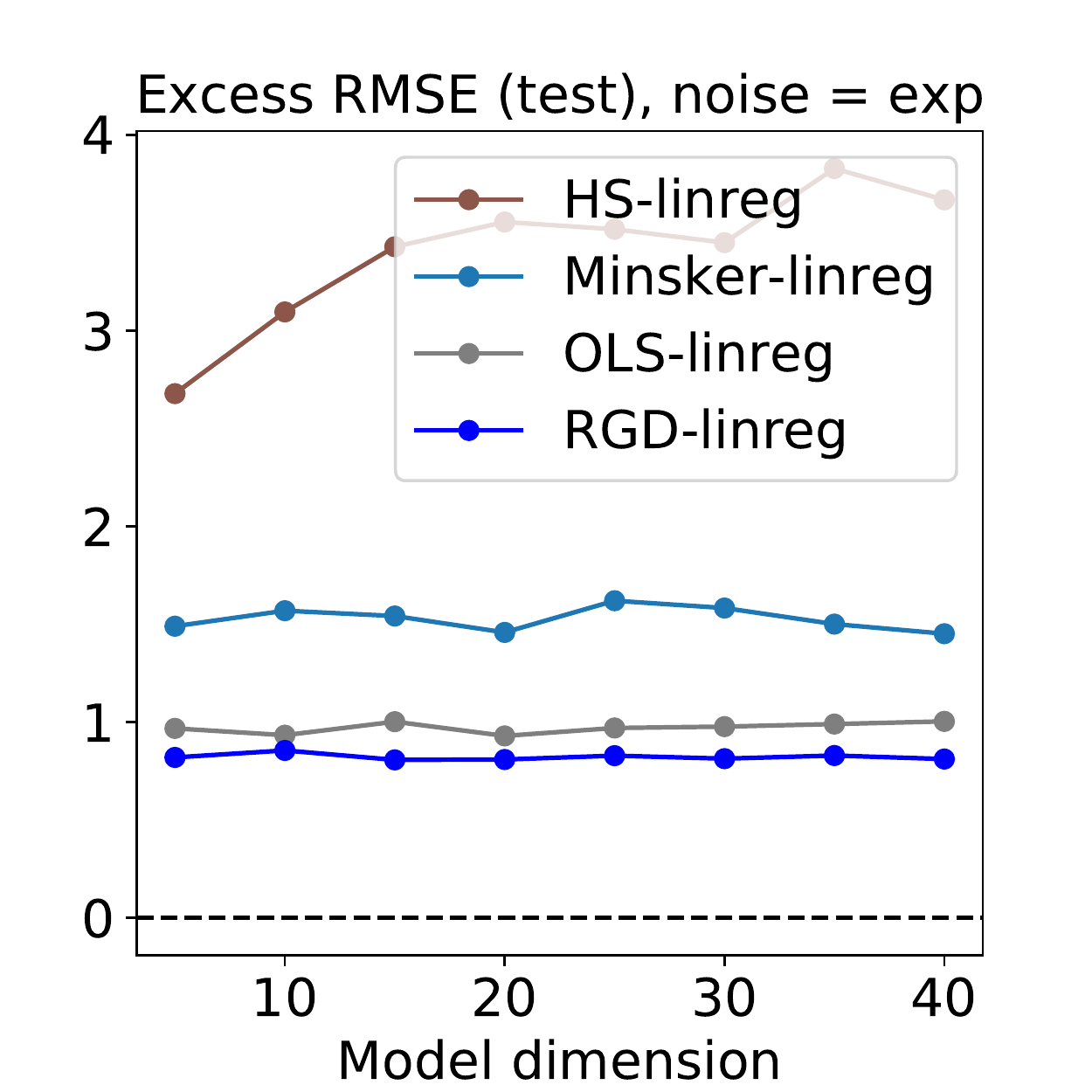}\\
\includegraphics[width=0.25\textwidth]{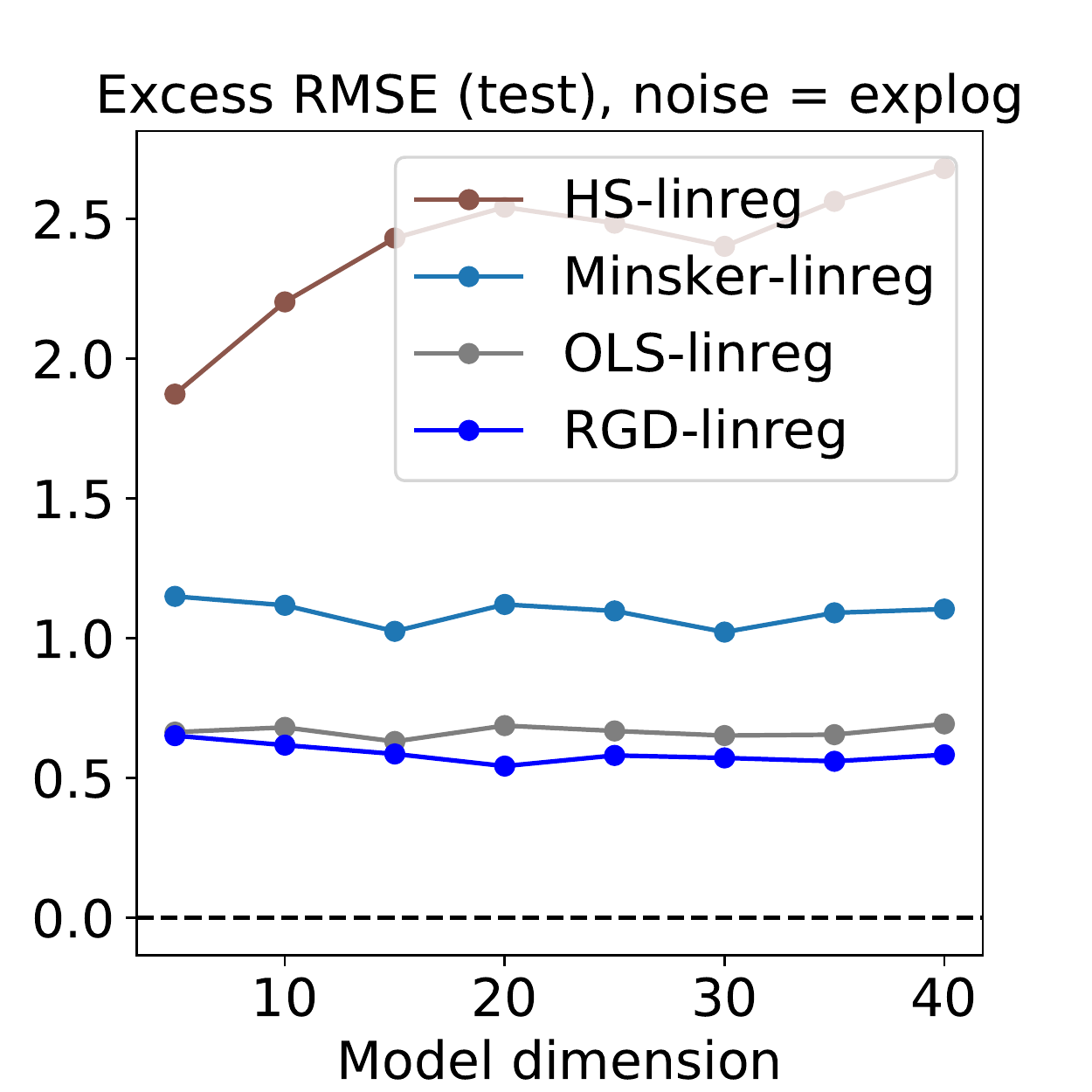}\,\includegraphics[width=0.25\textwidth]{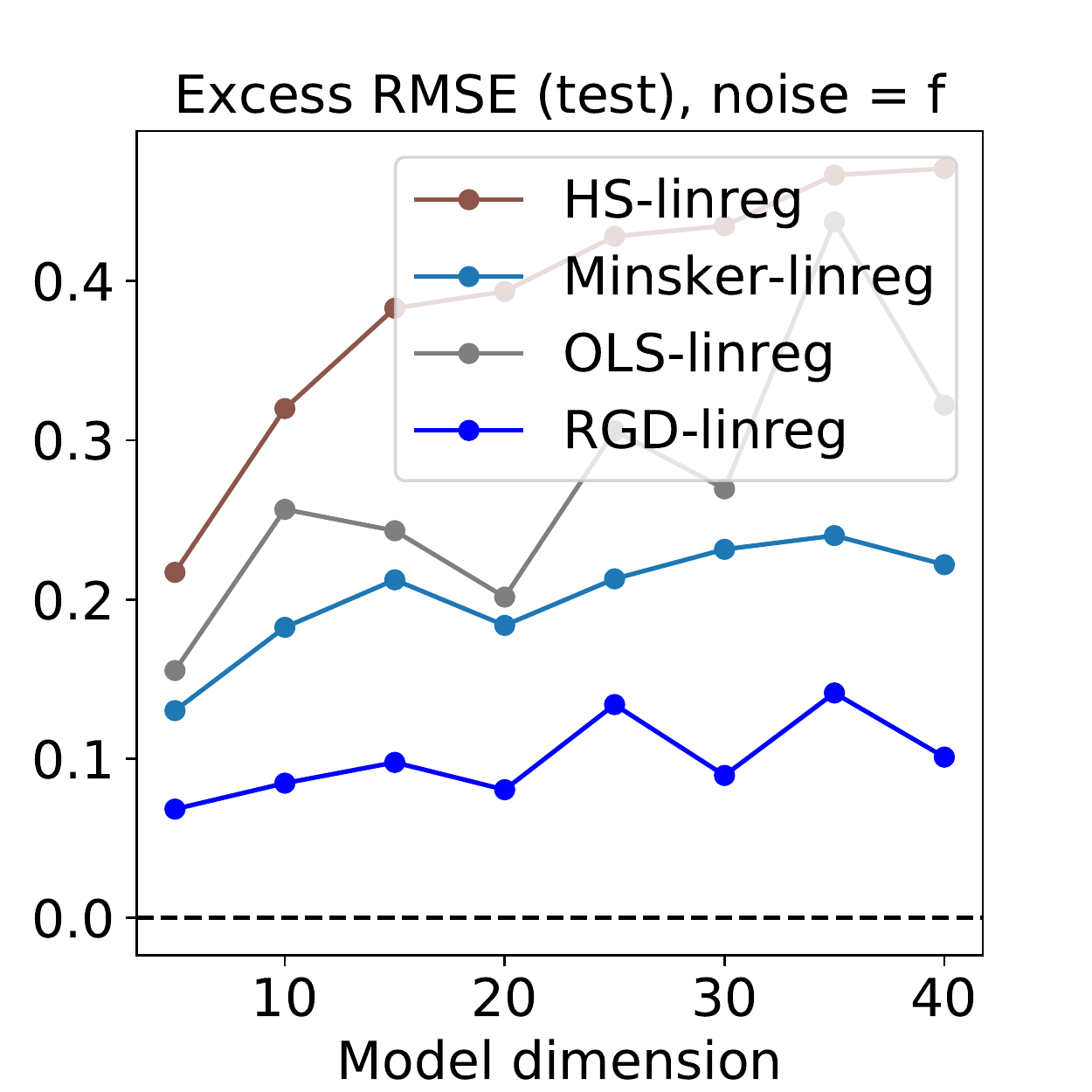}\,\includegraphics[width=0.25\textwidth]{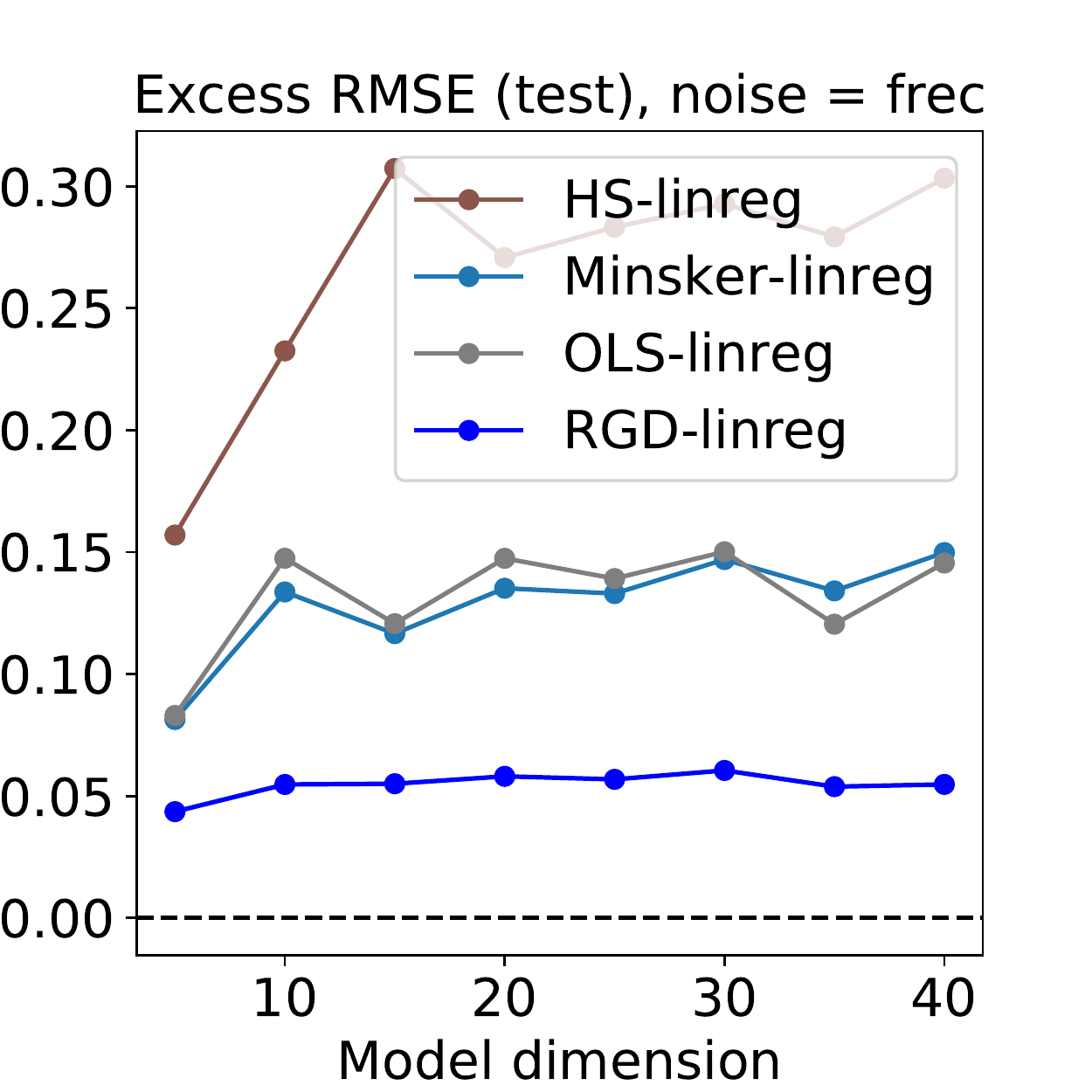}\,\includegraphics[width=0.25\textwidth]{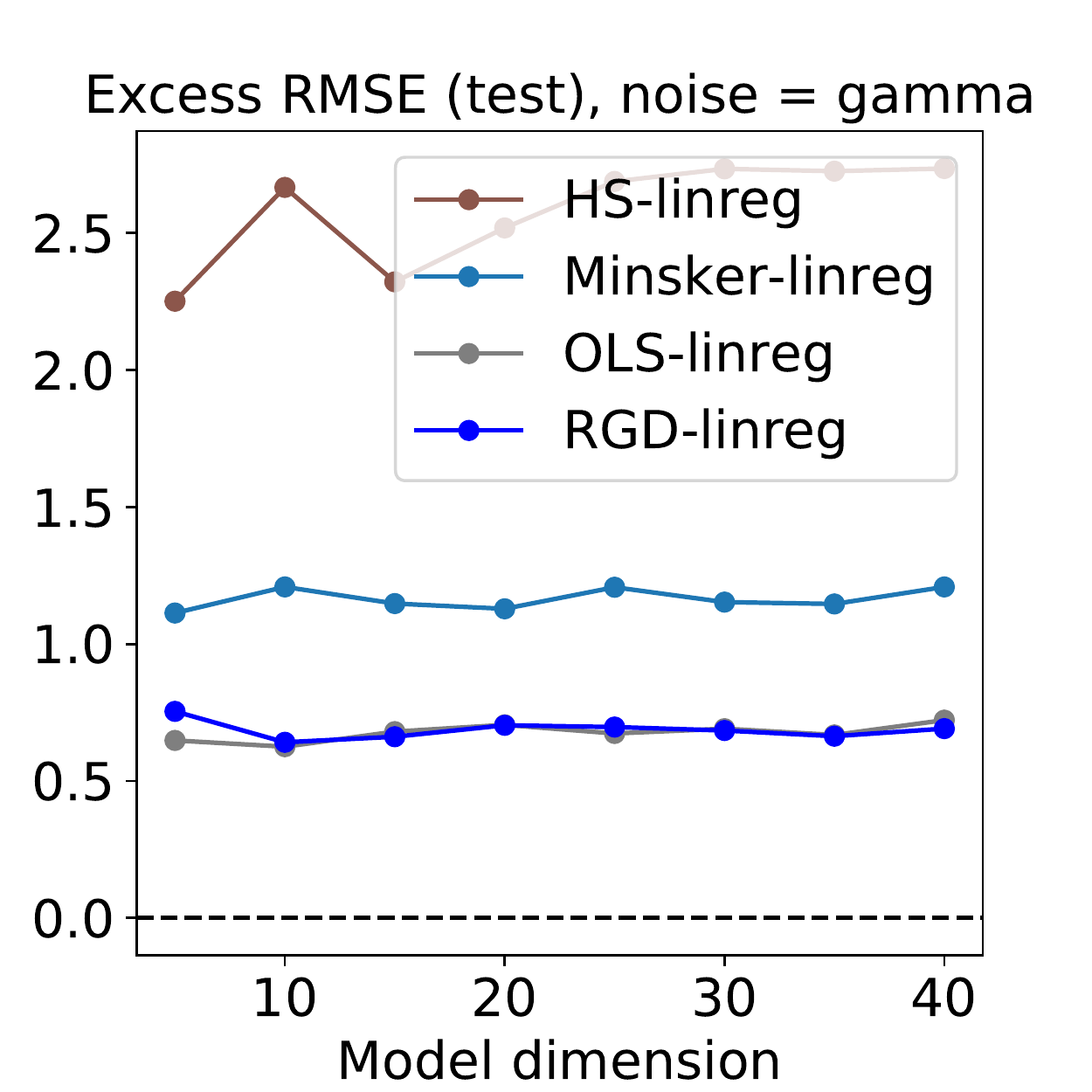}\\
\includegraphics[width=0.25\textwidth]{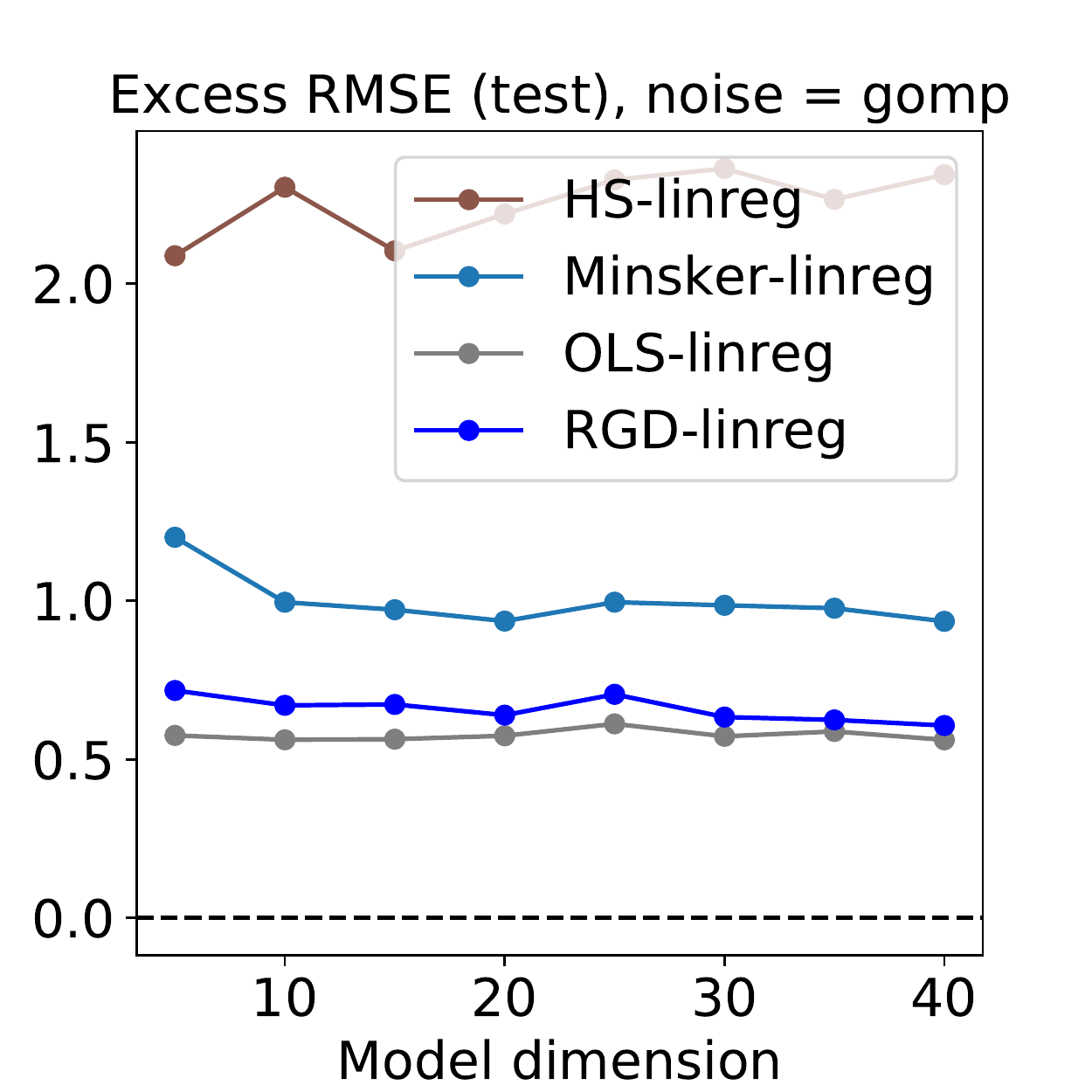}\,\includegraphics[width=0.25\textwidth]{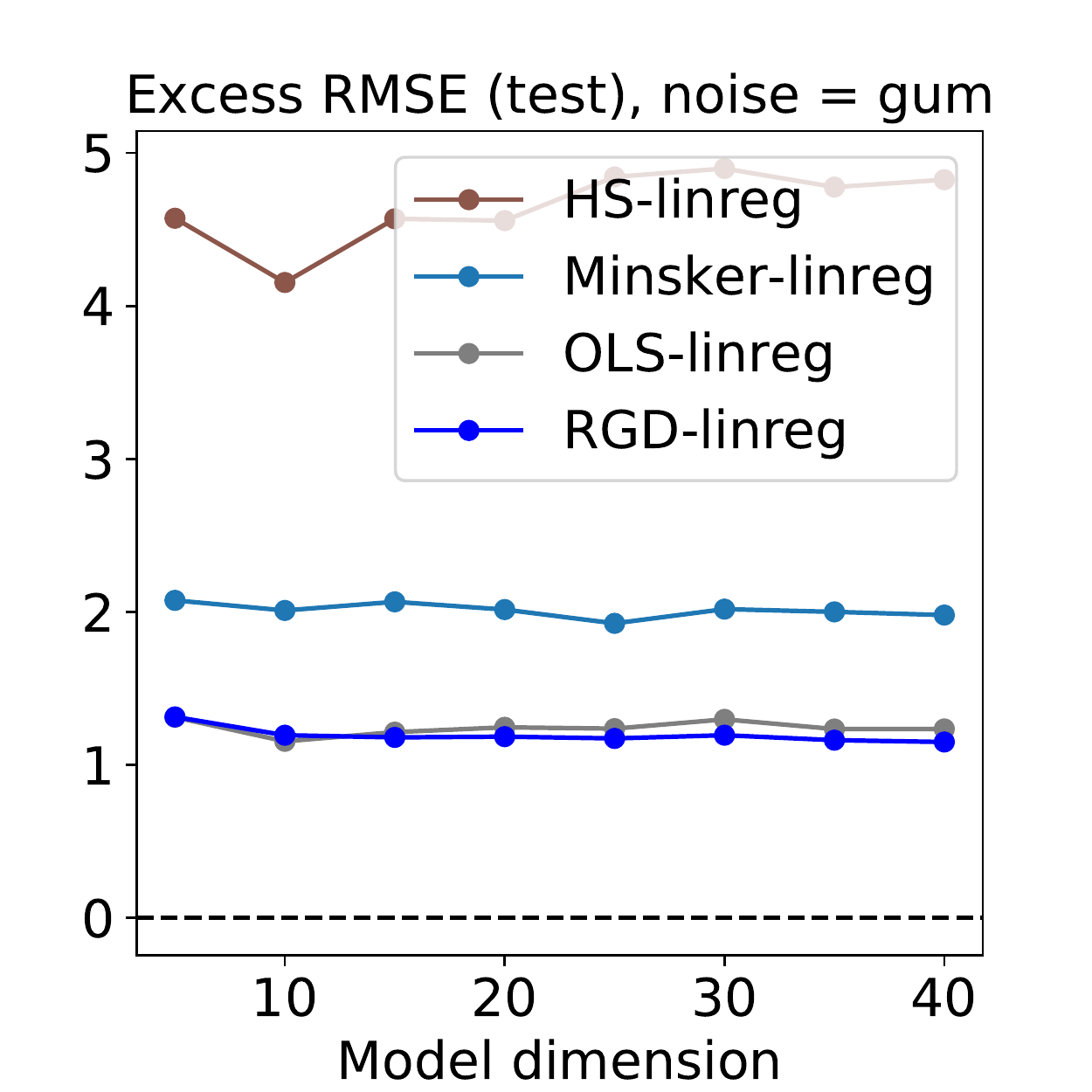}\,\includegraphics[width=0.25\textwidth]{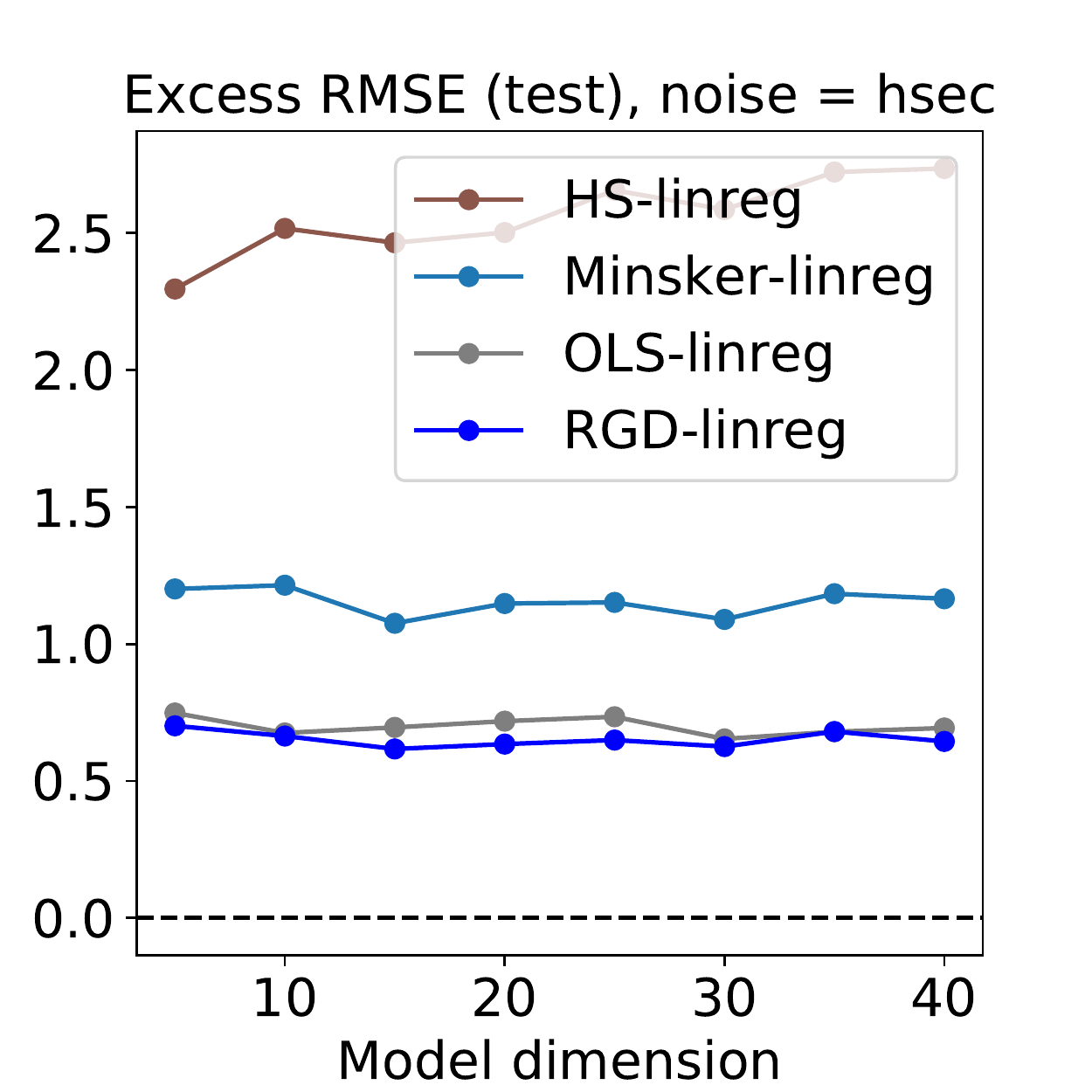}\,\includegraphics[width=0.25\textwidth]{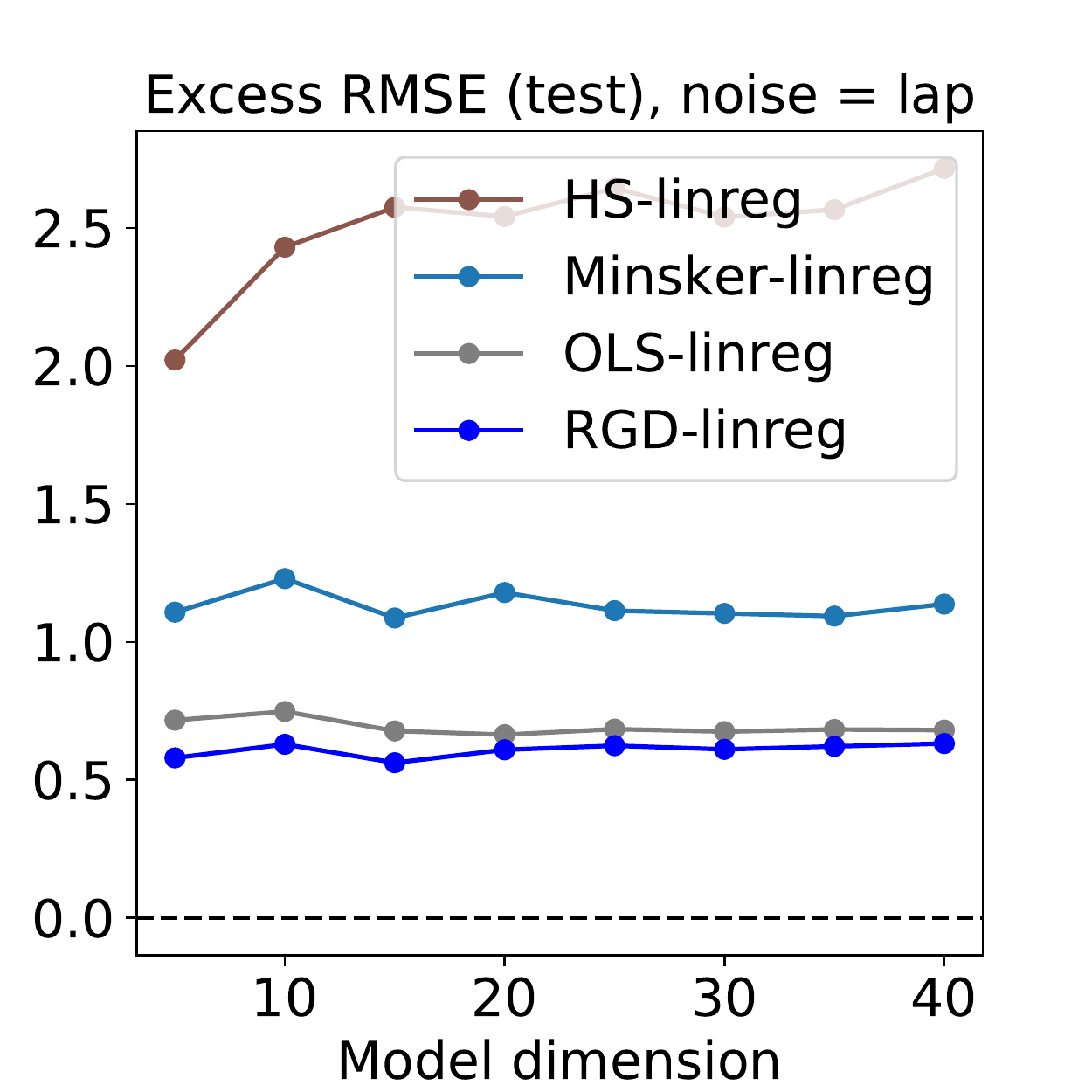}\\
\includegraphics[width=0.25\textwidth]{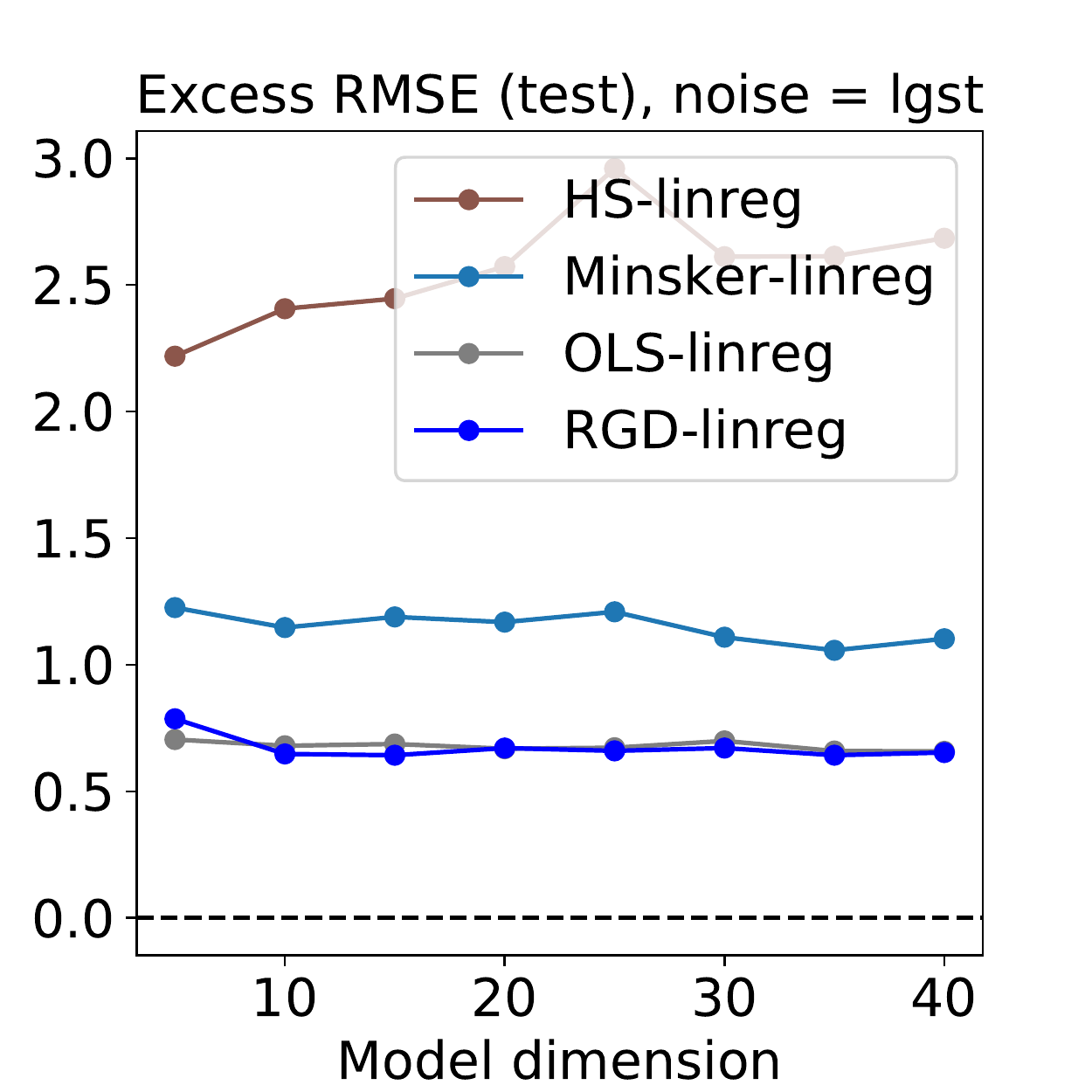}\,\includegraphics[width=0.25\textwidth]{linreg_overDim_risk_llog}\,\includegraphics[width=0.25\textwidth]{linreg_overDim_risk_lnorm}\,\includegraphics[width=0.25\textwidth]{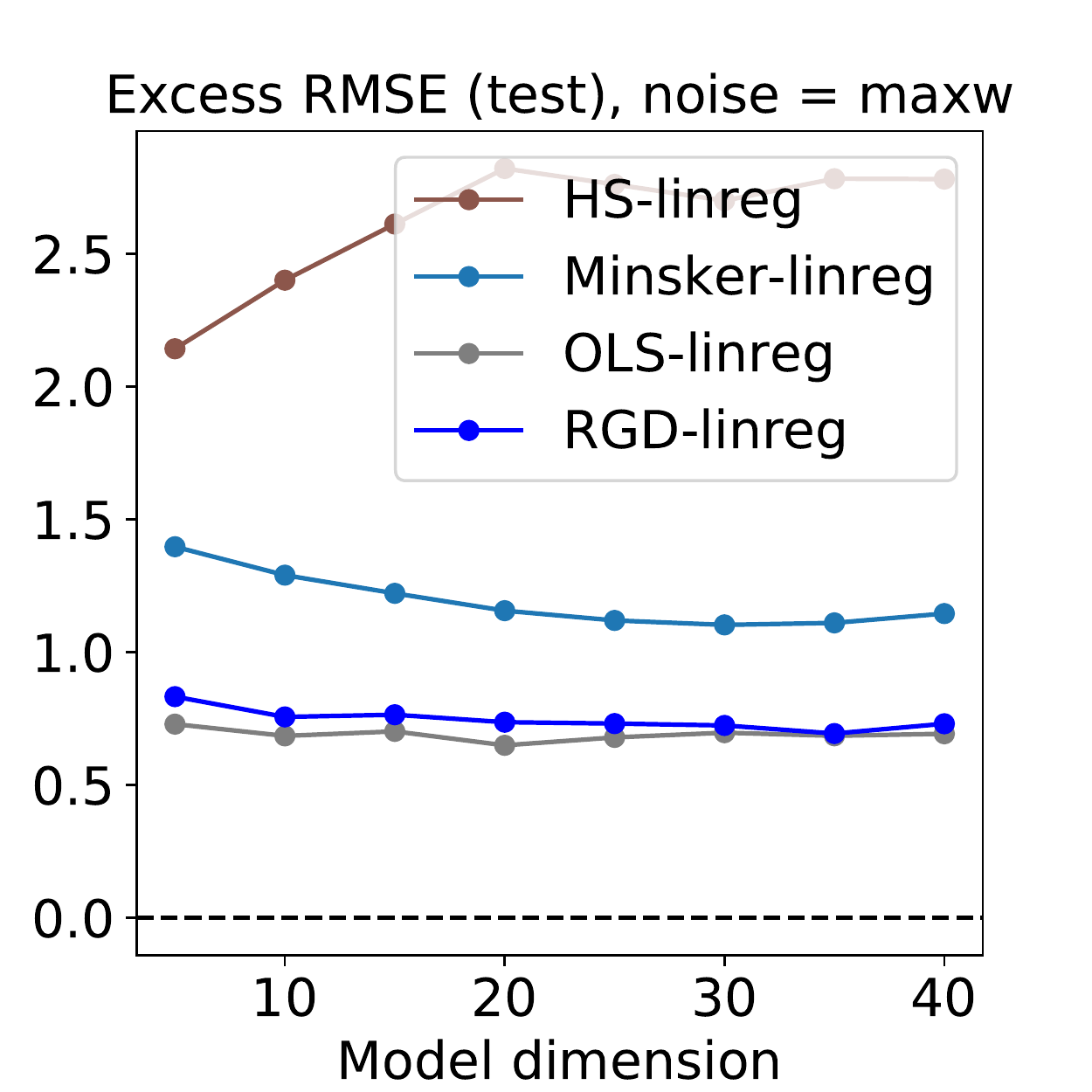}\\
\includegraphics[width=0.25\textwidth]{linreg_overDim_risk_norm}\,\includegraphics[width=0.25\textwidth]{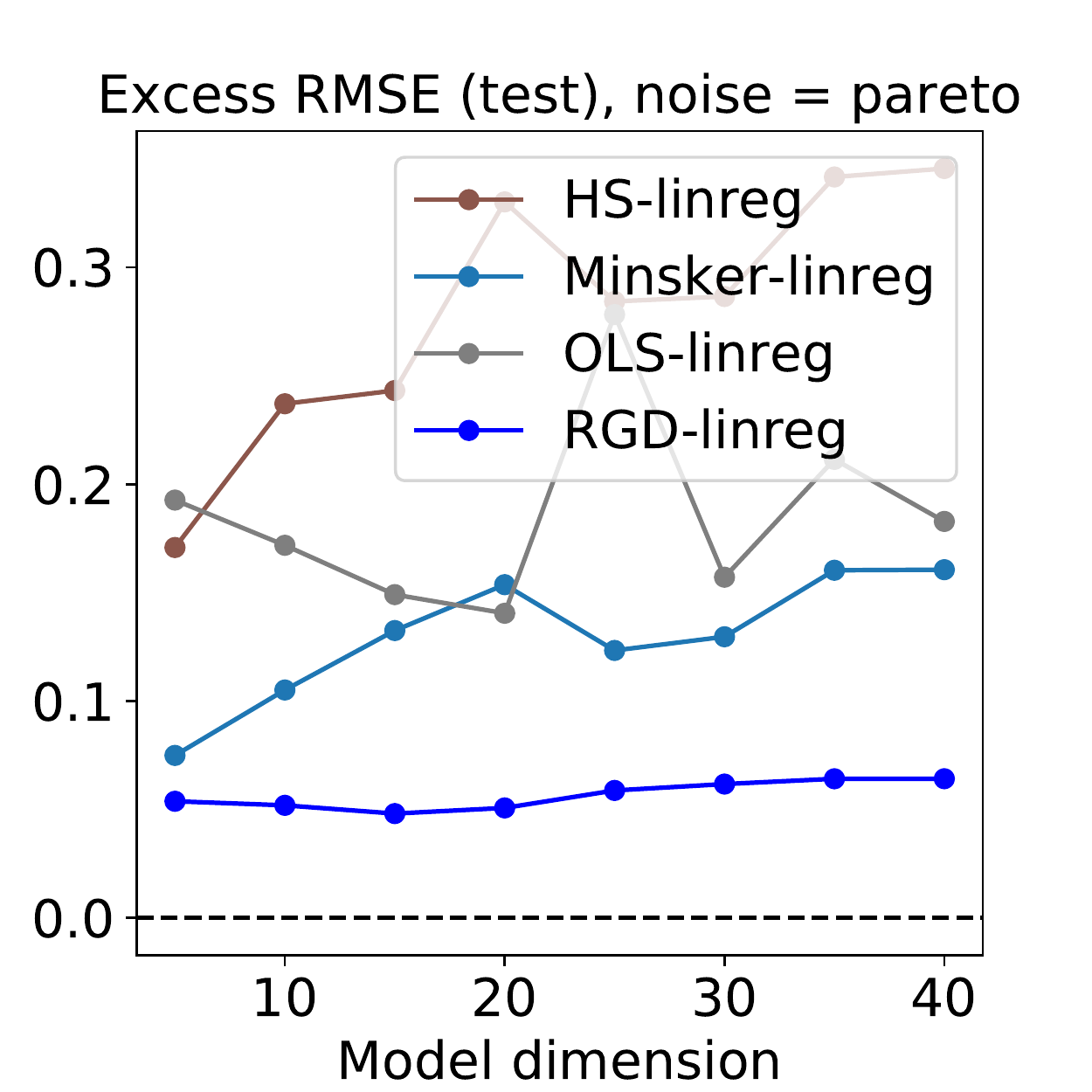}\,\includegraphics[width=0.25\textwidth]{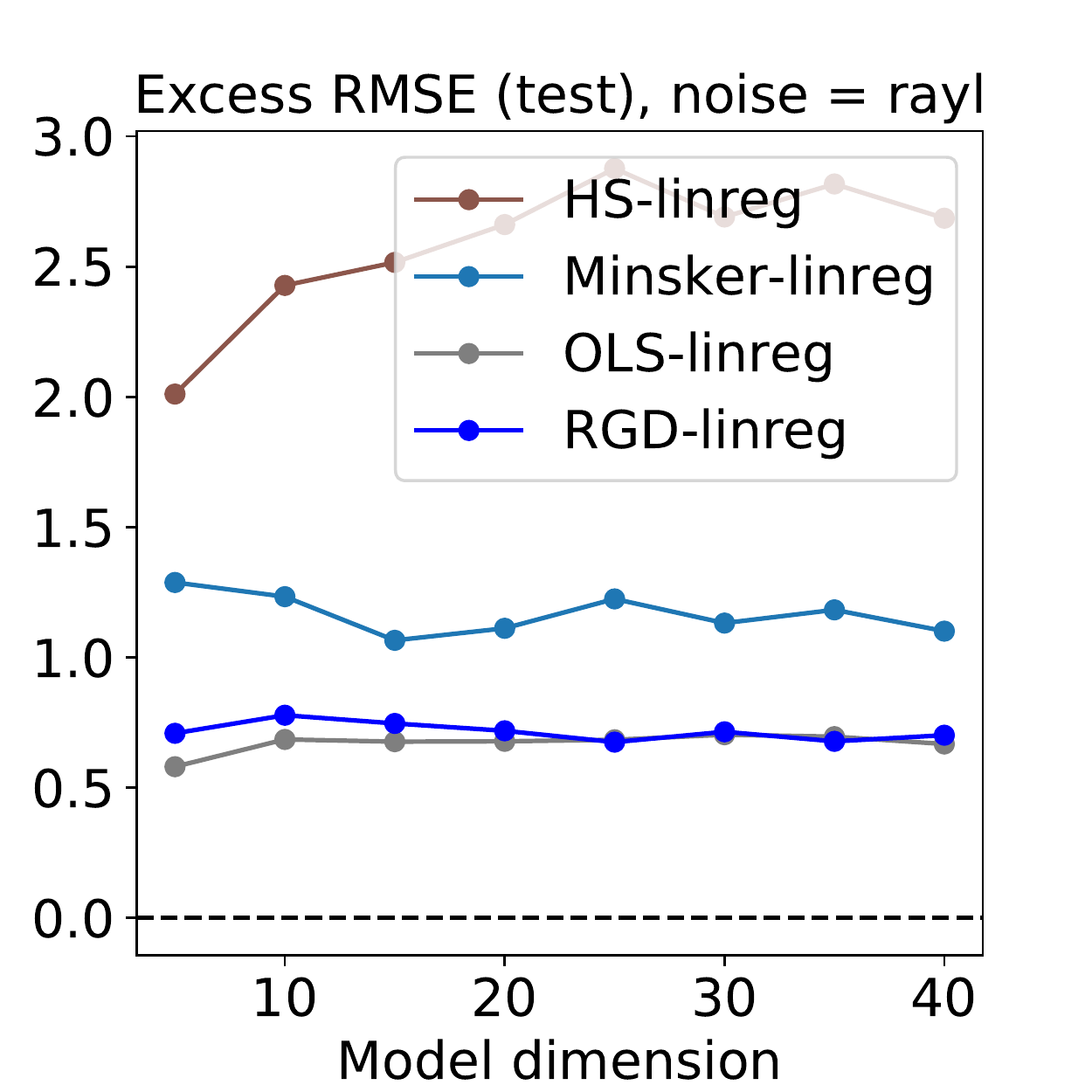}\,\includegraphics[width=0.25\textwidth]{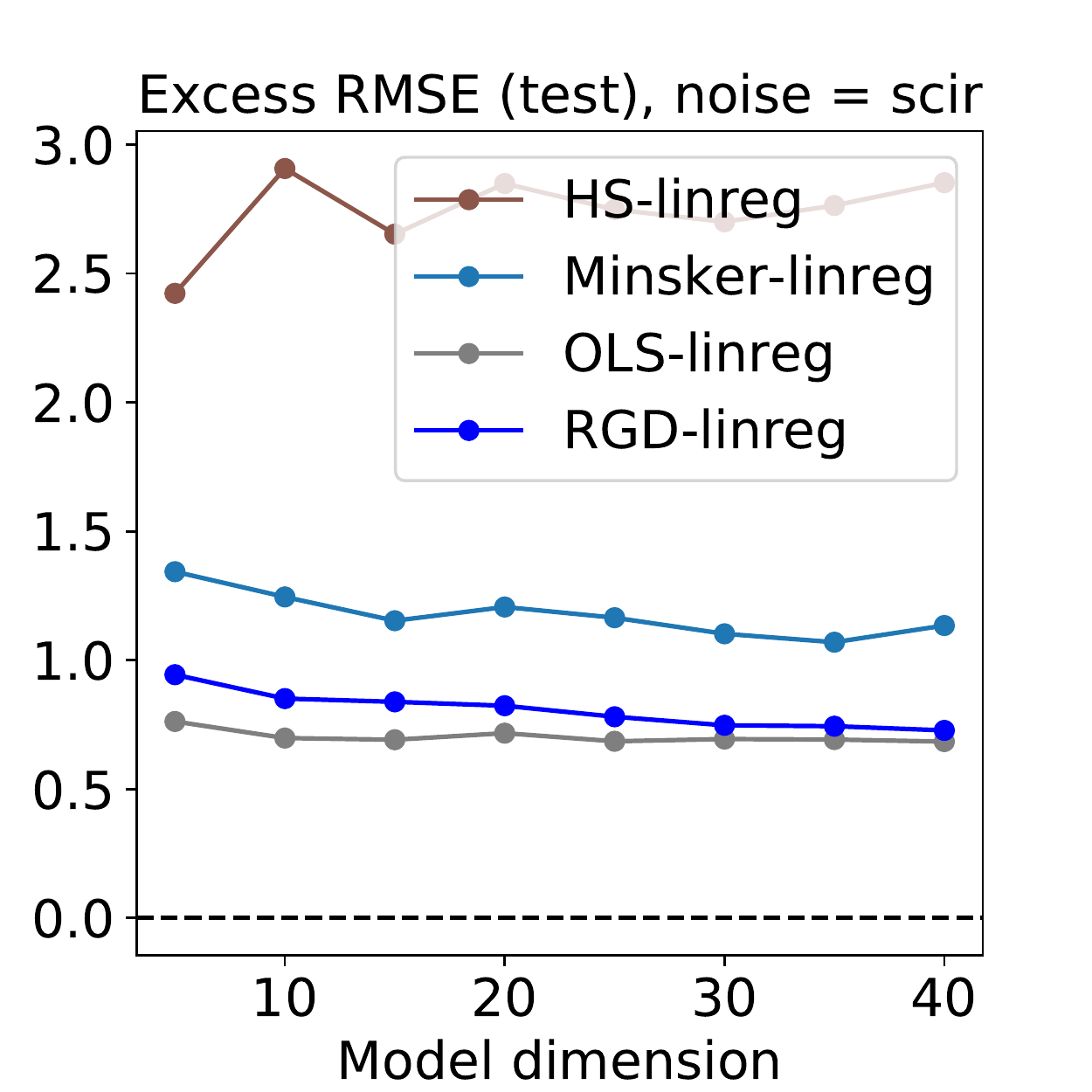}
\caption{Prediction error over dimensions $5 \leq d \leq 40$, with ratio $n/d = 6$ fixed, and noise level = $8$. Each plot corresponds to a distinct noise distribution.}
\label{fig:overDim_all_distros_1}
\end{figure}

\clearpage

\begin{figure}[t]
\centering
\includegraphics[width=0.25\textwidth]{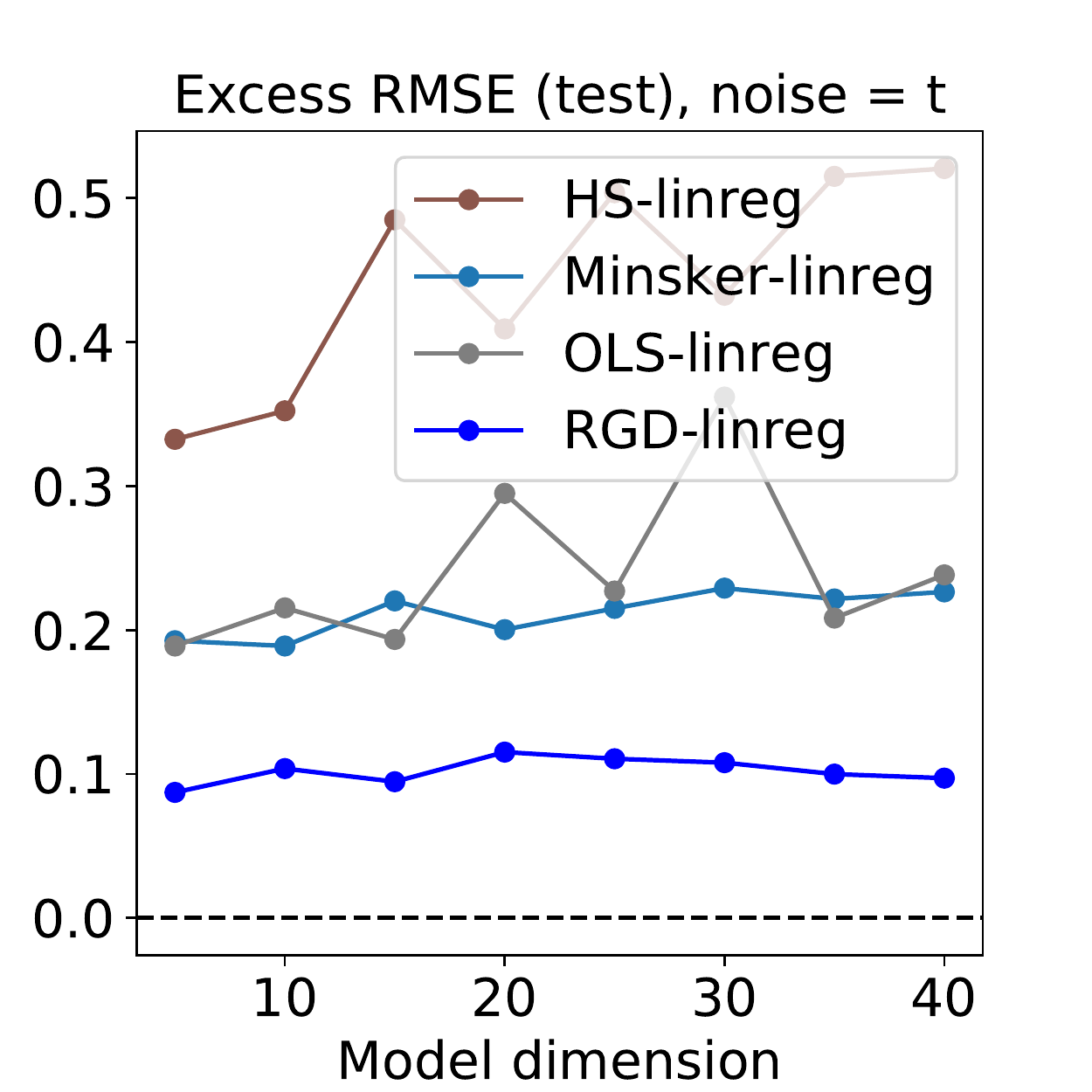}\,\includegraphics[width=0.25\textwidth]{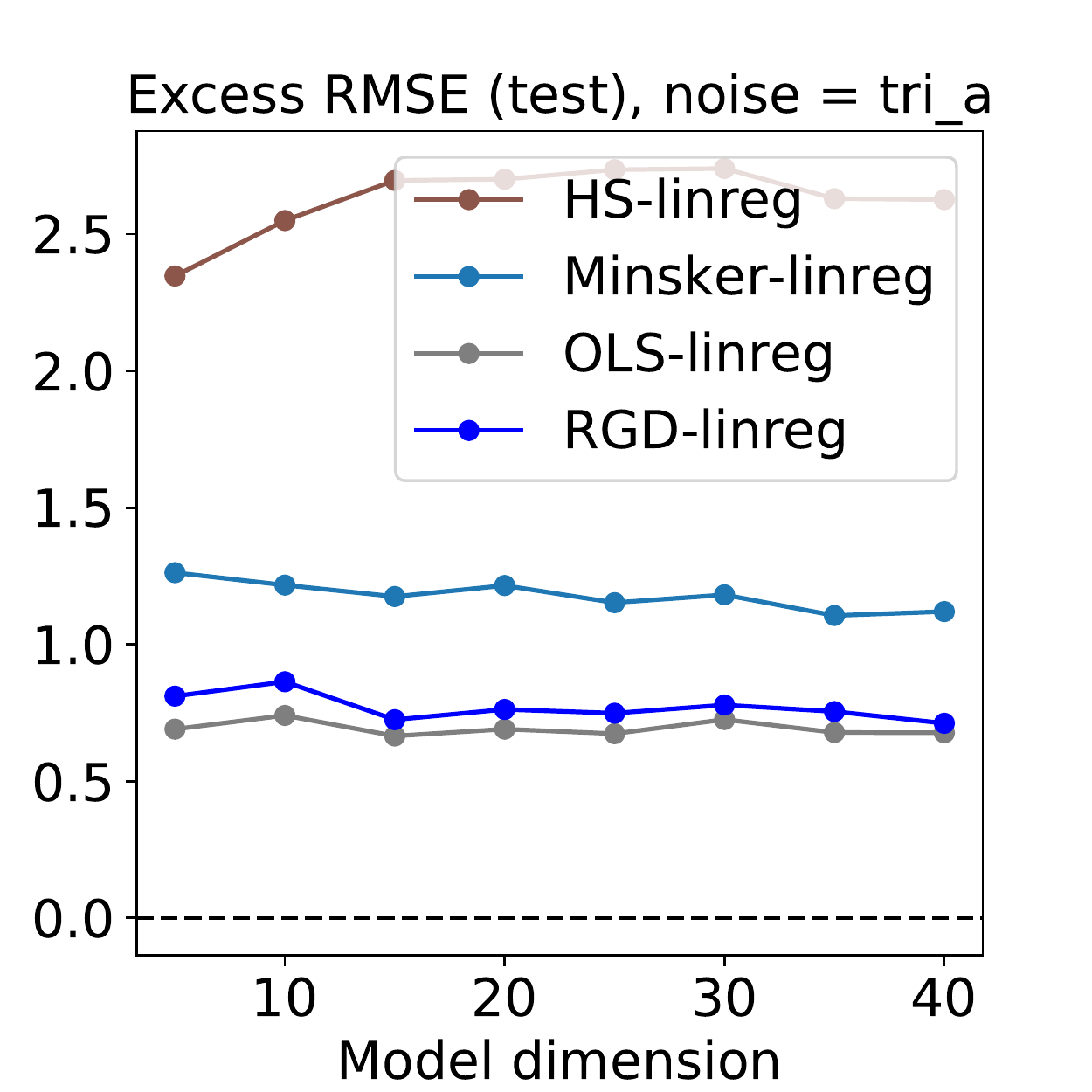}\,\includegraphics[width=0.25\textwidth]{linreg_overDim_risk_tri_s}\,\includegraphics[width=0.25\textwidth]{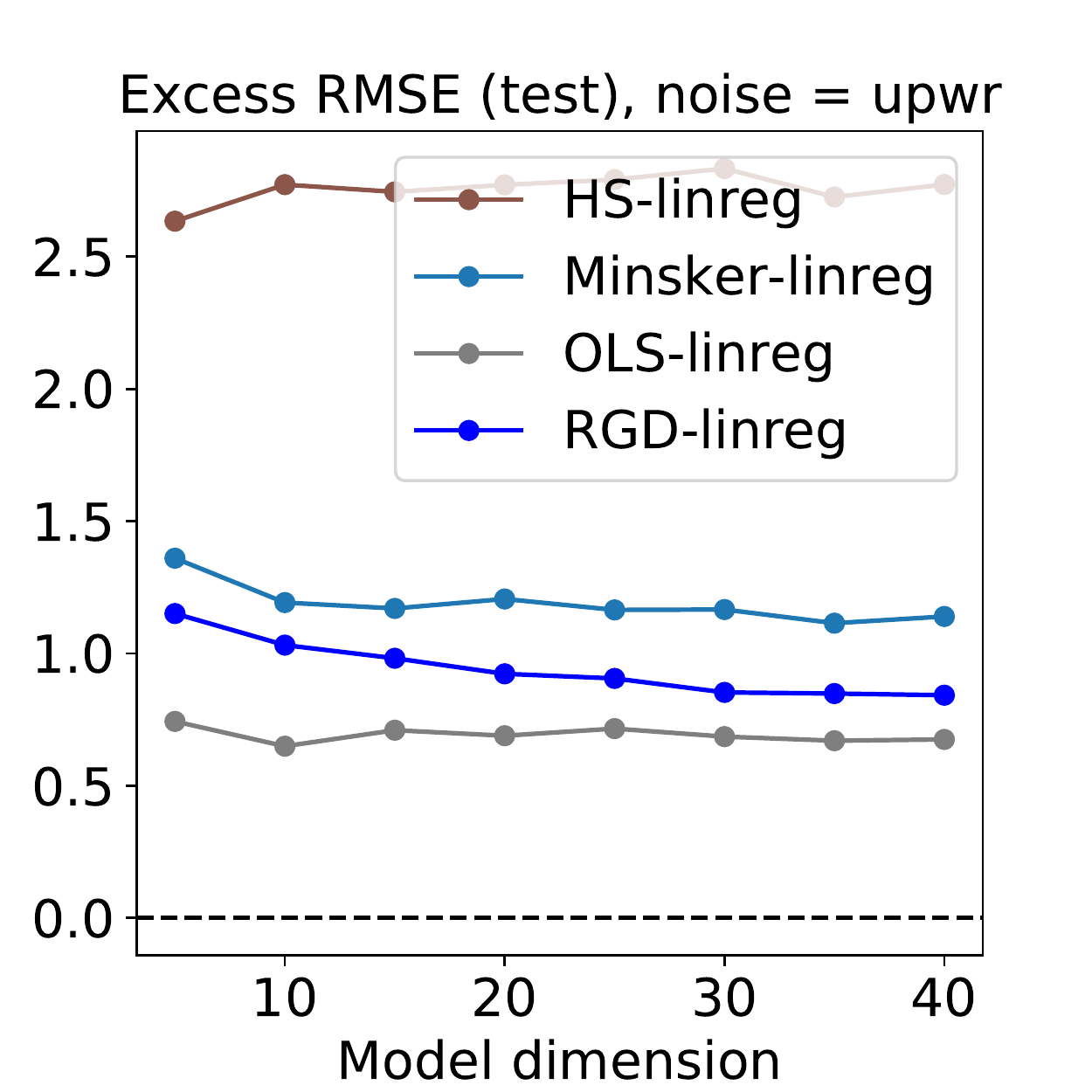}\\
\includegraphics[width=0.25\textwidth]{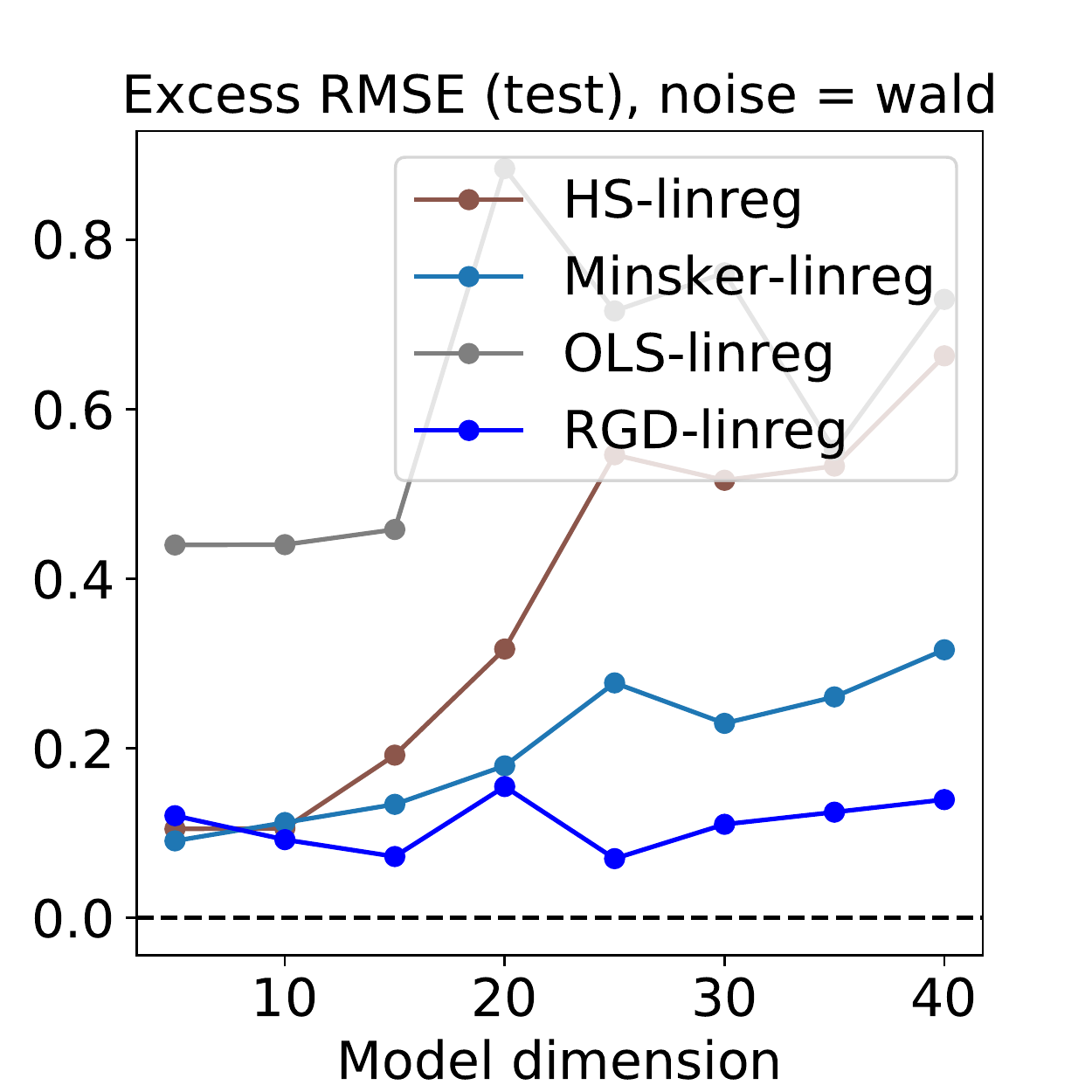}\,\includegraphics[width=0.25\textwidth]{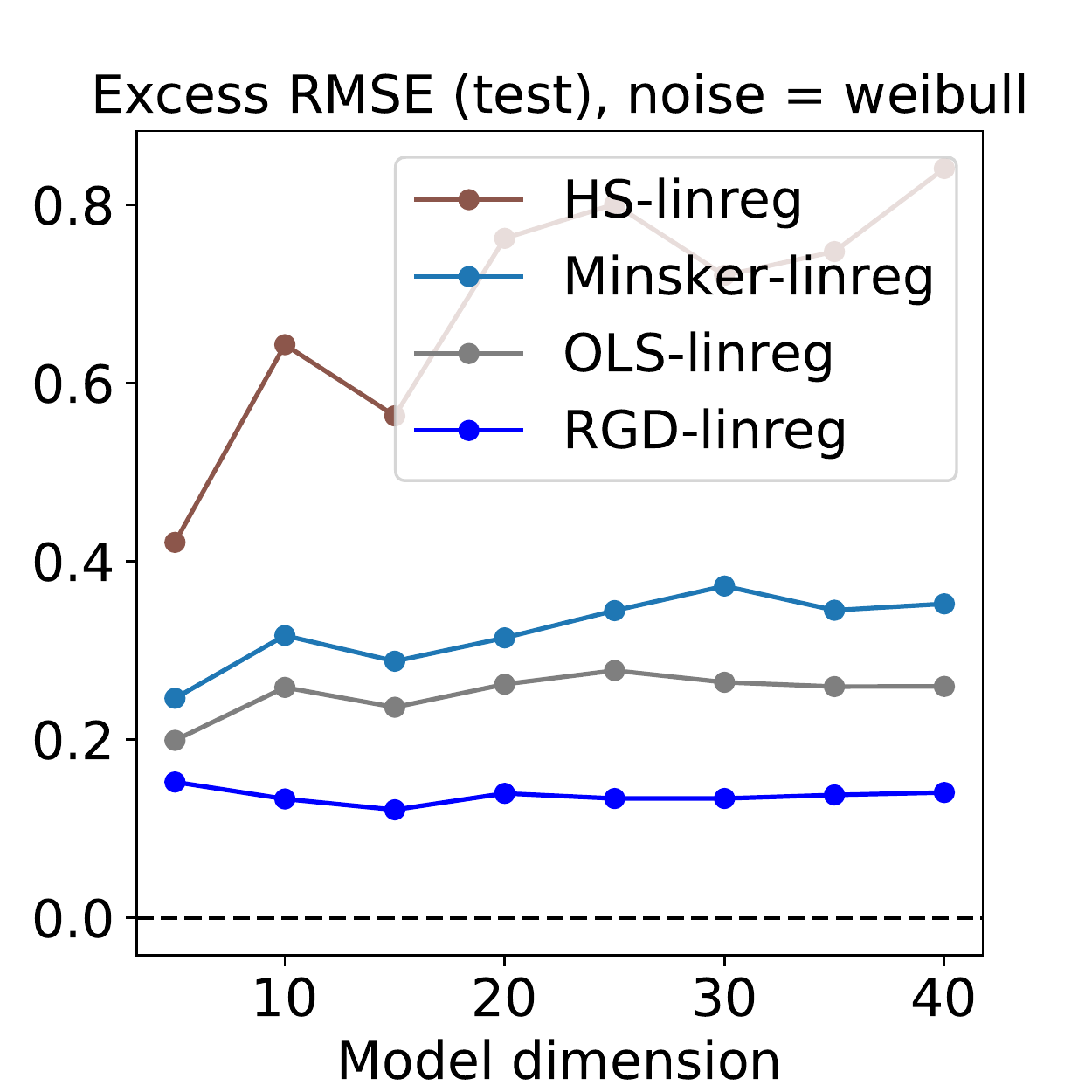}
\caption{Prediction error over dimensions $5 \leq d \leq 40$, with ratio $n/d = 6$ fixed, and noise level = $8$. Each plot corresponds to a distinct noise distribution.}
\label{fig:overDim_all_distros_2}
\end{figure}

\clearpage

{\small
\bibliographystyle{apalike}
\bibliography{refs_rgdmult.bib}

\begin{thebibliography}{}

\bibitem[Brownlees et~al., 2015]{brownlees2015a}
Brownlees, C., Joly, E., and Lugosi, G. (2015).
\newblock Empirical risk minimization for heavy-tailed losses.
\newblock {\em Annals of Statistics}, 43(6):2507--2536.

\bibitem[Catoni, 2004]{catoni2004SLT}
Catoni, O. (2004).
\newblock {\em Statistical learning theory and stochastic optimization: Ecole
  d'Et{\'e} de Probabilit{\'e}s de Saint-Flour XXXI-2001}, volume 1851 of {\em
  Lecture Notes in Mathematics}.
\newblock Springer.

\bibitem[Catoni, 2012]{catoni2012a}
Catoni, O. (2012).
\newblock Challenging the empirical mean and empirical variance: a deviation
  study.
\newblock {\em Annales de l'Institut Henri Poincar{\'e}, Probabilit{\'e}s et
  Statistiques}, 48(4):1148--1185.

\bibitem[Catoni and Giulini, 2017]{catoni2017a}
Catoni, O. and Giulini, I. (2017).
\newblock Dimension-free {PAC}-{B}ayesian bounds for matrices, vectors, and
  linear least squares regression.
\newblock {\em arXiv preprint arXiv:1712.02747}.

\bibitem[Chen et~al., 2017a]{chen2017a}
Chen, Y., Su, L., and Xu, J. (2017a).
\newblock Distributed statistical machine learning in adversarial settings:
  {B}yzantine gradient descent.
\newblock {\em arXiv preprint arXiv:1705.05491}.

\bibitem[Chen et~al., 2017b]{chen2017b}
Chen, Y., Su, L., and Xu, J. (2017b).
\newblock Distributed statistical machine learning in adversarial settings:
  {B}yzantine gradient descent.
\newblock {\em Proceedings of the ACM on Measurement and Analysis of Computing
  Systems}, 1(2):44.

\bibitem[Daniely and Shalev-Shwartz, 2014]{daniely2014a}
Daniely, A. and Shalev-Shwartz, S. (2014).
\newblock Optimal learners for multiclass problems.
\newblock In {\em 27th Annual Conference on Learning Theory}, volume~35 of {\em
  Proceedings of Machine Learning Research}, pages 287--316.

\bibitem[Feldman, 2016]{feldman2016a}
Feldman, V. (2016).
\newblock Generalization of {ERM} in stochastic convex optimization: The
  dimension strikes back.
\newblock In {\em Advances in Neural Information Processing Systems 29}, pages
  3576--3584.

\bibitem[Finkenst{\"a}dt and Rootz{\'e}n, 2003]{finkenstadt2003Extreme}
Finkenst{\"a}dt, B. and Rootz{\'e}n, H., editors (2003).
\newblock {\em Extreme Values in Finance, Telecommunications, and the
  Environment}.
\newblock CRC Press.

\bibitem[Hsu and Sabato, 2016]{hsu2016a}
Hsu, D. and Sabato, S. (2016).
\newblock Loss minimization and parameter estimation with heavy tails.
\newblock {\em Journal of Machine Learning Research}, 17(18):1--40.

\bibitem[Johnson and Zhang, 2013]{johnson2013a}
Johnson, R. and Zhang, T. (2013).
\newblock Accelerating stochastic gradient descent using predictive variance
  reduction.
\newblock In {\em Advances in Neural Information Processing Systems 26}, pages
  315--323.

\bibitem[Kolmogorov, 1993]{kolmogorov1993SelectWorks3}
Kolmogorov, A.~N. (1993).
\newblock $\varepsilon$-entropy and $\varepsilon$-capacity of sets in
  functional spaces.
\newblock In Shiryayev, A.~N., editor, {\em Selected Works of
  {A}.~{N}.~{K}olmogorov, {V}olume {III}: Information Theory and the Theory of
  Algorithms}, pages 86--170. Springer.

\bibitem[Lecu{\'e} and Lerasle, 2017]{lecue2017a}
Lecu{\'e}, G. and Lerasle, M. (2017).
\newblock Learning from {MOM}'s principles.
\newblock {\em arXiv preprint arXiv:1701.01961}.

\bibitem[Lecu{\'e} et~al., 2018]{lecue2018a}
Lecu{\'e}, G., Lerasle, M., and Mathieu, T. (2018).
\newblock Robust classification via mom minimization.
\newblock {\em arXiv preprint arXiv:1808.03106}.

\bibitem[Lin and Rosasco, 2016]{lin2016a}
Lin, J. and Rosasco, L. (2016).
\newblock Optimal learning for multi-pass stochastic gradient methods.
\newblock In {\em Advances in Neural Information Processing Systems 29}, pages
  4556--4564.

\bibitem[Luenberger, 1969]{luenberger1969Book}
Luenberger, D.~G. (1969).
\newblock {\em Optimization by Vector Space Methods}.
\newblock John Wiley \& Sons.

\bibitem[Lugosi and Mendelson, 2016]{lugosi2016a}
Lugosi, G. and Mendelson, S. (2016).
\newblock Risk minimization by median-of-means tournaments.
\newblock {\em arXiv preprint arXiv:1608.00757}.

\bibitem[Lugosi and Mendelson, 2017a]{lugosi2017a}
Lugosi, G. and Mendelson, S. (2017a).
\newblock Regularization, sparse recovery, and median-of-means tournaments.
\newblock {\em arXiv preprint arXiv:1701.04112}.

\bibitem[Lugosi and Mendelson, 2017b]{lugosi2017b}
Lugosi, G. and Mendelson, S. (2017b).
\newblock Sub-gaussian estimators of the mean of a random vector.
\newblock {\em arXiv preprint arXiv:1702.00482}.

\bibitem[Minsker, 2015]{minsker2015a}
Minsker, S. (2015).
\newblock Geometric median and robust estimation in {B}anach spaces.
\newblock {\em Bernoulli}, 21(4):2308--2335.

\bibitem[Nalisnick et~al., 2015]{nalisnick2015a}
Nalisnick, E., Anandkumar, A., and Smyth, P. (2015).
\newblock A scale mixture perspective of multiplicative noise in neural
  networks.
\newblock {\em arXiv preprint arXiv:1506.03208}.

\bibitem[Nesterov, 2004]{nesterov2004ConvOpt}
Nesterov, Y. (2004).
\newblock {\em Introductory Lectures on Convex Optimization: A Basic Course}.
\newblock Springer.

\bibitem[Nocedal and Wright, 1999]{nocedal1999a}
Nocedal, J. and Wright, S. (1999).
\newblock {\em Numerical Optimization}.
\newblock Springer Series in Operations Research. Springer.

\bibitem[Prasad et~al., 2018]{prasad2018a}
Prasad, A., Suggala, A.~S., Balakrishnan, S., and Ravikumar, P. (2018).
\newblock Robust estimation via robust gradient estimation.
\newblock {\em arXiv preprint arXiv:1802.06485}.

\bibitem[Shalev-Shwartz and Ben-David, 2014]{shalev2014a}
Shalev-Shwartz, S. and Ben-David, S. (2014).
\newblock {\em Understanding Machine Learning: From Theory to Algorithms}.
\newblock Cambridge University Press.

\bibitem[Srivastava et~al., 2014]{srivastava2014a}
Srivastava, N., Hinton, G., Krizhevsky, A., Sutskever, I., and Salakhutdinov,
  R. (2014).
\newblock Dropout: a simple way to prevent neural networks from overfitting.
\newblock {\em Journal of Machine Learning Research}, 15(1):1929--1958.

\bibitem[Vardi and Zhang, 2000]{vardi2000a}
Vardi, Y. and Zhang, C.-H. (2000).
\newblock The multivariate {$L_{1}$}-median and associated data depth.
\newblock {\em Proceedings of the National Academy of Sciences},
  97(4):1423--1426.

\end{thebibliography}
}

\end{document}